\documentclass[1pt]{article}

\usepackage[letterpaper,margin=1in]{geometry}
\usepackage{amsthm}
\usepackage{microtype}
\usepackage{newpxtext,newpxmath} 
\usepackage{wrapfig}
\usepackage{sourcesanspro}       
\usepackage{inconsolata}         
\usepackage{enumitem,booktabs,graphicx}
\usepackage[hidelinks]{hyperref}
\usepackage[nameinlink]{cleveref}
\usepackage{empheq}   
\usepackage{array}
\usepackage{needspace} 

\usepackage[table]{xcolor}
\definecolor{flatblue}{HTML}{EAF2FF} 

\usepackage{natbib}

\usepackage{amsmath,amsfonts,bm}
\usepackage{pifont}

\def\eqref#1{equation~\ref{#1}}
\def\Eqref#1{Equation~\ref{#1}}

\def\1{\bm{1}}

\def\mA{{\bm{A}}}
\def\mB{{\bm{B}}}
\def\mC{{\bm{C}}}
\def\mD{{\bm{D}}}

\def\mF{{\bm{F}}}
\def\mG{{\bm{G}}}
\def\mH{{\bm{H}}}

\def\mI{{\bm{I}}}

\def\mK{{\bm{K}}}
\def\mL{{\bm{L}}}
\def\mM{{\bm{M}}}

\def\mP{{\bm{P}}}
\def\mQ{{\bm{Q}}}
\def\mR{{\bm{R}}}
\def\mS{{\bm{S}}}

\def\mU{{\bm{U}}}
\def\mV{{\bm{V}}}
\def\mW{{\bm{W}}}
\def\mX{{\bm{X}}}
\def\mx{{\bm{x}}}
\def\mY{{\bm{Y}}}
\def\mZ{{\bm{Z}}}

\def\mLambda{{\bm{\Lambda}}}
\def\mSigma{{\bm{\Sigma}}}

\DeclareMathAlphabet{\mathsfit}{\encodingdefault}{\sfdefault}{m}{sl}
\SetMathAlphabet{\mathsfit}{bold}{\encodingdefault}{\sfdefault}{bx}{n}

\def\gJ{{\mathcal{J}}}

\def\gL{{\mathcal{L}}}

\def\gO{{\mathcal{O}}}

\newcommand{\E}{\mathbb{E}}

\newcommand{\R}{\mathbb{R}}

\newcommand{\Var}{\mathrm{Var}}

\newcommand{\Cov}{\mathrm{Cov}}

\definecolor{ForestGreen}{rgb}{0.13, 0.55, 0.13}
\definecolor{coolblack}{rgb}{0.0, 0.18, 0.39}
\DeclareMathOperator*{\argmax}{arg\,max}
\DeclareMathOperator*{\argmin}{arg\,min}

\DeclareMathOperator{\sign}{sign}

\newcommand{\loss}{\mathcal{L}}

\newcommand{\cov}{\bm{\Sigma}}
\newcommand{\noise}{\bm{\epsilon}}
\newcommand{\mNG}{\tilde{\bm{G}}}

\newcommand{\mEG}{\bar{\mG}}
\DeclareMathOperator{\polar}{\mathrm{Polar}}
\DeclareMathOperator{\tr}{\mathrm{Tr}}
\DeclareMathOperator{\EVD}{\mathrm{EVD}}

\DeclareMathOperator{\var}{\mathrm{Var}}
\DeclareMathOperator{\diag}{\mathrm{Diag}}
\DeclareMathOperator{\erf}{\mathrm{erf}}

\newcommand{\arosink}{ARO-Sinkhorn}
\newcommand{\alignscore}{S_A}
\newcommand{\Corr}{\mathrm{Corr}}

\newtheorem{definition}{Definition}[section]
\newtheorem{theorem}{Theorem}[section]

\newtheorem{lemma}[theorem]{Lemma}

\newtheorem{assumption}{Assumption}[section]

\definecolor{primary}{HTML}{1F3A56}  
\definecolor{accentBlue}{HTML}{3D8BFD}   
\definecolor{softbgBlue}{HTML}{EAF2FF}   

\definecolor{softbgPink}{RGB}{255, 230, 240}  
\definecolor{accentPink}{RGB}{230, 80, 120}

\definecolor{softbgPurple}{HTML}{F2ECFF}  
\definecolor{accentPurple}{HTML}{463668}  

\definecolor{soapgreen}{HTML}{E8F5E9}   %
\definecolor{spluspurple}{HTML}{F3E8FF} %
\definecolor{muonorange}{HTML}{FFE9D6}  %

\definecolor{softbgGreen}{HTML}{F7FFF2}   
\definecolor{accentGreen}{HTML}{07641D}   

\definecolor{AROBase}{HTML}{3B82F6}    
\definecolor{AROFull}{HTML}{2654A0}    
\definecolor{Muon}{HTML}{94A3B8}       
\definecolor{Dion}{HTML}{F97316}       
\definecolor{Eigen}{HTML}{93BAFA}      
\definecolor{AdamW}{HTML}{10B981}      

\hypersetup{colorlinks=true, linkcolor=accentBlue, citecolor=accentBlue, urlcolor=accentBlue}

\usepackage{subfigure}
\usepackage{caption}
\usepackage{subcaption}

\usepackage{algorithm}
\usepackage{algorithmic}

\usepackage{titlesec}
\titleformat{\section}{\Large\bfseries\sffamily}{\thesection}{0.6em}{}
\titleformat{\subsection}{\large\bfseries\sffamily}{\thesubsection}{0.6em}{}
\titleformat{\subsubsection}{\normalsize\bfseries\sffamily}{\thesubsubsection}{0.6em}{}
\setlist{itemsep=0.25em, topsep=0.3em}

\newtheorem{remark}{Remark}

\usepackage{titling}
\usepackage[most]{tcolorbox}
\newcommand{\name}{ARO} 
\tcbset{enhanced, frame hidden, boxrule=0pt}

\newcommand{\wrapfill}{%
  \par
  \ifnum\value{WF@wrappedlines}>0
    \addtocounter{WF@wrappedlines}{-1}%
    \null\vspace{\arabic{WF@wrappedlines}\baselineskip}%
    \WFclear
  \fi
}

\newcommand{\appendixtitle}[1][Appendix]{%
  \clearpage
  \appendix
  \phantomsection                           
  \addcontentsline{toc}{section}{#1}        
  \begin{center}
    {\LARGE\bfseries\sffamily #1\par}       
  \end{center}
  \vspace{0.8em}
}

\preauthor{\par\centering\rmfamily\large}
\postauthor{\par\centering\vspace{0.25em}}
\predate{} \postdate{}

\newlength{\boxedabstractindent}
\newenvironment{boxedabstract}
{%
  \begin{tcolorbox}[colback=softbgBlue, left=10pt, right=10pt, top=8pt, bottom=8pt, arc=1mm]%
  {\centering\large\bfseries Abstract\par}%
  \vspace{0.4\baselineskip}%
  \setlength{\parskip}{0.6\baselineskip}%
  \setlength{\parindent}{\boxedabstractindent}
  \noindent\ignorespaces
}
{%
  \end{tcolorbox}\vspace{0.5em}%
}

\usepackage{fancyhdr}
\newtcolorbox{questions}[1][]{
  enhanced,
  breakable,
  colback=softbgBlue,     
  colframe=accentBlue, 
  boxrule=0.6pt,
  arc=1mm,
  left=10pt, right=10pt, top=10pt, bottom=10pt,
  fonttitle=\sffamily\bfseries,
  boxed title style={colback=white, colframe=accentPurple, boxrule=0.6pt, arc=1mm},
  #1
}

\newtcolorbox{remarks}[1][]{
  enhanced,
  breakable,
  colback=softbgBlue,      
  colframe=accentBlue,   
  boxrule=0.6pt,
  arc=1mm,
  left=10pt, right=10pt, top=10pt, bottom=10pt,
  fonttitle=\sffamily\bfseries,
  boxed title style={colback=white, colframe=accentBlue, boxrule=0.6pt, arc=1mm},
  #1
}

\newtcolorbox{findings}[1][]{
  enhanced,
  breakable,
  colback=softbgPink,    
  colframe=accentPink,   
  boxrule=0.6pt,
  arc=1mm,
  left=10pt, right=10pt, top=10pt, bottom=10pt,
  fonttitle=\sffamily\bfseries,
  boxed title style={colback=white, colframe=accentPink, boxrule=0.6pt, arc=1mm},
  #1
}

\newtcolorbox{takeaways}[1][]{
  enhanced,
  breakable,
  colback=white,     
  colframe=black,
  boxrule=1.0pt,
  frame style={draw=black, fill=none},
  arc=1mm,
  left=10pt, right=10pt, top=10pt, bottom=10pt,
  #1
}

\newcounter{rq}

\usepackage{enumitem}
\newlist{rqlist}{enumerate}{1}
\setlist[rqlist]{%
  label=\textbf{Q\arabic*:},     
  leftmargin=2.6em,
  itemsep=0.25em, topsep=0.25em,
  before=\setcounter{rqlisti}{\value{rq}},
  after=\global\setcounter{rq}{\value{rqlisti}}%
}

\makeatletter

\newcommand{\ms}[1]{\textsuperscript{\textcolor{accentBlue}{\ensuremath{#1}}}}

\newcommand{\titlefootnote}[1]{%
  \g@addto@macro\@titlefootnotes{%
    \par\noindent\footnotesize #1%
  }%
}
\newcommand{\@titlefootnotes}{} 

\newcommand{\printtitlefootnotes}{%
  \begingroup
  \renewcommand{\thefootnote}{}%
  \footnotetext{%
    \@titlefootnotes
  }%
  \endgroup
}

\renewcommand{\maketitle}{%
  \begin{center}
    {\LARGE\bfseries\sffamily \@title \par}
    \vspace{0.8em}
    {\large\rmfamily \@author \par}
  \end{center}
  \vspace{0.8em}
  \printtitlefootnotes
}

\makeatother

\title{\textbf{ARO: A New Lens On Matrix Optimization For Large Models}}

\author{%
  Wenbo Gong\ms{1, \star},\;
  Javier Zazo\ms{1},\;
  Qijun Luo\ms{2, \ddagger},\; \\
  Puqian Wang\ms{3, \ddagger},\; 
  James Hensman\ms{1},\;
  Chao Ma\ms{1, \star}
  \\[0.4em]
  {\fontsize{10.5}{12}\selectfont\emph{
    \ms{1}Microsoft Research \quad
    \ms{2}The Chinese University of Hong Kong, Shenzhen \quad
    \ms{3}University of Wisconsin{-}Madison
  }}
}

\titlefootnote{\textcolor{accentBlue}{\ensuremath{\star}} Equal contribution.}
\titlefootnote{\textcolor{accentBlue}{\ensuremath{\ddagger}} Work by Qijun Luo and Puqian Wang was done during their internships at Microsoft Research Cambridge.}
\titlefootnote{Corresponds to: Chao Ma (\href{mailto:chaoma@microsoft.com}{chaoma@microsoft.com}) and Wenbo Gong (\href{mailto:wenbogong@microsoft.com}{wenbogong@microsoft.com}).}

\begin{document}

\maketitle

\begin{boxedabstract}
Matrix-based optimizers have attracted growing interest for improving LLM training efficiency, with significant progress centered on orthogonalization/whitening based methods. While yielding substantial performance gains, a fundamental question arises: can we develop new paradigms beyond orthogonalization, pushing the efficiency frontier further? We present \textbf{Adaptively Rotated Optimization (\name{})}, a new matrix optimization framework that treats gradient rotation as a first class design principle. \name{} accelerates LLM training by performing normed steepest descent in a rotated coordinate system, where the rotation is determined by a novel norm-informed policy. This perspective yields update rules that go beyond existing orthogonalization and whitening optimizers, improving sample efficiency in practice. To make comparisons reliable, we propose a rigorously controlled benchmarking protocol that reduces confounding and bias. Under this protocol, {\name{} consistently outperforms AdamW (by  1.3 $\sim$1.35$\times$) and orthogonalization methods (by 1.1$\sim$1.15$\times$) in LLM pretraining at up to 8B activated parameters, and up to $8\times$ overtrain budget, without evidence of diminishing returns}. Finally, we discuss how \name{} can be reformulated as a symmetry-aware optimizer grounded in rotational symmetries of residual streams, motivating advanced designs that enable computationally efficient exploitation of cross-layer/cross module couplings.
\end{boxedabstract}

\begin{figure}[h]
    \centering
\includegraphics[width=0.8\linewidth]{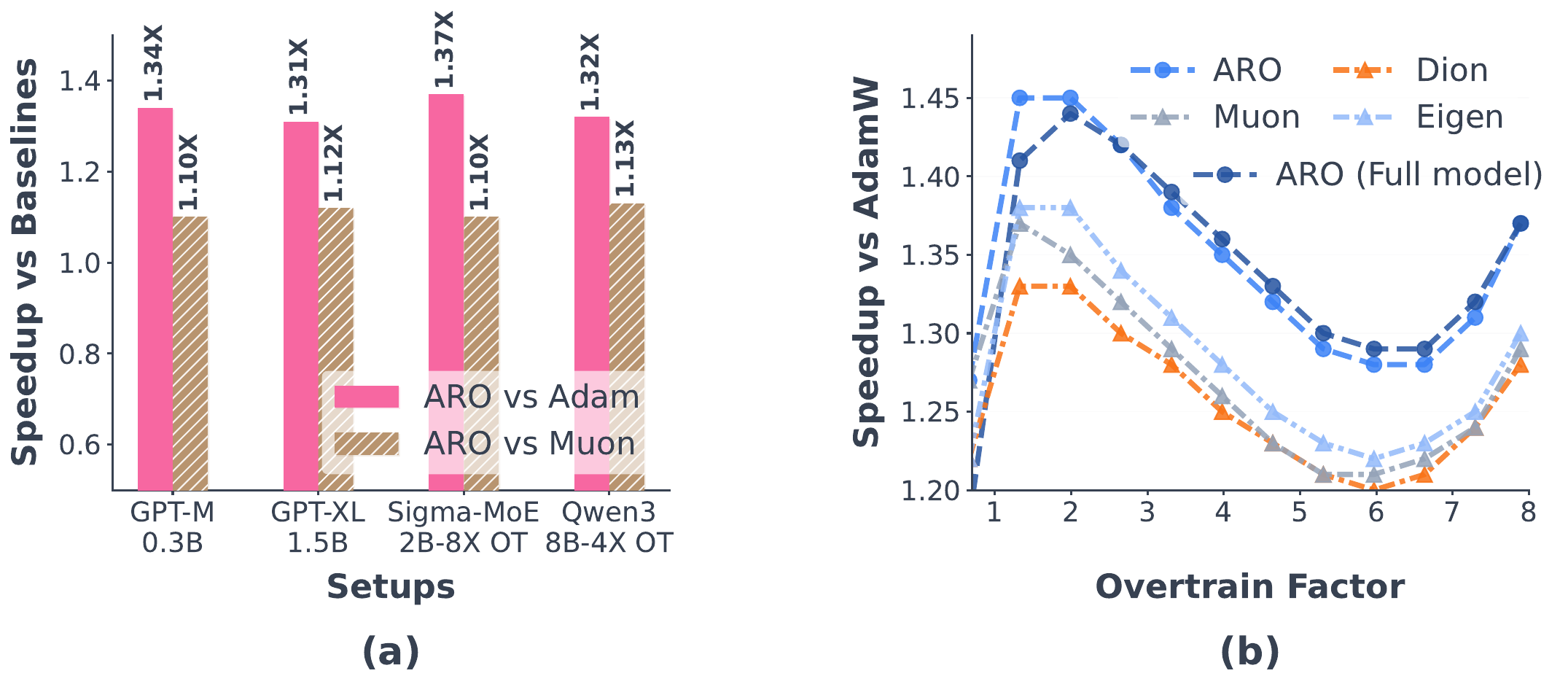}
    \caption{Scaling results preview of \name{} under a rigorous, controlled benchmarking protocol.
\textbf{(a)} \textbf{\name{} delivers consistent, non-diminishing speedups over AdamW and Muon} across model scales (up to 8B activated parameters) and training budgets (up to $8\times$ overtrain, denoted by $\texttt{OT}$).
\textbf{(b)} In the Sigma-MoE-2B \citep{hu2025sigmamoetinytechnicalreport} example, \textbf{\textcolor{AROBase}{\name{}}} consistently outperforms AdamW,
orthogonalization-based methods (\textbf{\textcolor{Muon}{Muon}}, \textbf{\textcolor{Dion}{Dion}}),
as well as conventional eigenvector-rotated optimizers (\textbf{\textcolor{Eigen}{Eigen}}) across overtraining budgets.
Moreover, \name{} is applicable to \textbf{all matrix parameters}
(\textbf{\textcolor{AROFull}{\name{} full-model}} mode, including embeddings and LM heads), yielding better long-horizon performance than \textbf{\textcolor{AROBase}{\name{}}} applied to hidden layers only.
}
    \label{fig: frontpage}
\end{figure}

\tableofcontents

\newpage
\section{Introduction}
Training large language models \citep{achiam2023gpt, guo2025deepseek, yang2025qwen3} is computationally demanding due to their scale and the complexity of the optimization landscape. In frontier model training, optimization has been dominated by the Adam family \citep{kingma2014adam, loshchilov2017decoupled, hernandez2025apertus}. Recently, \emph{matrix-based} optimizers \citep{ carlson2015stochastic, carlson2015stochasticb, carlson2015preconditioned, gupta2018shampoo, vyas2024soap, anil2020scalable, jordan2024muon, ma2024swan, pethick2025training, scetbon2025gradient, gong2025towards, ahn2025dion, frans2025really} have attracted growing interest for their potential to improve training efficiency. Early successful explorations of new optimizers in frontier model training \citep{liu2025muon, zeng2025glm, team2025kimi} reflects a trend toward geometry-aware optimizers for large-scale training.

To date, a majority of matrix optimization methods for LLMs literature has concentrated on orthogonalization /whitening-based approaches \citep{yang2008principal, carlson2015stochastic, carlson2015stochasticb, carlson2015preconditioned, gupta2018shampoo, anil2020scalable, jordan2024muon, ma2024swan, pethick2025training, gong2025towards, ahn2025dion, vyasimproving, lu2025understanding, eschenhagen2025purifying, xie2025structured, an2025asgo}. While there has been important progress in improving Muon \citep{lau2025polargrad, ahn2025dion, ma2024swan, pethick2025training, he2025root, shulgin2025beyond, amsel2025polar, si2025adamuon, li2025normuon, behrouz2025nested, wen2025hyperball, xie2026controlled, wang2026olionapproachinghadamardideal, yang2026prismstructuredoptimizationanisotropic}, the majority of those advances retain orthogonalization as the core backbone, composing it with complementary components such as adaptive scaling \citep{ si2025adamuon, li2025normuon, yang2026prismstructuredoptimizationanisotropic}, multi-scale momentum \citep{behrouz2025nested, ueaj_multiscale_muon}, additional constraints \citep{pethick2025training, wen2025hyperball, xie2026controlled}, or other add-ons \citep{wang2026olionapproachinghadamardideal}. Therefore, it remains unclear if orthogonalization is the canonical primitive for matrix optimization, or one effective instance of a broader principle that leads to qualitatively different designs. This leads to a fundamental question: 

\begin{takeaways}
    \begin{center}
        \emph{Can we develop new matrix optimization paradigms that move beyond orthogonalization, \\ and thereby push the efficiency frontier further?}
    \end{center}
\end{takeaways}

Additionally, early explorations with matrix optimizers also reveals several practical challenges. First, matrix-based approaches are not yet \emph{full-model} optimizers in common practices: key parameter groups, such as embeddings and LM head parameters, are often excluded and continue to use element-wise Adam-style updates \citep{zhao2024galore, jordan2024muon, liu2025muon, pethick2025training}. This have been identified as one of the top open challenges in \citep{liu2025muon}. While these parameter groups are smaller in sheer count than hidden layers, leaving important parts of the model governed by legacy heuristics is an indication of gaps in our understanding of LLM optimization. 

Second, most existing matrix optimizers are designed and applied at a per-layer granularity: each weight matrix is updated in isolation. As a result, they largely ignore cross-layer and cross-module geometric couplings. Full-parameter second-order methods \citep{abreu2025potential} in principle exploits these couplings, but are infeasible at large scale. This motivates the need for new principles for exploiting cross-layer geometry economically.

Finally, recent benchmarking studies \citep{wen2025fantastic, semenov2025benchmarking} raise a concerns about optimizer evaluation bias: with careful tuning, the speed-up and loss improvements of a number of new optimizers over AdamW diminishes at scale. While advanced techniques has been proposed to mitigate this issue \citep{wen2025hyperball}, their speedup over AdamW is still much more modest compared with the claimed $2$--$3\times$ gains in early reports \citep{liu2025muon, ahn2025dion, liu2023sophia,  ma2024swan}.  Therefore, it raises need for carefully controlled benchmarking of LLM optimizers at large-scale over-training regime, that is less commonly covered by related works.

Together, these observations underscore the need for new algorithms and principles that (i) produce novel update rules beyond gradient orthogonalization/whitening, further pushing the efficiency frontier; (ii) enable the use of unified update rule across LLM parameter types; (iii) offer mechanisms to exploit cross-layer/cross-module couplings economically, and (iv)  are benchmarked under carefully controlled, large-scale overstrained settings to debias confoundings and address diminishing return risks.

\subsection{Key insights: rotation as first class design principle} \label{sec: insights}
In this paper, we present our effort towards addressing these challenges. Our core idea is to leverage \emph{gradient rotation} as a unified interface for optimization across all model components. The starting observation is: many recent orthogonalization or whiteninig based matrix optimizers such as SOAP \citep{vyas2024soap}, Muon \citep{jordan2024muon}, and SPlus \citep{frans2025stable} can be reframed as running some version of Adam on a rotated coordiante system. Removing auxiliary designs such as moving average accumulations, and focus on one-sided preconditioning, the general form of such rotated optimization can be formulated as \textbf{rotated steepest descent} (\Cref{subsec: steepest descent under rotation}))
\begin{align} \label{eq: intro}
\Delta \mW_t &\propto -\eta \mR_tf_t(\mR_t^{\top} \mG_t), 
\end{align}
where $\eta$ is the stepsize, $\mW_t$ is a $m$ by $n$ weight matrix to be trained, and $\mG_t$ the corresponding gradient matrix. $\mR_t$ is a $m$ by $m$ orthonormal \emph{coordinate rotation matrix} $\mR_t \in SO(m)$, and $f_t$ is a nonlinear projection function $\R^{m \times n} \mapsto \R^{m \times n}$ characterizing a \emph{base optimizer} \footnote{In general, any optimizer can be written as a (stateful) function that maps the current gradient to the weight updates. See the discussion in \Cref{sec: warmup} for more details.} that ususally corresponds to steepest descent under certain norm. Intuitively, \Cref{eq: intro} performs steepest descent in a rotated coordinate system, and then rotate back to the original coordinate.

From this simplified setting, existing works exclusively consider $\mR_t$ to be the eigenvectors of $\mG_t \mG_t^\top$, or left singular vectors of $\mG_t$ \citep{nguyen2025improving}. Throughout the paper, we refer to this class of rotations as \emph{eigen-rotations}. Then, by varying $f_t$ within the class of Adam variants (Adam, SignGD \citep{bernstein2018signsgd} and row-normalized GD\footnote{We view row normalization as a special case of Adafactor \citep{shazeer2018adafactor}, and thus as a low-rank variant of Adam.} \citep{wensron, shazeer2018adafactor} ), we obtain SOAP/SPlus/Muon, respectively. See \Cref{tab: rotated_optimizers} for detail, and \cref{app: rotated_optimizers} for derivations.

\begin{table}[t]
\centering
\small
\setlength{\tabcolsep}{10pt}
\renewcommand{\arraystretch}{1.10}
\begin{tabular}{@{}
 >{\centering\arraybackslash}m{0.25\linewidth}
 >{\centering\arraybackslash}m{0.30\linewidth} 
 >{\centering\arraybackslash}m{0.40\linewidth}
@{}}

\toprule
\textbf{Method} & \(\mR_t \in SO(m)\) & \(f_t(\mX)\) \\
\midrule
SOAP \citep{vyas2024soap}
& \texttt{Eigenvectors}$(\mG_t\mG_t^\top)$
& $\text{Adam}(\mX)$ \\
\addlinespace[0.6em]
SPlus \citep{frans2025stable}
& \texttt{Eigenvectors}$(\mG_t\mG_t^\top)$
& $\text{Sign}(\mX)$ \\
\addlinespace[0.6em]
Spectral descent/Idealized Muon \citep{carlson2015preconditioned, jordan2024muon}
& \texttt{Eigenvectors}$(\mG_t \mG_t^\top)$
& $\sqrt{n}\,Q(\mX)^{-1}\mX,$ \newline  $Q(\mX) := \text{Diag}(\| \mX_{1,:}\|_2,...,\| \mX_{m,:}\|_2)$\\
\addlinespace[0.6em]
\rowcolor{flatblue}
\name{} (Ours)
& $\mathrm{QR}\!\big(\,\mG_t\, f_t( \mR_{t-1}^\top \mG_t)^\top \,\big)$, where $\mathrm{QR}(\cdot)$ returns the Q factor of its input
& General $f_t$ that is not rotational equivariant. In this paper we consider $\text{Sinkhorn}(\mX)$ \citep{scetbon2025gradient}, as well as $\text{Adam}(\mX)$, $\text{Sign}(\mX)$, and $\sqrt{n}\,Q(\mX)^{-1}\mX$.  \\
\bottomrule
\end{tabular}
\caption{Matrix optimizers unified as rotated optimization, $\Delta \mW_t \propto -\eta\,\mR_t f_t(\mR_t^\top\mG_t)$. We consider one-sided (left) rotations, and removed auxiliary designs such as moving averages/accumulations on $\mG_t$ or $\mG_t \mG_t^\top$ for clarity. Under this simplification, many existing methods (SOAP, Muon, SPlus) can be reformulated as rotated optimization: T
they rotate $\mG_t$ using eigenvectors of $\mG_t \mG_t^\top$/singular vectors of $\mG_t$, and apply specific $f_t$ on the rotated gradients. {\color[HTML]{007FFF}\name{} instead consider general $f_t$, and propose to use a non-eigen rotation informed by $f_t$, that does not necessarily align with the eigenvectors of $\mG_t \mG_t^\top$/singular vectors of $\mG_t$}.}
\label{tab: rotated_optimizers}
\end{table}

The form of \Cref{eq: intro} is not new, it is implicit in both whitening-based preconditioners \citep{vyas2024soap, frans2025stable} and spectral/orthogonalization updates of linear layers \citep{bernstein2024old}. 
We shift the emphasis from any metric-derived or norm-derived instantiation to the shared structural formula across these methods: perform steepest descent in an adaptively rotated coordinate system.
We hypothesize that a substantial fraction of the empirical gains of recent matrix optimizers stems from the rotation mechanism itself, rather than from \emph{a priori} prescribed norm (e.g., spectral norm) or to a paired fixed eigen-rotation and $f_t$. 
This hypothesis is partially supported by empirical evidence that: i) non-spectral descent variants (SOAP/Splus/Shampoo) perform competitively/even better than Muon \citep{frans2025really, wen2025fantastic}; ii) in our ablation in \Cref{subsec: universal effective} \Cref{fig: m1.2}, we show that sample efficiency is affected more by rotation than the choice of $f_t$.  

This leads to the core insight and methodology of our paper. Instead of treating the rotated steepest descent view \Cref{eq: intro} as a re-interpretation of existing methods, we elevate it to a first class principle to construct new first order optimizers. This allows us to consider general choices of $\mR_t$ beyond eigen-rotation, and general $f_t$ beyond Adam families, rather than constraining to canonical designs derived from structured second order methods \citep{gong2025towards, xie2025structured, morwani2024new} or whitening \citep{lu2025understanding, frans2025really}. We begin with the following questions:

\begin{questions}
\begin{rqlist}
  \item For a given base optimizer $f_t$, can we design rotation policies beyond canonical eigen-rotations - policies informed by $f_t$ that deliver improved robustness and efficiency?
  \item Can rotations enable the use of a single update rule across all LLM matrix parameters? Can it perform competitively or better than hybrid approaches (AdamW for non-hidden layers)?
  \item Can rotations act as a economic interface to enable exploitation of cross-layer and cross-module geometric couplings?
\end{rqlist}
\end{questions}

From our perspective, in general $\mR_t$ can be considered as a \emph{rotation policy} conditioned on the choice of \emph{projection rule} $f_t$. Under this lens, existing eigen-rotation based methods appear as special cases where $\mR_t$ and $f_t$ are decoupled. This opens the design space to \emph{non-eigen-rotations} coupled to $f_t$ - a capability that prior work does not capture. In our work, we provide a concrete, new rotation policy (\Cref{sec: aro}), namely \textbf{\name{} (Adaptively Rotated Optimization)} that gives positive answers to all above questions in our empirical validation. A preview of \name{} update rule is given in \Cref{tab: rotated_optimizers}.

Finally, after validating our methodology and insights, it becomes clear that the gradient rotations play a fundamental role in improving sample efficiency. A natural question to ask is the question of \emph{why}:

\begin{questions}
\begin{rqlist}
  \item Why do gradient rotations emerge as an effective primitive for LLM optimization?
  \item What does this suggest about the organizing principle for matrix optimizer design, and about what “matrix optimization” is fundamentally doing?
\end{rqlist}
\end{questions}

We present am extended discussion towards understanding the principles behind gradient rotation (\Cref{sec:symmetry}). We reveal a deep connection between rotated steepest descent and the parameter symmetries of neural networks.
We show that \name{} update rule can be derived as a variant of symmetry teleportation \citep{zhao2022symmetry} that exploits rotational symmetries of loss landscape induced by LLM architectures. This leads to a new candidate hypothesis for understanding and improving matrix optimizers from the lens of parameter symmetries, an important global structure of loss landscape often overlooked in the literature. 

\subsection{Contributions}

Building on those insights, we present \textbf{Adaptively Rotated Optimization (\name{})}, a full-stack, unified matrix optimization framework that speeds up LLM training by applying updates in a rotated coordinate system.
Our contributions are as follows:

\begin{itemize}
    \item \textbf{Rotation as first class design principle for efficient LLM training}. \name{} treats rotation as a unifying first class design principle for sample efficiency, and proposed a general gradient rotation policy that wraps any given base optimizer (e.g., \citep{kingma2014adam, bernstein2018signsgd, shazeer2018adafactor, scetbon2025gradient}) in a unified interface, and adaptively rotates the gradient into a new coordinate system, informed by the base optimizer. These non-eigen-rotations result in updates that are both more sample efficient and robust (\Cref{subsec: universal effective}), compared with their eigen-based variants, including gradient orthogonalization-based methods. We show that \name{} provides a single unified update rule for all LLM parameter classes -- including non-hidden layer parameters such has embeddings and LM heads -- closely matching or outperforming hybrid setups, and enabling consistent optimization across the entire model parameters (\Cref{subsec: Exp Sigma-MoE});

    \item \textbf{State-of-the-art pretraining performance with consistent gains at scale} We benchmark \name{} on both dense and MoE models across scales (up to 8B activated parameters),  and training budgets (up to 8$\times$ over-train). To reduce evaluation bias and ensure realistic comparisons, we propose practical guidelines for controlled benchmarking (\Cref{exp: guidelines}) and tune AdamW baselines end-to-end, avoiding extrapolation or hyperparameter transfer whenever feasible. Under our carefully aligned setting, \name{} delivers a consistent $1.3\sim 1.35\times$ speedup over AdamW, and $1.1\sim1.15\times$ speedup over Muon, across scales. These results complement and refine prior benchmarking studies \citep{wen2025fantastic, semenov2025benchmarking}, addressing the challenge of obtaining controlled, large-scale evaluations of new optimizers. With our distributed FSDP2 and Megatron-LM implementation, the end-to-end wall-clock time overhead over Adam is kept under 3\%.

    \item \textbf{Symmetry hypothesis: a new lens on matrix optimization for LLMs}. We show that \name{} can be elegantly derived from classic \emph{symmetry teleportation} technique \citep{zhao2022symmetry,zhao2023improving} that exploits loss landscape rotational symmetries (induced by inter-layer residual stream symmetries \citep{ashkboos2024slicegpt} and intra-layer $Q$-$K$ symmetries \citep{da2025hide}) for a more favorable local update. Based on this observation, we hypothesize that a substantial part of modern matrix optimization is potentially organized around exploiting global, architecture-induced parameter symmetries \citep{zhao2022symmetry,zhao2023improving, ashkboos2024slicegpt, da2025hide} (\Cref{sec:symmetry}).
    This perspective offers a compact interface for exploiting inter and intra-layer cross parameter couplings, and motivates new optimizer designs for \name{} supported by empirical experiments. These include: 
    \begin{itemize}
        \item The use of globally or locally shared \name{} rotations across layers and modules, for improved performance. This offers a new axis for improving matrix optimizers beyond layer-wise updates, while reducing computational overhead; 
        \item A fine-grained module-specific orientation and transposition scheme for rotated updates;
        \item A principled explanation for why rotated updates can be applied effectively to all matrix parameter groups, including non-hidden-layer matrices.
    \end{itemize}
\end{itemize}

Taken together, our contributions position gradient rotations and loss landscape symmetries in a new lens for understanding matrix optimization in LLMs, grounded on global structures of architectures. We hope this view sharpens new directions beyond orthogonalization-based methods towards designing more efficient, robust, and scalable training algorithms.

\section{Warmup: Rotated Steepest Descent} \label{sec: warmup}

\subsection{Normed steepest descent }
\label{subsec: normed steepest descent}

Let $\mW_t \in \mathbb{R}^{m \times n}$ denote the model parameters and $\mathcal{L}(\mW_t)$ the objective to be minimized (e.g., cross-entropy). Let $\mG_t \in \mathbb{R}^{m \times n}$ be the gradient matrix. We use the Frobenius inner product $\langle \mA, \mB\rangle := \mathrm{tr}(\mA^\top \mB)$ for matrices. The classical steepest descent method under a chosen norm $\|\cdot\|$ computes the update direction $\Delta\mW_t$ by solving the regularized linearized problem (see \citep{bernstein2024old}):
\begin{equation*}
  \Delta \mW_t^\star
   =  \argmin_{\Delta \mW_t}
   \left[  \langle \mG_t, \Delta \mW_t \rangle
     + \tfrac{\lambda}{2}\|\Delta \mW_t\|^2 \right],
  \label{eq: quad_surrogate}
\end{equation*}
where $\lambda>0$ is a regularization  parameter and $\|\cdot\|$ is an arbitrary norm on $\mathbb{R}^{m\times n}$. The minimizer admits the closed form
\begin{equation*}
  \Delta\mW_t^\star
  \propto -\argmax_{\|\mZ_t\|\leq 1} \;\langle \mG_t, \mZ_t \rangle.
  \label{eq: steepest_descent}
\end{equation*}
The solution can be expressed in compact form as $\Delta\mW_t^\star\propto-f(\mG_t)$, where $f$ is a \emph{projection function} $f(\mG_t) : = \argmax_{\|\mZ_t\|\leq1} \langle \mG_t, \mZ_t \rangle $ and given by
\begin{equation} \label{eq: projection}
    f(\mG_t) =  \nabla_{\mG_t} \|\mG_t\|_* ,
\end{equation}
where $\|\cdot\|_*$ is the dual norm of the chosen norm $\|\cdot\|$, and $\nabla \|\cdot\|_*$ is the subgradient of the dual norm. 

\begin{remarks}
    \begin{remark}[momentum-first design of normed steepest descent] \label{remark: momentum_first_nsd} A common way to introduce momentum to normed steepest descent is to follow a \textbf{momentum-first} design, i.e., an exponential moving average (EMA) filter is first applied on $\mG_t$ to obtain
    $$\mM_t = \text{EMA}[\mG_t] := \beta \mM_{t-1} + (1 - \beta) \mG_t.$$
    Then, projection function is applied in a post-momentum fashion and obtain the final update rule
    $$\Delta\mW_t^\star \propto -f(\mM_t).$$
    We will come back to more discussions on the momentum-first designs in \Cref{sec: aro_overview}, \Cref{remark: momentum_first}.
    \end{remark}
\end{remarks}

\begin{remarks}
    \begin{remark}[Stateless and stateful projection functions]
        In the derivation of normed steepest descent, $f(\cdot)$ is a \textbf{stateless} function that does not maintain any internal optimizer states. For example, for $\|\cdot\| = \|\cdot\|_\infty$, $f(\cdot) = \text{Sign}(\cdot)$ \citep{bernstein2024old}.
        Corresponding stateful, adaptive extensions (e.g., $\text{Sign}(\cdot) \Rightarrow$ Adam) can be found in recent work on adaptive normed descent \citep{veprikov2025preconditioned, xie2025tale}, as well as prior work of whitening with structured Fisher information~\citep{gong2025towards}. In the stateful setup, $f(\cdot)$ depends on time $t$ and maintains an internal state that can be updated during optimization. We use $f_t$ to denote the general stateful projection functions. Throughout \Cref{sec: warmup,sec: aro} we assume that $f_t$ is stateless without loss of generality. In \Cref{sec: handling stateful optimizer}, we provide an extended discussion with stateful $f_t$ examples.
    \end{remark}
\end{remarks}

Throughout the paper we will use the term \emph{base optimizer} or \emph{projection function} interchangeably to refer to $f$ or $f_t$, since it is usually served as the base optimizer of \name{} that will be rotated.

\subsection{Rotated steepest descent}
\label{subsec: steepest descent under rotation}

\paragraph{Loss in rotated coordinates.} 
Let $\mathcal{L}(\mW_t)$ be some loss function defined on matrix parameters $\mW_t$.
Let $\mR_t \in SO(m) := \{\mR_t\in\mathbb{R}^{m\times m}\,|\,\mR_t^\top \mR_t=\mI,\ \det(\mR_t)=1\}$ be an orthogonal rotation acting on the row space of $\mW_t$. Introduce rotated coordinates $\mZ_t \in \mathbb{R}^{m\times n}$ via
$$
\mW_t \;=\; \mR_t\,\mZ_t.
$$
Optimizing over $\mW_t$ is then equivalent to optimizing over $\mZ_t$ under the \emph{rotated loss}
$$
\widetilde{\mathcal{L}}_{\mR_t}(\mZ_t)
\;=\;
\mathcal{L}\!\big(\mR_t\,\mZ_t\big).
$$

\paragraph{Gradients in rotated coordinates.}
By the chain rule,
$$
\nabla_{\mZ_t}\widetilde{\mathcal{L}}_{\mR_t}(\mZ_t)
\;=\;
\mR_t^\top \nabla_{\mW_t}\mathcal{L}(\mR_t\mZ_t)
\;=\;
\mR_t^\top \mG_t,
$$
so the gradient is simply left-multiplied by $\mR_t^\top$ in the rotated coordinates.

\paragraph{Applying a base rule in rotated coordinates.}
Applying the base optimizer $f_t$ in the rotated coordinates with step size $\eta$, we have:
$$
\Delta \mZ_t = -\eta f_t(\mR_t^\top \mG_t).
$$
Mapping back to original coordinates:
$$
\Delta \mW_t = \mR_t \Delta \mZ_t = -\eta \mR_t f_t(\mR_t^\top \mG_t).
$$

\paragraph{General rotated-optimizer form.}
Including step size $\eta>0$, the update becomes:
\begin{empheq}[box=\fbox]{equation}\label{eq: rotated-optimizer}
\Delta \mW_t = -\eta \mR_t f_t(\mR_t^{\top} \mG_t)
\end{empheq}
We call it \textbf{rotated steepest descent}. This template recovers existing works that exclusively set  $\mR_t$ to be \emph{eigen-rotation}, i.e., the left eigenvectors of $\mG_t \mG_t^\top$ / singular vectors of $\mG_t$; and set $f_t$ to be variants of Adam, leading to different matrix optimizers (Adam $\rightarrow$ SOAP \citep{vyas2024soap}, Row-Normalization $\rightarrow$ Muon \citep{jordan2024muon}, SignGD $\rightarrow$ Splus \citep{frans2025stable}), see \Cref{tab: rotated_optimizers}.

\section{Methodology: Adaptively Rotated Optimization}  \label{sec: aro}
  
\subsection{\name{} update overview}  \label{sec: aro_overview}

We consider \Cref{eq: rotated-optimizer} as a starting principle for designing matrix optimizers.
The base optimizer $f_t$ specifies how gradients are transformed in a rotated coordinate system and the rotation $\mR_t \in SO(m)$ specifies the coordinate system itself.
Most existing matrix optimizers (after turning off accumulations) choose $\mR_t$ to be an eigen-rotation (i.e., eigenvectors of $\mG_t\mG_t^\top$/singular vectors of $\mG_t$), and then consider a particular choice of $f_t$ independently. In contrast, we propose to treat $\mR_t$ as a \emph{rotation policy conditioned on $f_t$}, enabling non-eigen rotations that are explicitly informed by the chosen projection rule.

\paragraph{Proposed method.}
We introduce \textbf{Adaptively Rotated Optimization (\name{})}: a family of matrix optimizers that instantiates \Cref{eq: rotated-optimizer} and {adapts the rotation using the geometry induced by $f_t$}. Intuitively, \name{} selects $\mR_t$ to better align the rotated coordinate system with directions that yield a larger \emph{instantaneous} loss decrease under the update induced by $f_t$ (see \Cref{sec: instantaneous} for derivation). Ignoring auxiliary moving-average accumulations for clarity, the \name{} update proceeds in two steps:
\begin{equation*}
\begin{aligned}
\textbf{(Rotation selection)}\qquad
\mR_{t} &= \mathrm{QR}\!\big(\,\mG_t\, f_t( \mR_{t-1}^\top \mG_t)^\top \,\big),\\ 
\textbf{(Rotated update)}\qquad
\Delta \mW_t \;&=\; -\,\eta\;\mR_t\; f_t\!\big(\mR_t^\top \mG_t\big),
\end{aligned}
\end{equation*}
where $\mathrm{QR}(\cdot)$ returns the left orthonormal factor (the $\mQ \in SO(m)$ factor) of its matrix argument. %

\paragraph{Intuition.} The rotation update forms the product $\mG_t\, f_t(\mR_{t-1}^\top\mG_t)^\top $, which can be viewed as a cross-alignment/cross-Gram matrix between the raw gradient update $\mG_t$ and the update proposed by $f_t$ in the previous rotated coordinates. Applying $\mathrm{QR}$ to it yields an orthonormal basis spanning the subspace of gradient directions that are coupled to the base optimizer’s transformation in the rotated frame. \name{} rotates gradients not only according to the geometry of $\mG_t$, but also according to how the chosen base optimizer would like to transform $\mR_{t-1}^\top\mG_t$.
Compared with classic power iteration $\mR_t = \text{QR}(\mG_t\mG_t^\top \mR_{t-1})$ \citep{vyas2024soap, gong2025towards}, here  we inserted an additional $f_t$ function to project the rotated gradient $\mR_{t}^\top \mG_t$ factor on the right. Setting the $f_t$ to identity from our rotation selection equation reduces the resulting equation to standard power-iteration update for eigenvector computation. In both \Cref{sec: instantaneous} and \Cref{subsec: discussion denoising} we show that \name{} rotation can be understood as \emph{improving the loss decrease rate over the eigen-rotations}.

\paragraph{Adding momentum.} In practice, we incorporate momentum into our update rule, and utilize shifted cholesky QR decomposition \citep{fukaya2020shifted} to speedup the matrix factorization (\Cref{sec:chol_qr}). In this paper we consider a momentum-first approach (see \Cref{remark: momentum_first_nsd}, \Cref{remark: momentum_first} and \Cref{tab: momentum_first_updates}) as it has been shown successful in \citep{jordan2024muon}. This gives:

\begin{empheq}[box=\fbox]{equation}\label{eq: gf}
\begin{aligned}
\mR_{t} &= \text{QR}(\mM_t f_t(\mR_{t-1}^\top \mM_t)^\top) \\
\Delta \mW_t & = -\eta \mR_{t} f_t(\mR_{t}^\top \mM_t )
\end{aligned}
\end{empheq}

$\mM_t = \text{EMA}[\mG_t] := \beta \mM_{t-1} + (1 - \beta) \mG_t$ is the momentum of the gradient with momentum coefficient $\beta$. This realizes a non-eigen, optimizer-aware rotation that can be used uniformly across all matrix parameters. Below we discuss the relation to gradient orthogonalization, and the choice of $f_t$.

Note that although the \name{} update rule may appear complex at first glance, in \Cref{sec:symmetry} we will show that it can be derived based on fundamental symmetry assumptions about the loss landscape. In this section we focus on the instantaneous loss decrease perspective since it needs the least mathematical preparation.

\begin{table}[t]
\centering
\scriptsize
\setlength{\tabcolsep}{5pt}
\renewcommand{\arraystretch}{1.18}
\begin{tabular}{@{}
 >{\centering\arraybackslash}m{0.14\linewidth}
 >{\centering\arraybackslash}m{0.42\linewidth}
 >{\centering\arraybackslash}m{0.42\linewidth}
@{}}
\toprule
\textbf{Method} & \textbf{Original} & \textbf{Momentum-first variant} \\
\midrule

\rowcolor{soapgreen}
SOAP \citep{vyas2024soap}
&
$ \begin{aligned}
\mC_t &= \text{EMA}[\mG_t\mG_t^\top],\quad \mR_t=\texttt{Eigenvectors}(\mC_t)\\
\mM_t &= \text{EMA}[\mG_t],\quad \mV_t=\text{EMA}\!\big[(\mR_t^\top \mG_t)^{\odot 2}\big]\\
\Delta\mW_t &\propto -\,\mR_t\!\Big(\mR_t^\top \mM_t \oslash \sqrt{\mV_t}\Big)
\end{aligned} $
&
$ \begin{aligned}
\mM_t &= \text{EMA}[\mG_t],\quad \mR_t=\texttt{Eigenvectors}(\mM_t\mM_t^\top)\\
\mV_t &= \text{EMA}\!\big[(\mR_t^\top \mM_t)^{\odot 2}\big]\\
\Delta\mW_t &\propto -\,\mR_t\!\Big(\mR_t^\top \mM_t \oslash \sqrt{\mV_t}\Big)
\end{aligned} $
\\

\addlinespace[0.45em]
\rowcolor{spluspurple}
SPlus \citep{frans2025stable}
&
$ \begin{aligned}
\mC_t &= \text{EMA}[\mG_t\mG_t^\top],\quad \mR_t=\texttt{Eigenvectors}(\mC_t)\\
\Delta\mW_t &\propto -\,\mR_t\,\text{Sign}(\mR_t^\top \mM_t)
\end{aligned} $
&
$ \begin{aligned}
\mM_t &= \text{EMA}[\mG_t],\quad \mR_t=\texttt{Eigenvectors}(\mM_t\mM_t^\top)\\
\Delta\mW_t &\propto -\,\mR_t\,\text{Sign}(\mR_t^\top \mM_t)
\end{aligned} $
\\

\addlinespace[0.45em]
\rowcolor{muonorange}
Muon (with idealized orthogonalization) \citep{jordan2024muon, ahn2025dion}
&
\multicolumn{2}{>{\centering\arraybackslash}m{\dimexpr 0.84\linewidth + 2\tabcolsep\relax}}{%
$ \begin{aligned}
\mM_t &= \text{EMA}[\mG_t],\quad \mR_t=\texttt{Eigenvectors}(\mM_t\mM_t^\top)\\
f_{\text{RN}}(\mX) &= \sqrt{n}\,Q(\mX)^{-1}\mX\\
\Delta\mW_t &\propto -\,\mR_t\, f_{\text{RN}}(\mR_t^\top \mM_t)
\end{aligned} $} \\

\addlinespace[0.45em]
\rowcolor{flatblue}
\name{} (Ours)
&
\multicolumn{2}{>{\centering\arraybackslash}m{\dimexpr 0.84\linewidth + 2\tabcolsep\relax}}{%
$ \begin{aligned}
\mM_t &= \text{EMA}[\mG_t],\quad
\mR_t=\mathrm{QR}\!\big(\mM_t\, f_t(\mR_{t-1}^\top \mM_t)^\top\big)\\
\Delta\mW_t &\propto -\,\mR_t\, f_t(\mR_t^\top \mM_t)
\end{aligned} $} \\
\bottomrule
\end{tabular}
\caption{Comparing the original vs.\ momentum-first implementations for matrix optimizers (one-sided/left rotation only). Original implementation usually involves storing additional buffers for the accumulation of $\mG_t\mG_t^\top$, where momentum-first implementation (popularized by Muon) uses only one buffer shared between the momentum and rotation estimation. For SOAP, its momentum-first implementation interestingly yields a two-fold Adam (AdaMomentum) rule \citep{wang2021rethinking} (i.e., $\mV_t$ is accumulated using $\mM_t$ rather than $\mG_t$), which has been observed to yield better generalization than standard Adam. See \Cref{sec: handling stateful optimizer} for detailed discussion. }
\label{tab: momentum_first_updates}
\end{table}

\begin{remarks}
    \begin{remark}[momentum-first design of rotations] \label{remark: momentum_first} Both the update from \Cref{eq: gf} and Muon optimizer follow a \textbf{momentum-first rotation} design (see \Cref{tab: momentum_first_updates}), akin to the proposal from Adadiag \citep{nguyen2025improving}. This means that when estimating rotations using $\mR_{t} = \text{QR}(\mM_t f_t(\mR_{t}^\top \mM_t)^\top)$, we reuse the same momentum buffer $\mM_t$ that is being used in the update equation $\Delta \mW_t = -\eta \mR_{t} f_t(\mR_{t}^\top \mM_t )$. Similarly, for eigen-space rotations, a momentum-first implementation results in the power iteration $\mR_{t} = \text{QR}(\mM_t \mM_t^\top \mR_{t-1})$.
    This is opposed to the original eigen-rotation implementation used in SOAP \citep{vyas2024soap} (in the one-sided case) $\mR_{t} = \text{QR}(\text{EMA} [\mG_t \mG_t^\top] \mR_{t-1})$, where SOAP accumulates a separate EMA buffer $\text{EMA} [\mG_t \mG_t^\top]$ different from $\mM_t \mM_t^{\top}$.
    In our exploratory analysis we have found that momentum-first implementations give similar performance to the original implementation, while consuming much less optimizer state memory.
    Therefore, throughout this paper for all variants of \name{} and baselines (including SOAP, SPlus, Muon), we always use momentum-first implementation.
    \end{remark}
\end{remarks}

\begin{figure}
    \centering
    \includegraphics[width=1.0\linewidth]{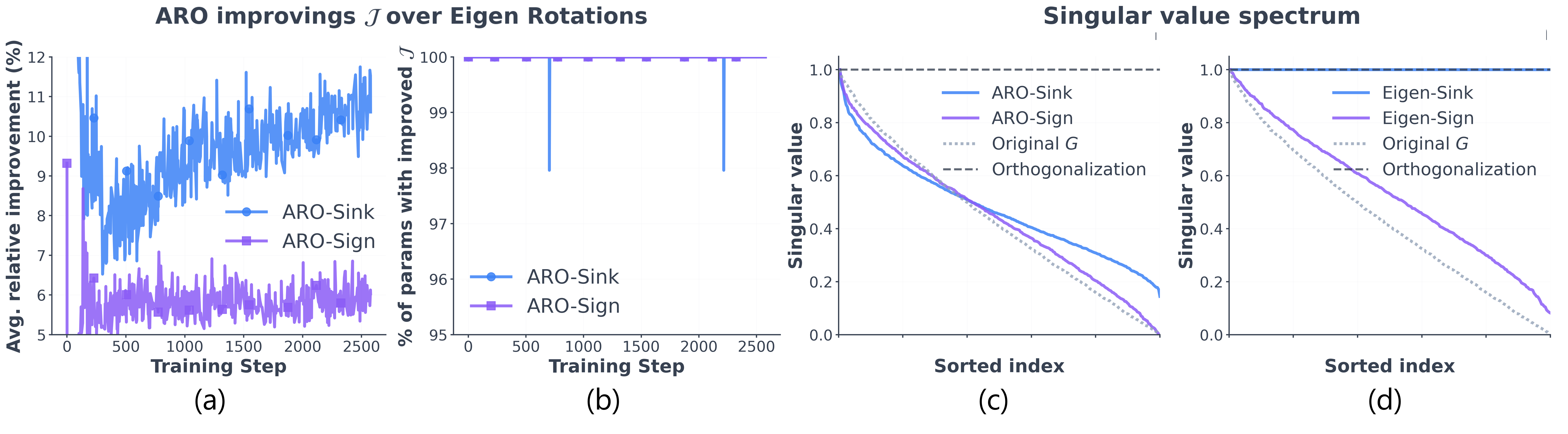}
    \caption{Properties of \name{} rotations. \textbf{(a)-(b) ARO rotations give higher rotation objective value.} During 130M nanoGPT training with \name{} using SinkGD \citep{scetbon2025gradient} or SignGD \citep{bernstein2018signsgd} as base optimizers, we evaluate at each step $t$ and for each momentum matrix $\mM_t$ the objective $\mathcal{J}(\mR;\mM_t,f)$ for the ARO rotation $\mR_t^{\text{ARO}}$ (\Cref{eq: gf_r}) and the exact eigen-rotation baseline $\mR_t^{\text{eig}}$ (\Cref{eq: eig_rotation}). (a) shows the mean relative gain $\big(\mathcal{J}(\mR_t^{\textsc{aro}})-\mathcal{J}(\mR_t^{\text{eig}})\big)/|\mathcal{J}(\mR_t^{\text{eig}})|$ (percent) over all parameters, and (b) shows the fraction of parameters with $\mathcal{J}(\mR_t^{\textsc{aro}})>\mathcal{J}(\mR_t^{\text{eig}})$. ARO scores higher for most parameters throughout training. \textbf{(c)-(d) \name{} updates are non-orthonormal.}
    We plot the singular value spectrum (normalized to have maximum 1) of the update matrix on a random Gaussian gradient. (c) Both \arosink{} and ARO-Sign updates deviate from orthogonalization. (d) Using $\mR_t^{\text{eig}}$ instead over SinkGD/SignGD base optimizers yields rotated updates closer to orthonormal, with Eigen-Sink exactly orthonormal. Exact definitions of baselines can be found in \Cref{sec: baselines}.} 
    \label{fig: stats_of_aro}
\end{figure}

\subsection{\name{} as instantaneous loss decrease rate optimization}  \label{sec: instantaneous}

\paragraph{Instantaneous loss decrease rate.} Consider the gradient flow of rotated steepest descent with base optimizer~$f_t$:  
\begin{equation}  
\dot{\mW_t} \;=\; -\,\mR_t\, f_t(\mR_t^\top \mG_t), \qquad \mR_t\in SO(m).  
\label{eq: ARO update}
\end{equation}  
The corresponding instantaneous loss decrease is  
\begin{equation}  
\frac{d}{dt}\,\loss(\mW_t)\;=\;\langle \mG_t, \dot{\mW_t}\rangle\;=\;-\,\underbrace{\langle \mG_t,\mR_t f_t(\mR_t^\top \mG_t)\rangle}_{\mathcal{J}(\mR_t;\mG_t,f)}.  \label{eq: J}
\end{equation}  
Intuitively, improving $\mathcal{J}(\mR_t;\mG_t,f)= \mathrm{tr}(\mG_t^\top \mR_t f(\mR_t^\top \mG_t))$ selects the rotation that results in better instantaneous loss decrease rate, hence potentially better sample efficiency. However, directly optimizing $\mathcal{J}$ over $SO(m)$ is difficult. For example, for $f_t=\text{Sign}(\cdot)$ this becomes classic $l_1$-PCA problem which is NP hard \citep{markopoulos2017efficient}. 

\paragraph{Approximate solution.} We therefore adopt a one-shot, closed-form rotation rule obtained by \emph{decoupling} the two appearances of $\mR_t$ in $\mathcal{J}$. At iteration $t$ we evaluate the projected term using a previous rotation value, and optimize only the remaining linear dependence on $\mR_t$. That is, we solve
\begin{equation}  
\mR_t^\star \;\in\;\argmax_{\mR_t\in SO(m)}\; \mathrm{tr}\big(\mG_t\,f_t(\mR_{t-1}^\top\mG_t )^\top \,\mR_t^\top\big). \label{eq: procrustes} 
\end{equation} 
The maximization problem in \Cref{eq: procrustes} is known as the orthogonal Procrustes problem \citep{schonemann1966generalized, myronenko2009closed}. Its solution is given by 
\begin{equation}
    \mR_{t}^\star = \mathrm{Orthogonalize}\!\big(\,\mG_t\, f_t( \mR_{t-1}^\top \mG_t)^\top \,\big).
\end{equation}
In \name{}, we choose QR decomposition as the orthogonalization operation due to its superior practical performance and robustness to gradient noise (more discussion is deferred to \cref{subsec: discussion denoising}). This gives:
\begin{equation}
    \mR_{t}^{\text{ARO}} = \mathrm{QR}\!\big(\,\mG_t\, f_t( {\mR_{t-1}^{\text{ARO}}}^\top \mG_t)^\top \,\big) .\label{eq: gf_r}
\end{equation}

\paragraph{\name{} update.} We apply $f_t$ in the new basis, turn on the accumulation on $\mG_t$ following momentum-first design (i.e., replacing $\mG_t$ with $\mM_t$, see \Cref{tab: momentum_first_updates}). This gives the \name{} update from \Eqref{eq: gf}:
\begin{align*}
\mR_{t}^{\text{ARO}} &= \text{QR}(\mM_t f_t({\mR_{t-1}^{\text{ARO}}}^\top \mM_t)^\top) \\
\Delta \mW_t & = -\eta \mR_{t}^{\text{ARO}} f_t({\mR_{t}^{\text{ARO}}}^\top \mM_t )
\end{align*}
In practice, we found that \Cref{eq: gf_r} consistently improves $\mathcal{J}(\mR_t;\mM_t,f)$ over the canonical eigen-rotation solution (under momentum-first design presented in \Cref{remark: momentum_first}), 
\begin{equation}
    \mR_t^{\text{eig}} = \texttt{Eigenvectors}[\mM_t \mM_t^T]. \label{eq: eig_rotation}
\end{equation}
See \Cref{fig: stats_of_aro} (a-b) for details. Therefore, \name{} rotation can be understood as improving the eigen-rotation toward a better loss decrease rate.

\begin{remarks}
    \begin{remark}[Discrete case]
        So far the objective $\mathcal{J}$ is derived from the continuous-time rate with infinitesimal step size. A more proper treatment is to use descent lemma and perform discrete-time analysis for finite step size, which we defer to \Cref{sec: importance_of_symmetry_in_ours}.
    \end{remark}
\end{remarks}

\begin{remarks}
    \begin{remark}[\name{} has no effect on rotational-equivariant $f_t$] \label{remark: no_effect_on_rot_equivariant} If $f_t$ of choice is rotational-equivariant, we have $\mR_{t} f_t(\mR_{t}^\top \mM_t ) = \mR_{t} \mR_{t}^\top f_t( \mM_t ) = f_t( \mM_t ) $. That is, \name{} has no effect compared with the base optimizer. Examples of such $f_t$ include spectral descent, gradient descent, and similar.
    \end{remark}
\end{remarks}

\subsection{Extended features of \name{}}

Treating gradient rotations as a core design principle not only yields new update rules in \name{}, but also unlocks additional capabilities for \name{}. We defer the motivation and detailed discussion of these features to \Cref{sec:symmetry}, since they rely on further prerequisites. Below we provide a brief preview.

\paragraph{Full-model mode}
\name{} supports a full-model mode, meaning that we apply \name{} to essentially all matrix or vector-valued parameters. The only exceptions are scalar parameters, which are rare in modern LLM architectures. For tensors of order $\leq 2$, including embedding and LM head parameters, we apply the same rotation rule as in \Cref{eq: gf}. For 1D parameters, we set $\mR_t=\mI$ (no rotation), or optionally sharing a rotation estimated from other related modules. We validate this feature at large scale in \Cref{subsec: Exp Sigma-MoE} and discuss it as a corollary of the symmetry hypothesis in \Cref{sec:symmetry}.

\paragraph{Exploiting cross-layer and cross-module couplings}
In \Cref{sec:symmetry}, we formalize rotations as a unified and economical interface for exploiting correlations across parameters from different layers and modules. In particular, these couplings are induced by symmetry structures of LLM loss landscapes and can be implemented as architecture $+$ topology-derived constraints on rotation estimation.  This will be discussed in \Cref{sec:symmetry} and validated in \Cref{sec: symmetry_predictions}. Related works are discussed in \Cref{sec: related}.

\subsection{Gradient orthogonalization as a special case of \name{}} \label{sec: ortho_as_special_case}

Below we discuss the connections between \name{} and gradient orthogonalization techniques. 

\paragraph{\name{} produce non-orthonormal updates.} The \name{} update $\Delta \mW_t $ given by \Cref{eq: gf} is not orthonormal in general. As shown in \Cref{fig: stats_of_aro} (c-d), \name{} significantly deviates from orthogonalization on randomly sampled gradients. Moreover, compared to \name{} rotations, eigen-rotations drive updates towards orthonormal updates. In \Cref{subsec: universal effective} and \Cref{subsec: Exp Sigma-MoE}, we show that \name{}-families consistently outperform both Eigen-families as well as gradient orthogonalization, indicating that the canonical choice of eigen rotations or gradient orthogonalization might not be optimal.

\paragraph{Gradient orthogonalization is a special case.} \name{} reduces to gradient orthogonalization when $f_t$ is chosen to be row normalization
\begin{align}\label{eq:rownorm-optimizer}
    f_{RN}(\mX) = \sqrt{n} Q(\mX)^{-1} \mX,
\end{align}
where $\mX$ has shape $m$ by $n$, $Q(\mX) = \text{Diag}(\| \mX_{1,:}\|_2,...,\| \mX_{m,:}\|_2)$ is the diagonal matrix whose diagonal coefficients are the $l_2$-norm of the rows of $\mX$. The optimizer corresponds to
steepest descent under max-row norm $\|\mW_t\| =\frac{\max\limits_{i\in[|1,m|]} \Vert W_{i,:}\Vert_2}{\sqrt{n}}$; see \citep{scetbon2025gradient, bernstein2024old}.
The equivalence between \name{} and gradient orthogonalization under $f_t(\mX) = \sqrt{n} Q(\mX)^{-1} \mX $ follows after noticing
$$\mR_{t} = \text{QR}(\mM_t f_{RN}(\mR_{t-1}^\top \mM_t)^\top) = \text{QR}(\mM_t \mM_t^\top \mR_{t-1} Q(\mR_{t-1}^\top \mM_t)^{-1}) = \text{QR}(\mM_t \mM_t^\top \mR_{t-1}), $$
which reduces to one step of standard power iteration of solving left-singular subspace $\mU_t$ of $\mM_t = \mU_t \mS_t \mV_t^\top$. At convergence of power iteration, we get $\mR_t = \mU_t$. Therefore, $$\Delta \mW_t  = -\eta \mU_{t} f_{RN}(\mU_{t}^\top \mM_t) = -\eta \mU_{t} Q(\mS_t \mV_t^\top)^{-1} (\mS_t \mV_t^\top) = -\eta \mU_t\mV_t^\top,$$ which is exactly the idealized  gradient orthogonalization/spectral descent step. Note that this is technically not equivalent to the Muon's Newton-Schulz (NS) implementation: as suggested by \citep{ahn2025dion}, the power-iteration-based approach might offer slightly better orthogonalization quality compared with NS. In our experiments, we found that depending on the setup, power-iteration-based approach can be both better (\Cref{fig: m4}) or worse (\Cref{fig:sigma speedup}) than NS-based approach.

\begin{remarks}
\begin{remark}[Understanding gradient orthogonalization]
Our analysis also clarifies how gradient orthogonalization fits into the \name{} framework. Viewed through \name{}, gradient orthogonalization corresponds to a rotated update applied to steepest descent under the max-row norm, where the rotation arises from the spectral norm geometry. This perspective separates the role of the rotation from the choice of base norm itself. Consistent with this interpretation, \name{} uncovers a broader optimizer family that outperforms gradient orthogonalization (\Cref{subsec: universal effective}, \Cref{exp: gpt}, \Cref{fig:sigma speedup}).
\end{remark}
\end{remarks}

\section{Practical Considerations}

\subsection{Fast and stable rotation with shifted Cholesky QR}
\label{sec:chol_qr}
Computing the QR decomposition required for rotations is essential in \name{}, but standard QR decomposition is expensive for large batched matrices. We adopt a shifted Cholesky-QR (SCQR) method \citep{fukaya2020shifted} for efficiency while maintaining numerical stability, as well as better performance on non-Nvida GPUs. Given a matrix $\mA \in \mathbb{R}^{m \times n}$, we first form the regularized Gram matrix
$$
\mP = \mA^\top \mA + \epsilon \mI_n,
$$
where $\epsilon \mI_n$ is a small regularization term to improve numerical stability and conditioning. We then compute its Cholesky factorization $\mP = \mL\mL^\top$ and obtain the orthonormal basis via
$$
\mQ = \mA \mL^{-1}.
$$
Since $\mL$ is triangular, $\mL^{-1}$ can be solved efficiently with forward substitution in $\mathcal{O}(n^2)$. This approach is significantly faster than full QR when $\mP$ is well-conditioned, and we fall back to standard QR only for batches where Cholesky fails or produces non-finite values. In practice, this drop-in replacement achieves near-QR accuracy with substantial speedups with rare fall backs, making re-estimating the rotation \emph{at every training step}  practical for large-scale training. 

When writing the manuscript, we found that the concurrent work of \citep{ahn2025dion} also implemented Cholesky QR (without regularization) as an alternative QR decomposition. However, we found that in our setup, the vanilla Cholesky QR falls back to QR more frequently at large scale settings and becomes less efficient.

\subsection{SinkGD: a new base optimizer for rotated optimization} \label{sec: sink}

The \name{} update rule \Cref{eq: gf} is very general and can in principle be applied to any base optimizer $f_t$;
we found that \name{} universally improves the performance of many base optimizers to a competitive level (\Cref{subsec: universal effective}, \Cref{fig: m1.1}). This allows us to extend beyond Adam variants, and consider advanced base optimizer such as SinkGD \citep{scetbon2025gradient}. SinkGD is a lightweight, stateless matrix optimizer recently proposed in our prior work, and achieved a good balance between simplicity and practical performance \citep{scetbon2025gradient}. We name \name{}-rotated SinkGD optimizer as \emph{\arosink{}}.

\paragraph{SinkGD: a lightweight, stateless base optimizer.}
Given the gradient $\mG_t \in \R^{m \times n}$ and iteration count $L \ge 1$, SinkGD applies $L$ rounds of alternating row- and column-wise $\ell_2$ normalization. The correspoding $f_t = f_{\text{Sink}}$ base function is given by \Cref{alg: sink}, and the original (un-rotated) SinkGD optimizer is given by \cref{alg: sinkgd}. Intuitively, SinkGD performs iterative row and column-wise normalization on gradient matrices before update. In the original SinkGD algorithm presented in \citep{scetbon2025gradient}, the row and column normalization are performed in a alternating manner, which corresponds to a reparameterization of the classical Sinkhorn-Knopp algorithm \citep{sinkhorn1967concerning}. In this paper, we made a simple modification where the row and column normalization are performed simultaneously (\Cref{alg: sink}). This is inspired by the variant presented in \citep{knight2014symmetry}, which gives slightly better performance. Following \citep{scetbon2025gradient}, we use Sinkhorn iteration $L=5$.

\begin{figure}[h!]
    \centering
    \begin{minipage}[t]{0.49\textwidth}
        \centering
        \begin{algorithm}[H]
           \caption{$f_\text{Sink}(\mG,L)$}
           \label{alg: sink}
        \begin{algorithmic}
           \STATE Initialize $\mX= \mG$.
           \FOR{$\ell=1$  {\bfseries to} $L$}
           \STATE $Q(\mX) =\text{Diag}(\Vert \mX_{1,:}\Vert_2, \dots, \Vert \mX_{m,:}\Vert_2)$
           \STATE $R(\mX) =\text{Diag}(\Vert \mX_{:,1}\Vert_2,\dots,\Vert \mX_{:,n}\Vert_2)$
           \STATE $\mX \propto Q(\mX)^{-1}\mX R(\mX)^{-1}$
           \ENDFOR
            \STATE Return $\mX$
        \end{algorithmic}
        \end{algorithm}
    \end{minipage}
    \hfill
    \begin{minipage}[t]{0.49\textwidth}
        \centering
        { 
        \begin{algorithm}[H]
           \caption{SinkGD update rule}
           \label{alg: sinkgd}
        \begin{algorithmic}
            \STATE {\bfseries Input:} $T\geq 1$ the number of updates, $(\eta_t)_{0\leq t\leq T}$ the global step-sizes, $\mathcal{L}$ the loss to minimize, and $L\geq 1$ the number $f_{\text{Sink}}$ iterations.
           \FOR{$t=1$  {\bfseries to} $T$}
           \STATE $\mG_t = \nabla_{\mW_t}\gL(\mW_t)$
            \STATE $\mW_{t+1} \gets \mW_t - \eta_t f_{\text{Sink}}(\mG_t, L) $
            \ENDFOR
            \STATE Return $\mW_t$
        \end{algorithmic}
        \end{algorithm}
        \vspace{-1.3em}
        }
    \end{minipage}
\end{figure}

\paragraph{\arosink{}.} The update rule of \arosink{} is given by:
\begin{align*}
\mR_{t} &= \text{QR}(\mM_t f_{\text{Sink}}(\mR_{t-1}^\top \mM_t)^\top) \\
\Delta \mW_t & = -\eta \mR_{t} f_{\text{Sink}}(\mR_{t}^\top \mM_t )
\end{align*}

\paragraph{Why SinkGD?} The main reasons for choosing SinkGD as our main base optimizer are as follows.

\begin{itemize}
    \item SinkGD is competitive optimizer while being \emph{stateless}, i.e., it does not maintain any moving average states as in AdamW. As shown in \citep{scetbon2025gradient}, SinkGD is able to match the performance of full AdamW in a small scale setting. This is a very convenient property for \name{} since the additional memory overhead due to $f_t$ is eliminated. With $f_{\text{Sink}}$, \arosink{} only need to maintain two optimizer states: the rotation $\mR_t$, and the momentum $\mM_t$, which has the same memory overhead as AdamW.
    \item $f_{\text{Sink}}$ is not rotational equivariant, a property needed for \name{} to be applicable (see \Cref{remark: no_effect_on_rot_equivariant}); other matrix optimizers such as Muon are already rotational equivariant and cannot be further boosted by \name{}.
    \item Many base optimizers such as Adam use variance adaptation mechanisms, while SinkGD does not. By using $f_{\text{Sink}}$, the resulting \arosink{} is a pure geometric solution, allowing us to disentangle from other orthogonal contributions from variance adaptation \citep{frans2025really}.
\end{itemize}

\paragraph{Computation, communication and memory overheads.} Applying the $\gO(mn)$-cost $f_{\text{Sink}}$ operation twice in \arosink{} could introduce some overhead. However, in practice, under large-batch distributed training, forward and backward passes dominate the total cost. As a result, \arosink{} achieves nearly the same throughput as AdamW, with only about a $1\%$ slowdown in large-scale models (\Cref{subsec: qwen3-8B}, \Cref{fig: Qwen throughput}). The communication required by $f_{\text{Sink}}$ is fully shared with that of $\mR_t$, so it incurs no additional distribution cost. Finally, since we target large-scale distributed training where optimizer states are typically sharded, the AdamW-like memory overhead is not a limiting factor, and further memory optimization would bring limited benefit.  Note that both claims above does not necessarily apply to small scale settings (small model, small batch size, no sufficient parallelization etc), where optimizer time is more dominant compared with forward-backward time.

\subsection{Handling other stateful base optimizers $f_t$}
\label{sec: handling stateful optimizer}

In previous sections, without loss of generality, we have treated the base optimizer $f_t(\cdot)$ as a stateless projection function. As discussed in \Cref{remark: momentum_first}, $f_t$ acts on the momentum $\mM_t$ (or its rotated version $\mR_t^\top \mM_t$). Many practical optimizers, however, maintain internal statistics \citep{xie2025tale, veprikov2025preconditioned}, and prior work \citep{frans2025really, vyas2024soap, gong2025towards} shows that such stateful $f_t$ can improve gradient whitening via variance adaptation. In this section, we describe how \name{} can rely on stateful base projection functions.

\paragraph{Projection function for Adam.} 
According to \Cref{remark: momentum_first_nsd}, let
$$
\mM_t = \beta \mM_{t-1} + (1 - \beta)\,\mG_t
$$
denote the first-order EMA of the gradient. We treat $\mM_t$ as \emph{input} to the base optimizer and maintain the second-moment accumulator $\mV_t$ depend only on momentum input. At step $t$, the update rule is given by
$$
\mV_t = \beta_2 \mV_{t-1} + (1 - \beta_2)\,\mM_t \odot \mM_t,
$$
where $\odot$ represents element-wise product. Because $\mV_t$ is built from $\mM_t$ instead of the raw gradient $\mG_t$, this update coincides with two-fold Adam (AdaMomentum) \citep{wang2021rethinking}, which has been observed to yield better generalization than standard Adam. 

Given $\mM_t$ and the current state $\mV_t$,  we can write down the corresponding stateful projection function of (two-fold) Adam:
$$
f_t(\mM_t;\mV_t) := \frac{\mM_t}{\sqrt{\mV_t}},
$$
with the square root applied element-wise. The parameter update produced by Adam in our framework is therefore
$$
\Delta \mW_t \propto - f_t(\mM_t;\mV_t) = -\frac{\mM_t}{\sqrt{\mV_t}}.
$$

To emphasize the dependence on the state we write $f_t(\mM_t;\mV_t)$, but conceptually $f_t$ is still a projection in its first argument. The first-order EMA $\mM_t$ is \emph{not} counted as part of the internal state of $f_t$.

\paragraph{General stateful projection function.} 
The Adam example illustrates a pattern that covers many adaptive optimizers. Under a momentum-first design, we can define a stateful based optimizer with:
\begin{itemize}
    \item internal state tensors $\{\mS_t^{(k)}\}_{k=1}^K$,
    \item a state-update rule
    $$
    \{\mS_t^{(k)}\}_{k=1}^K
    = \Phi_t\big(\{\mS_{t-1}^{(k)}\}_{k=1}^K,\,\mM_t\big),
    $$
    \item and a projection map
    $$
    f_t(\mM_t;\{\mS_t^{(k)}\}_{k=1}^K),
    $$
    which returns the search direction at step $t$.
\end{itemize}
By convention, the states are always updated using the same quantity that $f_t$ acts on: either $\mM_t$ or a rotated version (e.g., $\mR_t^\top \mM_t$). This ensures that all statistics in $\{\mS_t^{(k)}\}_{k=1}^K$ are computed from the same EMA-filtered gradients that drive the projection. For Adam in the previous paragraph, we have $K=1$ with $\mS_t^{(1)} = \mV_t$, $\Phi_t$ given by the EMA update $\mV_t = \beta_2 \mV_{t-1} + (1 - \beta_2)\,\mM_t \odot \mM_t$, and $f_t(\mM_t;\mV_t) = \mM_t/\sqrt{\mV_t}$.

\paragraph{Integration with \name{}.} 
The subtlety is that the \name{} update \Cref{eq: gf} defines the new rotation $\mR_t$ using the projection of the \emph{old} rotated momentum $\mR_{t-1}^\top \mM_t$. 
At iteration $t$ we have the current momentum $\mM_t$, the previous rotation $\mR_{t-1}$, and the current internal states $\{\mS_{t-1}^{(k)}\}_{k=1}^K$. To compute the new rotation and the update while keeping the states consistent, we conceptually perform two steps:

\begin{enumerate}
    \item \textbf{Look-ahead under the old rotation.}  
    We first form the momentum in the old rotated basis, $\mR_{t-1}^\top \mM_t$. Conceptually, we apply a single state-update step
    $$
        \{\tilde{\mS}_t^{(k)}\}_{k=1}^K
        = \Phi_t\big(\{\mS_{t-1}^{(k)}\}_{k=1}^K,\mR_{t-1}^\top \mM_t\big),
    $$
    and evaluate the lookahead direction
    $$
        \tilde{\mD}_t
        = f_t\big(\mR_{t-1}^\top \mM_t;\{\tilde{\mS}_t^{(k)}\}_{k=1}^K\big).
    $$
    Implementation-wise, this step is carried out \emph{functionally}: the temporary states are realized only in scratch memory via $\Phi_t$ and are not written back to $\{\mS_{t-1}^{(k)}\}_{k=1}^K$. We then plug $\tilde{\mD}_t$ into \Cref{eq: gf} to obtain the new rotation $\mR_t$.
    \item \textbf{Update states with the new rotation.}  
    Once $\mR_t$ is obtained, we project the momentum in the new rotated basis, $\mR_t^\top \mM_t$, and then update the optimizer states
    $$
        \{\mS_t^{(k)}\}_{k=1}^K
        = \Phi_t\big(\{\mS_{t-1}^{(k)}\}_{k=1}^K,\mR_t^\top \mM_t\big).
    $$
    Using these updated states, we compute the actual parameter update
    $$
        \Delta \mW_t
        = -\eta\,\mR_t f_t\big(\mR_t^\top \mM_t;\{\mS_t^{(k)}\}_{k=1}^K\big).
    $$
\end{enumerate}

Returning to the Adam example, the temporary state used to compute the new rotation is
$$
\mV_t^{\mathrm{tmp}}
= \beta_2 \mV_{t-1} + (1 - \beta_2)(\mR_{t-1}^\top \mM_t)\odot(\mR_{t-1}^\top \mM_t),
$$
and \Cref{eq: gf} uses the corresponding look-ahead direction $f_t(\mR_{t-1}^\top \mM_t;\mV_t^{\mathrm{tmp}})$. After $\mR_t$ is computed, we do not keep $\mV_t^{\mathrm{tmp}}$ as a persistent state; instead, we update the second moment with the final rotated momentum,
$$
\mV_t = \beta_2 \mV_{t-1} + (1 - \beta_2)(\mR_t^\top \mM_t)\odot(\mR_{t-1}^\top \mM_t),
$$
and then form the update direction $f_t(\mR_t^\top \mM_t;\mV_t)$. \Cref{alg: stateful optimizer} summarizes an implementation of this procedure for a general stateful optimizer. In practice we implement the temporary state update functionally, without materializing an additional copy of the optimizer state; for EMA-type states (e.g., the second moment of Adam) the lookahead direction
$$
f_t\big(\mR_{t-1}^\top \mM_t;\,\Phi_t(\{\mS_{t-1}^{(k)}\}_{k=1}^K,\mR_{t-1}^\top \mM_t)\big)
$$
can be computed directly from $\{\mS_{t-1}^{(k)}\}_{k=1}^K$ and $\mR_{t-1}^\top \mM_t$ using a scratch buffer.

\begin{algorithm}[H]
   \caption{Handling stateful optimizer with momentum-first design}
   \label{alg: stateful optimizer}
\begin{algorithmic}
   \STATE {\bfseries Input:} Momentum $\mM_t$, rotation $\mR_{t-1}$, base optimizer $f_t$ with states $\{\mS_{t-1}\}_{1}^K$ and state-update rule $\Phi_t$.
   \STATE $\tilde{\mD}_t \gets f_t\big(\mR_{t-1}^\top \mM_t;\,\Phi_t(\{\mS_{t-1}\}_{1}^K,\mR_{t-1}^\top \mM_t)\big)$ \hfill (lookahead; temporary states in scratch)
   \STATE $\mR_t \gets \text{Cholesky\_QR}\big(\mM_t \tilde{\mD}_t^\top\big)$
   \STATE $\{\mS_t\}_{1}^K \gets \Phi_t(\{\mS_{t-1}\}_{1}^K, \mR_t^\top \mM_t)$
   \STATE $\Delta \mW_t \gets -\eta\,\mR_t f_t\big(\mR_t^\top \mM_t;\{\mS_t\}_{1}^K\big)$
   \STATE Return $\Delta \mW_t$, $\mR_t$, $\{\mS_t\}_{1}^K$
\end{algorithmic}
\end{algorithm}

\begin{remarks}
    \begin{remark}[Imperfect update of internal states]
        \Cref{alg: stateful optimizer} deliberately keeps the state-update rule abstract. For many adaptive methods, a ``perfect'' update that is fully consistent with the changing rotations is computationally intractable. Consider Adam as an example. At time $t$, the second-moment accumulator $\mV_{t-1}$ conceptually aggregates the history $\{(\mR_i^\top \mM_i)^2\}_{i=1}^{t-1}$, where each term is computed under a different rotation $\mR_i$. To make $\mV_t$ perfectly consistent with a new rotation $\mR_t$, one would have to (i) undo all past rotations to recover $\{\mM_i\}_{i=1}^{t-1}$, (ii) apply $\mR_t$ to each of them, and (iii) recompute the EMA of their element-wise squares. This would require storing every past momentum $\mM_i$ and rotation $\mR_i$, which is impractical.

        In practice, following \citep{vyas2024soap, gong2025towards, zhao2024galore}, we use the simple approximate update
        \begin{align*}
            \mV_t = \beta_2 \mV_{t-1} + (1 - \beta_2) (\mR_t^\top \mM_t)\odot(\mR_t^\top \mM_t),
        \end{align*}
        which empirically yields good performance. Alternative correction rules have been proposed \citep{robert2024ldadam}; a systematic study of their effect, and the search for improved schemes, is left for future work.
    \end{remark}
\end{remarks}

\subsection{Hyperparameter transfer and alignment for benchmarking}
\label{sec:hyperparam}
Due to the dominance of AdamW optimizers in AI model training, its learning rate setting is relatively well-understood across wide range of model scales. Thus, for any given new proposed optimizer, directly transferring the hyperparameters from Adam enables a good enough shortcut to start with. Historically \citep{agarwallearning} first proposed the idea of learning rate grafting from Adam optimizers. Prior work \citep{gong2025towards, scetbon2025gradient} derived that a wider whitening-based optimizers and normalization/projection-based optimizers should be aligned to have constant RMS update norms. \citep{liu2025muon, kexuefm-11267} further derived that the AdamW has approximately $0.2$ constant RMS norm. Therefore, for all non-Adam baselines including \name{} considered in this paper, we adopt the following re-normalization rule:
\begin{align*}
    \Delta \mW_t \leftarrow 0.2 \frac{\Delta \mW_t}{\|\Delta \mW_t\|}\sqrt{mn}
\end{align*}

Apart from learning rate transfer, another crucial reason for such renormalization in our paper is to strictly align all methods to have the same update budged (same RMS norm), hence isolating the update direction to be the only factor driving the optimization performance.

\subsection{Distributed implementation}
\label{sec: distributed implementation}

Standalone optimization is not sufficient for training frontier-sized LLMs; distributed training is a practical requirement. However, modern distributed training pipelines \citep{shoeybi2019megatron,rasley2020deepspeed,paszke2019pytorch} are highly optimized for element-wise optimizers (e.g., Adam/AdamW) and do not natively support matrix-valued optimizers such as ARO. In the following, we summarize how we adapt two popular pipelines---PyTorch FSDP2 and Megatron-LM \citep{shoeybi2019megatron,paszke2019pytorch}---to support distributed training with ARO.

We abuse terminology and use parameters (or parameter shards) to refer to model parameters together with their associated gradients and optimizer states. 

\subsubsection{FSDP2 with PyTorch}
\label{subsec: FSDP2}

FSDP2 is a data-parallel (DP) scheme that shards model weights and optimizer states across a DP group to reduce per-GPU memory usage. During forward/backward, FSDP2 materializes full parameters when needed and then re-shards gradients/states, so we can assume that at the optimizer step each GPU has access to its local parameter/gradient/state shards.

\paragraph{Vanilla approach.}
A simple approach is to all-gather the shards so that every GPU reconstructs a complete copy of each matrix parameter when needed, runs ARO locally, and then applies updates to its local shards. While straightforward, this incurs substantial overhead: communication scales poorly with the number of DP ranks since all ranks participate in full-matrix reconstruction, and computation is redundantly repeated across GPUs.

\paragraph{Round-Robin distributed implementation.}
To make ARO efficient under FSDP2, we adopt a \emph{Round-Robin ownership} strategy inspired by the distributed Muon described in \cite{lim2025motif}. The key idea is to avoid reconstructing all matrices on all ranks. Instead, we partition matrix parameters into batches and assign a unique \emph{owner} GPU to each parameter within a batch. For each batch:
(i) we gather only the necessary shards to the owner ranks;
(ii) each owner executes the ARO update for its assigned parameters; and
(iii) we perform a scatter step so that updated parameter shards are placed back onto the correct ranks for the next iteration.


\begin{remarks}
\begin{remark}[Overlapping communication and computation.]
    The round-robin schedule naturally supports pipelining: owner devices compute the ARO update for the current batch, the system gathers shards for the next batch on a separate stream, and delays the scatter until the update is ready. In practice, this overlap is important for hiding the cost and keeping optimizer-step latency close to AdamW.
\end{remark}
\end{remarks}

\begin{remarks}
\begin{remark}[Practical throughput]
With this round-robin FSDP2 implementation, the throughput of ARO is comparable to AdamW when scaling up (see \cref{fig: Qwen throughput} and \cref{fig: sigma throughput} for 8B dense and 2B MoE, respectively), indicating that the additional computation and communication can be amortized across DP ranks.  
\end{remark}

\end{remarks}

\subsubsection{Megatron-LM}
\label{subsec: megatron-lm}

Megatron-LM \citep{shoeybi2019megatron} introduces additional parallelism (e.g., Tensor parallelism) and employs bucketed buffers for efficient communication. We provided a prototype implementation of ARO that supports Megatron-LM, extending the same high-level round-robin principle to the joint DP$\times$TP setting.

Concretely, our Megatron-LM integration combines three ideas:
(i) \textbf{Round-robin ownership} similar to FSDP 2 to shard ARO compute across ranks (avoiding redundant per-rank updates);
(ii) \textbf{fused TP gathering} to avoid per-parameter TP collectives calls across TP ranks; and
(iii) \textbf{boundary-only DP gathering} that leverages Megatron-LM's flattened buffers, so DP communication focuses on the small set of cross-boundary parameters rather than gathering/scattering the full parameter list. 
\begin{remarks}
    \begin{remark}
        We want to emphasize that this only represents a preliminary prototype implementation that only supports DP, TP and pipeline parallelism. Megatron-LM also supports other parallelisms (e.g., expert parallelism, context parallel, etc), different setup of GPU mesh grid will have different impacts on practical throughput. We leave the comprehensive benchmarking to future work for Megatron-LM. For simpilicity, for all experiments that require distributed training in this paper, we use Pytorch FSDP 2 implementation. 
    \end{remark}
\end{remarks}


\section{Experiments} \label{exp}

\subsection{Return to the basics: general guidelines for fair and controlled benchmarking}  \label{exp: guidelines}

We evaluate optimizers mainly on LLM pretraining, spanning model sizes from 130M to 8B parameters. Our goal is to make comparisons \emph{fair and controlled}: differences in loss should solely come from the optimizer, not from other confoundings or mismatched settings. Concretely, we adopt the following rules across runs:
\begin{itemize}
    \item \textbf{Always perform mixed precision training} (BF16 training with FP32 master weight). While mixed precision training is default in pretraining, we emphasize this since the pure BF16 training is popular in memory efficient optimizer literature \citep{zhao2024galore,  zhu2024apollo, ma2024swan}, where we have observed unrealistic high speedup numbers over AdamW, or less reliable conclusions, due to the use of BF16 master weights. 
    \item \textbf{Benchmark at realistic scale (batch size and context length).} Depending on model size and architecture, we use sufficiently large batchsizes (from 1M to 14M tokens per batch) and sequence length (from 1024 to 4096 context length) for a more realistic benchmarking. 
    \item \textbf{Train long enough to reveal long-horizon behavior.} Across all experiments and models, we always consume training tokens for at least 1$\times$ Chinchilla laws, and for certain model sizes (such as 2B and 8B models), we overtrain the model for up to 8$\times$ Chinchilla regime to get enough signal on the long term efficiencies of every optimizer (and to validate the risks of diminishing returns at scale).
    \item \textbf{Align lr schedules and step counts.} To properly estimate the sample efficiency of different methods, we align the learning rate schedule of all methods and verify all methods are trained with the same number of iterations, and learning rates are fully decayed (to target) \citep{kaddour2023no}. This is to avoid the evaluation bias in the common ``target loss to reach'' approach \citep{ahn2025dion, loshchilov2024ngpt}, where proposed optimizers enter and finish the learning rate decay phase earlier than AdamW baseline. While the rankings of methods may be unaffected, this favors the early lead of an algorithm and over-estimate speedups. 
    \item \textbf{Align non-hidden-layer optimizers.} Matrix optimizers are often used in a hybrid setup, applied only to hidden-layer parameters. To compare hidden-layer optimizers fairly, the optimizer on non-hidden-layer (non-matrix, 1D) parameters must be aligned. In all evaluations under hybrid setup, we fix AdamW for the non-hidden-layer parameters, rather than SignGD/Lion \citep{pethick2025training,ahn2025dion} -- to avoid conflating the observed gains. This rule does not apply for full model setups (same optimizer is applied to all matrix parameters). 
    \item \textbf{End-to-end learning rate tuning without scaling law extrapolation.} Good tuning matters for \emph{all} methods, including baselines, but exhaustive tuning becomes prohibitively expensive for large-model overtraining. Instead of extrapolating hyperparameters from scaling-law fits (which can be sensitive and introduce bias) as in \citep{wen2025fantastic}, we run end-to-end learning-rate sweeps for AdamW on the full overtraining runs and pick the learning rate that minimizes the final loss. We then transfer AdamW learning rates to other methods using RMS-norm matching (\Cref{sec:hyperparam}) following \citep{liu2025muon}. We use this protocol throughout, except for our largest 8B run where an end-to-end sweep is beyond our computational budget at the time of drafting, and we instead rely on established settings from prior work.
    \item \textbf{Control effective step size via aligned RMS-norm budget.} Continuing the previous point,  we use the RMS norm matching method \citep{liu2025muon} (\Cref{sec:hyperparam}) to ensure all methods operate under a comparable update RMS-norm budget to the end-to-end tuned Adam baseline. This is helpful to reduce the bias caused by significantly different effective learning rates, or the use of  per-layer learning-rate schemes (common in low-rank optimization practice \citep{zhao2024galore}). We found that this method is also sufficient to establish non-diminishing return over Adam across scales. An alternative, also promising method is to use $\mu$P based approaches for robust hyperparameter transfer \citep{qiu2025hyperparameter}, which might further improve performance \citep{xie2026controlled}. We leave it as future work since it requires optimizer-specific derivations and relies on certain assumptions that is not satisfied in our setup (e.g., MoE training).

\end{itemize}

\subsection{\name{} and baseline specifications} \label{sec: baselines}

\paragraph{\name{} family.}
Our main method is \name{} (\Cref{eq: gf}). Different choices of the base optimizer projection function $f_t$ define a family of \name{} instantiations:
\begin{itemize}
    \item \textbf{\arosink{}} (default): $f_t = f_{\text{Sink}}$ from \Cref{sec: sink}.
    \item \textbf{\name{}-Adam}: $f_t$ is the momentum-first, two-fold Adam projection from \Cref{sec: handling stateful optimizer}.
    \item \textbf{\name{}-Sign}: $f_t$ is the element-wise sign map.
    \item \textbf{\name{}-RowNorm}: $f_t$ is row-wise matrix normalization. As shown in \Cref{sec: ortho_as_special_case}, this is mathematically equivalent to orthogonalization. We keep this due to completeness.
\end{itemize}

For each instantiation, we evaluate two deployment modes:
\begin{itemize}
    \item \textbf{Hybrid setup:} Applying \name{} only to hidden-layer weight matrices; optimize embeddings, LM head, layer norms, and other non-matrix parameters with AdamW.
    \item \textbf{Full-model setup:} Applying \name{} to all matrix parameters (tensors with order $\leq 2$). For 1D parameters, we use \name{} with $\mR_t=\mI$ (i.e., no rotation).
\end{itemize}

For orthogonalization, we consider two QR implementations depending on the experiment goal: 
\begin{itemize}
    \item \textbf{Standard QR:} Using the default implementation in pytorch. It is slower but numerically more stable.
    \item \textbf{Shifted Cholesky-QR (SCQR):} a faster variant \cite{fukaya2020shifted} (see \Cref{sec:chol_qr}).  
\end{itemize}
Unless otherwise stated, we use the fast Shifted Cholesky-QR variant for \name{} family. We will also benchmark the standard QR variant to understand the lossless performance of \name{}.

\paragraph{Baselines}
Across experiments, we compare \name{} with AdamW, gradient orthogonalization, and eigen-rotation baselines. To isolate the effect of their core matrix update mechanism, we omit complementary add-ons such as variance adaptation \citep{si2025adamuon, li2025normuon, yang2026prismstructuredoptimizationanisotropic}, multi-scale momentum \citep{behrouz2025nested, ueaj_multiscale_muon}, or additional constraints \citep{pethick2025training, wen2025hyperball, xie2026controlled} that can be composed on both \name{} and other baselines. Our baselines are defined as follows:
\begin{enumerate}
    \item \textbf{AdamW} \citep{loshchilov2017decoupled}.
    \item \textbf{Muon} (Moonlight version) \citep{liu2025muon}: a strong baseline from the orthogonalization family.
    \item \textbf{Full-rank Dion} \citep{ahn2025dion}: using full-rank orthogonalization with standard QR decomposition. This baseline is included to separate algorithmic effects from numerical differences with orthogonalization solvers.
    \item \textbf{Eigen-rotated family} (momentum-first, one-sided): Eigen-rotated optimizers directly rotates the gradient using eigen rotation. In our momentum-first implementation principle (\Cref{remark: momentum_first}), this is done by setting $\mR_t$ to be the left singular vectors of $\mM_t$ (or equivalently, eigen vectors of $\mM_t \mM_t^\top$). The update rule is given  by
    \begin{align*}
        \mR_{t} &= \text{Eigenvectors}((\mM_t \mM_t^\top)) \approx \text{PowerIter}(\mM_t \mM_t^\top) \\
        \Delta \mW_t & = -\eta \mR_{t} f_{t}(\mR_{t}^\top \mM_t ).
    \end{align*}

We will always apply 1 step power iteration using $R_{t-1}$ as initialization: $\mR_t = \text{QR}(\mM_t\mM_t^\top \mR_{t-1})$ \citep{vyas2024soap}, unless specified (e.g., in \Cref{sec: aro}, \cref{fig: stats_of_aro} we used exact eigen vectors via SVD). By varying the base projection $f_t$, the eigen-rotated family recovers one-sided, momentum-first versions of many whitening optimizers (including SOAP, SPlus, and Muon, see \Cref{tab: momentum_first_updates} for details.). We use this family as a control reference point: it lets us compare \name{} against eigenvector-based rotations while holding fixed the rest of the design choices (one-sided rotation, momentum, and the eigenvector solver based on power iteration). We consider:
\begin{itemize}
    \item \textbf{Eigen-Adam:} $f_t$ is the two-fold Adam projection from \Cref{sec: handling stateful optimizer}. This can be viewed as one-sided, momentum-first variant of SOAP using power iteration (\Cref{tab: momentum_first_updates}).
    \item \textbf{Eigen-Sign:} $f_t$ is elementwise sign. This corresponds to a one-sided, momentum-first implementation of SPlus \citep{frans2025stable} using power iteration (\Cref{tab: momentum_first_updates}).
    \item \textbf{Eigen-RowNorm:} $f_t$ is row-wise normalization. This is equivalent to the full rank Dion \citep{ahn2025dion}, which is equivalent to idealized orthogonalization if power iteration is run until convergence. 
    \item \textbf{Eigen-Sinkhorn:} $f_t = f_{\text{Sink}}$. 
\end{itemize}

As with \name{}, each eigen-rotated optimizer is evaluated in both hybrid and full-model setups. Depending on the experiment, we also report targeted ablations:
\begin{itemize}
    \item \textbf{No-rotation variants}: base optimizers without rotation.
    \item \textbf{QR implementation variants}: Swapping the QR routine (standard QR vs.\ Shifted Cholesky-QR).
\end{itemize}
\end{enumerate}

\paragraph{Metric.}
Across the experiments, we will use one or some of the following metrics for evaluating the optimizers' performance:
\begin{enumerate}
    \item \textbf{Speedup factor}: This measure the convergence speed of target optimizer against the baseline, which is computed as the step ratio of baseline over target optimizer when achieving a target loss. 
    \item \textbf{Loss}: This can be either training or validation loss depending on the setup. 
    \item \textbf{Throughput}: This measures the wall-clock speed of the optimizer, which can be measured as the \emph{token processed per second} or \emph{time taken per iteration}. 
\end{enumerate}

\subsection{Understanding the effectiveness of \name{} rotation}
\label{subsec: universal effective}

\begin{figure}
    \centering
    \includegraphics[width=1.0\linewidth]{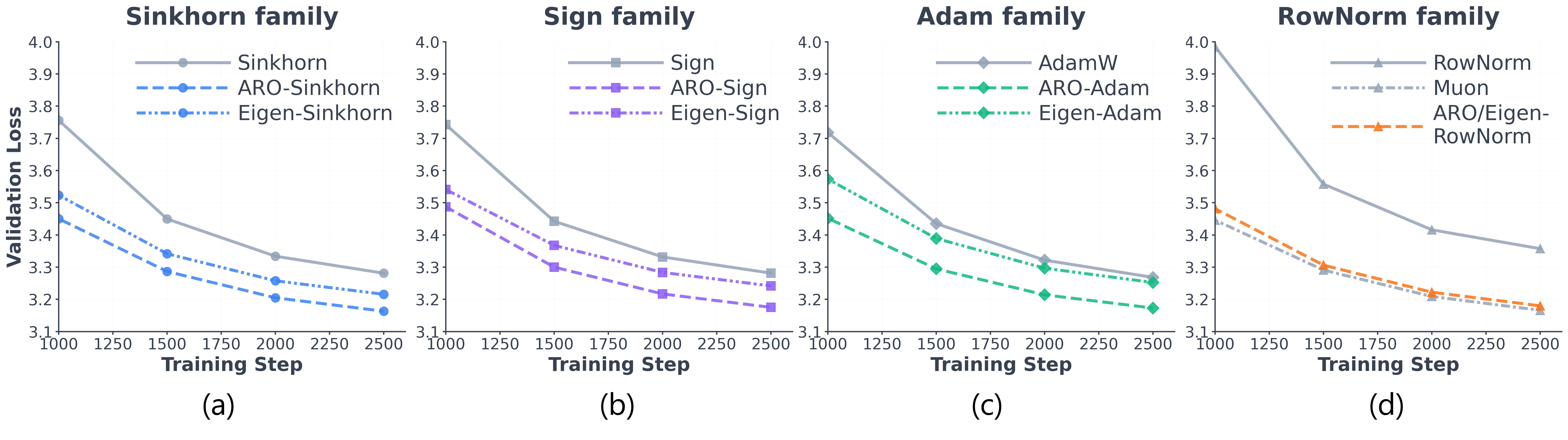}
    \caption{Comparing \name{} rotation with eigen-rotations, across various of base optimizers. \textbf{Results indicates that the non-eigen-rotations of \name{} consistently outperform eigen-rotations, and improves all base optimizers to a competitive level}. Experiments were done on GPT2-124M model. By default, we use SCQR implementation across all \name{} and eigen-rotation methods. Ablations on QR implementation can be found in \cref{fig: m2.1}. }
    \label{fig: m1.1}
\end{figure}

In this section, we start with a smaller scale models (GPT2-124M), and focus on validating the effectiveness of the non-eigen-rotations of \name{}, across a variety of base optimizers defined.  For formal benchmarking and comparisons of baselines under practical large-scale setup, as well as the full-model training capabilities of \name{}, we defer to \Cref{exp: gpt}, \Cref{subsec: Exp Sigma-MoE}, and \Cref{subsec: qwen3-8B}.

\paragraph{Setup and how guidelines are followed.} We consider 124M parameter GPT-2 model following the setup of NanoGPT repo \citep{Karpathy2022}, trained on OpenWebText dataset with 1$\times$ Chinchilla law regime and 1024 context length. All methods uses the hybrid setup (AdamW is used for non-hidden layer parameters, as well as 1D parameters). We follow the guidelines in \Cref{exp: guidelines} to perform mixed precision training, align lr schedules and non-hidden layer optimizers across all methods, optimally tune the AdamW learning rates end-to-end, and then transfer to all rotation based methods. 

Results are shown in \Cref{fig: m1.1}, \Cref{fig: m2.1} and \Cref{fig: m2.2}. Our core findings are as follows:

\begin{findings} 

\textbf{Findings 1: Under SCQR, \name{} consistently outperforms eigen-rotations across base optimizers, and improves all base optimizers to a competitive level. Both \name{} and eigen-rotation methods consistently outperform non-rotated base optimizers.}    
\end{findings}

 \Cref{fig: m1.1} (a)-(d) shows the performance of \name{} and eigen-rotated baselines, across all four base optimizers (Sinkhorn, Sign, Adam, and RowNorm) under the fast SCQR implementation (\Cref{sec:chol_qr}).  Rotation strategies improve the performance of base optimizers universally. Moreover, \name{}-rotated versions consistently outperform Eigen-rotated versions by a large margin. 

\begin{findings}
    \textbf{Findings 2: Across all base optimizers tested, loss landscape rotation is the top contributor to optimizer performance gains.}
\end{findings}

\Needspace{18\baselineskip}
\begin{wrapfigure}{r}{0.45\linewidth}
    \centering
    \includegraphics[width=\linewidth]{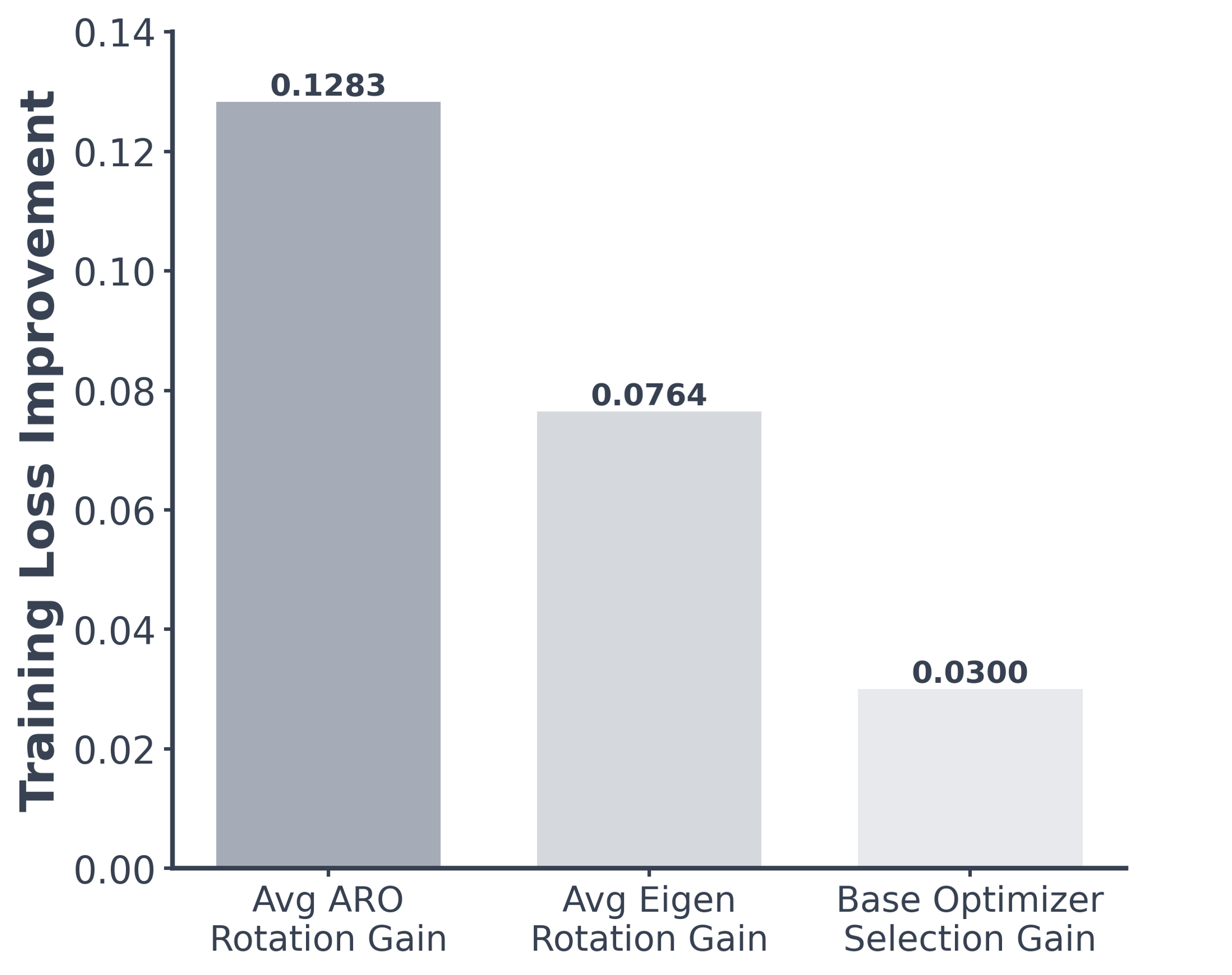}
    \caption{The impact of the choice of rotation policies and base optimizers on training loss. \textbf{Rotations have significant impact on performance.}}
    \label{fig: m1.2}
\end{wrapfigure}

As discussed in \Cref{sec: insights}, many modern optimizers can be written in the form
$ \Delta \mW_t \propto - \eta \, \mR_t f_t( \mR_t^{\top} \mG_t ) $
for a chosen rotation $ \mR_t $. The performance of this very general optimizer class is therefore determined by two factors: the choice of base optimizer $ f_t $ and the choice of rotation policy $ \mR_t $. To isolate the impact of rotations, we compare the effect of $ \mR_t $ against the effect of $ f_t $. Specifically, we compute the averaged loss difference between \name{}-rotated methods and their non-rotated versions across four base optimizers. Losses are averaged over the last $ 5\% $ of training steps. We also compute the loss difference between the best-performing base optimizer and the mean performance across all base optimizers, again averaged over the last $ 5\% $ of steps. Results on the 124M GPT-2 model show that \name{}-rotations and eigen-rotations have a larger impact on loss than the choice of base optimizer. This supports the view of \Cref{sec: insights} that rotation policies should be treated as a core design principle.


\par \wrapfill \par

\begin{remarks}
    \begin{remark} (Other important factors) Technically speaking, a rotated optimizer is determined by three elements: 
    \begin{enumerate}
        \item A rotation policy (e.g., eigen-rotation, \name{}-rotation)
        \item Base optimizer $f$;
        \item A orientation rule that determines, for each parameter in the model, which side will the rotation be applied.
    \end{enumerate}
    In this section we mainly discussed the first two factors. The third apspect is also important, which will be deferred to \Cref{sec: symmetry_predictions} as it requires more pre-requisite discussions. 
    \end{remark}
\end{remarks}

\begin{findings} 
\textbf{Findings 3: the performance gain of \name{} rotations over eigen-rotation is twofold: 
\begin{itemize}
    \item it produces geometrically better rotations regardless of which QR scheme is used;
    \item it massively improves the numerical stability of the computation and makes SCQR usable.
\end{itemize}
Together, \name{} rotation enables better rotation with less compute.}
\end{findings}

\begin{figure} [!h]
    \centering
\includegraphics[width=0.85\linewidth]{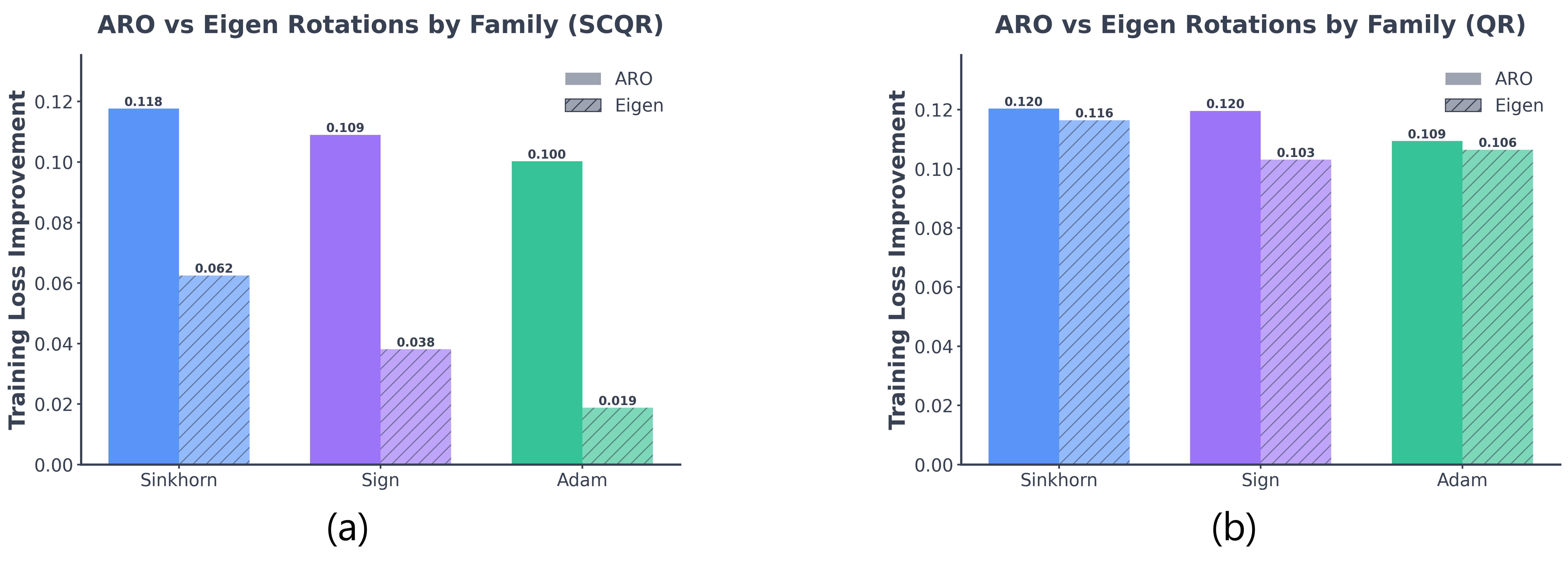}
    \caption{The performance of different rotation policies under both SCQR (left) and QR (right) implementations, across different base optimizers. RowNorm family is omitted as \name{} coincide with eigen-rotation, as shown in \Cref{sec: ortho_as_special_case}. \textbf{The results suggests that \name{} not only provides a better rotation direction, but also enables fast QR computation by improving the conditioning.}}
    \label{fig: m2.1}
\end{figure}

Next, we conduct an ablation study to understand the sources of the performance gains of \name{} rotations over eigen-rotation. Under the SCQR implementation, the improvement from \name{} may arise from two factors: i) producing a geometrically better rotated direction, and ii) improving the numerical conditioning of the SCQR decomposition through the additional normalization term in $\mM_t f_t(\mR_t^\top \mM_t)^\top$. To separate these effects, we repeat the same set of experiments on the 124M GPT-2 model but replace the fast SCQR implementation with an standard QR decomposition, which is much more well-behaved and usually does not suffer from numerical conditioning issues. This allows us to compare the loss improvement from \name{}-rotations against that of eigen-rotations, across all base optimizers, without conflating gains from SCQR conditioning. Results are shown in \Cref{fig: m2.1}. We observe that when switching from SCQR to standard QR, the gap between \name{}-rotations and eigen-rotations becomes much more moderate across all base optimizers. This indicates that a large portion of the improvements of \name{}-rotations is indeed from improved conditioning. However, even under standard QR, the gap is significant. On Sign and Sinkhorn families, the loss gap is $\approx 0.02$ and $\approx 0.008$, respectively, both are often considered as significant in LLM training. On Adam base optimizer class, the gap is smaller $\approx 0.003$, partially justifying that why in the literature eigen-rotation is often used for Adam-like base optimizers. Together, our results suggests that \name{} not only provides a better rotation direction, but also enables fast QR computation by improving the conditioning. 

\begin{remarks}
    \begin{remark}
        To further ablate the effect of improved numerical conditioning with \name{}, we have also compared \name{} rotations with other variants that are not derived from maximizing \name{} objective as in \Cref{sec: instantaneous}. For example, instead of doing $\mR_{t} = \text{QR}(\mM_t f_t(\mR_{t}^\top \mM_t)^\top)$, one can also try $\mR_{t} = \text{QR}(f_t(\mM_t) f_t(\mR_{t}^\top \mM_t)^\top)$ or $\mR_{t} = \text{QR}( f_t(\mM_t  \mM_t^\top \mR_t))$, both will also improve the numerical conditioning of SCQR. However, in our experiments we found that they have worse performance compared with \name{}, under both standard QR and SCQR.
    \end{remark}
\end{remarks}

\begin{findings} 
\textbf{Findings 4: At least locally, the \name{} objective $\mathcal{J}(\mR_t;\mM_t,f)$ predicts the performance ranking of different rotation strategies.}
\end{findings}

\Needspace{18\baselineskip}
\begin{wrapfigure}{r}{0.45\linewidth}
    \vspace{-1.2em}
    \centering
    \includegraphics[width=\linewidth]{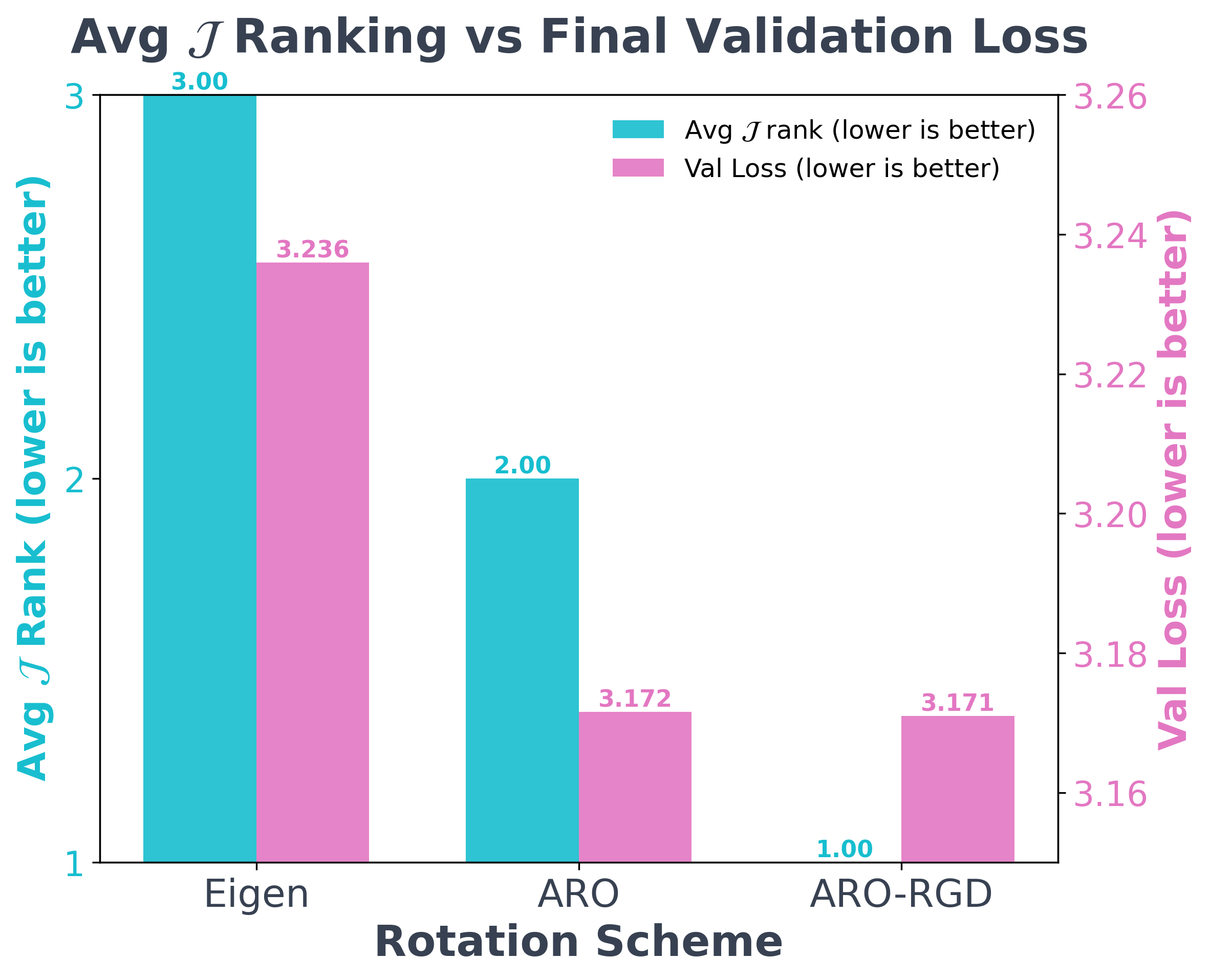}
    \caption{Alignment between relative rankings of  $\mathcal{J}(\mR_t;\mM_t,f)$ and final loss, across 3 different rotation schemes. \textbf{$\mathcal{J}(\mR_t;\mM_t,f)$ positively correlates with pretraining performance.}}
    \label{fig: m2.2}
    \vspace{-1.2em}
\end{wrapfigure}

Finally, we provide preliminary analysis on the correlation between the maximization of $\mathcal{J}(\mR_t;\mM_t,f)$, and the final performance of the resulting rotated optimizers. On the same GPT2-124M model pretraining setup, we use the sign base optimizer, and consider the following rotation policies:

\begin{itemize}
    \item \name{} rotation with SCQR using \Cref{eq: gf};
    \item Eigen-rotation, given by the eigenvectors of the momentum
    $$  \mR_t^\star = \text{Eigenvectors}(\mM_t\mM_t^\top) $$
    We use 1-step power iteration to compute eigenvectors.
    \item \name{}-RGD rotation, that uses 10-step Riemannian gradient descent (with QR retraction) to locally improve $\mathcal{J}(\mR_t;\mM_t,f)$ using $\mR_{t-1}$ as initialization at time $t$.    
\end{itemize}

\par \wrapfill \par
We train 3 different runs using these rotation policies, respectively. Note that directly comparing raw $\mathcal{J}(\mR_t;\mM_t,f)$ values across policies is not meaningful, because each optimizer follows a different trajectory with distinct gradient statistics. To make $\mathcal{J}$ comparable across rotation policies, we focus on one optimization trajectory at a time. At each step along its trajectory, we take the current gradient from a intermediate layer, compute the hypothetical update direction for all three rotation methods, and log their $\mathcal{J}(\mR_t;\mM_t,f)$ values. These candidate updates are discarded after logging. For robustness, we repeat the analysis and average the resulting \emph{relative} $\mathcal{J}(\mR_t;\mM_t,f)$ rankings across all 3 runs 
(lower rank number, e.g., rank=1 denotes the best). The results obtained along the \name{}-Sign trajectory are shown in \Cref{fig: m2.2}.

We observe that, for the rotation policies considered, larger $\mathcal{J}(\mR_t;\mG_t,f)$  (i.e., lower rank number) are associated with better final loss. However, although \name{}-RGD achieves a higher $\mathcal{J}(\mR_t;\mG_t,f)$ than the default \name{} rotation, its improvement in training loss is negligible. This suggests diminishing returns from further optimizing $\mathcal{J}(\mR_t;\mG_t,f)$ and we conjecture the possibility that aggressively maximizing $\mathcal{J}(\mR_t;\mG_t,f)$ may eventually hurt optimization performance. This is consistent with the argument in \citep{yun2023riemannian} that maximizing $\mathcal{J}(\mR_t;\mG_t,f)$ correlate with increased sharpness.

\begin{takeaways}
    \textbf{Takeaways for \Cref{subsec: universal effective}: \name{} rotations consistently outperform eigen-rotations and no-rotation base optimizers. The rotation policy has the most significant impact on optimization performance; its quality is locally aligned with $\mathcal{J}$ objective, reinforcing rotations as a primitive design principle.}
\end{takeaways}

\subsection{GPT2-XL-1.5B pretraining} \label{exp: gpt}

\begin{figure}[!h]
    \centering
\includegraphics[width=1.03\linewidth]{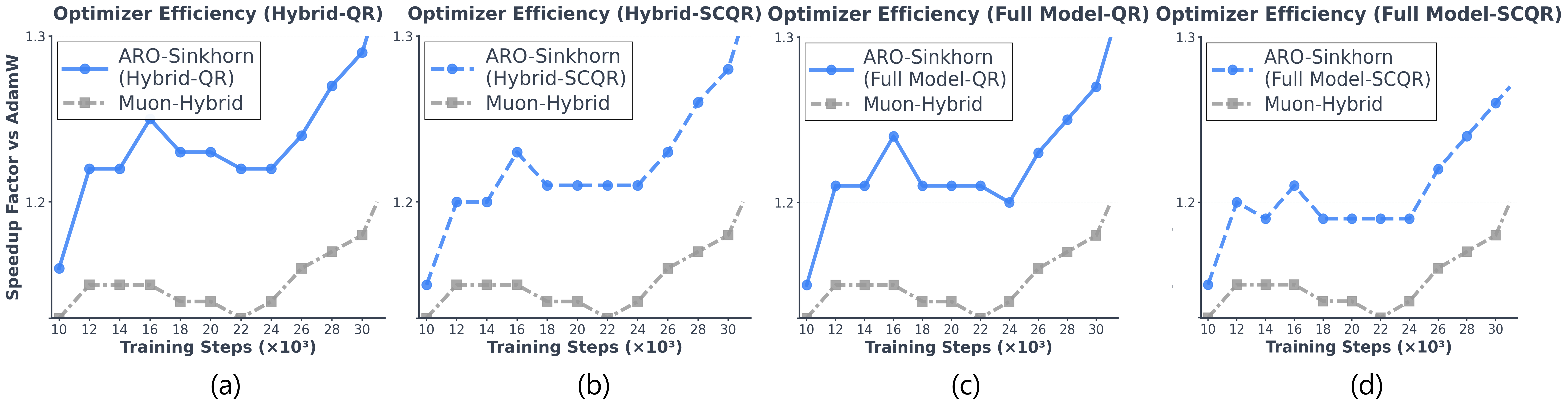}
    \caption{Tracking the speedup factor of \arosink{} against AdamW during training. We consider all 4 regimes of $\{ \text{Hybrid}, \text{Full model} \}$ $\times$ $\{ \text{standard QR}, \text{fast SCQR}\}$ setups. Speedup at $t$ is calculated by $\frac{t}{t_{\text{ARO}}}$, where $t_{\text{ARO}}$ is the number of steps \arosink{} required to reach the loss of AdamW at $t$. Since Muon diverges under full model setup (\Cref{fig: m6}), we compare to Muon-hybrid baselines in all regimes. \textbf{\arosink{} significantly outperforms both AdamW and Muon across all 4 setups, reaching up to around $1.3\times$ speedup over AdamW, while Muon achieves around $1.2\times$ speedup.}}
    \label{fig: m3}
\end{figure}

From this section, we proceed to formally benchmark the performance of \name{} and various of baselines, as well as the full-model capabilities of \name{}, starting with a GPT2-XL-1.5B model.

\paragraph{Setup and how guidelines are followed.} We consider GPT2-XL-1.5B model following the setup of NanoGPT \citep{Karpathy2022}, trained on FineWeb dataset with 1$\times$ Chinchilla law regime and 1024 context length. Most methods are compared both with the hybrid setup and the full model setup. We follow the guidelines in \Cref{exp: guidelines} with mixed precision training, aligned lr schedules, optimally tuned AdamW learning rates end-to-end, and then transferred to all rotation based methods. Note that the alignment of non-hidden-layer optimizers is only enforced under hybrid setup, and does apply to full model setup.

\begin{figure}[!h]
    \centering
\includegraphics[width=0.99\linewidth]{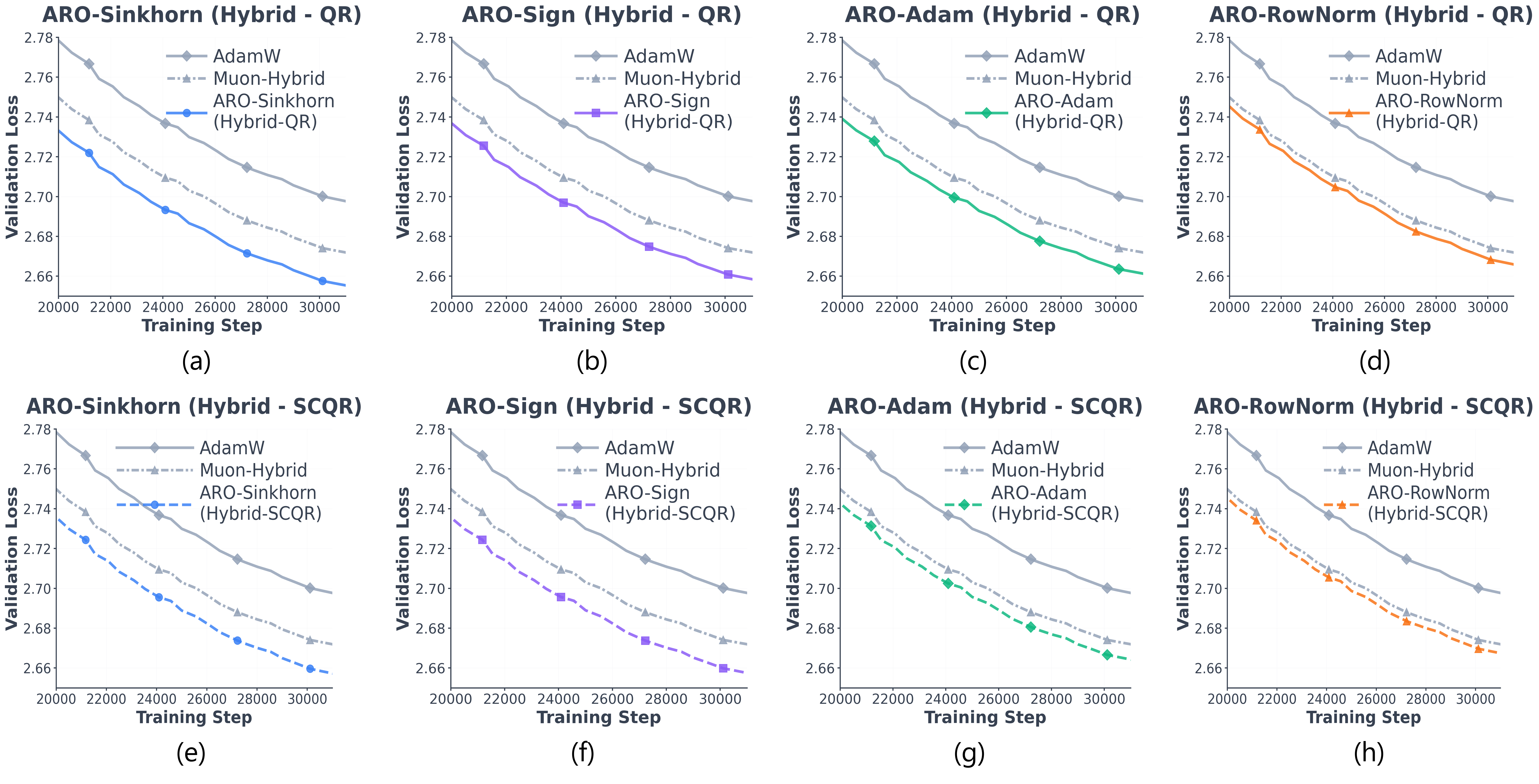}
    \caption{Benchmarking the performance of \name{} on GPT2-XL-1.5B, compared to AdamW and Muon optimizers, across all four base optimizers (Sinkhorn, Sign, Adam, and RowNorm), under the hybrid setup. \textbf{\name{} family consistently outperforms AdamW and Muon; and \arosink{} gives the best performance while being robust to the choice of QR schemes.}}
    \label{fig: m4}
\end{figure}

\begin{findings}
\textbf{Findings 5: On GPT2-XL-1,5B, under the hybrid setup, \name{} family consistently outperforms AdamW and Muon. \arosink{} gives the best performance while being robust to the choice of QR schemes, reaching $\sim 1.3\times$ against AdamW, while Muon achieves $~1.2\times$ speedup.}
\end{findings}

\Cref{fig: m4} shows the performance of \name{} in comparison with AdamW and Muon optimizers, across all four base optimizers (Sinkhorn, Sign, Adam, and RowNorm), under the hybrid setup. Particularly, \Cref{fig: m4} (a)-(d) use the standard QR implementation, and (e)-(h) use the fast SCQR implementation (\Cref{sec:chol_qr}).  We observe that:
\begin{itemize}
    \item The \name{} family (Sinkhorn, Sign, Adam) consistently outperforms both AdamW and Muon baseline. The only exception is the RowNorm family where \name{} gives near identical performance under standard QR, as analyzed by \Cref{sec: ortho_as_special_case}.
    \item Within \name{} family, \arosink{} gives the best performance. As shown in \Cref{fig: m3} (a), \arosink{} (Hybrid-QR) reaches around $\sim 1.3\times$ speedup against AdamW at the end of training; while Muon achieves around $~1.2\times$ speedup.
    \item Switching from standard QR (\Cref{fig: m4} (a)-(d)) to fast SCQR (\Cref{fig: m4} (e)-(h)) in general result in performance degradation. This is more pronounced in Adam family, while \arosink{} and \name{}-Sign are more resilient to the change. \emph{This reveals the potential limitation of Adam as base optimizer.}
    \item \Cref{fig: m3} (b) tracks the speedup factor of \arosink{} (hybrid-SCQR) against AdamW, during training. Compared with standard QR (\Cref{fig: m3} (a)), \arosink{} (hybrid-SCQR) has lower speedup during training. However, the gap shrinks at the end of training, and \arosink{} (hybrid-SCQR) reaches a $1.28\times$ speedup factor. This justifies the use of SCQR, especially under the default Sinkhorn base optimizer.
\end{itemize}

\begin{figure}[!h]
    \centering
\includegraphics[width=0.99\linewidth]{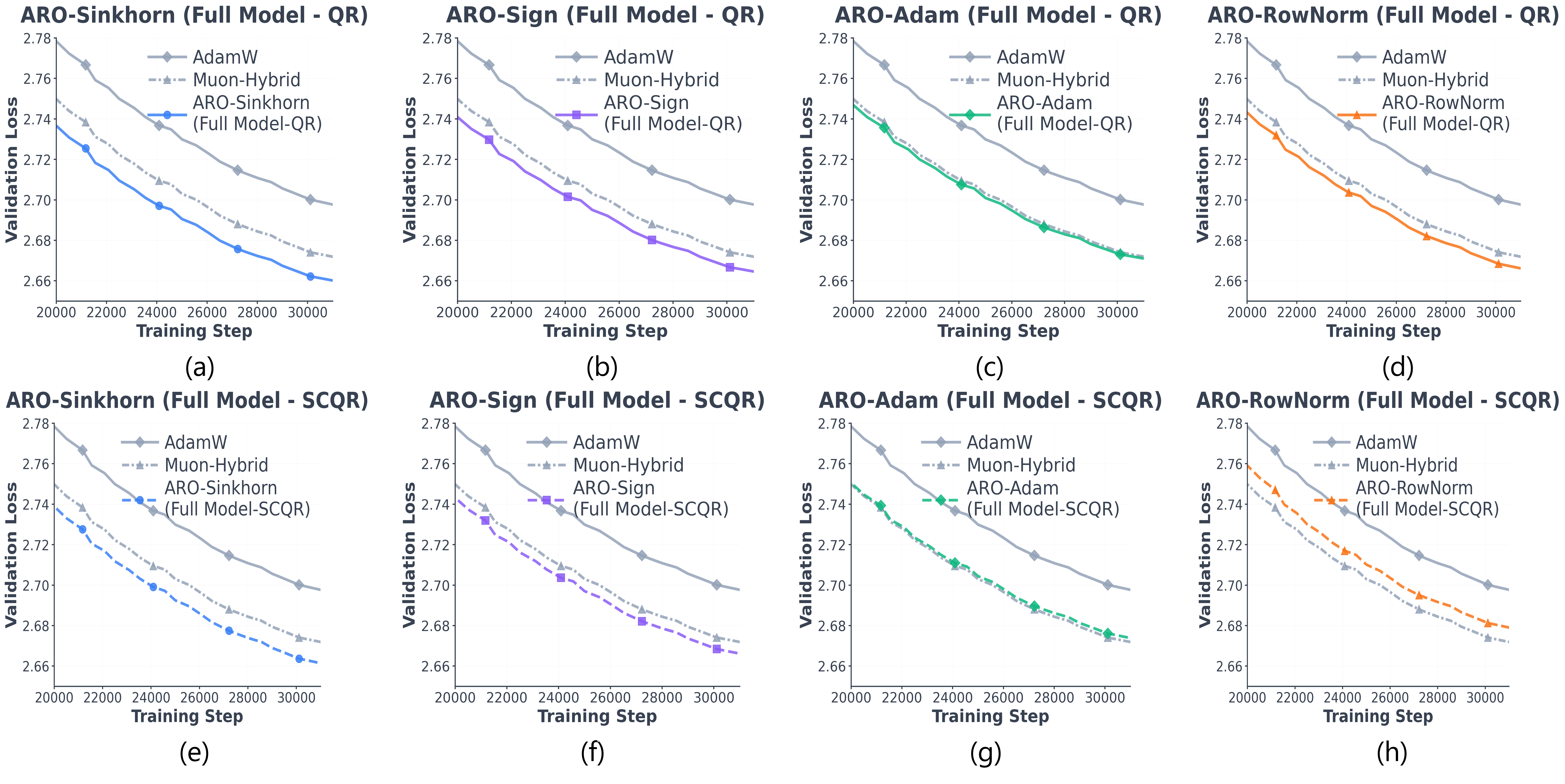}
    \caption{Performance of \name{} on pretraining GPT2-XL-1.5B, in comparison with AdamW and Muon, across all four base optimizers (Sinkhorn, Sign, Adam, and RowNorm), under the \textbf{full model} optimization setup. \textbf{\name{} family consistently performs well, while Muon diverges when applied to all matrix parameters (\Cref{fig: m6}). Particularly, \arosink{} gives the best and robust performance under full model setup, making full model setup a viable option; while other base optimizers show various limitations. }}
    \label{fig: m5}
\end{figure}

\begin{findings}
    \textbf{Findings 6: On GPT2-XL-1.5B, under the full-model setup, \name{}-family consistently performs well, while Muon diverges. Particularly, \arosink{} gives the best and robust performance under full model setup, outperforming the Muon$+$Adam-hybrid; while other base optimizers show various limitations. This result reveals the importance of both \name{} rotation, as well as choosing base optimizer wisely.}
\end{findings}

\Needspace{18\baselineskip}
\begin{wrapfigure}{r}{0.35\linewidth}
    \vspace{-1.2em}
    \centering
\includegraphics[width=\linewidth]{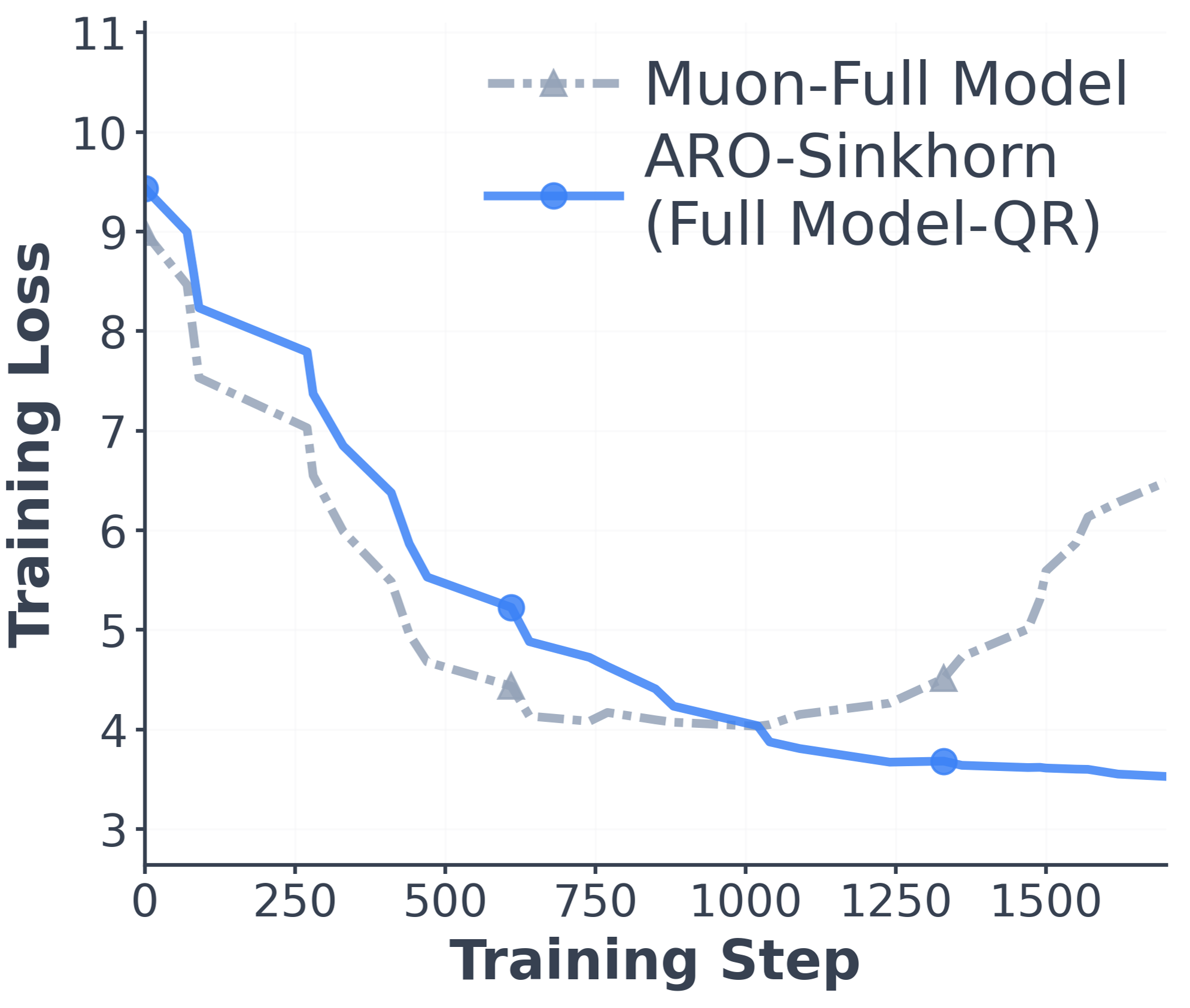}
    \caption{Muon diverges when applied to all matrix parameters. In our analysis, we infer that \textbf{this is due to both the numerical quality of Newton-Schulz solver, as well as the limitation of its RowNorm base optimizer.}}
    \label{fig: m6}
    \vspace{-1.2em}
\end{wrapfigure}

\Cref{fig: m5} shows the full-model optimization performance of \name{} in comparison with AdamW and Muon optimizers, across all four base optimizers (Sinkhorn, Sign, Adam, and RowNorm). Since Muon diverges under full model setup (\Cref{fig: m6}), we still compare to Muon-hybrid baselines. We observe that:

\begin{itemize}
    \item Under full-model setup, Muon diverges (\Cref{fig: m6}), while \name{} family still manages to converge and outperforms AdamW, especially under standard QR (\Cref{fig: m5} (a) - (d)). 
    \item Comparing to their hybrid versions (\Cref{fig: m4}), \name{}-Adam and \name{}-Sign both suffer from much more pronounced degradation under full model setup. This reveals that while \name{} rotation is beneficial for full model optimization, the choice of base optimizer is also important. \emph{Adam/SignGD as base optimizers seem to have strong limitations in a full model optimization setup}.
    \item \arosink{} performs robustly well under full model setup. As shown in \Cref{fig: m3} (c), \arosink{} (full model-QR) reaches around $1.27\times$ speedup against AdamW, outperforming Muon-hybrid. 
    \item Moreover, when switching from standard QR to fast SCQR, the \name{}-RowNorm (i.e., Dion) takes a significant performance hit, while \arosink{} is much more robust.  Combining with previous observations, this also reveals that \emph{the failure of Muon under full model setup is due to both the numerical quality of Newton-Schulz solver, as well as the limitation of its RowNorm base optimizer.}
\end{itemize}

\par \wrapfill \par
\begin{remarks}
    \begin{remark}
        For the full model setup, we mainly showed that \name{} family, especially \arosink{} performs competively; but it did not showed better performance than hybrid setup. In fact, in \Cref{subsec: Exp Sigma-MoE} we will further push to a 2B-parameter model $+$ 8$\times$ over-train scale, and show that at the first $1\times$ Chinchilla law regime, full model training indeed performs worse; however, when we continue to over-train the model, full model training with \name{} started to outperform its hybrid versions, and gives the best performance. 
    \end{remark}
\end{remarks}

\begin{takeaways}
    \textbf{Takeaways for \Cref{exp: gpt}: on large-scale pretraining tasks with GPT2-XL-1.5B, \name{} (especially \arosink{}) significantly outperforms both AdamW and Muon, and makes full model training a viable option.}
\end{takeaways}

\subsection{Sigma-MoE-2B pretraining}
\label{subsec: Exp Sigma-MoE}
To scale up ARO and evaluate its effectiveness across different architectures, we adopt \arosink{} to train a large-scale MoE language model, Sigma \citep{qu2025sigma}, with $2$B total parameters ($200$M activated). This architecture enables extreme sparsity while maintaining top-tier performance. We choose an MoE architecture to compare against dense LLMs and evaluate ARO optimizer performance. We compare \arosink{} against the following representative baselines: (1) AdamW; (2) Muon; (3) Dion; (4) Eigen-Sinkhorn; and (5) ARO-Adam. In addition, we include a full-model training setup for \arosink{}. 

\paragraph{Setup and how guidelines are followed.}
We follow all guidelines detailed in \cref{exp: guidelines}. Apart from the full-model setup of \name{}, we use hybrid setup for all other optimizers. We use $2$K sequence length, $4M$ token batch sizes and $24000$ steps, corresponding to $100$B training tokens. Although there is no single universally accepted compute optimal ratio for MoE model, $100$B can be regarded as long-horizon setup \citep{ludziejewski2025joint}. 
We use the Nemontron-CC dataset \citep{su2025nemotron} with \emph{kind2=actual} split. We choose $5E-4$ as the learning rate for AdamW based on the end-to-end hyperparameter tuning and transfer to others following \cref{sec:hyperparam}. Due to the lack of validation dataset and the fact that we never exhaust the entire training data, we compute the speed-up using the training loss. For detailed setup, refer to \cref{subapp: sigma setup}.

\paragraph{Distributed training.} We use standard distributed data parallel (DDP) for all methods with A100 NVIDIA GPUs. 

\begin{figure}
\begin{minipage}[t]{0.45\linewidth}
\centering
    \includegraphics[width=1\textwidth]{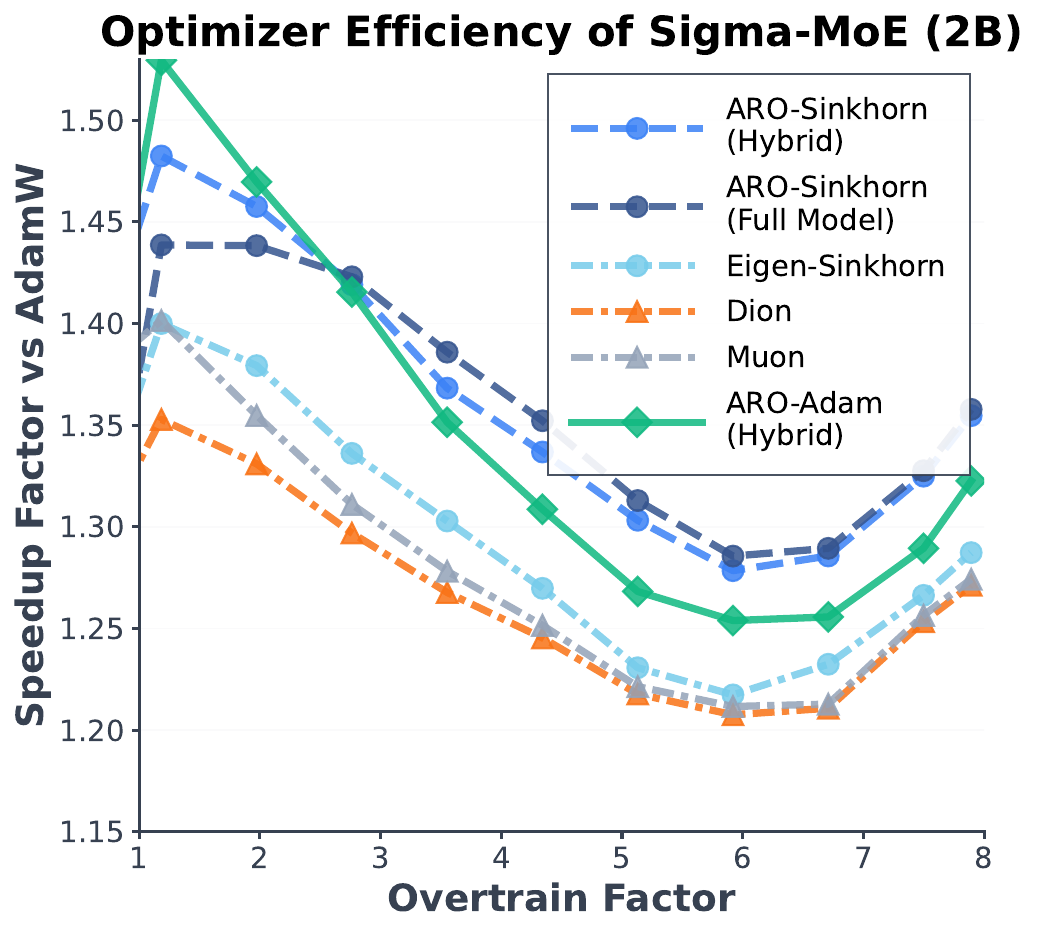}
    \caption{The speedup factors of different optimizers compared to AdamW at different overtrain factors. The $1\times$ overtrain is computed as $20\sqrt{\text{total param}\times \text{activated param}}$. The target loss is the loss value achieved by AdamW at the chosen data scale.}
    \label{fig:sigma speedup}
\end{minipage}\hfill
\begin{minipage}[t]{0.45\linewidth}
\centering
    \includegraphics[width=1\textwidth]{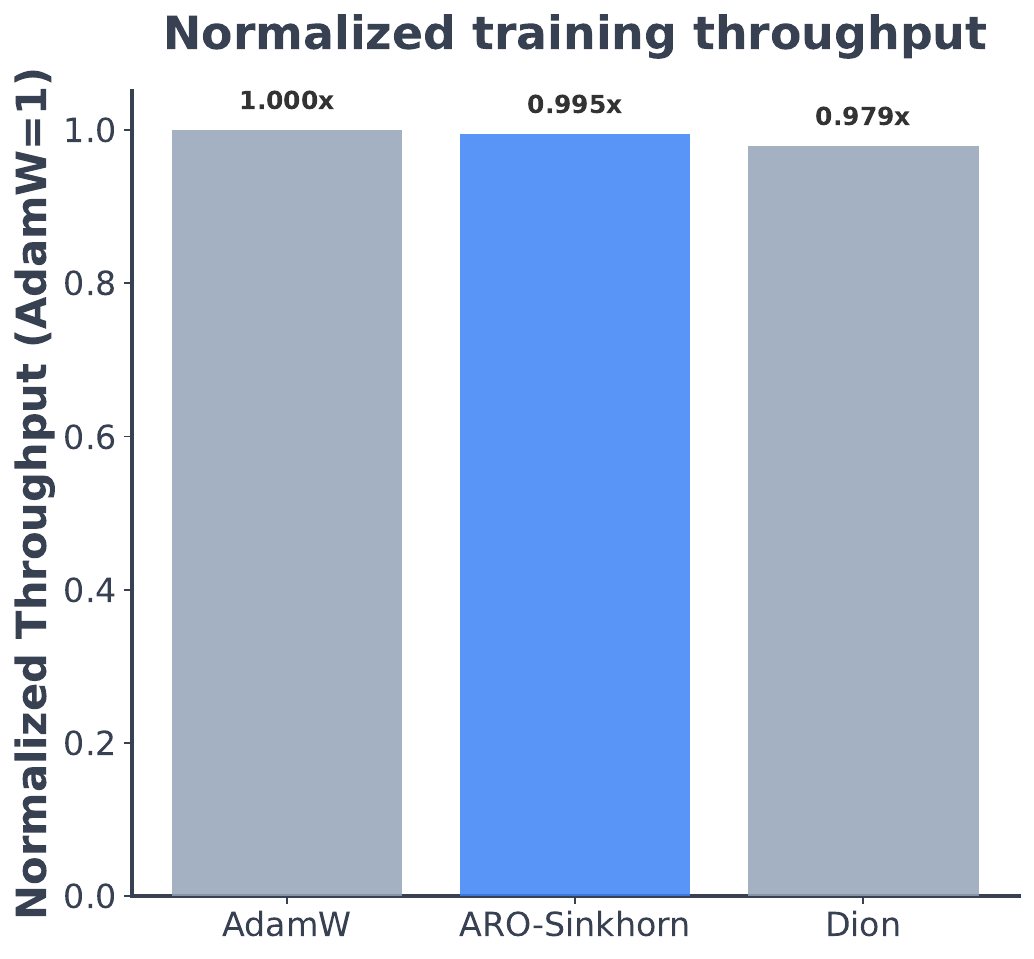}
    \caption{The normalized throughput averaged over entire training for different optimizers. Since the \emph{Hybrid} and \emph{full-model} setup only differ on a small number of parameters to be updated, their throughput are almost identical. This throughput is equivalent to optimizers' wall-clock speed.}
    \label{fig: sigma throughput}
\end{minipage}
\end{figure}

\begin{findings}
\textbf{Findings 7: The advantage of \arosink{} and \name{}-rotation observed on dense models in \Cref{subsec: universal effective} and \Cref{exp: gpt} continue to hold on MoE architecture. Particularly,  \arosink{}  outperforms other baselines, with no diminishing returns at 2B + 8$\times$ over-train regime.}
\end{findings}

\Cref{fig:sigma speedup} plots the speedup factors of different optimizers compared to AdamW at different over-training regimes (measured from the same 8$\times$ over-train run). We observe that:
\begin{itemize}
    \item \name{} family (i.e., \arosink{} and ARO-Adam) consistently achieve speedups over AdamW ($> 1.35 \times $) and outperform other orthogonalization-based methods (Muon, Dion) ($\sim 1.1\times$).
    \item Within \name{} family, \arosink{} shows a clear advantage over Eigen-Sinkhorn, confirming Finding~1 in \cref{subsec: universal effective}, demonstrating that ARO-based rotations are consistently more effective than commonly used eigen-based rotations, even when scaling up the model. 
    \item Within the family of Eigen-rotation methods, Eigen-Sinkhorn slightly outperforms Muon and Dion, suggesting that SinkGD serves as a stronger base optimizer than RowNorm. 
    \item Importantly, we observe that at least at the 2B + 8$\times$ over-train regime, \arosink{} does not suffer from the diminishing returns observed in \cite{wen2025fantastic}. Compared with GPT-XL experiments (1.5B + 1 $\times$ compute optimal regime, reaching around 1.3$\times$ speedup \Cref{fig: m3}), the speedup factor vs AdamW on Sigma MoE reaches around 1.36$\times$ at the end of training.  The speedup curve at different over-training regimes in \Cref{fig:sigma speedup} does not follow a monotonically diminishing trend either; it follows a pattern of increase $\rightarrow$ decrease $\rightarrow$ increase,  as we continue to over-train the model.  \emph{The same pattern is observed in subsequent 8B model overtrained runs from \Cref{subsec: qwen3-8B}, indicating that the there is no strong evidence of diminishing return. }
\end{itemize}

\begin{findings}
\textbf{Findings 8: \arosink{} applied to all matrix parameters (including non-hidden layer parameters), outperforms its hybrid variant in the long term. }
\end{findings}

In \Cref{exp: gpt}, we showed that under 1$\times$ compute optimal regime, when applied to all matrix parameters, \name{} (especially \arosink{})  performs competitively to its hybrid version (non-hidden layer parameters trained by AdamW). The results for 2B Sigma-MoE + 8$\times$ over-train regime show that the true benefit of full-model optimization arises at over-train regime. Particularly, in \Cref{fig:sigma speedup}, we observe that from the 3$\times$-4$\times$ over-training stage, \arosink{}-full-model mode starts to outperform its hybrid counterpart, and remains strictly better for the rest of training. This evidence highlights \emph{the potential of using \name{} as a unified update rule across all matrix parameters, addressing the open challenge described in \cite{liu2025muon}}.

\begin{findings}
\textbf{Findings 9: The throughput of \arosink{} is similar to AdamW}
\end{findings}
Better data efficiency/speedup in steps does not necessarily translate into practical wall-clock gains when the per-step computation is more expensive than that of the baselines. \Cref{fig: sigma throughput} reports the normalized training throughput against AdamW, averaged over the entire training run. We observe that \arosink{} achieves throughput comparable to Muon, Dion, and AdamW. On average, \arosink{} incurs only a $0.5\%$ total overhead relative to AdamW. This indicates that the overhead of \arosink{} is manageable, and the improved data efficiency/speedup of \arosink{} can be effectively translated into wall-clock speedups.

\begin{takeaways}
    \textbf{Takeaways for \Cref{subsec: Exp Sigma-MoE}: on large-scale pretraining of Sigma-MoE-2B + up to 8$\times$ overtrain budget, \name{} significantly outperforms AdamW, gradient orthogonalization, and eigen-rotation variants. Particularly, \arosink{} with full model mode gives the best long horizon performance.}
\end{takeaways}

\subsection{Qwen3-8B pretraining}
\label{subsec: qwen3-8B}
To test the performance of ARO when scaling up even further, we pretrain an 8B Qwen3 dense model \citep{yang2025qwen3} with SlimPajama \citep{cerebras2023slimpajama}. Based on the performance analysis in previous sections and the expensive nature of training a large model, we pick \arosink{} in the ARO family and compare it with AdamW and Muon at the time of drafting.  

\paragraph{Setup and how guidelines are followed.} 
We use a hybrid setup for ARO and mixed precision training. Following \Cref{exp: guidelines}, we train with a larger $4$K sequence length and a $14$M-token batch size, using SlimPajama which is approximately $4\times$ the Chinchilla compute-optimal regime \citep{hoffmann2022training}. Since end-to-end learning-rate tuning is infeasible at the 8B-plus overtraining scale, we follow the hyperparameter setting used for LLaMA~3.1~8B \citep{grattafiori2024llama}: a peak learning rate of $3\times10^{-4}$ with a cosine decay schedule, linear warm-up, and $400$ warm-up steps. We use decoupled weight decay with coefficient $0.1$. Finally, we align \arosink{} and Muon via RMS-norm matching to AdamW (\Cref{sec:hyperparam}). We use the entire validation dataset to compute the validation loss, containing $500M$ tokens. For detailed setup, please refer to appendix \cref{subapp: Qwen3 8B setup}.
\paragraph{Distributed training.}
We adopt the Pytorch FSDP2 distributed strategy with efficient ARO implementation (\cref{subsec: FSDP2}) and B200 NVIDIA GPUs. All tensors are sharded on the first dimension.

\begin{figure}
    \centering
    \begin{minipage}[t]{0.3\linewidth}
        \includegraphics[width=1\textwidth]{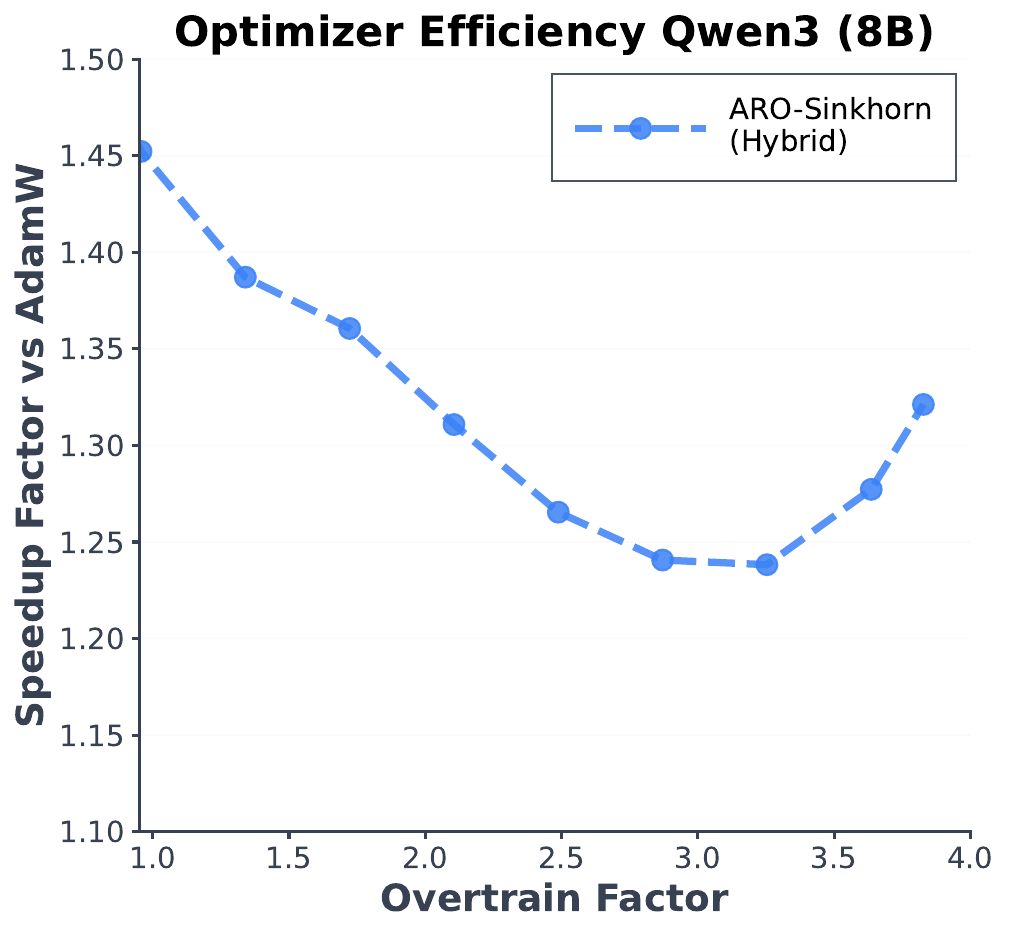}
        \caption{The validation speedup factors of \arosink{} compared to AdamW at different overtrain factors ($1\times$ is Chinchilla compute optimal). The target loss is the AdamW loss at the chosen scale.}
        \label{fig: Qwen speedup}
    \end{minipage}\hfill
    \begin{minipage}[t]{0.3\linewidth}
        \includegraphics[width=1\textwidth]{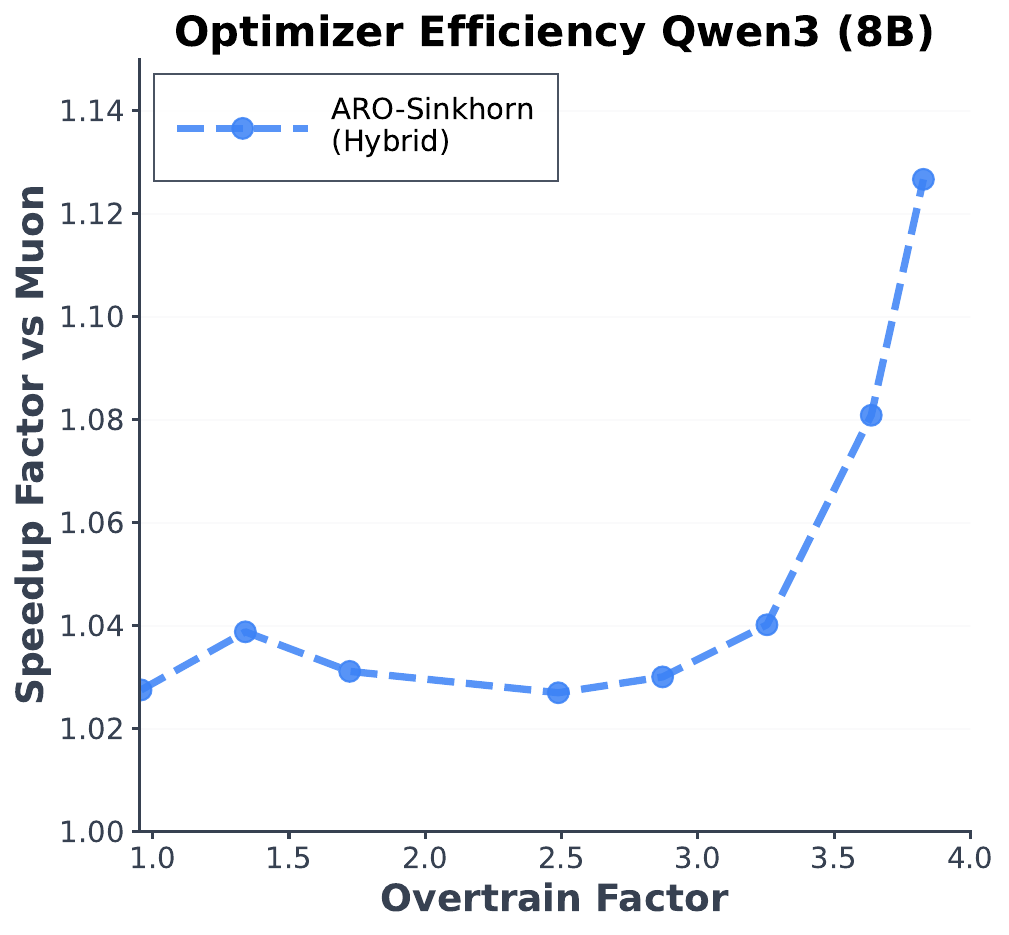}
        \caption{The validation speedup factors of \arosink{} compared to Muon at different overtrain factors ($1\times$ is Chinchilla compute optimal). The target loss is the Muon loss at the chosen scale.}
        \label{fig: Qwen speedup Muon}
    \end{minipage}\hfill
    \begin{minipage}[t]{0.3\linewidth}
        \includegraphics[width=1\textwidth]{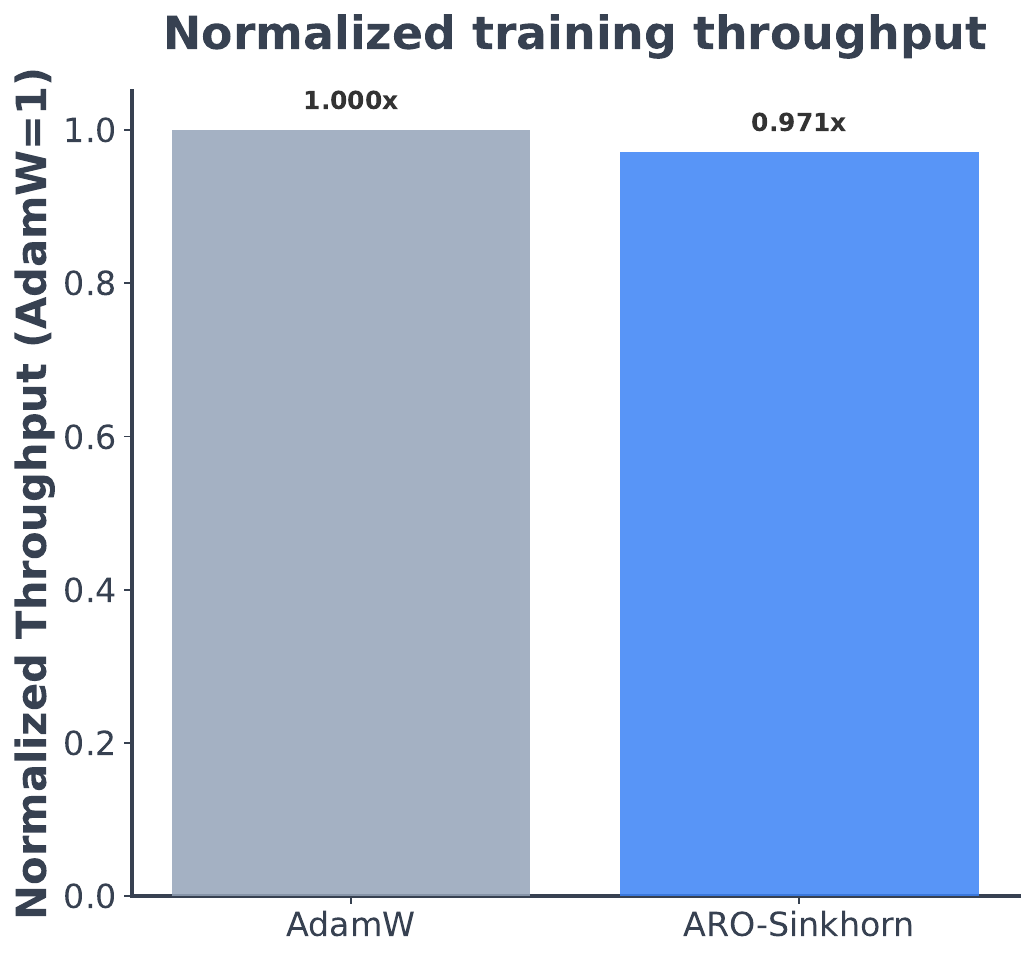}
        \caption{The throughput of \arosink{} compared to AdamW averaged across the entire training. }
        \label{fig: Qwen throughput}
    \end{minipage}
\end{figure}

\begin{findings}
\textbf{Findings 10: At 8B model + 4$\times$-overtrain scale, \arosink{} achieves non-trivial speed-up compared to AdamW and Muon with no evidence of diminishing returns.}
\end{findings}
\Cref{fig: Qwen speedup,fig: Qwen speedup Muon} reports the validation speedup factors of \arosink{} relative to AdamW and validation curve v.s. Muon, respectively. We observe:
\begin{itemize}
    \item \arosink{} achieves non-trivial speed-up compared to both AdamW and Muon. At the end of training, \arosink{} reaches a speedup factor of around 1.32$\times$ and 1.13$\times$ respectively, both consistent with results on smaller scales in previous sections. 
    \item Compared with \Cref{fig:sigma speedup} in \cref{subsec: Exp Sigma-MoE}, we observe that despite the difference in architecture (MoE vs dense) and scale (2B vs 8B), and overtrain setup (8$\times$ vs 4 $\times$), the speedup curve over AdamW follows a similar and robust pattern. Both initially reaches around 1.5$\times$ at the early $1/4$ stage, then decrease until around $3/4$ of the entire training; then start increase again at the later stages of training, both reaching around 1.35$\times$ speedup.  
\end{itemize}

Together, these observations suggest that for \name{} \emph{we do not observe evidence of diminishing returns observed in \cite{wen2025fantastic}}. The overall speedup does not exhibit a diminishing-return behavior with respect to data scale, and its speedup appears to be robustly transferable across model architectures (i.e., MoE and dense models) and model sizes (i.e., 2B and 8B), in an over-trained setup. \textbf{We also note that, when compared on a similar pretraining FLOPs budget, the data-efficiency gains from \arosink{} are comparable with recent architectural improvements such as Canon \citep{Allen2025-resonate} and mHC \citep{xie2025mhc}.}

Finally, compared with prior work on hyperparameter transfer for matrix optimizers \citep{qiu2025hyperparameter}, \textbf{our results across scales is obtained i) following general guidelines in \Cref{exp: guidelines} under standard parameterization, without optimizer-specific $\mu P$ scaling; ii) verified at a much larger scale model + up to over-trained regime, which is not covered in \citep{qiu2025hyperparameter}}. Further deriving the exact $\mu P$ scaling rules for \name{} may further improve performance, which we leave for future work.

\begin{findings}
\textbf{Findings 11: At larger scale \arosink{} has similar throughput as AdamW translating step-wise speed-up to wall-clock time advantages}
\end{findings}
\Cref{fig: Qwen throughput} plots the normalized training throughput averaged over the entire training against AdamW. While the pure computation time consumed by \name{} is larger than AdamW, we found that the forward/backward computation takes the majority of time during training. As a result, the end-to-end training throughput of \arosink{} is $3.0\%$ lower than that of AdamW. Similar to Findings~2 in \cref{subsec: Exp Sigma-MoE}, this overhead is quite manageable in practice, and the data efficiency speedup of \arosink{} can be translated to wall-clock time advantages.  We note that under settings where forward/backward computation is much less dominant (e.g., small global batch), the overhead of \arosink{} might be significantly larger than the reported number.

\begin{takeaways}
\textbf{Takeaway for \Cref{subsec: qwen3-8B}: On large-scale pretraining of Qwen-dense-8B with up to a $4\times$ overtraining budget, \arosink{} delivers a strong and consistent speedup over both AdamW and Muon, consistent with the patterns observed in our smaller-scale experiments in previous sections.}
\end{takeaways}

\subsection{Scaling analysis: putting it all together}

Finally, we focus on understanding the scaling behavior of \name{}, by aggregating and extending results from previous sections. We use \name{}-Sinkhorn variant (hybrid, SCQR) in this section.

\begin{figure}
    \centering
    \begin{minipage}[t]{0.4\linewidth}
        \includegraphics[width=1\textwidth]{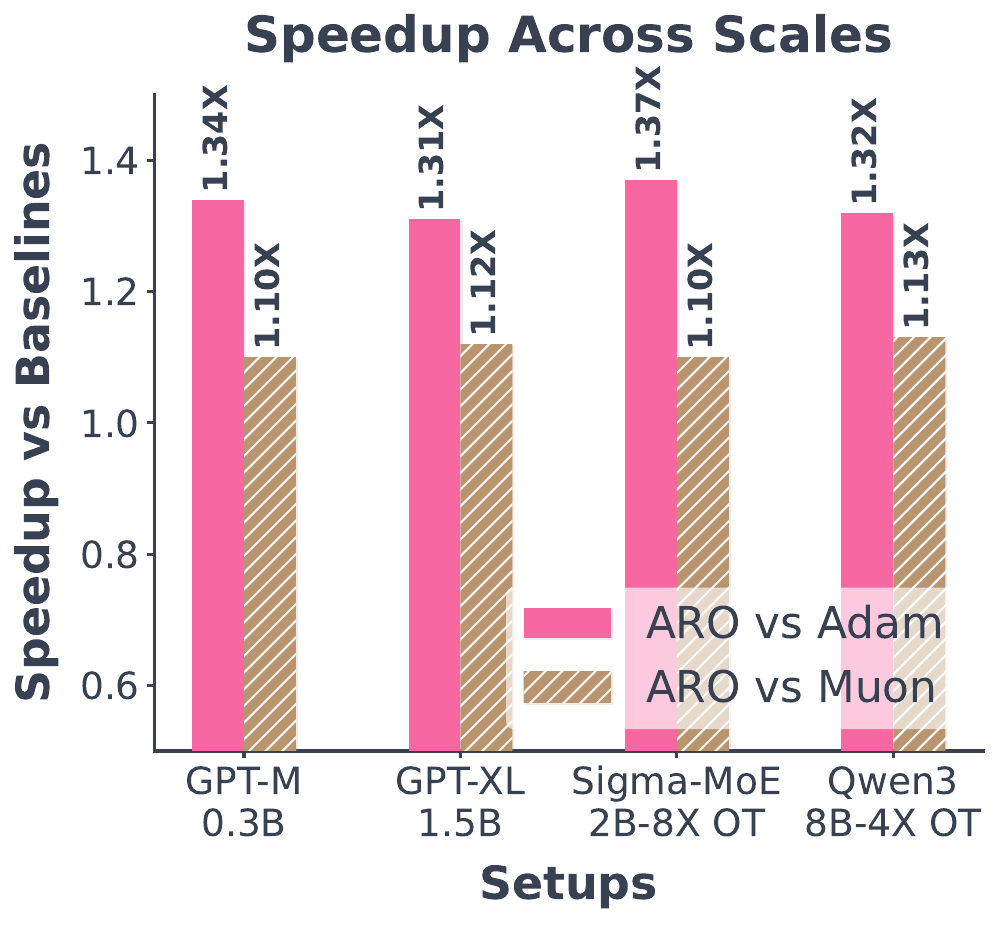}
        \caption{Speedup of \name{} over AdamW and Muon across model scales and training regimes. Consistent gains are observed from small GPT-style models to larger Sigma-MoE and Qwen variants, spanning standard (1$\times$) and overtrained (4$\times$, 8$\times$ Chinchilla) settings.}
        \label{fig: speedup_overall}
    \end{minipage}\hfill
    \begin{minipage}[t]{0.4\linewidth}
        \includegraphics[width=1\textwidth]{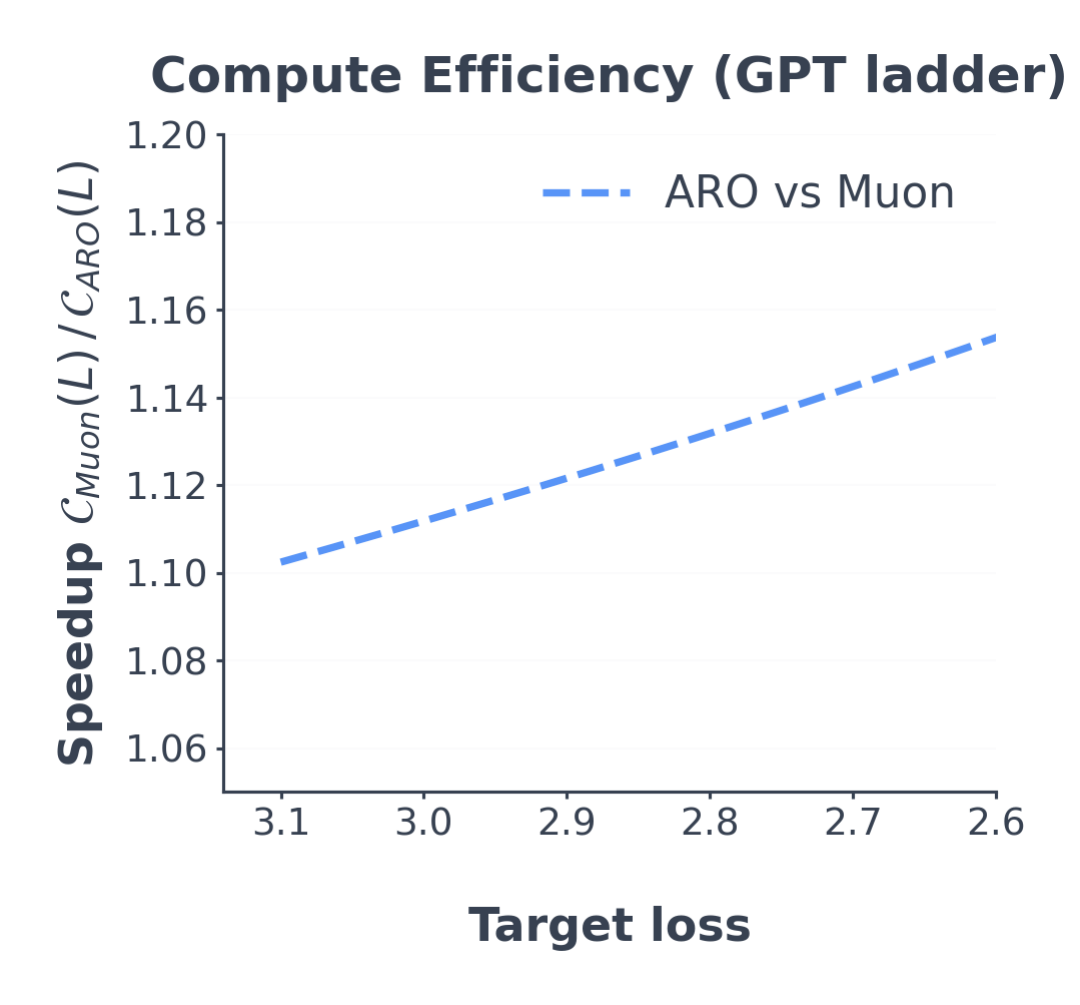}
        \caption{Compute efficiency at fixed target loss on the GPT-2 ladder. We plot $C_{\text{Muon}}(L)/C_{\text{\name{}}}(L)$, where $C(L)$ is the estimated training FLOPs required to reach loss $L$. Values $>1$ indicate better efficiency. Results show \name{} is consistent $1.1$--$1.15\times$ more efficient than Muon.}
        \label{fig: scaling}
    \end{minipage}\hfill
\end{figure}

\paragraph{Scaling across diverse setups and budgets} \Cref{fig: speedup_overall}(a) shows the speedup of \name{}-Sinkhorn relative to AdamW and Muon across a diverse set of model scales and training regimes covered in previous sections. We observe consistent improvements across all four settings, spanning small to large models and both standard (1$\times$ Chinchilla) and overtraining budgets (4$\times$, 8$\times$). Notably, the magnitude of speedup remains stable as scale increases, with no evidence of diminishing returns when moving to larger architectures or longer training horizons. This consistency indicates that \textbf{the gains from \name{}-Sinkhorn are robust to both model size and compute budget, and persist under aggressive overtraining regimes commonly used in modern LLM scaling.}

\paragraph{Compute efficiency vs.\ Muon over the GPT ladder.}
Finally, we evaluate the computational efficiency of \name{}-Sinkhorn relative to Muon, under a slightly different metric. We train both methods on a GPT-2 ladder consisting of four model sizes (124M, 350M, 700M, and 1.5B) under a fixed training protocol, and measure the total training compute $C$ (in FLOPs) required to reach a given validation loss $L$. Following standard practice, we fit a two-parameter scaling law of the form $L(C) = A\,C^{-\alpha}$ to the loss--compute trajectories across model sizes, for both \name{}-Sinkhorn and Muon.  Inverting the fitted relations yields the compute required to reach a target loss, $C(L) = (A / L)^{1/\alpha}$, which allows us to define the compute speedup at loss $L$ as $S(L) = C_{\text{Muon}}(L) / C_{\name{}}(L)$. We report $S(L)$ over the interpolation range $L \in [2.6, 3.1]$, avoiding extrapolation beyond the observed losses. As shown in \Cref{fig: scaling}, \name{}-Sinkhorn consistently achieves $S(L) > 1$, corresponding to a $1.1\times$--$1.15\times$ less compute required than Muon, with the speedup increasing as the target loss decreases. This trend indicates improved sample efficiency for \name{}-Sinkhorn that persists and slightly strengthens along the GPT ladder.

\paragraph{Cautions (interpreting speedup).}
In this paper so far we have used two metrics of speedup: \emph{matched-run step speedup} (\Cref{fig: m3}, \Cref{fig:sigma speedup}, \Cref{fig: Qwen speedup}, \Cref{fig: Qwen speedup Muon}, \Cref{fig: speedup_overall}), defined at training step $t$ as $t/t_{\text{method}}$, where $t_{\text{method}}$ is the step at which the proposed method first attains the baseline loss at step $t$ under identical training settings (model/data/schedule); and \emph{compute-to-target-loss speedup} (\Cref{fig: scaling}) on the GPT-2 ladder, estimated as $C_{\text{baseline}}(L)/C_{\text{method}}(L)$ from inverted scaling fits. The latter metric could overstate absolute gains because a faster optimizer may reach the target loss earlier and effectively traverse the decay phase with fewer epochs, implicitly rewarding early lead; we therefore treat these ratios as qualitative evidence for additional insight. In both cases, they should not be interpreted at face value or compared directly across pipelines.

\section{Why Rotations: The Symmetry Hypothesis for Matrix Optimization}
\label{sec:symmetry}

Previous sections presented \name{} as an optimizer family built around gradient rotations and validated it empirically, while intentionally postponing the theory. With the rotated-optimization abstraction developed earlier in hand, we now return to the ``why'' questions that we asked in \Cref{sec: insights}:

\setcounter{rq}{2}

\begin{questions}
\begin{rqlist}
  \item \textbf{Why do gradient rotations emerge as an effective primitive for LLM optimization?}
  \item \textbf{What does this suggest about the organizing principle for matrix optimizer design, and about what “matrix optimization” is fundamentally doing?}
\end{rqlist}
\end{questions}

These two questions are tightly linked. Under the rotated-optimization abstraction, a large portion of modern ``matrix optimization'' can be expressed as choosing a \emph{rotation policy} for the gradient and then applying a base update rule in that rotated coordinate system. Then, understanding \emph{why rotations help} is also a direct route to understanding \emph{what matrix optimization is capturing}.  Much of the matrix optimization literature motivates specific rotation/geometry choices via non-Euclidean descent and matrix-norm control (e.g., spectral descent \citep{carlson2015preconditioned,carlson2015stochastic,carlson2015stochasticb,bernstein2024old,jordan2024muon}) or via structured curvature/whitening surrogates \citep{gong2025towards,xie2025structured,gupta2018shampoo,morwani2024new,zhang2024transformers}. Yet these interpretations still depend on choosing particular ``oracles'' (norms, Kronecker factorizations, whitening rules, eigendirections) among many plausible alternatives, so it remains unclear \emph{a priori} why certain canonical choices should be preferred on modern LLM architectures such as Transformers.

Our methodological advancement in \name{} sharpens this ambiguity. Empirically, \name{} improves a wide range of base optimizer rules, and several of its best-performing instantiations do not reduce to either
(i) spectral descent, nor
(ii) eigen-rotation/whitening updates that explicitly track eigendirections of second-moment surrogates. 
This motivates two hypothesis. First, there should not be strong intrinsic privilege attached to any \emph{single} geometric oracle (e.g., the spectral norm) as a starting point for understanding why matrix optimizers work. Second, successful constructions are not arbitrary: the fact that \name{} works robustly across many settings, while going beyond standard templates, indicates that some more fundamental organizing principle is at play.

In this section, we hypothesize that a strong candidate is global \emph{symmetry structures} of loss landscape, that is, transformations of model parameters that leave the network function (and hence the loss) unchanged \citep{zhao2022symmetries,zhao2022symmetry,zhao2025symmetry,ziyin2025parameter, ashkboos2024slicegpt}. These symmetries organize the loss landscape into equivalence classes (orbits) and induce large, structured subsets of loss level sets. 

We show that, under suitable conditions, \name{}-family (hence its specific instantiations) can be recast within the symmetry teleportation (ST) framework \citep{zhao2022symmetry}. ST exploits loss-invariant symmetries to move along level sets and thereby reshape the local descent geometry, and has been shown to improve optimization behavior in practice \citep{zhao2023improving}.  Particularly, we show that \name{} can be derived as non-Euclidean symmetry teleportation under \emph{rotational symmetry groups}, a symmetry structure that exists in transformer architectures \citep{ashkboos2024slicegpt}. Therefore, we advance and hypothesize a unifying viewpoint:

\begin{takeaways}
    \begin{center}
\textbf{Symmetry Hypothesis} \emph{: a substantial part of matrix optimizers can be derived from, understood through, and improved using global symmetry properties of the loss landscape induced by model architectures.}
\end{center}

\end{takeaways}

From this viewpoint, several effective matrix optimizers (including \name{}, Muon, and even AdamW) naturally emerge as symmetry-exploiting procedures for specific invariances that (at least approximately) exists in deep neural networks, revealing a non-trivial connection that leads to the core nature of matrix optimization. We hypothesize that optimizer design can and should leverage those architecture-specific symmetries as intrinsic global structure. 

\subsection{Loss landscape symmetries and Symmetry teleportation}

A large fraction of optimization theory for deep learning is phrased in terms of \emph{local} differential information of the loss landscape (gradients, smoothness, curvature surrogates). In contrast, \emph{symmetries} are \emph{global} structures: they describe transformations of parameters that leave the loss unchanged, thereby organizing the loss landscape into large families of equivalent points. This global viewpoint is particularly natural in overparameterized models, where many distinct parameter values represent the same input--output function. 

\paragraph{Loss landscape symmetry and parameter symmetry.}
Let $\gL:\mathcal{W}\to\R$ be a loss function defined on a parameter space $\mathcal{W}$ (for us, typically $\mathcal{W}\subseteq \R^{m\times n}$ for a matrix parameter). Let $G$ be a group acting on $\mathcal{W}$; we write the action as
$
g\cdot \mW \in \mathcal{W}
$
for $g\in G$ and $\mW\in\mathcal{W}$. We say that $\gL$ is \emph{$G$-invariant} (or has a \emph{$G$-symmetry}) if
\begin{equation}
\gL(g\cdot \mW) \;=\; \gL(\mW),
\qquad \forall \mW\in\mathcal{W},~\forall g\in G.
\label{eq: G_invariance}
\end{equation}
For a fixed $\mW$, the \emph{orbit} of $\mW$ is the set
$
\mathcal{O}_G(\mW):=\{g\cdot\mW: g\in G\}.
$
Under \Cref{eq: G_invariance}, the entire orbit $\mathcal{O}_G(\mW)$ lies on the same loss level set, i.e., every point in the orbit achieves the same loss value.

Such symmetries are widespread in neural networks: many neural network architectures admit parameter-space symmetries, including discrete symmetries (e.g., hidden unit permutations) and continuous symmetries (e.g., rotations, rescalings and other reparametrizations), arising from the compositional and overparameterized nature of deep models \citep{zhao2022symmetry, zhao2023improving, zhao2022symmetries, zhao2025symmetry, ziyin2025parameter}. These symmetries induce large, structured level sets in parameter space. Importantly, while the loss is constant along an orbit, the \emph{optimization dynamics} (gradients and update directions) can vary dramatically across points on the same orbit, which motivates exploiting symmetries to accelerate training.

\paragraph{Symmetry teleportation (ST).}
Symmetry teleportation \citep{zhao2022symmetry, zhao2023improving} is built on a simple idea to exploit symmetries:

\begin{quote}
\emph{If we can move to a different point with the same loss but a more favorable gradient geometry, then subsequent optimization steps can improve.}
\end{quote}

In other words, ST treats symmetries as additional degrees of freedom to improve optimization, which is as opposed to gauge invariant methods such as natural gradient descent \citep{amari2016information}. In its basic form, ST augments gradient-based optimization by occasionally applying a symmetry transformation that preserves the loss, but changes the local properties of the gradient field. In particular, ST teleport the current weight along the orbit to a point with ``steeper'' gradient (hence faster instantaneous convergence), then take a gradient step there.

To make this concrete, consider standard gradient descent
\begin{equation}
\mW_{t+1} \;=\; \mW_t \;-\;\eta\,\nabla \gL(\mW_t),
\label{eq: gd}
\end{equation}
where $\eta>0$ is the step size. Under symmetry \Cref{eq: G_invariance}, one can move $\mW_t$ to any $\widetilde{\mW}_t\in\mathcal{O}(\mW_t)$ without changing $\gL$; the canonical choice of ST chooses the orbit element that maximizes the Euclidean gradient norm (which corresponds to local instantaneous loss improvement rate \citep{zhao2022symmetry}):
\begin{equation}
g_t \in \argmax_{g\in G}\ \big\|\nabla \gL(g\cdot \mW_t)\big\|_2,
\qquad 
\widetilde{\mW}_t := g_t\cdot \mW_t.
\label{eq: st_choice}
\end{equation}
Then ST applies a gradient step at the teleported point, and (optionally) maps back by undoing the group action.

To make the mechanism explicit, it is helpful to rewrite ST as a \emph{teleport--optimize} cycle. At iteration $t$, given $\mW_t$, choose a symmetry element $g_t\in G$, teleport to $\widetilde{\mW}_t:=g_t\cdot \mW_t$, take a gradient step there:
\begin{equation}
\mW_t \xrightarrow{\text{choose }g_t}
\widetilde{\mW}_t := g_t\cdot \mW_t
\ \xrightarrow{\text{GD step}}\ 
{\mW}_{t+1}:=\widetilde{\mW}_t - \eta\,\nabla \gL(\widetilde{\mW}_t).
\label{eq: st_euclid}
\end{equation}

When $\gL$ is exactly $G$-invariant, teleportation itself does not change the loss value, i.e.\ $\gL(\widetilde{\mW}_t)=\gL(\mW_t)$. This decouples the act of moving far/moving fast, from the loss decrease, allowing ST to search globally within a level set for points with improved local geometric properties encountered by the optimizer. 

Empirically and theoretically, ST is appealing as it can accelerate convergence and improve generalization in various settings \citep{zhao2023improving,zhao2025symmetry}. Also, in an idealized setting, ST behaves similarly to Newton's method locally:

\begin{remarks}
\begin{remark}[Symmetry teleportation locally aligns with the update direction of Newton's method]
A striking connection established in \citep{zhao2023improving} is that, under strict convexity assumption on the objective, teleporting to a \emph{critical point on a loss level set} can make the subsequent gradient step behave like a (damped) Newton step. Concretely, let $w'$ solve the constrained maximization
$$
w' \in \argmax_{w}\ \frac{1}{2}\|\nabla \gL(w)\|_2^2
\quad\text{s.t.}\quad \gL(w)=\gL(w_0),
$$
and assume $\nabla \gL(w')\neq 0$. Then \citep{zhao2023improving} shows that one gradient step from $w'$ can be written as
\begin{equation}
w_1 \;=\; w' - \gamma \nabla \gL(w')
\;=\; w' - \gamma \lambda_0 \,\nabla^2\gL(w')^{-1}\nabla \gL(w'),
\label{eq: st_newton_like}
\end{equation}
for some scalar $0\le \lambda_0 \le \lambda_{\max}(\nabla^2\gL(w_0))$. In other words, once teleportation reaches such a level-set critical point, the subsequent first-order step aligns with the Newton direction up to an (unknown) scalar factor. This provides a principled explanation for why symmetry-based ``jumps'' can inject higher-order behavior without explicitly computing Hessians.
\end{remark}
\end{remarks}

\subsection{\name{} as generalized symmetry teleportation} \label{sec: deriving_optimizer}

We now connect \name{} to ST from a symmetry viewpoint. We will:
(i) specify the symmetry group relevant for matrix parameters (we focus on rotational symmetries);
(ii) generalize ST to non-Euclidean base rules (so that continuous rotational symmetries become exploitable);
(iii) show that \name{} is an instantiation of generalized ST under rotational symmetry.

\paragraph{(one-sided) Rotational symmetries.}
Let $\mW\in\R^{m\times n}$ denote a matrix parameter, and we focus on \emph{left} orthogonal transformations acting on the row space. Let
$
SO(m) := \{\mR\in\R^{m\times m} : \mR^\top\mR=\mI,\ \det(\mR)=1\}.
$
We say $\gL$ has \emph{one-sided rotation invariance} if
\begin{equation}
\gL(\mR\mW) = \gL(\mW), \qquad \forall \mR\in SO(m),~\forall \mW\in\R^{m\times n}.
\label{eq: left_rot_invariance}
\end{equation}

Variants of this invariance exist widely in modern architectures. For example, recent work has identified exact one-sided rotational invariances in transformers \citep{ashkboos2024slicegpt}, which play a crucial role in model quantization. We will return to the transformers architecture and show how they contain rich rotational symmetries that can be expressed in exact same form as \Cref{eq: left_rot_invariance} in detail in \Cref{sec: importance_of_symmetry_in_ours}.

\paragraph{Generalized symmetry teleportation (GST).}
The original ST formulation \citep{zhao2022symmetry} is most naturally paired with Euclidean geometry, e.g., by selecting $g$ to maximize $\|\nabla \gL(g\!\cdot\!\mW)\|_2$ along a symmetry orbit. For the rotational symmetries considered above, this Euclidean objective can become \emph{uninformative}: the Euclidean gradient norm can be constant along the orbit of orthogonal/rotational transformations. To obtain a nontrivial selection signal, we generalize ST by (i) allowing a non-Euclidean base update rule and (ii) allowing the orbit search under general criterion for teleportation.

Recall that in \Cref{subsec: normed steepest descent}, we defined a base optimizer $f(\mG)$ via
\begin{equation}
f(\mG)\in \argmax_{\|\mZ\|\le 1}\langle \mG,\mZ\rangle,
\qquad\text{equivalently}\qquad
f(\mG)=\nabla_{\mG}\|\mG\|_*,
\label{eq: f_dual}
\end{equation}
where $\|\cdot\|$ is a chosen norm on $\mathbb{R}^{m\times n}$ and $\|\cdot\|_*$ is its dual.

GST is simply ST with this non-Euclidean base rule. At iteration $t$, for a symmetry action $g\cdot \mW$ and step size $\eta$, GST performs:
\begin{equation}
\mW_t \xrightarrow{\text{choose }g_t}
\widetilde{\mW}_t := g_t\cdot \mW_t
\ \xrightarrow{\text{steepest step}}\ 
{\mW}_{t+1}:=\widetilde{\mW}_t - \eta\, f\!\big(\nabla \gL(\widetilde{\mW}_t)\big).
\label{eq: gst_general}
\end{equation}

The only remaining design choice is how to select $g_t$. We can use symmetry as an additional degree of freedom and choose any $g_t$-selection criteria based on our preference.
Given a  criteria function $\gJ$ of choice, the teleportation is solved via  
\begin{equation}
g_t \in \argmax_{g\in\mathcal{G}} \gJ(g), 
\quad \text{or generally choose $g_t$ to improve }\gJ_t(g_t)\ \text{relative to a baseline.}
\label{eq: gst_choose_g_general}
\end{equation}

A natural canonical criterion, in direct analogy with classical ST, is the \emph{dual-norm} magnitude of the gradient at the teleported point:
\begin{equation}
\gJ_t(g) :=\;\big\|\nabla \gL(g\cdot \mW_t)\big\|_*.
\label{eq: gst_choose_g_dual}
\end{equation}
When $\|\cdot\|$ is Euclidean, $\|\cdot\|_*=\|\cdot\|_2$ and \Cref{eq: gst_choose_g_dual} reduces to the classical ST objective; when $\|\cdot\|$ is not invariant to the group action, $\|\nabla \gL(g\cdot\mW)\|_*$ can vary substantially with $g$, yielding a nontrivial orbit-selection signal.

\paragraph{\name{} as (generalized) symmetry teleportation.}
We now show \name{} is a specific instantiation of GST. Assume $\gL$ is left-rotation invariant as in \Cref{eq: left_rot_invariance}, and consider the left action of $SO(m)$ on $\mW\in\mathbb{R}^{m\times n}$:
$$
g\cdot \mW := g\,\mW,
\qquad g\in SO(m).
$$
Let $\mR_t\in SO(m)$ denote an arbitrary rotation matrix produced by some rotation policy. Since $SO(m)$ is closed under transpose and $\mR_t^{-1}=\mR_t^\top$, we may parameterize the applied group element without loss of generality as
$$
g_t := \mR_t^\top \in SO(m).
$$
Teleportation then maps the iterate to an equivalent representative
$$
\widetilde{\mW}_t = g_t\cdot \mW_t = \mR_t^\top \mW_t.
$$
By left-rotation invariance, $\gL(\widetilde{\mW}_t)=\gL(\mW_t)$, and the gradient in teleported coordinates satisfies the chain rule
$$
\nabla_{\widetilde{\mW}}\gL(\widetilde{\mW}_t)
=\nabla_{\mW}\gL(\mR_t^\top \mW_t)
=\mR_t^\top \nabla_{\mW}\gL(\mW_t)
=\mR_t^\top \mG_t,
$$
where $\mG_t:=\nabla_{\mW}\gL(\mW_t)$. A base step in the teleported coordinates is therefore
$$
\widetilde{\mW}_{t+1}=\widetilde{\mW}_t-\eta\, f(\mR_t^\top \mG_t).
$$

Next, consider the functional used in \name{} (\Cref{sec: instantaneous}):
\begin{equation}
\mathcal{J}(\mR;\mG_t,f)
:= \langle \mG_t,\ \mR f(\mR^\top \mG_t)\rangle.
\label{eq: exact_max_J}
\end{equation}
Using $\langle \mA,\mB\rangle=\mathrm{tr}(\mA^\top \mB)$ and orthogonality of $\mR$, we obtain
\begin{equation}
\mathcal{J}(\mR;\mG_t,f)
= \langle \mR^\top \mG_t,\ f(\mR^\top \mG_t)\rangle
= \|\mR^\top \mG_t\|_*,
\label{eq: J_is_dual}
\end{equation}
where the last equality follows from the defining property of $f$ as the normed-steepest-descent projection (cf.\ \Cref{eq: f_dual}). Hence, any rotation policy that increases $\mathcal{J}(\mR_t;\mG_t,f)$ is exactly implementing the GST principle of increasing the dual-norm proxy $\|\nabla \gL(g\cdot \mW_t)\|_*$ along the orbit (specialized here to $SO(m)$ acting by left multiplication).

Finally, since $\gL$ is invariant under $SO(m)$, we may optionally map back to the original coordinates after the base step, without changing the represented model or the loss.
This yields
\begin{equation}
\mW_{t+1}
= \mR_t\,\widetilde{\mW}_{t+1}
= \mR_t(\mR_t^\top \mW_t - \eta f(\mR_t^\top \mG_t))
= \mW_t - \eta\,\mR_t f(\mR_t^\top \mG_t),
\label{eq: gst_equals_rotated_update}
\end{equation}
which is exactly the rotated steepest descent update \Cref{eq: rotated-optimizer}. Therefore, \name{} can be viewed as a special case of GST under left-rotational symmetry $+$ canonical criterion (dual-norm of the gradient at the teleported point).  This equivalence is conceptually important: it suggests that \name{}, a large class of matrix optimizers, can be derived and understood as a family of algorithms that exploit the underlying symmetries of the loss landscape. This is done by choosing a base geometry $f$ and searching along the symmetry orbits for better optimization behavior. In short, we have showed:

\begin{takeaways}
    \textbf{Takeaway: \name{} can be derived as (generalized) symmetry teleportation, exploiting rotational symmetries of the loss landscape.  }
\end{takeaways}

\paragraph{Partial maximization in \name{}.}
The practical \name{} rule \Cref{eq: gf} does not aggressively push $\mathcal{J}(\mR;\mG_t,f)$ to the exact optimum. Instead, it performs a conservative projection at each step, as derived in \Cref{sec: aro}:
\begin{equation}
\mR_t \;=\; \mathrm{QR}\!\big(\,\mG_t\, f(\mR_{t-1}^\top \mG_t)^\top \,\big),
\label{eq: partial_max_recall}
\end{equation}
which consistently improves $\mathcal{J}(\mR;\mG_t,f)$ over eigen-rotation solutions (\Cref{fig: stats_of_aro}, (a-b)). We view this design choice as essential in practice:
\begin{itemize}
    \item \textbf{Tractability.} Exact maximization over $SO(m)$ is expensive for general $f$; \name{} replaces this with a single QR-based step, conditioned on $\mR_{t-1}$
    \item \textbf{Robustness.} In \Cref{sec: importance_of_symmetry_in_ours} we will show that the objective $\mathcal{J}$ of \name{}'s rotation policy is derived via descent lemma; under symmetry assumptions, it forms a lower bound of loss improvement under rotated steepest descent. Hence, it is a local surrogate, and under non infinitesimal learning rates and noisy gradients, aggressive maximization can be overly myopic. 
\end{itemize}

We conjecture that our partial maximization approach is result of stability-convergence speed trade off, sitting in-between the classic ST that aims at exactly maximizing $\mathcal{J}$, and recent variants \citep{yun2023riemannian} that minimizes $\mathcal{J}$ for reduced sharpness. We more details to an extended study in \Cref{subsec: discussion denoising} where we show the importance of our partial maximization approach. 

\paragraph{Choosing base optimizer.} In the next remark, we will discuss the choice of base optimizers from the perspective of symmetries.

\begin{remarks}
\begin{remark}[Choosing a proper base optimizer $f_t$]
To leverage a symmetry group (e.g., left rotations), the base geometry must break that symmetry. Formally, if the underlying norm $\|\cdot\|$ of $f_t$ is $SO(m)$-invariant, then $\|\mR^\top \mG\|_*=\|\mG\|_*$ for all $\mR$, and maximizing \Cref{eq: exact_max_J} is vacuous. Thus, a necessary condition for nontrivial rotation selection is:
\begin{equation}
\|\mR^\top \mG\|_* \ \text{varies depending on } \mR.
\label{eq: need_break_sym}
\end{equation}

This provides a concrete way to compare candidate base optimizers via their intrinsic invariances. Below we summarize three base optimizers used repeatedly in this work.

\begin{itemize}
    \item \textbf{RowNorm.} Define the RowNorm geometry by the norm
    \begin{equation}
    \|\mX\|_{\text{row}} := \max_{i\in[m]} \frac{\|\mX_{i,:}\|_2}{\sqrt{n}}.
    \label{eq: rownorm}
    \end{equation}
    This norm is invariant to left row permutations and to right orthogonal transforms:
    \begin{equation}
    \|\mP\mX\|_{\text{row}}=\|\mX\|_{\text{row}}\ \ \forall \mP\in\mathsf{Perm}(m),
    \qquad
    \|\mX\mQ\|_{\text{row}}=\|\mX\|_{\text{row}}\ \ \forall \mQ\in O(n).
    \label{eq: rownorm_invar}
    \end{equation}
    Thus RowNorm breaks left rotational invariance (good for our purposes) and still keeps right rotational invariance.
    
    \item \textbf{SignGD/Adam.} SignGD corresponds to the elementwise $\ell_\infty$ geometry
    \begin{equation}
    \|\mX\|_{\infty} := \max_{i,j} |\mX_{ij}|.
    \label{eq: linfty}
    \end{equation}
    The elementwise $\ell_\infty$ norm is also variant under left one-sided rotations. However, it is invariant to permutations of all entries (i.e., $\mathsf{Perm}(mn)$ acting on $\mathrm{vec}(\mX)$). 
    %
    \item \textbf{SinkGD.} SinkGD \citep{scetbon2025gradient} is a multi-normalization update. Its convex relaxation can be viewed as a hybrid, single norm geometry
    \begin{equation} 
     \|\mX\|_{\text{Sink}} := \max\!\Big\{ \max_i \tfrac{\|\mX_{i,:}\|_2}{\sqrt{n}},\ \max_j \tfrac{\|\mX_{:,j}\|_2}{\sqrt{m}} \Big\}.
    \label{eq: sink_relax}
    \end{equation}
    It is straightforward to verify that this hybrid norm is invariant to left row permutations $\mathsf{Perm}(m)$ and right column permutations $\mathsf{Perm}(n)$, but is not invariant to any one-sided rotations in general. Compared with $l_\infty$ geometry, the invariance of this SinkGD hybrid norm is much smaller.
    %
\end{itemize}

Among these three, we observe that SinkGD has the following properties:

\begin{itemize}
    \item \textbf{Left transformations:} By our convention, the left handside is where \name{} rotation is applied. Among these three options, SinkGD retains a smaller symmetry subgroup than SignGD, thereby leaving more flexibility for \name{}'s rotation rule. 
    \item \textbf{Right transformations:} SinkGD preserves right-permutation invariance, but not rotational invariance. Under our orientation convention, the right side corresponds to the feature dimension (e.g., the hidden dimension in MLP layers), where hidden units have permutation symmetry, but not rotational invariance. Therefore, on thise right side where \name{} is not applied, SinkGD is consistent with the architectural inductive bias.
\end{itemize}

\arosink{} empirically provides the strongest performance comapred the above variants. We conjecture that these properties could be a useful criterion for choosing $f_t$ within the \name{} class, and leave further theoretical analysis for future work.

\end{remark}
\end{remarks}

\subsection{The role of symmetry in \name{}'s rotation policy}
\label{sec: importance_of_symmetry_in_ours}

So far, symmetry structures on loss landscape has served as a theoretical bridge between \name{} and (generalized) symmetry teleportation. We now explicitly derive the mechanistic role of symmetry in \name{}, and show that  it makes the $\mathcal{J}$ objective used in \name{} a \emph{valid} criterion for choosing $\mR_t$.

\paragraph{Background: descent lemma and the loss decrease bound.}
Fix any norm $\|\cdot\|$ on $\mathbb{R}^{m\times n}$ (the norm underlying the base optimizer $f$) and its dual $\|\cdot\|_*$. A standard smoothness assumption in this geometry is:
\begin{equation}
\gL(\mW+\Delta)\ \le\ \gL(\mW) + \langle \nabla \gL(\mW),\Delta\rangle + \frac{\beta}{2}\|\Delta\|^2,
\qquad \forall \mW,\Delta,
\label{eq: smoothness_norm}
\end{equation}
for some $\beta>0$.
Under \Cref{eq: smoothness_norm}, the normed steepest descent step $\Delta=-\eta f(\mG)$ satisfies the classical one-step bound
\begin{equation}
\gL(\mW) - \gL(\mW-\eta f(\mG))
\ \ge\
\eta \|\mG\|_* - \frac{\beta}{2}\eta^2,
\label{eq: descent_lemma_bound}
\end{equation}
and the optimal step size $\eta=\|\mG\|_*/\beta$ yields the familiar bound
\begin{equation}
\gL(\mW) - \gL(\mW-\eta f(\mG))
\ \ge\
\frac{1}{2\beta}\|\mG\|_*^2.
\label{eq: descent_certificate}
\end{equation}
which depends on both the dual norm of gradient, and the smoothness constant that did not appear in our objective $\mathcal{J}$ in \name{}.

\paragraph{Symmetry makes \name{}'s rotation objective self-consistent.}
Now consider the rotated update used by \name{} (ignoring momentum for clarity):
$$
\mW_{t+1} = \mW_t - \eta\,\mR_t f(\mR_t^\top \mG_t),
\qquad \mR_t\in SO(m).
$$
Using the same analysis in \Cref{subsec: steepest descent under rotation} \Cref{eq: gst_equals_rotated_update}, this is equivalent to optimize rotated loss
$$
\widetilde{\gL}_{\mR_t}(\mZ) := \gL(\mR_t\mZ),
$$
with
$$
\mZ_{t+1} = \mZ_t - \eta\,f(\mR_t^\top \mG_t),
\qquad \mR_t\in SO(m).
$$
but with time-dependent $\mR_t$. Repeating the descent-lemma argument on $\widetilde{\gL}_{\mR_t}(\mZ)$, we obtain the bound
\begin{equation}
\widetilde{\gL}_{\mR_t}(\mZ_t) - \widetilde{\gL}_{\mR_t}(\mZ_{t+1})
\;\ge\;
\frac{1}{2\beta_{\mR}}\;\|\mR^\top \mG\|_*^2,
\end{equation}
equivalently,
\begin{equation} \label{eq: rotated_descent_bound}
\gL(\mW_t) - \gL(\mW_{t+1})
\;\ge\;
\frac{1}{2\beta_{\mR}}\;\|\mR^\top \mG\|_*^2,
\end{equation}

Therefore, in \name{} (and in general any rotational optimizer), whenever we have a time-dependent rotation policy $\mR_t$, we are also changing the smoothness constant. Therefore, if we select rotation via optimizing for $\|\mR^\top \mG\|_*$, there is no guarantee that the lower bound on loss decrease is improved. If the underlying loss function is strongly sensitive to the choice of $\mR_t$, the underlying  $2\beta_{\mR}$ may explode when maximizing $\|\mR^\top \mG\|_*$, making the bound $\frac{1}{2\beta_{\mR}}\;\|\mR^\top \mG\|_*^2$ smaller.

The importance of symmetry comes in at this point.  If $\gL$ is left-rotation invariant, then $\gL(\mR^\top \mW)=\gL(\mW)$ for all $\mR$, meaning the teleported loss is \emph{the same function} in the new coordinates. Then, the smoothness parameter is unchanged across rotations:
$$
\beta_{\mR} \equiv \beta
\qquad\text{for all }\mR\in SO(m),
$$
because teleportation does not change the loss landscape being optimized, only the representation of the gradient within a symmetry orbit.  Therefore, applying \name{} (or any rotational otpimizer) yields
\begin{equation}
\gL(\mW_t) - \gL(\mW_{t+1})
\;\ge\;
\frac{1}{2\beta}\,\|\mR^\top \mG\|_*^2.
\label{eq: rotated_descent_bound}
\end{equation}
Hence, improving $\|\mR^\top \mG\|_*$ (equivalently improving $\mathcal{J}$ by \Cref{eq: J_is_dual}) is a \emph{principled} way to improve the guaranteed instantaneous progress.  This provides a derivation of the crucial role of symmetry:
\begin{quote}
\emph{Under symmetry, the rotation-selection objective of \name{} matches the descent-lemma guarantee because the smoothness term is invariant across rotations.}
\end{quote}

We hereby note that while we have been using $\mathcal{J} = \|\mR^\top \mG\|_*^2$ to guide rotation policy as canonical choice, the same reasoning (regarding the core role of symmetry) applies to any rotational optimizers beyond \name{}, that involves manipulating dual norms (including eigen-rotations widely used in literature).

\paragraph{What if symmetries does not hold/only hold approximately?} It is important to note that, the above descent lemma argument is quite general, applies to cases where symmetry assumption does not hold/only hold approximately. Symmetry only gives a sufficient but not necessary condition to treat $\beta_{\mR_t}$ as constant. In practice this is a matter of trade-off. See remark below.

\begin{remarks}
\begin{remark}[Symmetry in practice: the trade-off of symmetry breaking] \label{remark: symmetry_tradeoff}
Real-world models rarely satisfy exact invariance \Cref{eq: left_rot_invariance}; instead, they exhibit \emph{approximate} symmetries.
From \Cref{eq: rotated_descent_bound}, rotation selection is implicitly a trade-off:
we seek transformations that increase the dual-norm proxy $\|\mR^\top \mG\|_*^2$ while keeping the effective smoothness constant $\beta_{\mR}$ controlled.
This perspective suggests why mild symmetry breaking (or slightly enlarging the transformation family beyond exact symmetries) can still be very useful in practice: if the induced increase in $\beta_{\mR}$ is modest compared with the gain in $\|\mR^\top \mG\|_*^2$ , the net bound can improve. This can be served as a criteria of whether a particular transformation/relaxation of symmetry should be used.

Meanwhile, this insight also aligns with our use of the partial maximization strategy (as in \Cref{eq: gf}), since it is a conservative method that avoids over-committing to the myopic objective, especially when the corresponding symmetries only hold approximately for the loss.
\end{remark}
\end{remarks}

\subsection{Rotational symmetries in residual streams of neural networks} \label{sec: transformer_symmetry}

So far we have been discussing symmetries at an abstract level, and demonstrated how \name{} can be derived as ST under one-sided rotational symmetry assumptions. One crucial aspect we have not discussed in detail is how does this property emerge from model architectures. Below we show that such rotational symmetries exist as a result of matrix multiplications and residual skip connections in neural networks. We note that while parameter symmetries of neural networks/transformers have been extensively studied \citep{wangmaximal, da2025hide, zhang2025beyond}, the majority of them focuses on invariance induced by Q-K inner products per attention block. On the other hand, residual-stream induced-symmetry focuses on global, cross-layer coupling that is relatively under-explored in gauge symmetry literature. Without loss of generality, below we use modern transformer architecture as an example to demonstrate such residual stream induced rotational symmetries. This is based on well-established results from the quantization literature \citep{ashkboos2024slicegpt}.  We will first re-state the symmetry in the parameter and tensor conventions used by common transformer implementations, such as \texttt{nanoGPT}; then, we will discuss how this aligns with how \name{} is applied to LLMs in our previous experiments in \Cref{exp}.

\paragraph{Transformer architecture}
Let $d$ be the embedding dimension, $T$ the sequence length, and $V$ the vocabulary size. Let $\mW_{\rm tok}\in\R^{V\times d}$ and $\mW_{\rm pos}\in\R^{T_{\max}\times d}$ denote token and positional embeddings. Given tokens $(s_1,\dots,s_T)$, define the initial residual stream $\mX^{(0)}\in\R^{T\times d}$ row-wise by
\begin{equation}
\mX^{(0)}_{t,:}
\;:=\;
\mW_{\rm tok}[s_t,:]\;+\;\mW_{\rm pos}[t,:],
\qquad t=1,\dots,T.
\label{eq: rms_gpt_init}
\end{equation}
We use RMSNorm as a scale-free normalization on each token row,
\begin{equation}
\mathrm{RMSNorm}(\mx)
\;:=\;
\frac{\mx}{\|\mx\|_2/\sqrt{d}},
\qquad
\mathrm{RMSNorm}(\mX)\ \text{applied row-wise.}
\label{eq: rmsnorm_def}
\end{equation}
A pre-norm transformer block $\ell\in\{1,\dots,L\}$ with attention and MLP weights
$$
\mW_Q^{(\ell)},\mW_K^{(\ell)},\mW_V^{(\ell)},\mW_O^{(\ell)}\in\R^{d\times d},
\qquad
\mW_{\rm up}^{(\ell)}\in\R^{4d\times d},\ \ \mW_{\rm down}^{(\ell)}\in\R^{d\times 4d},
$$
computes (suppressing biases/dropout and keeping attention abstract):
\begin{align}
\mU^{(\ell)} &:= \mathrm{RMSNorm}(\mX^{(\ell-1)}), \label{eq: rms_gpt_unorm_attn}\\
\mQ^{(\ell)} &:= \mU^{(\ell)}(\mW_Q^{(\ell)})^\top,\quad
\mK^{(\ell)} := \mU^{(\ell)}(\mW_K^{(\ell)})^\top,\quad
\mV^{(\ell)} := \mU^{(\ell)}(\mW_V^{(\ell)})^\top, \label{eq: rms_gpt_qkv}\\
\mA^{(\ell)} &:= \mathrm{Attn}\!\big(\mQ^{(\ell)},\mK^{(\ell)},\mV^{(\ell)}\big), \qquad
\mY^{(\ell)} := \mA^{(\ell)}(\mW_O^{(\ell)})^\top, \label{eq: rms_gpt_attn_out}\\
\mX^{(\ell-\frac12)} &:= \mX^{(\ell-1)} + \mY^{(\ell)}, \label{eq: rms_gpt_res1}\\[2pt]
\mV^{(\ell)}_{\rm mlp} &:= \mathrm{RMSNorm}(\mX^{(\ell-\frac12)}), \qquad
\mH^{(\ell)} := \phi\!\big(\mV^{(\ell)}_{\rm mlp}(\mW_{\rm up}^{(\ell)})^\top\big), \label{eq: rms_gpt_mlp_hidden}\\
\mZ^{(\ell)} &:= \mH^{(\ell)}(\mW_{\rm down}^{(\ell)})^\top,\qquad
\mX^{(\ell)} := \mX^{(\ell-\frac12)} + \mZ^{(\ell)}. \label{eq: rms_gpt_res2}
\end{align}
Finally, with LM head $\mW_{\rm head}\in\R^{V\times d}$, the logits are
\begin{equation}
\mathrm{logits}
\;=\;
\mathrm{RMSNorm}(\mX^{(L)})\,\mW_{\rm head}^\top.
\label{eq: rms_gpt_head}
\end{equation}

In certain scenarios, $\mW_{\rm head}$ is tied with $\mW_{\rm tok}$. Without loss of generality, in this paper we don't make distinction between $\mW_{\rm head}$ and $\mW_{\rm tok}$.

\paragraph{Cross-layer global rotational symmetry of transformer residual streams.}
A key observation in \citep{ashkboos2024slicegpt} is that RMSNorm commutes with orthogonal changes of basis on the residual stream:
\begin{equation}
\mathrm{RMSNorm}(\mX\mR^\top)\;=\;\mathrm{RMSNorm}(\mX)\mR^\top,
\qquad \forall \mR\in O(d),
\label{eq: rms_commute}
\end{equation}
and this commutation enables an exact global invariance of the transformer function under a \emph{coupled} rotation of \emph{all} matrices that touch the residual stream. Concretely, \citep[Theorem~1]{ashkboos2024slicegpt} states that for any $\mR\in O(d)$, the RMSNorm transformer defined above is functionally invariant if we simultaneously transform parameters as
\begin{equation}
\boxed{
\begin{aligned}
\widetilde{\mW}_{\rm tok} &= \mW_{\rm tok}\mR^\top, \qquad
\widetilde{\mW}_{\rm pos} = \mW_{\rm pos}\mR^\top, \qquad
\widetilde{\mW}_{\rm head} = \mW_{\rm head}\mR^\top,\\
\widetilde{\mW}_Q^{(\ell)} &= \mW_Q^{(\ell)}\mR^\top,\quad
\widetilde{\mW}_K^{(\ell)} = \mW_K^{(\ell)}\mR^\top,\quad
\widetilde{\mW}_V^{(\ell)} = \mW_V^{(\ell)}\mR^\top,\quad
\widetilde{\mW}_O^{(\ell)} = \mR\,\mW_O^{(\ell)},\\
\widetilde{\mW}_{\rm up}^{(\ell)} &= \mW_{\rm up}^{(\ell)}\mR^\top,\qquad
\widetilde{\mW}_{\rm down}^{(\ell)} = \mR\,\mW_{\rm down}^{(\ell)},
\qquad \ell=1,\dots,L.
\end{aligned}}
\label{eq: transformer_coupled_rotation}
\end{equation}
That is, matrices that \emph{consume} the residual stream (appear as $\mX(\cdot)^\top$) are post-multiplied by $\mR^\top$, while matrices that \emph{produce} residual-stream outputs (and hence enter residual additions) are pre-multiplied by $\mR$. Because attention and MLP outputs are added back into the same residual stream, this invariance is generated via strong \textbf{cross-layer coupling across the entire network}: the same $\mR$ must be shared across embeddings, all layers, and the LM head to preserve the computation exactly \citep{ashkboos2024slicegpt}. 

This form of cross-layer, coupled rotational invariance is exactly equivalent to our starting definition of one-sided rotational invariance in \Cref{eq: left_rot_invariance}, as shown below.

\begin{remarks}
    \begin{remark}[Equivalence to \Cref{eq: left_rot_invariance} via stacked parameters]  \label{remark: stack}Next we show that the coupled rotational symmetry rule \Cref{eq: transformer_coupled_rotation} can be re-written as one-sided rotational invariance in the form of \Cref{eq: left_rot_invariance}. To see this, orient each parameter so that the residual-stream dimension $d$ is always the \emph{row} dimension (e.g., take transposes for ``input-side'' matrices). For instance, $(\mW_Q^{(\ell)})^\top\in\R^{d\times d}$ transforms as
$$
(\widetilde{\mW}_Q^{(\ell)})^\top
=
(\mW_Q^{(\ell)}\mR^\top)^\top
=
\mR\,(\mW_Q^{(\ell)})^\top,
$$
and similarly $\widetilde{\mW}_O^{(\ell)}=\mR\mW_O^{(\ell)}$, $(\widetilde{\mW}_{\rm tok})^\top=\mR(\mW_{\rm tok})^\top$, etc. If we concatenate these oriented matrices into a single ``stacked'' parameter
\begin{equation}
\mW_{\rm stack}
:=
\Big[
(\mW_{\rm tok})^\top,\ (\mW_{\rm pos})^\top,\ (\mW_{\rm head})^\top,\
\big\{(\mW_Q^{(\ell)})^\top,(\mW_K^{(\ell)})^\top,(\mW_V^{(\ell)})^\top,\mW_O^{(\ell)},(\mW_{\rm up}^{(\ell)})^\top,\mW_{\rm down}^{(\ell)}\big\}_{\ell=1}^L
\Big]\in\R^{d\times M},
\label{eq: transformer_stack}
\end{equation}
then \Cref{eq: transformer_coupled_rotation} becomes a \emph{one-sided rotation} 
\begin{equation}
\mW_{\rm stack}\ \mapsto\ \widetilde{\mW}_{\rm stack}=\mR\,\mW_{\rm stack},
\qquad \mR\in SO(d),
\label{eq: transformer_stack_left}
\end{equation}
which exactly matches our general form of one-sided rotational symmetry $\gL(\mR\mW_{\rm stack})=\gL(\mW_{\rm stack})$ in \Cref{eq: left_rot_invariance}.  This implies that in principle, we can use the $\mW_{\rm stack}$-representation of a RMSNorm-tranformer, and apply the \name{} rule to figure out the shared rotation $\mR_t$. This means that we will need to compute a single \name{}-rotation by treating $\mW_{\rm stack}$ as a single matrix parameter, and then perform \name{} update using the shared rotation.
\end{remark}
\end{remarks}

\paragraph{Chain-coupled local rotations when residual coupling is absent or weak.}
The global sharing of a single rotation $\mR$ in \Cref{eq: transformer_coupled_rotation} is enforced by residual additions: the skip term and the branch output must be expressed in the same rotated basis for $(\mX+\mY)$ (and later $(\mX+\mZ)$) to be equivariant.
When the skip contribution is absent or weak compared with the branch output--i.e.,, when
\begin{equation}
\mX^{(\ell-\frac12)}=\mX^{(\ell-1)}+\mY^{(\ell)} \approx \mY^{(\ell)}
\qquad\text{and}\qquad
\mX^{(\ell)}=\mX^{(\ell-\frac12)}+\mZ^{(\ell)} \approx \mZ^{(\ell)},
\label{eq: weak_or_no_skip}
\end{equation}
the computation no longer forces a single globally shared rotation.
In this regime, we may use a \emph{sequence} of rotations $\{\mR_\ell\}_{\ell=0}^{L}$ with $\mR_\ell\in O(d)$, and transform parameters as
\begin{equation}
\boxed{
\begin{aligned}
\widetilde{\mW}_{\rm tok} &= \mW_{\rm tok}\mR_0^\top,\qquad
\widetilde{\mW}_{\rm pos} = \mW_{\rm pos}\mR_0^\top,\qquad
\widetilde{\mW}_{\rm head} = \mW_{\rm head}\mR_L^\top,\\
\widetilde{\mW}_Q^{(\ell)} &= \mW_Q^{(\ell)}\mR_{\ell-1}^\top,\quad
\widetilde{\mW}_K^{(\ell)} = \mW_K^{(\ell)}\mR_{\ell-1}^\top,\quad
\widetilde{\mW}_V^{(\ell)} = \mW_V^{(\ell)}\mR_{\ell-1}^\top,\quad
\widetilde{\mW}_O^{(\ell)} = \mR_{\ell}\mW_O^{(\ell)},\\
\widetilde{\mW}_{\rm up}^{(\ell)} &= \mW_{\rm up}^{(\ell)}\mR_{\ell-1}^\top,\qquad
\widetilde{\mW}_{\rm down}^{(\ell)} = \mR_{\ell}\mW_{\rm down}^{(\ell)},
\qquad \ell=1,\dots,L,
\end{aligned}}
\label{eq: transformer_local_rotation}
\end{equation}
so that layer $\ell$ consumes activations rotated by $\mR_{\ell-1}$ and produces activations rotated by $\mR_{\ell}$.
Compared with the globally enforced rule \Cref{eq: transformer_coupled_rotation}, this ``chain-coupled'' family enlarges the transformation class by allowing $\mR_{\ell-1}\neq \mR_{\ell}$; when \Cref{eq: weak_or_no_skip} holds only approximately, this is an instance of the symmetry--flexibility trade-off discussed in \Cref{remark: symmetry_tradeoff}.
It also retains strictly more structure than the naive alternative that rotates each parameter matrix independently (ignoring any shared $\mR_\ell$ across the matrices in \Cref{eq: transformer_local_rotation}).
In \Cref{sec: symmetry_predictions} we provide preliminary empirical evidence that this chain-coupled rotation sharing achieves a strong performance, outperforming (i) the globally shared symmetry \Cref{eq: transformer_coupled_rotation} and (ii) the naive per-parameter rotation that ignores coupling.

\Needspace{18\baselineskip}
\begin{wrapfigure}{r}{0.35\linewidth}
    \vspace{-1.2em}
    \centering
    \includegraphics[width=\linewidth]{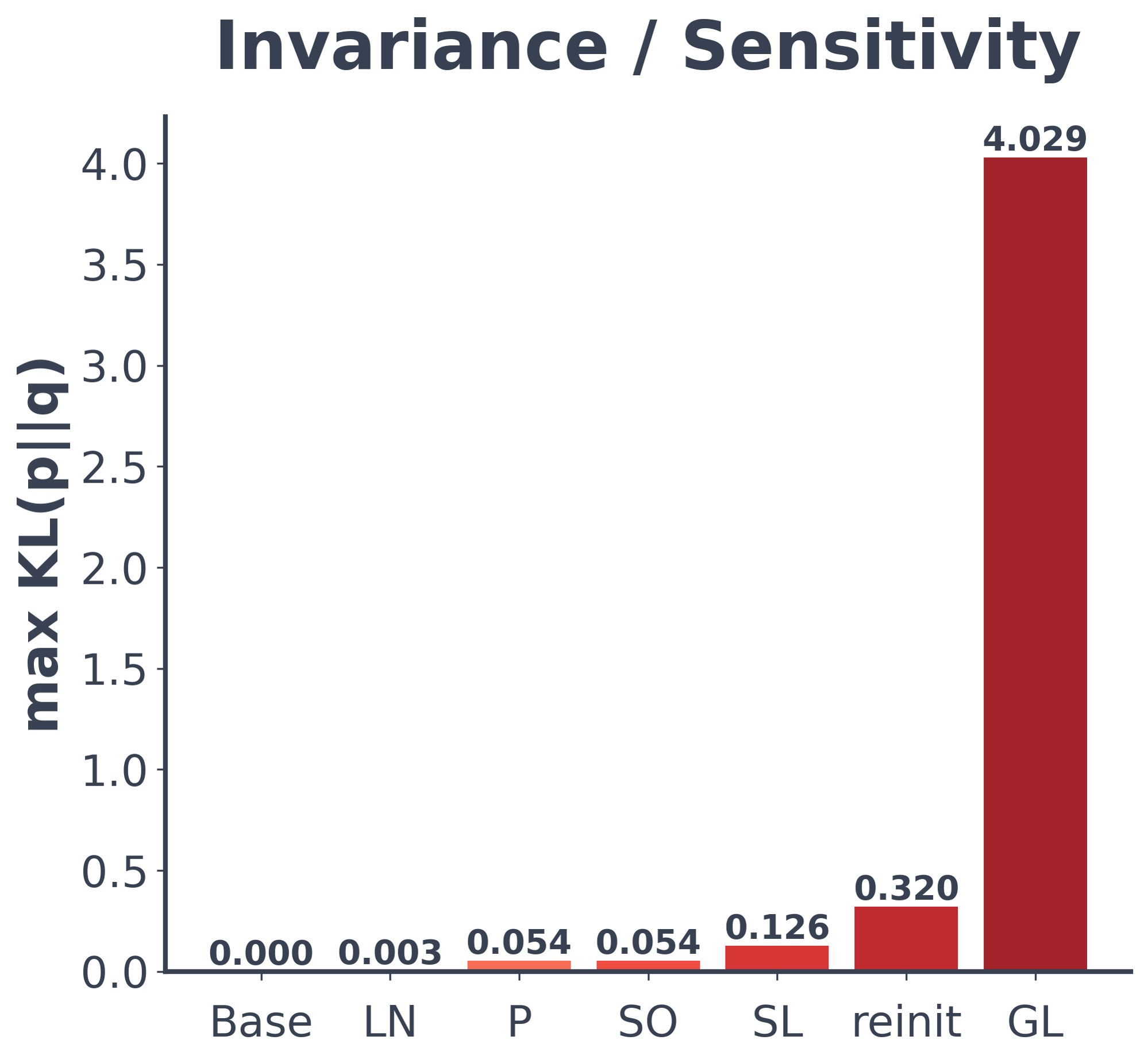}
    \caption{Benchmarking how a transformer model’s computational invariance breaks under different distortions. 
    }
    \label{fig:m7}
    \vspace{-2.2em}
\end{wrapfigure}

\paragraph{Robustness of the invariance}
For the rotational invariance of transformers under the coupled transformations in \Cref{eq: transformer_coupled_rotation} to hold exactly, two conditions are required:
\begin{itemize}
    \item Either (i) RMSNorm is used, or (ii) LayerNorm is used with unit weights, equivalently with LayerNorm weights fused into adjacent weights as in \citep{ashkboos2024slicegpt};
    \item the rotations applied across parameters in \Cref{eq: transformer_coupled_rotation} are \emph{identical} (globally shared).
\end{itemize}

In \Cref{fig:m7}, we quantify the extent to which invariance is broken when these conditions are violated. Concretely, we randomly initialize a small transformer (about 1M parameters) and measure how much its predictive distributions change after applying various parameter transformations. We define the distributional shift as the \emph{largest} change in the predicted next-token distribution, measured by KL divergence.  \texttt{Base} corresponds to an RMSNorm transformer with globally shared rotations applied exactly as in \Cref{eq: transformer_coupled_rotation}. As expected, this induces essentially no distribution shift, confirming the theorem of \cite{ashkboos2024slicegpt}. The \texttt{LN} baseline switches RMSNorm to LayerNorm. We find that the resulting shift remains very small, suggesting that in practice it can be reasonable to apply \name{}/ST to LayerNorm transformers without substantial degradation from symmetry breaking. To isolate the effect of violating global rotation sharing, we consider \texttt{P}, \texttt{SO}, \texttt{SL}, and \texttt{GL}: these apply transformations analogous to \Cref{eq: transformer_coupled_rotation}, except that each parameter receives its own independently sampled transform from the permutation group (\texttt{P}), the special orthogonal group (rotations, \texttt{SO}), the special linear group (\texttt{SL}), and the general linear group (\texttt{GL}), respectively. These baselines break the second condition above. We additionally include \texttt{reinit}, which re-initializes the model with a larger variance. Overall, we observe that the discrepancy induced by non-shared random rotations (\texttt{SO}) is significant, but relatively moderate compared with much larger shifts produced by other transformations and re-initialization.

\par \wrapfill \par
\paragraph{Aligning with \name{} practice}
In our training setting (\Cref{exp}), we apply \name{} to standard LayerNorm transformers and (by default) do not enforce global rotation sharing: each matrix parameter maintains its own $\mR_t$. From the symmetry viewpoint, these are two distinct departures from invariances above. Empirically, \Cref{fig:m7} suggests that the RMSNorm$\rightarrow$LayerNorm substitution alone induces only a very small distributional shift, whereas breaking rotation sharing yields a larger shift. As discussed in \Cref{remark: symmetry_tradeoff}, allowing parameter-specific rotations enlarges the transformation family (greater optimizer expressiveness) while moving farther from the exact symmetry class (more symmetry breaking), so the net effect is inherently a trade-off. In \Cref{sec: symmetry_predictions} we find that enforcing \emph{structured rotation sharing}--via the chain-coupled rule \Cref{eq: transformer_local_rotation}--improves performance relative to both globally shared rotations and naive per-parameter rotations.

\begin{takeaways}
    \textbf{Takeaways: symmetry as an interface for exploiting cross-layer/cross-module coupling.}  The rotational symmetries of transformers are induced by the cross-layer/module couplings in the model. As a result, in principle, \name{} rotations can be shared either globally (\Cref{eq: transformer_coupled_rotation}) or locally (\Cref{eq: transformer_local_rotation}). In \Cref{sec: symmetry_predictions} we provide empirical evidence on the effectiveness of rotation sharing strategies. Through rotation sharing, symmetry offers an natural and economic interface to exploit cross-layer couplings to improve computational efficiency.  \emph{To the best of our knowledge, \name{} is the first matrix optimizer for LLMs that explicitly exploits symmetry to efficiently account for cross-layer/module coupling.}
\end{takeaways}

\subsection{Application of symmetry hypothesis: new insights on optimizer designs} \label{sec: symmetry_predictions}

In previous discussions, we have demonstrated how the seemingly complicated update rules of \name{} can be derived from elegant assumptions on loss landscape symmetry, and how these assumptions are directly grounded in model architectures. Therefore, we can re-state symmetry hypothesis in a more accurate manner:

\begin{takeaways}
    \textbf{Symmetry Hypothesis:} \emph{a substantial part of matrix optimization can be derived, understood and improved from global symmetry properties of the loss landscape. Particularly:
    \begin{itemize}
        \item \name{} instantiations can be viewed as methods that exploit rotational symmetry of loss landscape, as additional degrees of freedom to help improve optimization (\Cref{sec: deriving_optimizer}). 
        \item Rotational symmetries exists in residual streams in neural networks, such as transformer architectures (\Cref{sec: transformer_symmetry}) under constraints (RMS norm $+$ globally/locally shared rotation). These constraints maybe relaxed depending the overall effect of symmetry breaking-optimizer expressiveness tradeoff (\Cref{remark: symmetry_tradeoff}).
    \end{itemize}
    }
\end{takeaways}

To further demonstrate and validate the practical usefulness of symmetry hypothesis, we present three new optimizer designs for \name{} implied by this theory. They include:
    \begin{itemize}
        \item exploiting cross-layer coupling via rotation sharing;
        \item The use of an unified rule for full model optimization;
        \item Optimizing rotation orientation.
    \end{itemize}

Next, we will state, analyze, and validate them respectively.

\begin{findings}
    \textbf{Design 1: due to cross-layer coupling of symmetries in transformers, in principle we can further improve the efficiency of \name{} with  global or local rotation sharing across parameters. }
\end{findings}

\paragraph{Analysis} Our analysis in \Cref{sec: transformer_symmetry} suggests that cross-layer coupling can be exploited to share rotations.
For example, in the global rotation sharing scheme (\Cref{eq: transformer_coupled_rotation}), all residual-stream matrices transform with the same left rotation, so we can compute a \emph{single} \name{} rotation by stacking them into
\begin{equation}
\mW_{\rm stack}
:=
\Big[
(\mW_{\rm tok})^\top,\ (\mW_{\rm pos})^\top,\ (\mW_{\rm head})^\top,\
\big\{(\mW_Q^{(\ell)})^\top,(\mW_K^{(\ell)})^\top,(\mW_V^{(\ell)})^\top,\mW_O^{(\ell)},(\mW_{\rm up}^{(\ell)})^\top,\mW_{\rm down}^{(\ell)}\big\}_{\ell=1}^L
\Big]\in\R^{d\times M},
\end{equation}
and applying \name{} rotation rules  using the corresponding stacked momentum $\mM_{t,{\rm stack}}$ to estimate global rotation $\mR_t$.
This requires one QR factorization of a $d\times d$ matrix (from \Cref{eq: gf}), hence $\mathcal{O}(d^3)$ per step.
By contrast, applying independent \name{} rotations to each block requires $(6L+3)$ such QR factorizations, costing $\mathcal{O}((6L+3)d^3)$.

The same procedure extends to the chain-coupled local sharing rule \Cref{eq: transformer_local_rotation}: one maintains a sequence of rotations $\{\mR_{\ell,t}\}_{\ell=0}^{L}$, each estimated from a \emph{local} stacked momentum that aggregates exactly the matrices that are coupled by $\mR_{\ell,t}$ in \Cref{eq: transformer_local_rotation}. Thus chain-coupled sharing requires only $(L+1)$ QR factorizations per step, costing $\mathcal{O}((L+1)d^3)$, interpolating between the fully global scheme ($1$ QR) and the naive per-parameter scheme ($(6L+3)$ QRs).

\paragraph{Validation}
To validate this conjecture, we conduct preliminary experiments on the $\sim$130M NanoGPT model \citep{Karpathy2022} and the NanoChat model \citep{nanochat}, following a setup similar to \Cref{subsec: universal effective}. As a baseline, we train both models using \arosink{} (with SCQR). We then evaluate \arosink{} under rotation sharing, including both global rotation sharing (\Cref{eq: transformer_coupled_rotation}) and chain-coupled local rotation sharing (\Cref{eq: transformer_local_rotation}). We also compare against Muon. For a controlled comparison, we follow the experimental guidelines in \Cref{exp: guidelines}, except for scale: we train with a batch size of 1M, context length of 1K, and $1\times$ Chinchilla-law tokens. For both models, we use learning rates tuned for AdamW and directly transfer them to all \name{} variants and Muon. We continue to use AdamW for embedding and LM head parameters, since under rotation sharing in small models, the large number of embedding columns can bias rotation estimation. Results are reported in \Cref{tab: rotation_share}.

\begin{table}[t]
\centering
\caption{Validation loss comparison across optimizers and rotation sharing schemes.}
\label{tab: rotation_share}
\begin{tabular}{lcccc}
\toprule
Architecture & Muon & ARO & ARO (global sharing) & ARO (chain-coupled) \\
\midrule
NanoGPT  & 3.161 & 3.150 & 3.158 & \textbf{3.148} \\
NanoChat & 3.136 & 3.132 & 3.126 & \textbf{3.116} \\
\bottomrule
\end{tabular}
\end{table}

We observe that:
\begin{itemize}
    \item \name{} with chain-coupled local rotation sharing consistently performs the best across two architectures. On NanoChat particularly, it \textbf{outperforms Muon with a loss gap of \textbf{$\sim$ 0.02} }, which is significant in our controlled setup.
     \item \name{} with global rotation sharing closely matches the performance of \name{} with per-parameter dedicated rotations. However, its relative behavior is architecture-dependent: it outperforms the vanilla \name{} on NanoChat, but underperforms on NanoGPT.
\end{itemize}

This suggest that, at least for the scale and setup we tested, symmetry-motivated rotation sharing has the potential to further boost the performance of \name{}, while reducing the computational overhead. This result shows a key distinction between the symmetry-driven optimizer design and many existing framework for matrix optimization \citep{bernstein2024old, gong2025towards} that usually models the layer-wise geometry independently.

\par \wrapfill \par

\begin{findings}
    \textbf{Design 2: \name{} can be applied to all matrix parameters, including embedding and lm-head parameters.}
\end{findings}

\paragraph{Analysis} From the analysis of rotational symmetry of GPT-like transformer model in \Cref{sec: transformer_symmetry}, we note that the rotational symmetries also applies to embedding/lm head parameters ($\mW_{\rm tok}$). From the perspective of symmetry, in terms of how rotation is applied, there is no fundamental distinction between embedding/lm head parameters and hidden layer parameters. Hence, in principle, (whether shared or not) the same \name{} rule can be applied to all matrix parameters without modification. 

\paragraph{Validation} This is validated at scale in our experiments in \Cref{exp}. On GPT-XL (\Cref{exp: gpt}), we showed that the \name{} family converges and outperforms Adam consistently under full model mode. Notably, on Sigma-2B model \arosink{} with full model training outperforms \arosink{} (for hidden layers) $+$ Adam (for embedding/lm head) hybrid (\Cref{fig:sigma speedup}). Together, they demonstrate the plausibility and benefit of performing rotated optimization across all layers.

\paragraph{Discussion} This observation reveals a key different between symmetry perspective and modular approaches for deep learning \citep{bernstein2024modular, pethick2025training}. Instead of assigning different atomic norms to different layers based on their semantics as in modular methods, our symmetry perspective justifies the use a unified update rule over all matrix parameters. While it could be beneficial to apply the same rotation policy over different base norms picked for different layers, our preliminary study in \Cref{subsec: universal effective} suggests the impact of the choice of base norms are much less pronounced than the effect of rotations.

\begin{findings}
    \textbf{Design 3: depending on the types of the module, their weights (hence corresponding gradients/momentum) need to be correctly oriented, before the \name{} update is applied. Particularly, in a common transformer setup (\Cref{sec: transformer_symmetry}), all weights should be transposed, except for $\mW_{\rm O}$ and $\mW_{\rm down}$}. 
\end{findings}

\paragraph{Analysis} \Cref{eq: transformer_stack} suggests that in the default one-sided left rotation setup of \name{}, we need to pre-transpose the weights of different parameters accordingly such that rotational symmetry in weight space can be established. In the common transformer setup presented in \Cref{sec: transformer_symmetry}, \Cref{eq: transformer_stack} implies that all weights must be transposed, except for $\mW_{\rm O}$ and $\mW_{\rm down}$. This is equivalent to \emph{keeping the orientation of weights unchanged, but transpose the gradients (except for $\mG_{\rm O}$ and $\mG_{\rm down}$ ) in the optimizer code, when they are being processed by \name{} rotations}; then transpose the \name{} update back during parameter update.

\paragraph{Validation} To verify the effectiveness of this new insight, we compare a range of different transpose/orientation rules on the 130M GPT model, following the model setup in \Cref{subsec: universal effective}. We compare the following strategies using \arosink{} (full model mode), listed below:

\Needspace{18\baselineskip}
\begin{wrapfigure}{r}{0.4\linewidth}
    \vspace{-1.2em}
    \centering
    \includegraphics[width=\linewidth]{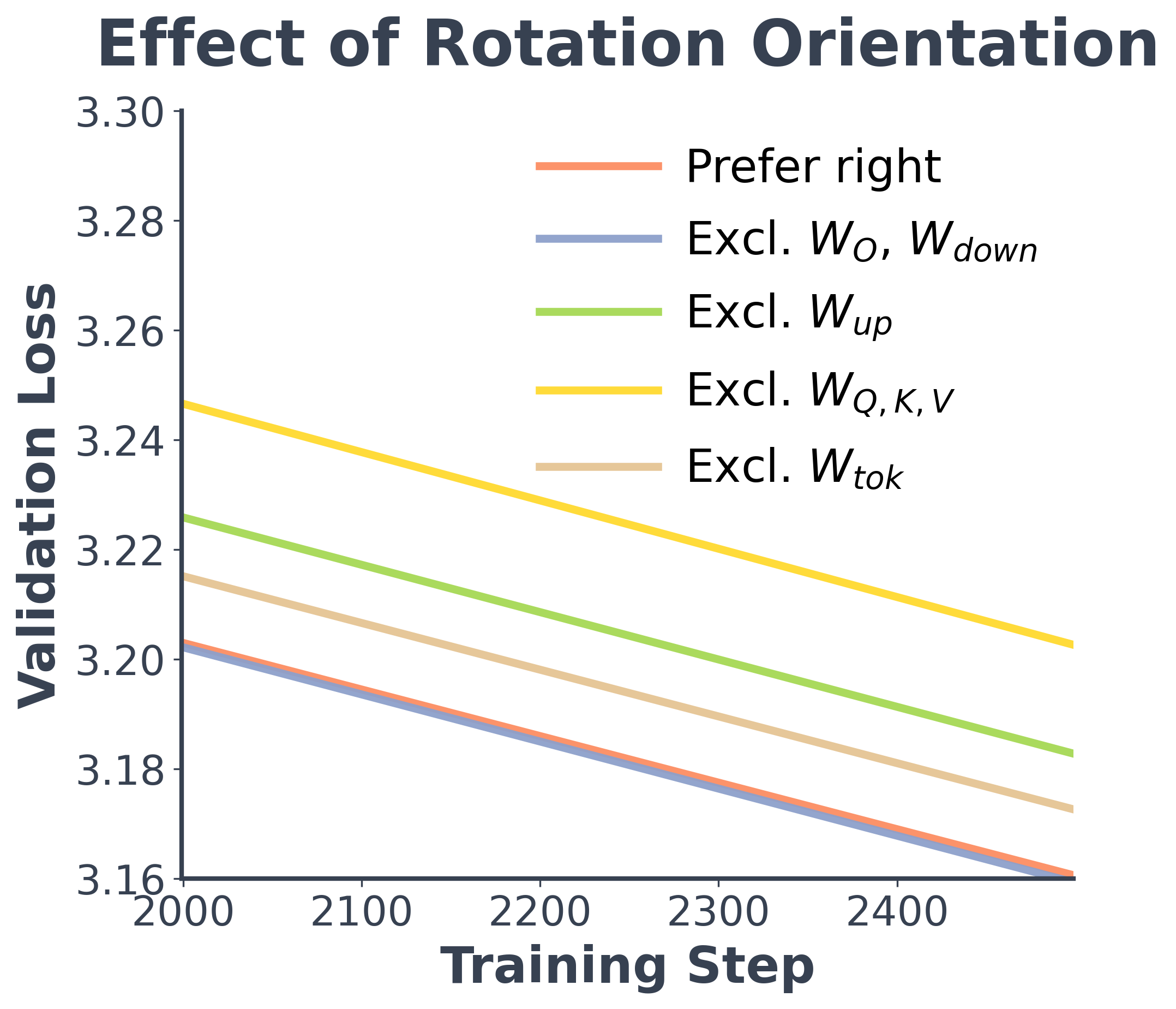}
    \caption{The impact of rotation orientation on training performance. \textbf{Excl. $\mW_{\rm O}$, $\mW_{\rm down}$, the rule that complies with rotational symmetry, gives the best performance.}}
    \label{fig: m8.2}
    \vspace{-4.2em}
\end{wrapfigure}

\begin{itemize}
    \item \textbf{Excl. $\mW_{\rm O}$, $\mW_{\rm down}$}: transpose all parameters, except for $\mW_{\rm O}$, $\mW_{\rm down}$. This is the correct rule that complies with our derived rotational symmetry relation.
    \item \textbf{Excl. $\mW_{\rm up}$}: transpose all parameters, except for up projection matrix of MLP blocks.
    \item \textbf{Excl. $\mW_{\rm Q}$, $\mW_{\rm K}$, $\mW_{\rm V}$}: transpose all parameters, except for attention weights.
    \item \textbf{Excl. $\mW_{\rm tok}$}: transpose all parameters, except for embedding/lm head parameters.
    \item \textbf{Prefer right}: transpose if $m \geq n$ and apply right rotation, otherwise left rotation. This is equivalent to transposing all parameters except $\mW_{\rm down}$.
\end{itemize}

Results are shown in \Cref{fig: m8.2}. We observe that: 

\begin{itemize}
    \item The pre-transpose strategies does have an impact on training performance across all different orientation strategies. The gap in final validation loss between the best and worst strategies  $\sim 0.04$. This shows orientation is an important factor of rotated optimization, compared with the impact of rotations (around $0.07$-$0.10$ gap depending on strategy) and base optimizers (around $0.03$ gap) in \Cref{fig: m1.2}.
    \item The orientation rule that complies with the rotational symmetry (\textbf{Excl. $\mW_{\rm O}$, $\mW_{\rm down}$}) performs the best across all baselines. 
\end{itemize}
Together, they  provide an preliminary evidence on the impact of symmetry on optimization effectiveness, as well as a practical guideline for determining orientation of gradient rotations.

\par \wrapfill \par

\begin{takeaways}
    \textbf{Takeaways: the symmetry hypothesis not only  provides an elegant principle from which  \name{} can be derived, but also provides new insights on optimizer designs that are not covered by prior work.}
\end{takeaways}

\subsection{Summary: what is matrix optimization?}

We return to the guiding question of understanding \emph{what is matrix optimization}. So far, in our paper we have shown that
\begin{itemize}
    \item \name{} with a broad class of base optimizers can significantly improve performance over Adam and Muon;
    \item The seemingly complicated rules of \name{} can be elegantly derived under symmetry assumptions over the loss landscape, as a generalized symmetry teleportation method;
    \item The symmetry view point produced novel predictions for optimizer design that are validated by preliminary empirical results;
\end{itemize}
Together, these observations justify symmetry as a promising principle introduced at the beginning of this section. In this view,  a substantial part of matrix optimizers can be derived from, understood through, and improved using global symmetry properties over the loss landscape induced by model architectures. This view hypothesizes that, matrix-wise update rules of matrix optimizers can naturally arise as a result of  the coupling of the following ingredients:
\begin{itemize}
\item a \emph{symmetry group} induced by the architecture (exact or approximate);
\item a \emph{base optimizer geometry} that intentionally breaks that symmetry so that orbit search becomes meaningful;
\item an \emph{orbit search strategy} that exploits the symmetry as an additional degree of freedom to improve optimization.
\end{itemize}

This perspective goes beyond a reinterpretation of \name{} and leads to concrete new design choices that further improve optimization (\Cref{sec: symmetry_predictions}). While these designs will require modification to scale up, we remain cautiously optimistic that symmetry could enable a more substantial advance beyond Muon and whitening-based methods, addressing an open challenge posed by \citep{frans2025really}.

As an ending remark of this section, we would like to note that while promising, symmetry is also unlikely to be the entire story behind large-scale training dynamics, and our current analysis primarily justifies local proxy objectives under symmetry-aligned transformations. Developing a fully predictive theory that connects architecture, optimization, and training outcomes (e.g., generalization) remains an open challenge.

\section{What Is a Good Rotation: A Behavioral Study on the Stability-Speed Tradeoff}
\label{subsec: discussion denoising}

Parameter symmetries provide a principled explanation for why rotations can improve optimization when training LLMs. Importantly, symmetry does not prescribe a single rotation policy. Instead, we have the flexibility to choose a policy that reinforces the behavior we care about in the target regime. This naturally leads to a key question: \textbf{What properties should a rotation have for training large language models?}

Rather than providing a complete recipe, this section focuses on understanding some potential desirable properties that the rotation should have in realistic training regimes - specifically, when gradients are stochastic due to mini-batch sampling. Earlier sections focused on an idealized, noise-free perspective. While analytically clean and simplified, it does not capture how stochastic noise interacts with rotation choices in realistic training regimes. As we will show, stochasticity can qualitatively change what rotations are effective. In particular, aggressive maximization of the instantaneous loss decrease rate $\mathcal{J}(\mR;\mG_t,f)
:= \langle \mG_t,\ \mR f(\mR^\top \mG_t)\rangle$ may not be desirable since it can result in training instability under noisy setting.

This section is organized around the following questions:
\begin{questions}
\begin{rqlist}
\item Why does the rotation that aggressively maximizes instantaneous loss decrease rate may cause instability under mini-batch noise?
\item Which classes of rotations mitigate these issues?
\item Can we find better rotations?
\end{rqlist}
\end{questions}

We summarize the main takeaways below:
\begin{takeaways}
\textbf{Takeaways: 
\begin{enumerate}
    \item Maximizing empirical instantaneous loss decrease rate under mini-batch noise may lead to unstable training behavior, inconsistent loss decrease and slow convergence.
    \item Eigen rotations greatly improves the stability of training. However, it comes with the cost of \textbf{smaller alignment magnitude} .
    \item \name{} rotations empirically achieves a better compromise between the training stability and alignment magnitude. 
\end{enumerate}
}
\end{takeaways}

We will introduce an important metric: \textbf{alignment score}. We have mentioned  \emph{training stability} several times. Alignment score is the metric that is closely related. 

\paragraph{Loss decrease, loss stability and alignment score.} Consider a dataset $\mathcal{D}$ and parameters $\mW$, one can define a loss function $\loss(\mW,\mathcal{D})$ over the entire dataset as the ground-truth loss. Any mini-batched version $\loss(\mW, \mathcal{B})$ is a stochastic approximation to it. This loss function can be highly non-linear. Therefore, the loss change under the parameter perturbation $\phi$ (i.e.~$\loss(\mW+\phi, \mathcal{D})-\loss(\mW, \mathcal{D})$) is hard to analyze. However, \textbf{within a small trust region of $\mW$} (i.e., the perturbation $\phi$ has small magnitude), the loss function can be safely approximated by its first-order expansion:
\begin{align*}
    \loss(\mW+\phi, \mathcal{D}) = \loss(\mW,\mathcal{D}) + \left\langle \nabla_\mW \loss(\mW,\mathcal{D}), \phi\right\rangle + o(||\phi||),
\end{align*}
With in the small enough trust region, the linear term dominates the high order terms and governs the loss change. Thus, we define this linear term as the alignment score:
\begin{align}
    \alignscore(\phi) = \left\langle\nabla_\mW \loss(\mW, \mathcal{D}), \phi\right\rangle.
    \label{eq: alignment score general}
\end{align}
Under the stochastic setup, the perturbation $\phi$ often comes from the output of a chosen optimizer, that is the function of the stochastic mini-batch gradient. This will make the loss $\loss(\mW+\phi,\mathcal{D})$ a stochastic objective.
Therefore, \textbf{under the small trust region assumption}, the expected loss change can be approximated by the expected alignment score:
\begin{align*}
    \E[\loss(\mW+\phi,\mathcal{D}) - \loss(\mW,\mathcal{D})] \approx \E[\alignscore(\phi)],
\end{align*}
which characterizes the expected loss decrease.

We consider training instability of the update $\phi$ under mini-batch noise by the inconsistencies of the post-update full-dataset loss $\loss(\mW+\phi, \mathcal{D})$ across minibatch draws. For sufficiently small $\phi$, this is governed by the variance of the first-order loss change $\alignscore(\phi)$. More formally:
\begin{align}
    \var(\loss(\mW+\phi,\mathcal{D})) =& \var(\loss(\mW,\mathcal{D})+\left\langle\nabla_\mW\loss, \phi\right\rangle+o(||\phi||)) \nonumber \\
    =&\E\left[\left(\loss(\mW,\mathcal{D})+\left\langle\nabla_\mW\loss, \phi\right\rangle+o(||\phi||)\right)^2\right] - \E\left[\loss(\mW,\mathcal{D})+\left\langle\nabla_\mW\loss, \phi\right\rangle+o(||\phi||)\right]^2\nonumber\\
    \approx& \E\left[\left\langle\nabla \loss, \phi\right\rangle^2\right] - \E\left[\left\langle\nabla \loss, \phi\right\rangle\right]^2\nonumber\\
    =& \var(\alignscore(\phi)).
    \label{eq: loss variance}
\end{align}

Therefore, the first two moments of the alignment score determine the leading-order behavior of the mean loss decrease and the batch-to-batch variability of the post-update loss. Since $\alignscore(\phi)$ is sensitive to the scale of $\phi$ and $\nabla_W \loss(\mW,\mathcal{D})$, we often consider its normalized version:
\begin{equation}
    \alignscore(\phi) = \frac{\left\langle\nabla_\mW \loss(\mW, \mathcal{D}), \phi\right\rangle}{\|\nabla_W\loss(\mW,\mathcal{D})\|_F\|\phi\|_F}.
    \label{eq: normalized alignment score}
\end{equation}

In the following, we abuse notation and use $\alignscore$ to denote its normalized version and assume the \name{} update, i.e.~$\phi = \mR f(\mR^T\mNG)$, where $\mNG$ is the noisy minibatch gradient.

\paragraph{Loss decrease rate with mini-batch noise}
In the noiseless setup, the loss decrease rate is defined as $\mathcal{J}(\mR;\mG_t,f):= \langle\mG_t, \mR f(\mR^\top\mG_t)\rangle$, which is equivalent to the unnormalized alignment score $\alignscore$ defined above. However, with mini-batch sampling, the clean gradient is no longer accessible. Therefore, we define the empirical instantaneous rate as $\tilde{\mathcal{J}}:=\mathcal{J}(\mR;\mNG_t, f) = \langle\mNG_t, \mR f(\mR^\top \mNG_t) \rangle$. 
Under this case, $\tilde{\mathcal{J}}$ is no longer equivalent to the unnormalized $\alignscore$ with $\mNG$, since $\mNG$ only appears in $\phi$ whereas $\tilde{\mathcal{J}}$ uses $\mNG$ for both terms. The instantaneous loss decrease rate under stochastic settings should be the unnormalized $\alignscore$ with $\mNG$. It is crucial to distinguish between those terms before diving into the following sections.

\subsection{Setups}
\label{subsubsec: discussion denoising: setups}

We introduce common setups, assumptions and notation used. Additional technical assumptions and formal definitions are deferred to \cref{subapp: denoising setups and assumptions}.

\paragraph{MNIST as a demonstration.}
To make the core phenomena easy to visualize, we consider a simple MNIST classification task using a two-layer MLP. Despite its simplicity and it does not induce rotation symmetry, this setting already exposes the key properties we care about. Detailed experimental settings are provided in \cref{subapp: MNIST setup}.

\paragraph{Sign as the base optimizer $f$.}
To simplify theoretical analysis and enable closed-form reasoning, we assume the base optimizer takes the form $f(\cdot)=\sign(\cdot)$ for certain claims. We leave the theoretical analysis of other $f_t$ to the future work.

\paragraph{Notations.}
We denote by $\mEG\in\R^{m\times n}$ the full-batch, noise-free gradient (i.e.~$\nabla \loss(\mW, \mathcal{D})$), and by $\mNG=\mEG+\noise$ (i.e.~$\nabla \loss(\mW, \mathcal{B})$) the mini-batch gradient with noise $\noise$. We further denote by $\mM\in\R^{m\times n}$ the momentum constructed from past mini-batch gradients.

\subsection{Aggressively improving decrease rate leads to unstable training}
\label{subsubsec: discussion denoise: failure of alternative projection}
In this section, we focus on one crucial argument: \textbf{Aggressively improving the empirical instantaneous loss decrease rate $\tilde{\mathcal{J}}$ may be undesirable under mini-batch noise, with the consequence of unstable training.} This analysis leads to the conclusion that low $\var(\alignscore)$ may be a desirable property in the presence of noise. 

Motivated by the generalized teleportation perspective in \cref{sec:symmetry}, a natural objective for selecting a rotation is to maximize the instantaneous loss decrease rate. However, in practice, the true instantaneous rate $\mathcal{J}$ is not accessible for large scale training. One only has access to its empirical version $\tilde{\mathcal{J}}$. We argue that directly maximizing this empirical rate may not always be desirable.

A fundamental difficulty is that maximizing it (i.e.,~$\tilde{\mathcal{J}}$) is, in general, analytically intractable under arbitrary choices of $f$. A representative example arises when $f=\sign(\cdot)$: in this case, maximizing $\tilde{\mathcal{J}}$ reduces to the well-known L1-PCA problem \citep{kwak2008principal}, for which no closed-form solution is known.

Rather than insisting on exact maximization, we consider a more practical question: \emph{Is it always beneficial to apply a rotation that increases $\tilde{\mathcal{J}}$?} An appealing aspect of this reformulation is that there exists a projection procedure with a guaranteed monotonic non-decreasing property of the objective. We refer to this procedure as the \emph{Polar scheme}.

Concretely, we consider the following two-step projection scheme:
\begin{enumerate}
    \item Fix $\mA_{t-1} = f(\mR_{t-1}^{*\top}\mNG)^\top$;
    \item Solve $\mR_t^* = \argmax_{\mR_t \in \mathcal{O}} \tr(\mNG_t \mA_{t-1} \mR_t^\top).$
\end{enumerate}

Step (2) corresponds to the classical Procrustes problem \citep{schonemann1966generalized,myronenko2009closed}, which admits a closed-form solution via polar decomposition:
\begin{align}
\mR^*_{t}
\;=\;
\text{Orthogonalize}(\mNG_t \mA_{t-1}).
\label{eq: polar solution}
\end{align}
When the orthogonalization operator is chosen as $\polar(\cdot)$. This solution achieves the maximum value and $\mR^*_t = \mU\mV^\top$ where $\mU$, $\mV$ are left and right singular matrix of $\mNG_t \mA_{t-1}$. Importantly, the resulting projection procedure enjoys a monotonic non-decreasing property of the objective.

\begin{remarks}
\begin{remark}[Monotonically non-decreasing property]
A key property of the above projection scheme is that it guarantees a monotonically non-decreasing objective
$
\tr\!\bigl(\mNG_t\, f(\mR_t^\top \mNG)^\top \mR_t^\top\bigr)$
under repeated projections, provided that $f$ is chosen as the dual norm derivative (cf.~\cref{eq: projection}). A detailed discussion of this property is provided in \cref{subapp: failure of polar}.
\label{remark: monotonic non-decreasing polar}
\end{remark}
\end{remarks}

Since the resulting $\mR_t^*$ is itself a rotation matrix and is guaranteed to improve the instantaneous loss decrease rate, we can directly evaluate the training behavior induced by this rotation.

\begin{findings}
    \textbf{Findings 1: With our MNIS experiment setup, the polar scheme works well with the absence of noise, but its performance are greatly impact under mini-batch training.}
\end{findings}

As shown in \Cref{fig: clean sign mnist}, the Polar scheme yields improved training behavior compared to other rotation choices in the absence of mini-batch noise.

This behavior is expected, since under a suitable and small enough learning rate, the magnitude of alignment is a good proxy for the loss change and the polar scheme explicitly guarantees a larger alignment. All baselines achieves similar training loss in the end at the level of $1E-8$. 

However, in practice, such as in LLM training, one does not have access to the clean gradient. Instead, only the noisy mini-batch gradient $\mNG$ is available, rather than $\mEG$. 
Importantly, different rotation schemes have different impacts to the training stability under mini-batch noise. 

\begin{wrapfigure}[18]{r}{0.35\linewidth}
\centering
\includegraphics[width=\linewidth]{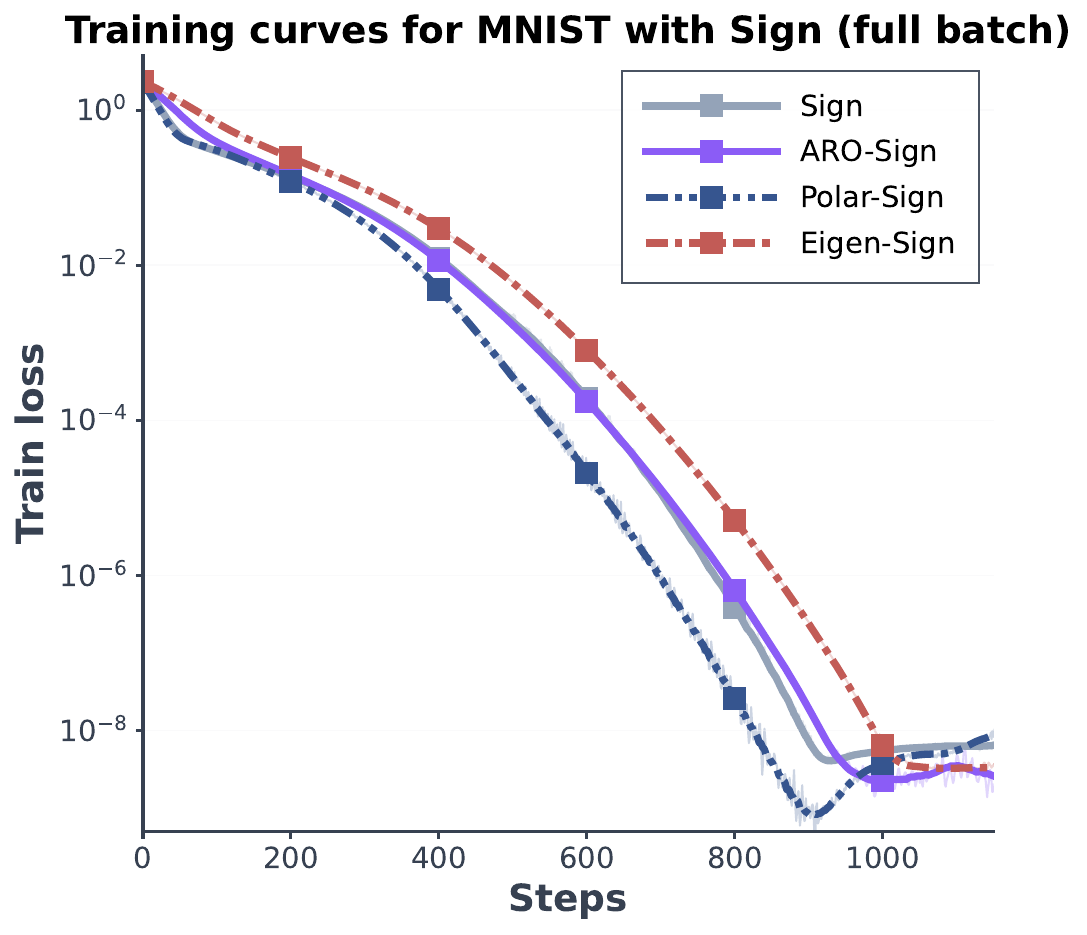}
\caption{Full-batch MNIST training curves with $f=\sign$ under different rotation choices.}
\label{fig: clean sign mnist}
\end{wrapfigure}

In contrast to the noise-free setting, \Cref{fig: MNIST} demonstrates that the training performances without rotation or with polar rotation are greatly impacted by the mini-batch noise, and only reaches the training loss of $1E-4$ to $1E-6$ (compared to $1E-8$ for clean setup).

By jointly inspecting the loss trajectories and the alignment score, we find that both polar projection and the unrotated baseline often induce substantial variability in these metrics.
Oscillations of the alignment score around zero indicate that the instantaneous training progress is not consistently descending: the optimizer frequently alternates between loss descent and loss ascent, which ultimately slows down overall convergence.

This behavior is rooted in the shift of the objective used to determine the rotation. Under clean setup, the objective $\mathcal{J}$ is equivalent to unnormalized $\alignscore$, which is a determinstic objective. Therefore maximizing it leads to a consistent high alignment score.
Under stochastic gradients, the empirical objective $\tilde{\mathcal{J}}$ is no longer equivalent to $\alignscore$. The polar scheme searches for a rotation that aligns the update with the noisy gradient $\mNG_t$, rather than with the noise-free gradient $\mEG_t$. As a consequence of mini-batch noise, $\mNG_t$ can deviate significantly from $\mEG_t$, and may even become anti-aligned with it. In such cases, the resulting $\alignscore$ is no longer guaranteed to be positive , and can instead induce instantaneous loss increase.

Unfortunately, for large models and datasets, such behavior is common. In particular, prior work has shown that momentum can become anti-aligned with the gradient in certain regimes \citep{becker2024momentum}. To make this phenomenon precise, we formalize how maximizing the empirical objective under stochastic gradients can lead to inconsistent training progress and, crucially, enlarge a \emph{danger region} in which \emph{loss ascent is guaranteed}.

\paragraph{Maximizing the empirical objective enlarges the \emph{danger region}.}

We define the \emph{danger region} as the region of noisy gradients or momentum for which the resulting update
$
\Delta \mW \;\propto\; \mR\, f(\mR^\top \mNG_t)
$
is \textbf{guaranteed} to be anti-aligned with the ground-truth gradient $\mEG_t$ (i.e.~$\alignscore$ is negative). When the observed gradient (e.g., $\mM_t$ or $\mNG_t$) lies in this region, the corresponding update induces an instantaneous increase in the true loss. \textbf{\Cref{thm: failure of polar} in \cref{subapp: failure of polar}} formally characterizes this behavior, with the proof deferred to \cref{subapp: proof of failure of polar}.

The theorem shows that maximizing the empirical objective can expand this danger region: $\mM_t$ or $\mNG_t$ becomes more likely to fall inside it. This provides a concrete explanation for the observed oscillations of the alignment score under polar scheme.

\begin{figure}
\centering
    \begin{minipage}{0.3\textwidth}
        \includegraphics[width=1\linewidth]{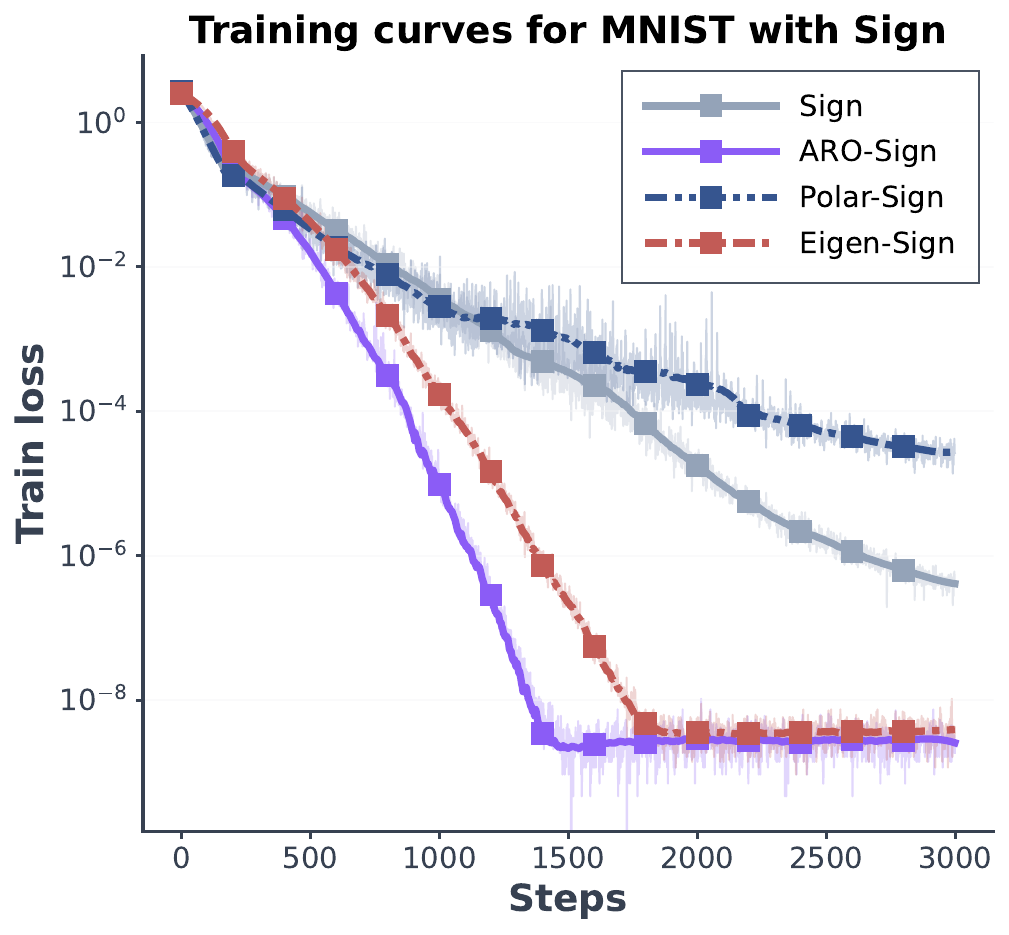}
    \end{minipage}\hfill
    \begin{minipage}{0.3\textwidth}
        \includegraphics[width=1\linewidth]{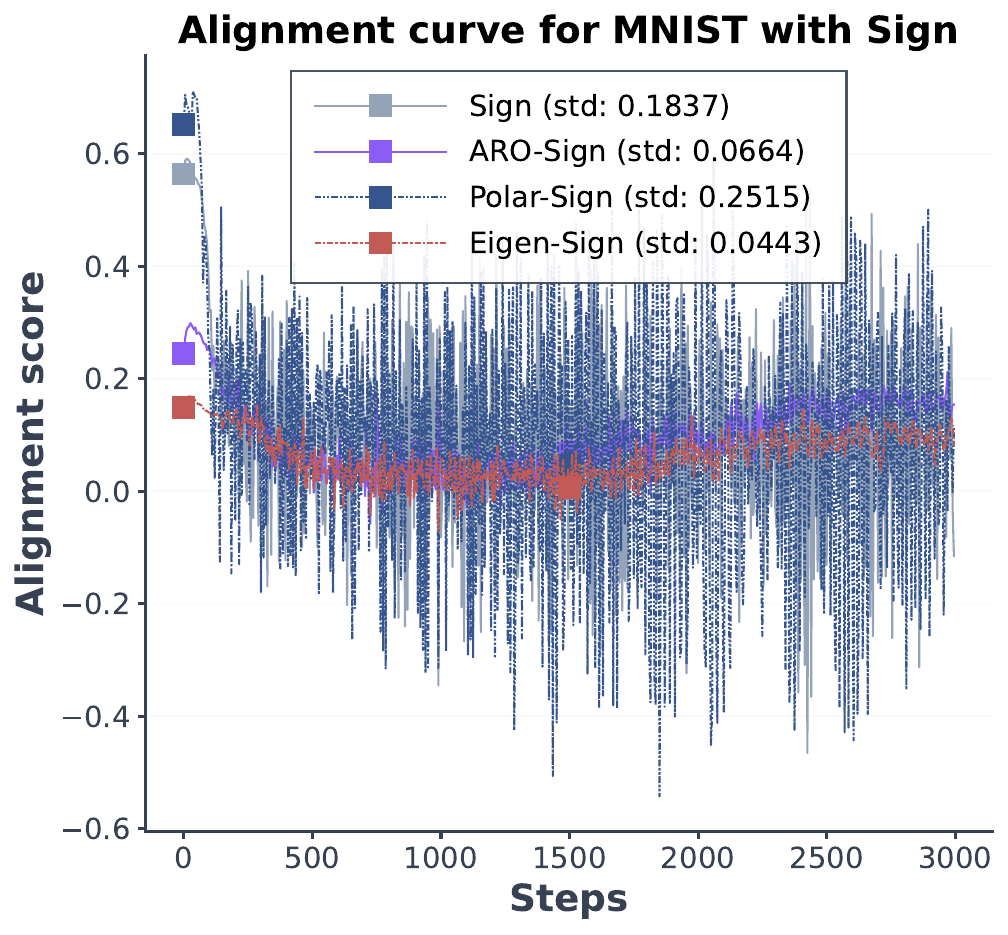}
    \end{minipage}\hfill
    \begin{minipage}{0.3\textwidth}
        \includegraphics[width=\linewidth]{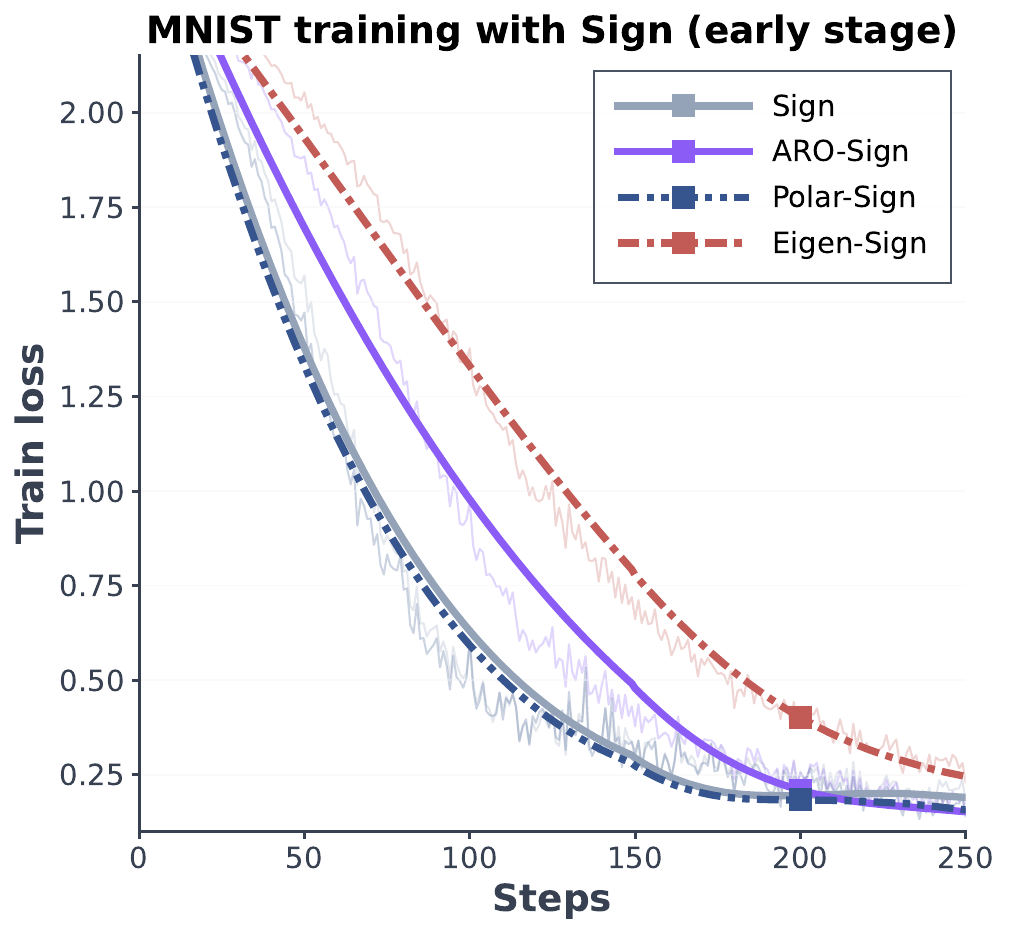}
    \end{minipage}
    \caption{MNIST training with $f_t=\sign$ and different rotation schemes. The \textbf{left} panel shows training loss curves (raw and smoothed). Due to the log axis, late-stage fluctuations of \name{}-Sign and Eigen-Sign may appear larger than their actual magnitude. The \textbf{middle} panel plots the alignment score during training, together with its \textbf{standard deviation} for each optimizer. The \textbf{right} panel plot the early-stage training curve with $f_t=\sign$. }
    \label{fig: MNIST}
\end{figure}

\subsection{Eigen-rotation improves stability but is overly conservative }
\label{subsubsec: discussion denoise: rotation as variancec reduction}
From the previous analysis, we observe that aggressively improving the empirical objective $\tilde{\mathcal{J}}$ may not be necessarily desirable in stochastic settings. 
Following this reasoning, reducing $\var(\alignscore)$ emerges as a more appropriate objective to reduce the sensitivity to minibatch noise, than solely increasing the empirical loss decrease rate. In other words, it is more desirable for an optimizer to achieve \textbf{stable and consistent loss improvement} rather than pursuing directions that are potentially faster but highly riskier.

As shown in \cref{fig: MNIST}, \name{} and eigen-rotation are much more robust to mini-batch noise compared to others, and greatly stabilize the training. In this simple setup, they even achieve the same loss level as the clean setup (i.e., 1E-8). 

\begin{findings}
\textbf{Finding 2: With our setup, eigen and \name{} rotations substantially reduce $\var(\alignscore)$, as well as realized loss curves. Their strong performance stems from making steady and reliable training progress.}
\end{findings}

Compared with polar projection and unrotated baselines, eigen and \name{} rotations produce significantly smaller variability in both loss curves and alignment scores. Notably, they tend to keep the alignment score \textbf{positive for most of the training process}, even though its magnitude is relatively small. This indicates that the optimizer consistently takes descent directions. We argue that this mechanism is one of the important drivers of their strong performance under noisy setup: rather than chasing aggressive but noisy updates, these methods make reliable progress per-step.

\paragraph{Eigen-rotation as optimal  $\var(\alignscore())$-reduction} \Cref{thm: optimal align score variance reduction} in \cref{subapp: rotation optimal variance reduction} shows that, for $f(\cdot) = \text{Sign}(\cdot)$ (i.e., under $l_\infty$ base geometry), eigen-rotation is a \textbf{local minimizer of $\operatorname{Var}(\alignscore)$} under a block-diagonal structure assumption on the gradient-noise covariance. The full derivation is provided in \cref{subapp: proof of minimium variance reduction}. This provides a principled explanation for the empirical effectiveness of eigen-rotations, and uncovers the choice of rotation selection criteria in symmetry teleportation that give rise to eigen rotations. It also gives an explanation on why eigen rotations are typically used in tandem with Sign/Adam-like base optimizer in SOAP/Shampoo/SPlus/Muon, since its optimal $\operatorname{Var}(\alignscore)$-reduction effect is true only under    $f(\cdot) = \text{Sign}(\cdot)$.

\begin{remarks}
\begin{remark}[Intuition behind the $\var(\alignscore)$ minimizer] Below we give a high-level intuition for why eigen-rotation is a local minimizer of $\operatorname{Var}(\alignscore)$. If we assume that the eigenvectors of $\mG\mG^{T}$ and the noise covariance coincide, then eigen-rotation aligns the noisy gradient with the principal noise directions, thereby maximally separating the noise component from the signal. This separation directly yields the minimum combined alignment variance. This assumption is partially supported by prior theoretical setups discussed in \cref{subapp: denoising setups and assumptions}.
\end{remark}
\end{remarks}

\begin{remarks}
\begin{remark}[Distinction from variance reduction techniques]
While our discussion on eigen rotations has been focused on its $\var(\alignscore)$-reduction effect, this is fundamentally different from variance-reduction techniques in stochastic optimization \citep{yuan2024mars, liu2025mars, cutkosky2019momentum, roux2012stochastic, johnson2013accelerating, nguyen2017sarah, zhou2020stochastic, fang2018spider}. Those variance reduction methods are typically based on methods like control variates and aim to reduce the variance of an estimator while preserving its expectation. In contrast, in our analysis, for different choices of rotations in the update
$$
\delta W \propto \mR f(\mR^T \mNG)
$$
generally induce different expected updates. This distinction highlights that rotation-based methods constitutes a qualitatively different mechanism and is orthogonal to control-variate approaches. Indeed, recent work \citep{liu2025mars} has begun to explore combining control variates \citep{yuan2024mars} with Muon (a special case of \name{}), suggesting that these techniques can be complementary.
\end{remark}
\end{remarks}

\begin{remarks}
\begin{remark}[Limitations]
Despite the empirical support, our theoretical observation on the optimal $\operatorname{Var}(\alignscore)$-reduction effect of eigen-rotation has some limitations. 
(1) The proof relies on the assumption that noise dominates the gradient signal and analyzes the variance by neglecting higher-order terms. However, we argue that this minimization effect may still holds when the gradient has stronger signal magnitude than the noise. In \cref{subapp: rotation optimal variance reduction}, we provide an synthetic experiment showing that eigen-rotation can still achieve low $\var(\alignscore)$ compared to other solutions when the signal-to-noise ratio is $10$. This may justify the behavior in early training stage of MNIST (\Cref{fig: MNIST} (middle)), where $\var(\alignscore)$-reduction effect already starts to show. 
(2) The current proof techniques are tightly coupled to the specific choice $f=\sign$ and do not readily generalize to other transformations. 
(3) In practice, the true expected gradient $\mEG$ is inaccessible for computing eigen rotations; instead, one must rely on noisy proxies such as $\mM$ or $\mNG$. The quantitative impact of using such proxies is not analyzed in this work, and we leave a careful treatment of this issue to future work.
\end{remark}
\end{remarks}

However, this desirable property comes with an important drawback: eigen rotation is overly conservative in terms of true loss decrease rate (i.e., alignment score $\alignscore$) under expectation. This effect is clearly reflected in the reduced magnitude of the alignment score observed in \cref{fig: MNIST}.

\begin{findings}
    \textbf{Finding 3: With our setup, we find eigen rotated optimizers is overly conservative at decreasing the loss compared to \name{}.}
\end{findings}

From \Cref{fig: MNIST}, eigen rotation often yields a smaller \emph{magnitude} of the alignment score than other methods. This behavior is also evident in right panel of \Cref{fig: MNIST}, where eigen rotation reduces the loss more slowly than \name{}, polar projection, and the unrotated baseline during the early phase of training.

Importantly, this behavior can be understood as an over-conservative effect induced by the $\operatorname{Var}(\alignscore)$-reduction property of the eigen-rotation. We characterize this conservativeness theoretically, under the same assumptions used in the variance minimization analysis.

\paragraph{Eigen rotation also acts as an alignment score minimizer}

\Cref{thm: rotation as loss rate minimizer} in \cref{subapp: rotation optimal variance reduction} shows that eigen rotation is also a local minimizer of the expected alignment score when $f=\sign$:
\begin{align*}
\E\left[\left\langle \mEG,\; \mR \sign\!\left(\mR^{\top}\mNG\right) \right\rangle\right],
\end{align*}
under the same assumptions as \cref{thm: optimal align score variance reduction}.

\begin{remarks}
\begin{remark}[Two sides of a coin]
\Cref{thm: optimal align score variance reduction} and \cref{thm: rotation as loss rate minimizer} reveal two complementary aspects of rotation: 
(1) $\var(\alignscore)$, and 
(2) $\E[\alignscore]$. 
Under the considered stochastic setting, no single rotation can simultaneously minimize $\var(\alignscore)$ and maximize $\E[\alignscore]$. 
Eigen rotation emphasizes the former at the expense of low magnitude, whereas polar projection or un-rotated baselilne (e.g. $\mR=\mI$) favors high magnitude while inducing large alignment variance. 
Empirically, variance minimization appears more favorable for training neural networks in stochastic regimes. We hypothesize that this is closely related to the loss landscape geometry, and defer a deeper theoretical explanation to future work.
\end{remark}
\end{remarks}

\begin{remarks}
    \begin{remark}[Limitation] Our theoretical result assumes a low signal-to-noise magnitude. However, 
        Similar to the $\var(\alignscore)$ reduction, we have shown through a synthetic experiment in \cref{subapp: rotation optimal variance reduction} that the $\E[\alignscore]$ property can, to certain extend, robustly extrapolate beyond this regime. 
    \end{remark}
\end{remarks}

\subsection{\name{} rotation: a speed--stability tradeoff}
\label{subsec: rotation loss decrease rate minimizer}

The analysis above highlights a potential limitation of eigen rotations, which are widely used in matrix optimizers. By reducing $\var(\alignscore)$, eigen rotations can substantially improve training stability. However, this variance reduction can also make the update overly conservative. Combining with our observations for the failure mode of Polar rotation, we therefore argue that a desirable rotation should avoid the two extremes:

\begin{takeaways}
\textbf{A desirable rotation should avoid the two extremes: it should not inflate alignment $S_A$ into the high $\var(\alignscore)$ regime (as Polar rotation), nor suppress alignment too strongly (as eigen rotation). Instead, it should increase $S_A$ moderately while preserving the $\var(\alignscore)$-reduction benefit.}
\end{takeaways}

Does \name{} satisfy this property? We empirically evaluate whether \name{} rotation alleviates this conservatism while preserving the stability benefits. As shown in \Cref{fig: MNIST}, \name{} rotation consistently achieves a higher alignment than eigen rotation, while still retaining the benefits of improved stability via $\var(\alignscore)$-reduction. \Cref{fig: stats_of_aro} shows that \name{} rotation, under both $f=\sign$ and $f=\text{sinkhorn}$, improves the empirical $\tilde{\mathcal{J}}$ relative to eigen rotation across all layers of the network. One should note that although empirical loss decrease rate $\tilde{\mathcal{J}}$ is not equivalent to alignment score as discussed in the beginning, we find it it is sufficient here to capture the qualitative behavior: \name{} improves loss decrease rate without requiring an aggressive increase in alignment. This leads us to the below finding:

\begin{findings}
\textbf{Finding 4: Empirically, \name{} rotation provides a good trade-off between $\var(\alignscore)$-reduction and loss decrease rate.}
\end{findings}

Overall, these results suggest that \name{} rotation induces a moderately higher alignment, which helps explain its stronger empirical performance across a wide range of experiments, while still retaining the stabilizing effect of $\var(\alignscore)$-reduction. Importantly, this behavior is consistent with the design of \name{}: rather than fully maximizing the rotation criterion (instantaneous loss decrease rate as a noisy surrogate of alignment $S_A$, see \Cref{sec: aro} and \Cref{sec: deriving_optimizer} ), \name{} performs only a partial maximization, targeting an improvement over eigen rotation while avoiding the overly conservative regime.

\subsection{Summary: the stability--speed tradeoff under noisy setup}
\label{subsec: denoising closing remakr}

The main driver of this section is to make an effort to the question: \textbf{What properties a desirable rotation should have}. We gradually build our answer by examining:
\begin{itemize}

\item Unlike the noise-free setting, an aggressive improvement in empirical loss decrease rate $\tilde{\mathcal{J}}$ may not be desirable in the presence of mini-batch noise, since it may induce high alignment variance and lead to unstable or inconsistent training progress as demonstrated in our preliminary experiment.

\item Eigen rotations aim to optimally reduce training instability under assumptions, resulting in stable and consistent loss decrease. However, it is often overly conservative, leading to a low alignment magnitude.

\item In contrast, empirically, \name{} rotation provides a better trade-off by avoiding the two extremes: it moderately improves alignment while preserving the desirable stabilizing property.

\end{itemize}

We emphasize that this section does not provide a complete recipe for choosing rotations. Instead, it offers a preliminary analysis of rotation through the lens of the stability--speed tradeoff under noisy settings, which serves as one plausible guiding principle. Many important questions remain open regarding this aspect, including but not limited to:

\begin{enumerate}

\item Are there alternative rotations that can achieve a better balance?

\item How to generalize our theoretical analysis to practical large scale LLM training?

\item Are there other aspects that are crucial for practical large scale LLM training?

\end{enumerate}

\section{Related Works}

\label{sec: related}
Below we briefly review related works in modern  matrix optimization for LLMs.

\paragraph{Gradient orthogonalization} Gradient orthogonalization methods optimize matrix-valued parameters by mapping each (momentum-) gradient to the closest orthogonal update, effectively equalizing progress across singular directions and improving conditioning relative to coordinatewise methods such as AdamW. Gradient orthogonalization methods are tightly connected to Shampoo \citep{gupta2018shampoo, anil2020scalable}: \citep{bernstein2024old} first observe that disabling Shampoo's gradient accumulation yields spectral descent \citep{carlson2015preconditioned, carlson2015stochastic, carlson2015stochastic}, and proposed efficient orthogonalization via Newton-Schulz iterations. These ideas were materialized as an practical optimizer (Muon) in \citep{jordan2024muon}, with an momentum-first implementation that resolves performance issues in prior work \citep{tuddenham2022orthogonalising}. This is subsequently modified for large scale distributed training with direct Adam grafting \citep{kexuefm-11267} and weight decay in \citep{liu2025muon}, which becomes the default Muon implementation used in our baselines.

\paragraph{Iterative normalization/iterative whitening} A direct extension to steepest descent under non-Euclidean norms (e.g., spectral descent/muon) is to consider multiple geometry simultaneously. In our prior work \citep{scetbon2025gradient}, we propose to perform steepest descent under multiple norm constraints, resulting in an iterative normalization scheme on  gradient matrices to improve update quality. This result in two practical instantiations: i) the combination of gradient orthogonalization and row-wise normalization, proposed in \citep{maswan}, and generalized in \citep{scetbon2025gradient}. Later, a similar idea was also explored in NorMuon \citep{li2025normuon}. ii) the SinkGD optimizer \Cref{sec: sink}, a lightweight, efficient matrix optimizer that applies a reparameterized version of Sinkhorn-Knopp algoritm to project gradients matrices \citep{sinkhorn1967concerning} before updates. Following on this, in \citep{swan2} we discussed the potential opportunity of using SinkGD in rotated space. This idea leads to the use of SinkGD as a base optimizer in \name{}.  Similarly, \citep{vyasimproving} proposed to iteratively whiten the gradients for improved performance from a different perspective. The author further conjectured the possibility of rotation sharing across layers. In our work, we explicitly derived concrete rotation sharing schemes based on rotational symmetries induced by cross-layer coupling in neural networks, and validated its empirical performance on \name{}.  

\paragraph{Improving Muon}  Following the initial successes of gradient orthogonalization, there has been a recent burst of work aimed at improving Muon. A dominant direction is to add variance/stepsize adaptation on top of orthogonal updates, including elementwise or blockwise second-moment scaling \citep{si2025adamuon, li2025normuon}, adaptive step size \citep{zhang2025adagrad}, as well as incorporating explicit variance-reduction techniques \citep{liu2025mars}. Complementarily, several works interpret/improves Muon through explicit or implicit regularization on weight space constraints \citep{pethick2025training, xie2026controlled, wen2025hyperball}. A third line studies additional momentum terms, proposing multi-timescale momentum constructions to better mix short- and long-horizon gradient information \citep{behrouz2025nested}. Finally, improving the numerical quality/efficency of orthogonalization is also an important direction \citep{amsel2025polar, ahn2025dion, lau2025polargrad}. Together, they offer substantial insights on practical optimizer design, especially on the choice of auxiliary elements overlooked by original implementations of Muon. Compared to those work, our work focus on pushing the frontier beyond the primitives of orthogonalization/spectral descent methods itself. Crucially, the performance gain of \name{} over Adam and Muon comes from its new geometric solution. We believe that in principle,  most of the advancements above for improving Muon (variance adaptation, multi-momentum buffer, Dion V2, weight decay alternatives with hyperball/spectral sphere, etc) can also be adapted to improve \name{}. We leave those extensions for future work.

\paragraph{K-FAC, and optimization on rotated spaces} K-FAC and its variants are among the earliest approaches to move beyond diagonal preconditioning in neural network optimization, using a Kronecker-factored approximation to curvature blocks to obtain efficient matrix preconditioning updates \citep{martens2015optimizing}. A complementary line of work exploits the induced eigenspaces of such structured curvature approximations, yielding algorithms that run first-order adaptivity in a rotated (approximately natural-gradient) basis \citep{george2018fast}, while related E-FAC-style designs introduce a diagonal preconditioner that is updated between expensive K-FAC inversion steps. The use of Kronecker factored structures also motivates a broad family of structured matrix/tensor preconditioners, such as Shampoo \citep{gupta2018shampoo, anil2020scalable, shi2023distributed}, SOAP \citep{vyas2024soap}. and SPlus \citep{frans2025stable}. Recently, \citep{vantowards} revisit Kronecker-factored eigenbasis (KFE) directly and proposed adaptive optimization in the KFE basis. In parallel, recent analyses reveals the basis dependence behavior of Adam's optimization performance \citep{maes2024understanding, xie2025adamexploitsellinftygeometryloss}. Particularly, \citep{maes2024understanding} argues that explaining Adam requires rotation-dependent assumptions and identifies eigen rotations that empirical improve performance of Adam. This line of work influenced our insight on treating as a first-class design principle, and automate the rotation selection policy informed by base optimizers.

\paragraph{Symmetry-aware optimization}
Over-parameterized neural network models often induce rich parameter symmetries where many distinct parameterizations implement the same function (e.g., permutations, rescalings, and internal basis changes). These symmetry structures has been extensively studied or exploited \citep{lecun2002gradient, hecht1990algebraic, kuurkova1994functionally, kuurkova1994functionally, garipov2018loss, draxler2018essentially, kunin2020neural, tanaka2021noether, entezari2021role, ainsworth2022git, zhao2022symmetries, zhao2025symmetry, wangmaximal, wang2025complete, da2025hide, zhang2025beyond, ashkboos2024slicegpt, liu2024spinquant, ashkboos2024quarot}, raising the question of how to utilize those information during optimization.  Existing methods respond in two complementary ways: \emph{exploiting} symmetry as additional degrees of freedom, or \emph{quotienting} it out via gauge-/reparameterization-invariant geometry. On the exploitation side, symmetry teleportation \citep{zhao2022symmetry} composes standard training steps with explicit loss-preserving group actions, teleporting along symmetry orbits to representatives where the base optimizer makes better progress (by maximizing gradient norm at teleported location); follow-up work reports faster convergence and improved generalization \citep{zhao2023improving}. Closely related ideas appear in permutation-aware mode connectivity and model merging \citep{entezari2021role, ainsworth2022git}, and in rotation-based matching/model fusion \citep{zhang2025beyond}.  On the quotient/invariant optimization side, 
natural gradient descent \citep{amari1998natural,amari2016information} performs steepest descent under the Fisher-Rao metric and is invariant to smooth reparameterizations. Recent work emphasizes that reparameterization-consistent dynamics can be obtained more generally by choosing an explicit Riemannian metric and transforming it correctly \citep{kristiadi2023geometry}. Grounded on transformers, quotient-Riemannian treatment were then used to explicitly derive symmetry-invariant metrics and projects gradients to remove orbit directions. \citep{da2025hide,wang2025gauge,wang2025complete}. The symmetry hypothesis of our work is heavily influenced by the above prior work. Building on these work, \name{} targets at residual-stream rotational symmetries that are extensively used in quantization \citep{ashkboos2024slicegpt,liu2024spinquant,ashkboos2024quarot}, but less explored in parameter symmetry literature. We further extends ST to non-Euclidean base optimizers to handle rotations, and uses a milder teleportation policy that improves a dual norm over eigen-solutions, sitting between the original gradient norm-maximization ST and its norm-minimization variant \citep{yun2023riemannian}.

\paragraph{Exploiting cross-layer coupling in matrix optimizers for LLMs} In the literature of LLM training, this is an under-explored topic due to the prohibitive memory and compute overhead of representing and applying non–block-diagonal preconditioners at LLM scale, especially under distributed training.  \citep{abreu2025potential} studied the potential of full Gauss-Newton methods for LLM pretraining, and found that layer-wise method closely match the performance of full GN method.  In contrast, \name{} uses a symmetry view to induce an economical coupling mechanism across layers and module types, improving performance while reducing overhead; suggesting that optimizer design should be grounded in architecture. Around the same time as our work, TEON \citep{zhang2026teontensorizedorthonormalizationlayerwise} extends Muon to higher-order tensors and couples adjacent layers by orthogonalizing stacked $QKV$ matrices. Mathematically, this is equivalent to sharing the eigen-rotations across different parameters based on \emph{module types}. The key difference to our work is three fold. First, our rotation sharing scheme enables exploiting both cross layer \emph{and} cross module-type coupling, grounded on architecture. In early development we have also explored similar, module type-based rotation sharing schemes. It is also an effective approach, but requires more computational overheads than our current approach. Second, our framework is base-optimizer agnostic and applies to any \name{} instantiation, handling very wide/rectangular matrices and global sharing more robustly. This opens more flexibility for setups that are not recommended in TEON for numerical stability concerns. Finally, materializing rotations in \name{} offers potential solutions for implementing rotation sharing with reduced communication cost.

\paragraph{Preconditioned SGD} PSGD is an relatively under-explored class of matrix optimizers that learns a non-diagonal preconditioner online via a Hessian-fitting criterion \citep{li2017preconditioned, li2018preconditioner, li2024stochastic}. PSGD has a similar design philosophy to our method in spirit - both tries to automate the learned preconditioning as a first-class design choice. PSGD is constructed around fitting a preconditioner to local curvature, whereas we explicitly utilize the rotated optimization form and find such rotations informed by general utility criteria. The two families overlap in special cases (both recovers gradient orthogonalization as special case), yet in general their objectives and algorithmic constructions differ.


\section{Conclusion and Future Directions}

This work introduces \textbf{Adaptively Rotated Optimization (\name{})}, a matrix optimization framework that advances the efficiency frontier of LLM pretraining over AdamW and gradient orthogonalization methods. By reformulating key existing matrix optimizers as instances of rotated steepest descent, we elevate gradient rotation to a first-class design principle. This perspective exposes a previously under-explored design space, leading to general rotation policies beyond eigen-rotations and base optimizers beyond Adam variants. Leveraging this view, we design a novel rotation policy that is informed by, and tailored to the base optimizer, yielding new update rules that extend beyond conventional eigen-rotation-based and orthogonalization approaches.

To ensure a more reliable and generalizable conclusions, we introduce rigorous benchmarking protocols to reduce confounding and bias. Under these protocols, \name{} consistently outperforms strong baselines (AdamW and Muon), across diverse LLM pretraining setups spanning model scales (up to 8B activated parameters), training budgets (up to 8 $\times$ overtrain), and architectures (both dense and MoE), with no clear signs of diminishing returns at scale.

In extended discussions, we discussed the theoretical principles that derives \name{}. We argue that the convergence of many key matrix optimizers to a special case of rotational forms might not be an coincidence. We derive \name{} as a form of generalized symmetry teleportation, that exploits exact or approximate rotational symmetries in NNs induced by cross-layer couplings. This leads to new optimizer designs that further exploits those couplings, validated in preliminary experiments. These motivates the \textbf{symmetry hypothesis}, an proposal that matrix optimization can be understood and improved through the lens of symmetry, a global structure largely overlooked in LLM optimization. 

Finally, \name{} opens up a new lens on understanding matrix optimization for large models. Promising future directions for developing next generation of optimizers may include:

\begin{itemize}
    \item Developing new rotation policies for \name{} that either further improves optimization performance, or reduce the computational overheads. 
    \item The search for new or exisiting base optimizer projection function $f$ that further improves the performance over our SinkGD. \name{} unlocks significant opportunities due to the large volume of available base optimizers in the literature.
    \item Further improving the performance of \name{} by adopting recent advancements in the Muon literature, such as variance adaptation, multi-momentum buffers, weight decay alternatives, or better hyperparameter transfer with $\mu$P, etc.
    \item Scaling up the rotation sharing schemes and rotation orientation strategies for \name{}, demonstrated in \Cref{sec: symmetry_predictions}, with necessary improvements and modifications. We believe this will opens up opportunities for algorithm - infrastructure co-design.
    \item Extending and improving \name{} by exploiting more symmetry structures of LLMs beyond residual stream rotations. For example, per-head rotational symmetries \citep{da2025hide}, as well as other symmetries characterized in  the literature \citep{wang2025complete}.
    \item Theoretical understanding on the symmetry breaking - optimizer expressiveness tradeoff.
\end{itemize}

\section*{Acknowledgment}

We are grateful to Ted Meeds and Aditya Nori for their continuous support and insightful feedback throughout the project. We thank Sushrut Karmalkar, Teodora Pandeva, Alicia Curth, and Riccardo Grazzi for valuable discussions and comments. We also thank Yordan Zaykov and Jonathan Tims for discussions on distributed training, GPU onboarding, and resource management that made this study possible. We would also like to thank Yeyun Gong and Peng cheng for their support and helpful discussion on testing \name{} on Sigma, as well as Yuchen Ding, Laihan Ren, Yaoxiang Wang for constructive feedback.  Finally, we thank Marvin F. da Silva and Bo Zhao for discussions on symmetry-aware optimization.

\bibliographystyle{plain} 
\bibliography{refs} 

\begin{thebibliography}{100}

\bibitem{abreu2025potential}
Natalie Abreu, Nikhil Vyas, Sham Kakade, and Depen Morwani.
\newblock The potential of second-order optimization for llms: A study with full gauss-newton.
\newblock {\em arXiv preprint arXiv:2510.09378}, 2025.

\bibitem{achiam2023gpt}
Josh Achiam, Steven Adler, Sandhini Agarwal, Lama Ahmad, Ilge Akkaya, Florencia~Leoni Aleman, Diogo Almeida, Janko Altenschmidt, Sam Altman, Shyamal Anadkat, et~al.
\newblock Gpt-4 technical report.
\newblock {\em arXiv preprint arXiv:2303.08774}, 2023.

\bibitem{agarap2018deep}
Abien~Fred Agarap.
\newblock Deep learning using rectified linear units (relu).
\newblock {\em arXiv preprint arXiv:1803.08375}, 2018.

\bibitem{agarwallearning}
Naman Agarwal, Rohan Anil, Elad Hazan, Tomer Koren, and Cyril Zhang.
\newblock Learning rate grafting: Transferability of optimizer tuning.

\bibitem{ahn2025dion}
Kwangjun Ahn and Byron Xu.
\newblock Dion: A communication-efficient optimizer for large models.
\newblock {\em arXiv e-prints}, pages arXiv--2504, 2025.

\bibitem{ainsworth2022git}
Samuel~K Ainsworth, Jonathan Hayase, and Siddhartha Srinivasa.
\newblock Git re-basin: Merging models modulo permutation symmetries.
\newblock {\em arXiv preprint arXiv:2209.04836}, 2022.

\bibitem{Allen2025-resonate}
Zeyuan {Allen-Zhu}.
\newblock {Physics of Language Models: Part 4.2, Canon Layers at Scale where Synthetic Pretraining Resonates in Reality}, 2025.
\newblock Code released at \url{https://github.com/facebookresearch/PhysicsLM4}.

\bibitem{amari1998natural}
Shun-Ichi Amari.
\newblock Natural gradient works efficiently in learning.
\newblock {\em Neural computation}, 10(2):251--276, 1998.

\bibitem{amari2016information}
Shun-ichi Amari.
\newblock {\em Information geometry and its applications}, volume 194.
\newblock Springer, 2016.

\bibitem{amsel2025polar}
Noah Amsel, David Persson, Christopher Musco, and Robert~M Gower.
\newblock The polar express: Optimal matrix sign methods and their application to the muon algorithm.
\newblock {\em arXiv preprint arXiv:2505.16932}, 2025.

\bibitem{an2025asgo}
Kang An, Yuxing Liu, Rui Pan, Yi~Ren, Shiqian Ma, Donald Goldfarb, and Tong Zhang.
\newblock Asgo: Adaptive structured gradient optimization.
\newblock {\em arXiv preprint arXiv:2503.20762}, 2025.

\bibitem{anil2020scalable}
Rohan Anil, Vineet Gupta, Tomer Koren, Kevin Regan, and Yoram Singer.
\newblock Scalable second order optimization for deep learning.
\newblock {\em arXiv preprint arXiv:2002.09018}, 2020.

\bibitem{ashkboos2024slicegpt}
Saleh Ashkboos, Maximilian~L Croci, Marcelo Gennari~do Nascimento, Torsten Hoefler, and James Hensman.
\newblock Slicegpt: Compress large language models by deleting rows and columns.
\newblock {\em arXiv preprint arXiv:2401.15024}, 2024.

\bibitem{ashkboos2024quarot}
Saleh Ashkboos, Amirkeivan Mohtashami, Maximilian~L Croci, Bo~Li, Pashmina Cameron, Martin Jaggi, Dan Alistarh, Torsten Hoefler, and James Hensman.
\newblock Quarot: Outlier-free 4-bit inference in rotated llms.
\newblock {\em Advances in Neural Information Processing Systems}, 37:100213--100240, 2024.

\bibitem{becker2024momentum}
Marlon Becker, Frederick Altrock, and Benjamin Risse.
\newblock Momentum-sam: Sharpness aware minimization without computational overhead.
\newblock {\em arXiv preprint arXiv:2401.12033}, 2024.

\bibitem{behrouz2025nested}
Ali Behrouz, Meisam Razaviyayn, Peilin Zhong, and Vahab Mirrokni.
\newblock Nested learning: The illusion of deep learning architectures.
\newblock {\em arXiv preprint arXiv:2512.24695}, 2025.

\bibitem{bernstein2024modular}
Jeremy Bernstein and Laker Newhouse.
\newblock Modular duality in deep learning.
\newblock {\em arXiv preprint arXiv:2410.21265}, 2024.

\bibitem{bernstein2024old}
Jeremy Bernstein and Laker Newhouse.
\newblock Old optimizer, new norm: An anthology.
\newblock {\em arXiv preprint arXiv:2409.20325}, 2024.

\bibitem{bernstein2018signsgd}
Jeremy Bernstein, Yu-Xiang Wang, Kamyar Azizzadenesheli, and Animashree Anandkumar.
\newblock signsgd: Compressed optimisation for non-convex problems.
\newblock In {\em International conference on machine learning}, pages 560--569. PMLR, 2018.

\bibitem{carlson2015stochastic}
David Carlson, Volkan Cevher, and Lawrence Carin.
\newblock Stochastic spectral descent for restricted boltzmann machines.
\newblock In {\em Artificial intelligence and statistics}, pages 111--119. PMLR, 2015.

\bibitem{carlson2015stochasticb}
David Carlson, Ya-Ping Hsieh, Edo Collins, Lawrence Carin, and Volkan Cevher.
\newblock Stochastic spectral descent for discrete graphical models.
\newblock {\em IEEE Journal of Selected Topics in Signal Processing}, 10(2):296--311, 2015.

\bibitem{carlson2015preconditioned}
David~E Carlson, Edo Collins, Ya-Ping Hsieh, Lawrence Carin, and Volkan Cevher.
\newblock Preconditioned spectral descent for deep learning.
\newblock {\em Advances in neural information processing systems}, 28, 2015.

\bibitem{cutkosky2019momentum}
Ashok Cutkosky and Francesco Orabona.
\newblock Momentum-based variance reduction in non-convex sgd.
\newblock {\em Advances in neural information processing systems}, 32, 2019.

\bibitem{da2025hide}
Marvin~F Da~Silva, Felix Dangel, and Sageev Oore.
\newblock Hide \& seek: Transformer symmetries obscure sharpness \& riemannian geometry finds it.
\newblock {\em arXiv preprint arXiv:2505.05409}, 2025.

\bibitem{draxler2018essentially}
Felix Draxler, Kambis Veschgini, Manfred Salmhofer, and Fred Hamprecht.
\newblock Essentially no barriers in neural network energy landscape.
\newblock In {\em International conference on machine learning}, pages 1309--1318. PMLR, 2018.

\bibitem{entezari2021role}
Rahim Entezari, Hanie Sedghi, Olga Saukh, and Behnam Neyshabur.
\newblock The role of permutation invariance in linear mode connectivity of neural networks.
\newblock {\em arXiv preprint arXiv:2110.06296}, 2021.

\bibitem{eschenhagen2025purifying}
Runa Eschenhagen, Aaron Defazio, Tsung-Hsien Lee, Richard~E Turner, and Hao-Jun~Michael Shi.
\newblock Purifying shampoo: Investigating shampoo's heuristics by decomposing its preconditioner.
\newblock {\em arXiv preprint arXiv:2506.03595}, 2025.

\bibitem{fang2018spider}
Cong Fang, Chris~Junchi Li, Zhouchen Lin, and Tong Zhang.
\newblock Spider: Near-optimal non-convex optimization via stochastic path-integrated differential estimator.
\newblock {\em Advances in neural information processing systems}, 31, 2018.

\bibitem{frans2025really}
Kevin Frans, Pieter Abbeel, and Sergey Levine.
\newblock What really matters in matrix-whitening optimizers?
\newblock {\em arXiv preprint arXiv:2510.25000}, 2025.

\bibitem{frans2025stable}
Kevin Frans, Sergey Levine, and Pieter Abbeel.
\newblock A stable whitening optimizer for efficient neural network training.
\newblock {\em arXiv preprint arXiv:2506.07254}, 2025.

\bibitem{fukaya2020shifted}
Takeshi Fukaya, Ramaseshan Kannan, Yuji Nakatsukasa, Yusaku Yamamoto, and Yuka Yanagisawa.
\newblock Shifted cholesky qr for computing the qr factorization of ill-conditioned matrices.
\newblock {\em SIAM Journal on Scientific Computing}, 42(1):A477--A503, 2020.

\bibitem{garipov2018loss}
Timur Garipov, Pavel Izmailov, Dmitrii Podoprikhin, Dmitry~P Vetrov, and Andrew~G Wilson.
\newblock Loss surfaces, mode connectivity, and fast ensembling of dnns.
\newblock {\em Advances in neural information processing systems}, 31, 2018.

\bibitem{george2018fast}
Thomas George, C{\'e}sar Laurent, Xavier Bouthillier, Nicolas Ballas, and Pascal Vincent.
\newblock Fast approximate natural gradient descent in a kronecker factored eigenbasis.
\newblock {\em Advances in neural information processing systems}, 31, 2018.

\bibitem{gong2025towards}
Wenbo Gong, Meyer Scetbon, Chao Ma, and Edward Meeds.
\newblock Towards efficient optimizer design for llm via structured fisher approximation with a low-rank extension.
\newblock {\em arXiv preprint arXiv:2502.07752}, 2025.

\bibitem{grattafiori2024llama}
Aaron Grattafiori, Abhimanyu Dubey, Abhinav Jauhri, Abhinav Pandey, Abhishek Kadian, Ahmad Al-Dahle, Aiesha Letman, Akhil Mathur, Alan Schelten, Alex Vaughan, et~al.
\newblock The llama 3 herd of models.
\newblock {\em arXiv preprint arXiv:2407.21783}, 2024.

\bibitem{guo2025deepseek}
Daya Guo, Dejian Yang, Haowei Zhang, Junxiao Song, Peiyi Wang, Qihao Zhu, Runxin Xu, Ruoyu Zhang, Shirong Ma, Xiao Bi, et~al.
\newblock Deepseek-r1 incentivizes reasoning in llms through reinforcement learning.
\newblock {\em Nature}, 645(8081):633--638, 2025.

\bibitem{gupta2018shampoo}
Vineet Gupta, Tomer Koren, and Yoram Singer.
\newblock Shampoo: Preconditioned stochastic tensor optimization.
\newblock In {\em International Conference on Machine Learning}, pages 1842--1850. PMLR, 2018.

\bibitem{he2025root}
Wei He, Kai Han, Hang Zhou, Hanting Chen, Zhicheng Liu, Xinghao Chen, and Yunhe Wang.
\newblock Root: Robust orthogonalized optimizer for neural network training.
\newblock {\em arXiv preprint arXiv:2511.20626}, 2025.

\bibitem{hecht1990algebraic}
Robert Hecht-Nielsen.
\newblock On the algebraic structure of feedforward network weight spaces.
\newblock In {\em Advanced Neural Computers}, pages 129--135. Elsevier, 1990.

\bibitem{hernandez2025apertus}
Alejandro Hern{\'a}ndez-Cano, Alexander H{\"a}gele, Allen~Hao Huang, Angelika Romanou, Antoni-Joan Solergibert, Barna Pasztor, Bettina Messmer, Dhia Garbaya, Eduard~Frank {\v{D}}urech, Ido Hakimi, et~al.
\newblock Apertus: Democratizing open and compliant llms for global language environments.
\newblock {\em arXiv preprint arXiv:2509.14233}, 2025.

\bibitem{hoffmann2022training}
Jordan Hoffmann, Sebastian Borgeaud, Arthur Mensch, Elena Buchatskaya, Trevor Cai, Eliza Rutherford, Diego de~Las Casas, Lisa~Anne Hendricks, Johannes Welbl, Aidan Clark, et~al.
\newblock Training compute-optimal large language models.
\newblock {\em arXiv preprint arXiv:2203.15556}, 2022.

\bibitem{hu2025sigmamoetinytechnicalreport}
Qingguo Hu, Zhenghao Lin, Ziyue Yang, Yucheng Ding, Xiao Liu, Yuting Jiang, Ruizhe Wang, Tianyu Chen, Zhongxin Guo, Yifan Xiong, Rui Gao, Lei Qu, Jinsong Su, Peng Cheng, and Yeyun Gong.
\newblock Sigma-moe-tiny technical report, 2025.

\bibitem{kexuefm-11267}
Su~Jianlin.
\newblock Why adam has update rms norm of 0.2, Sep 2025.

\bibitem{johnson2013accelerating}
Rie Johnson and Tong Zhang.
\newblock Accelerating stochastic gradient descent using predictive variance reduction.
\newblock {\em Advances in neural information processing systems}, 26, 2013.

\bibitem{jordan2024muon}
Keller Jordan, Yuchen Jin, Vlado Boza, Jiacheng You, Franz Cesista, Laker Newhouse, and Jeremy Bernstein.
\newblock Muon: An optimizer for hidden layers in neural networks, 2024.

\bibitem{kaddour2023no}
Jean Kaddour, Oscar Key, Piotr Nawrot, Pasquale Minervini, and Matt~J Kusner.
\newblock No train no gain: Revisiting efficient training algorithms for transformer-based language models.
\newblock {\em Advances in Neural Information Processing Systems}, 36:25793--25818, 2023.

\bibitem{Karpathy2022}
Andrej Karpathy.
\newblock \text{NanoGPT}.
\newblock \url{https://github.com/karpathy/nanoGPT}, 2022.

\bibitem{nanochat}
Andrej Karpathy.
\newblock nanochat: The best chatgpt that \$100 can buy, 2025.

\bibitem{kingma2014adam}
Diederik~P Kingma and Jimmy Ba.
\newblock Adam: A method for stochastic optimization.
\newblock {\em arXiv preprint arXiv:1412.6980}, 2014.

\bibitem{knight2014symmetry}
Philip~A Knight, Daniel Ruiz, and Bora U{\c{c}}ar.
\newblock A symmetry preserving algorithm for matrix scaling.
\newblock {\em SIAM journal on Matrix Analysis and Applications}, 35(3):931--955, 2014.

\bibitem{kristiadi2023geometry}
Agustinus Kristiadi, Felix Dangel, and Philipp Hennig.
\newblock The geometry of neural nets' parameter spaces under reparametrization.
\newblock {\em Advances in Neural Information Processing Systems}, 36:17669--17688, 2023.

\bibitem{kunin2020neural}
Daniel Kunin, Javier Sagastuy-Brena, Surya Ganguli, Daniel~LK Yamins, and Hidenori Tanaka.
\newblock Neural mechanics: Symmetry and broken conservation laws in deep learning dynamics.
\newblock {\em arXiv preprint arXiv:2012.04728}, 2020.

\bibitem{kuurkova1994functionally}
Vera Kurkova and Paul~C Kainen.
\newblock Functionally equivalent feedforward neural networks.
\newblock {\em Neural Computation}, 6(3):543--558, 1994.

\bibitem{kwak2008principal}
Nojun Kwak.
\newblock Principal component analysis based on l1-norm maximization.
\newblock {\em IEEE transactions on pattern analysis and machine intelligence}, 30(9):1672--1680, 2008.

\bibitem{lau2025polargrad}
Tim Tsz-Kit Lau, Qi~Long, and Weijie Su.
\newblock Polargrad: A class of matrix-gradient optimizers from a unifying preconditioning perspective.
\newblock {\em arXiv preprint arXiv:2505.21799}, 2025.

\bibitem{lecun2002gradient}
Yann LeCun, L{\'e}on Bottou, Yoshua Bengio, and Patrick Haffner.
\newblock Gradient-based learning applied to document recognition.
\newblock {\em Proceedings of the IEEE}, 86(11):2278--2324, 2002.

\bibitem{li2017preconditioned}
Xi-Lin Li.
\newblock Preconditioned stochastic gradient descent.
\newblock {\em IEEE transactions on neural networks and learning systems}, 29(5):1454--1466, 2017.

\bibitem{li2018preconditioner}
Xi-Lin Li.
\newblock Preconditioner on matrix lie group for sgd.
\newblock {\em arXiv preprint arXiv:1809.10232}, 2018.

\bibitem{li2024stochastic}
Xi-Lin Li.
\newblock Stochastic hessian fittings with lie groups.
\newblock {\em arXiv preprint arXiv:2402.11858}, 2024.

\bibitem{li2025normuon}
Zichong Li, Liming Liu, Chen Liang, Weizhu Chen, and Tuo Zhao.
\newblock Normuon: Making muon more efficient and scalable.
\newblock {\em arXiv preprint arXiv:2510.05491}, 2025.

\bibitem{lim2025motif}
Junghwan Lim, Sungmin Lee, Dongseok Kim, Taehyun Kim, Eunhwan Park, Jeesoo Lee, Jeongdoo Lee, Junhyeok Lee, Wai~Ting Cheung, Dahye Choi, et~al.
\newblock Motif 2 12.7 b technical report.
\newblock {\em arXiv preprint arXiv:2511.07464}, 2025.

\bibitem{liu2023sophia}
Hong Liu, Zhiyuan Li, David Hall, Percy Liang, and Tengyu Ma.
\newblock Sophia: A scalable stochastic second-order optimizer for language model pre-training.
\newblock {\em arXiv preprint arXiv:2305.14342}, 2023.

\bibitem{liu2025muon}
Jingyuan Liu, Jianlin Su, Xingcheng Yao, Zhejun Jiang, Guokun Lai, Yulun Du, Yidao Qin, Weixin Xu, Enzhe Lu, Junjie Yan, et~al.
\newblock Muon is scalable for llm training.
\newblock {\em arXiv preprint arXiv:2502.16982}, 2025.

\bibitem{liu2025mars}
Yifeng Liu, Angela Yuan, and Quanquan Gu.
\newblock Mars-m: When variance reduction meets matrices.
\newblock {\em arXiv preprint arXiv:2510.21800}, 2025.

\bibitem{liu2024spinquant}
Zechun Liu, Changsheng Zhao, Igor Fedorov, Bilge Soran, Dhruv Choudhary, Raghuraman Krishnamoorthi, Vikas Chandra, Yuandong Tian, and Tijmen Blankevoort.
\newblock Spinquant: Llm quantization with learned rotations.
\newblock {\em arXiv preprint arXiv:2405.16406}, 2024.

\bibitem{loshchilov2024ngpt}
Ilya Loshchilov, Cheng-Ping Hsieh, Simeng Sun, and Boris Ginsburg.
\newblock ngpt: Normalized transformer with representation learning on the hypersphere.
\newblock {\em arXiv preprint arXiv:2410.01131}, 2024.

\bibitem{loshchilov2017decoupled}
Ilya Loshchilov and Frank Hutter.
\newblock Decoupled weight decay regularization.
\newblock {\em arXiv preprint arXiv:1711.05101}, 2017.

\bibitem{lu2025understanding}
Yanqing Lu, Letao Wang, and Jinbo Liu.
\newblock Understanding soap from the perspective of gradient whitening.
\newblock {\em arXiv preprint arXiv:2509.22938}, 2025.

\bibitem{ludziejewski2025joint}
Jan Ludziejewski, Maciej Pi{\'o}ro, Jakub Krajewski, Maciej Stefaniak, Micha{\l} Krutul, Jan Ma{\l}a{\'s}nicki, Marek Cygan, Piotr Sankowski, Kamil Adamczewski, Piotr Mi{\l}o{\'s}, et~al.
\newblock Joint moe scaling laws: Mixture of experts can be memory efficient.
\newblock {\em arXiv preprint arXiv:2502.05172}, 2025.

\bibitem{maswan}
Chao Ma, Wenbo Gong, Meyer Scetbon, and Edward Meeds.
\newblock Swan: Sgd with normalization and whitening enables stateless llm training.
\newblock In {\em Forty-second International Conference on Machine Learning}.

\bibitem{swan2}
Chao Ma, Wenbo Gong, Meyer Scetbon, and Edward Meeds.
\newblock {S}{W}{A}{N}:{S}{G}{D} with {N}ormalization and {W}hitening {E}nables {S}tateless {L}{L}{M} {T}raining.
\newblock \url{https://chao-ma.org/wp-content/uploads/2025/09/swan-2.pdf}.
\newblock [Accessed 03-12-2025].

\bibitem{ma2024swan}
Chao Ma, Wenbo Gong, Meyer Scetbon, and Edward Meeds.
\newblock Swan: Preprocessing sgd enables adam-level performance on llm training with significant memory reduction.
\newblock {\em arXiv preprint arXiv:2412.13148}, 2024.

\bibitem{maes2024understanding}
Lucas Maes, Tianyue~H Zhang, Alexia Jolicoeur-Martineau, Ioannis Mitliagkas, Damien Scieur, Simon Lacoste-Julien, and Charles Guille-Escuret.
\newblock Understanding adam requires better rotation dependent assumptions.
\newblock {\em arXiv preprint arXiv:2410.19964}, 2024.

\bibitem{markopoulos2017efficient}
Panos~P Markopoulos, Sandipan Kundu, Shubham Chamadia, and Dimitris~A Pados.
\newblock Efficient l1-norm principal-component analysis via bit flipping.
\newblock {\em IEEE Transactions on Signal Processing}, 65(16):4252--4264, 2017.

\bibitem{martens2015optimizing}
James Martens and Roger Grosse.
\newblock Optimizing neural networks with kronecker-factored approximate curvature.
\newblock In {\em International conference on machine learning}, pages 2408--2417. PMLR, 2015.

\bibitem{mori2022power}
Takashi Mori, Liu Ziyin, Kangqiao Liu, and Masahito Ueda.
\newblock Power-law escape rate of sgd.
\newblock In {\em International Conference on Machine Learning}, pages 15959--15975. PMLR, 2022.

\bibitem{morwani2024new}
Depen Morwani, Itai Shapira, Nikhil Vyas, Eran Malach, Sham Kakade, and Lucas Janson.
\newblock A new perspective on shampoo's preconditioner.
\newblock {\em arXiv preprint arXiv:2406.17748}, 2024.

\bibitem{myronenko2009closed}
Andriy Myronenko and Xubo Song.
\newblock On the closed-form solution of the rotation matrix arising in computer vision problems.
\newblock {\em arXiv preprint arXiv:0904.1613}, 2009.

\bibitem{nguyen2017sarah}
Lam~M Nguyen, Jie Liu, Katya Scheinberg, and Martin Tak{\'a}{\v{c}}.
\newblock Sarah: A novel method for machine learning problems using stochastic recursive gradient.
\newblock In {\em International conference on machine learning}, pages 2613--2621. PMLR, 2017.

\bibitem{nguyen2025improving}
Son Nguyen, Bo~Liu, Lizhang Chen, and Qiang Liu.
\newblock Improving adaptive moment optimization via preconditioner diagonalization.
\newblock {\em arXiv preprint arXiv:2502.07488}, 2025.

\bibitem{paszke2019pytorch}
Adam Paszke, Sam Gross, Francisco Massa, Adam Lerer, James Bradbury, Gregory Chanan, Trevor Killeen, Zeming Lin, Natalia Gimelshein, Luca Antiga, et~al.
\newblock Pytorch: An imperative style, high-performance deep learning library.
\newblock {\em Advances in neural information processing systems}, 32, 2019.

\bibitem{pethick2025training}
Thomas Pethick, Wanyun Xie, Kimon Antonakopoulos, Zhenyu Zhu, Antonio Silveti-Falls, and Volkan Cevher.
\newblock Training deep learning models with norm-constrained lmos.
\newblock {\em arXiv preprint arXiv:2502.07529}, 2025.

\bibitem{qiu2025hyperparameter}
Shikai Qiu, Zixi Chen, Hoang Phan, Qi~Lei, and Andrew~Gordon Wilson.
\newblock Hyperparameter transfer enables consistent gains of matrix-preconditioned optimizers across scales.
\newblock {\em arXiv preprint arXiv:2512.05620}, 2025.

\bibitem{qu2025sigma}
Lei Qu, Lianhai Ren, Peng Cheng, Rui Gao, Ruizhe Wang, Tianyu Chen, Xiao Liu, Xingjian Zhang, Yeyun Gong, Yifan Xiong, et~al.
\newblock Sigma: An ai-empowered training stack on early-life hardware.
\newblock {\em arXiv preprint arXiv:2512.13488}, 2025.

\bibitem{rasley2020deepspeed}
Jeff Rasley, Samyam Rajbhandari, Olatunji Ruwase, and Yuxiong He.
\newblock Deepspeed: System optimizations enable training deep learning models with over 100 billion parameters.
\newblock In {\em Proceedings of the 26th ACM SIGKDD international conference on knowledge discovery \& data mining}, pages 3505--3506, 2020.

\bibitem{robert2024ldadam}
Thomas Robert, Mher Safaryan, Ionut-Vlad Modoranu, and Dan Alistarh.
\newblock Ldadam: Adaptive optimization from low-dimensional gradient statistics.
\newblock {\em arXiv preprint arXiv:2410.16103}, 2024.

\bibitem{roux2012stochastic}
Nicolas Roux, Mark Schmidt, and Francis Bach.
\newblock A stochastic gradient method with an exponential convergence \_rate for finite training sets.
\newblock {\em Advances in neural information processing systems}, 25, 2012.

\bibitem{scetbon2025gradient}
Meyer Scetbon, Chao Ma, Wenbo Gong, and Edward Meeds.
\newblock Gradient multi-normalization for stateless and scalable llm training.
\newblock {\em arXiv preprint arXiv:2502.06742}, 2025.

\bibitem{schonemann1966generalized}
Peter~H Sch{\"o}nemann.
\newblock A generalized solution of the orthogonal procrustes problem.
\newblock {\em Psychometrika}, 31(1):1--10, 1966.

\bibitem{semenov2025benchmarking}
Andrei Semenov, Matteo Pagliardini, and Martin Jaggi.
\newblock Benchmarking optimizers for large language model pretraining.
\newblock {\em arXiv preprint arXiv:2509.01440}, 2025.

\bibitem{shazeer2018adafactor}
Noam Shazeer and Mitchell Stern.
\newblock Adafactor: Adaptive learning rates with sublinear memory cost.
\newblock In {\em International Conference on Machine Learning}, pages 4596--4604. PMLR, 2018.

\bibitem{shi2023distributed}
Hao-Jun~Michael Shi, Tsung-Hsien Lee, Shintaro Iwasaki, Jose Gallego-Posada, Zhijing Li, Kaushik Rangadurai, Dheevatsa Mudigere, and Michael Rabbat.
\newblock A distributed data-parallel pytorch implementation of the distributed shampoo optimizer for training neural networks at-scale.
\newblock {\em arXiv preprint arXiv:2309.06497}, 2023.

\bibitem{shoeybi2019megatron}
Mohammad Shoeybi, Mostofa Patwary, Raul Puri, Patrick LeGresley, Jared Casper, and Bryan Catanzaro.
\newblock Megatron-lm: Training multi-billion parameter language models using model parallelism.
\newblock {\em arXiv preprint arXiv:1909.08053}, 2019.

\bibitem{shulgin2025beyond}
Egor Shulgin, Sultan AlRashed, Francesco Orabona, and Peter Richt{\'a}rik.
\newblock Beyond the ideal: Analyzing the inexact muon update.
\newblock {\em arXiv preprint arXiv:2510.19933}, 2025.

\bibitem{si2025adamuon}
Chongjie Si, Debing Zhang, and Wei Shen.
\newblock Adamuon: Adaptive muon optimizer.
\newblock {\em arXiv preprint arXiv:2507.11005}, 2025.

\bibitem{sinkhorn1967concerning}
Richard Sinkhorn and Paul Knopp.
\newblock Concerning nonnegative matrices and doubly stochastic matrices.
\newblock {\em Pacific Journal of Mathematics}, 21(2):343--348, 1967.

\bibitem{cerebras2023slimpajama}
Daria Soboleva, Faisal Al-Khateeb, Robert Myers, Jacob~R Steeves, Joel Hestness, and Nolan Dey.
\newblock {SlimPajama: A 627B token cleaned and deduplicated version of RedPajama}.
\newblock \url{https://www.cerebras.net/blog/slimpajama-a-627b-token-cleaned-and-deduplicated-version-of-redpajama}, 2023.

\bibitem{su2025nemotron}
Dan Su, Kezhi Kong, Ying Lin, Joseph Jennings, Brandon Norick, Markus Kliegl, Mostofa Patwary, Mohammad Shoeybi, and Bryan Catanzaro.
\newblock Nemotron-cc: Transforming common crawl into a refined long-horizon pretraining dataset.
\newblock In {\em Proceedings of the 63rd Annual Meeting of the Association for Computational Linguistics (Volume 1: Long Papers)}, pages 2459--2475, 2025.

\bibitem{tanaka2021noether}
Hidenori Tanaka and Daniel Kunin.
\newblock Noether’s learning dynamics: Role of symmetry breaking in neural networks.
\newblock {\em Advances in Neural Information Processing Systems}, 34:25646--25660, 2021.

\bibitem{team2025kimi}
Kimi Team, Yifan Bai, Yiping Bao, Guanduo Chen, Jiahao Chen, Ningxin Chen, Ruijue Chen, Yanru Chen, Yuankun Chen, Yutian Chen, et~al.
\newblock Kimi k2: Open agentic intelligence.
\newblock {\em arXiv preprint arXiv:2507.20534}, 2025.

\bibitem{tuddenham2022orthogonalising}
Mark Tuddenham, Adam Pr{\"u}gel-Bennett, and Jonathan Hare.
\newblock Orthogonalising gradients to speed up neural network optimisation.
\newblock {\em arXiv preprint arXiv:2202.07052}, 2022.

\bibitem{ueaj_multiscale_muon}
{ueaj}.
\newblock Multiscale muon.

\bibitem{vantowards}
Tycho~FA van~der Ouderaa, Mohamed Baioumy, Matt Beton, Seth Howes, Gelu Vrabie, and Alex Cheema.
\newblock Towards large scale training on apple silicon.
\newblock In {\em ES-FoMo III: 3rd Workshop on Efficient Systems for Foundation Models}.

\bibitem{veprikov2025preconditioned}
Andrey Veprikov, Arman Bolatov, Samuel Horv{\'a}th, Aleksandr Beznosikov, Martin Tak{\'a}{\v{c}}, and Slavomir Hanzely.
\newblock Preconditioned norms: A unified framework for steepest descent, quasi-newton and adaptive methods.
\newblock {\em arXiv preprint arXiv:2510.10777}, 2025.

\bibitem{vyas2024soap}
Nikhil Vyas, Depen Morwani, Rosie Zhao, Mujin Kwun, Itai Shapira, David Brandfonbrener, Lucas Janson, and Sham Kakade.
\newblock Soap: Improving and stabilizing shampoo using adam.
\newblock {\em arXiv preprint arXiv:2409.11321}, 2024.

\bibitem{vyasimproving}
Nikhil Vyas, Rosie Zhao, Depen Morwani, Mujin Kwun, and Sham Kakade.
\newblock Improving soap using iterative whitening and muon.

\bibitem{wangmaximal}
Hong Wang and Kelly Wang.
\newblock Maximal gauge symmetry in transformer architectures.

\bibitem{wang2025complete}
Hong Wang and Kelly Wang.
\newblock Complete characterization of gauge symmetries in transformer architectures.
\newblock In {\em NeurIPS 2025 Workshop on Symmetry and Geometry in Neural Representations}, 2025.

\bibitem{wang2025gauge}
Hong Wang and Kelly Wang.
\newblock Gauge fiber bundle geometry of transformers.
\newblock 2025.

\bibitem{wang2023theoretical}
Mingze Wang and Lei Wu.
\newblock A theoretical analysis of noise geometry in stochastic gradient descent.
\newblock {\em arXiv preprint arXiv:2310.00692}, 2023.

\bibitem{wang2021rethinking}
Yizhou Wang, Yue Kang, Can Qin, Huan Wang, Yi~Xu, Yulun Zhang, and Yun Fu.
\newblock Rethinking adam: A twofold exponential moving average approach.
\newblock {\em arXiv preprint arXiv:2106.11514}, 2021.

\bibitem{wang2026olionapproachinghadamardideal}
Zixiao Wang, Yifei Shen, and Huishuai Zhang.
\newblock Olion: Approaching the hadamard ideal by intersecting spectral and $\ell_{\infty}$ implicit biases, 2026.

\bibitem{wen2025hyperball}
Kaiyue Wen, Xingyu Dang, Kaifeng Lyu, Tengyu Ma, and Percy Liang.
\newblock Fantastic pretraining optimizers and where to find them 2.1: Hyperball optimization, 12 2025.

\bibitem{wen2025fantastic}
Kaiyue Wen, David Hall, Tengyu Ma, and Percy Liang.
\newblock Fantastic pretraining optimizers and where to find them.
\newblock {\em arXiv preprint arXiv:2509.02046}, 2025.

\bibitem{wensron}
Ziqing Wen, Yanqi Shi, Jiahuan Wang, Ping Luo, Linbo Qiao, Dongsheng Li, and Tao Sun.
\newblock Sron: State-free llm training via row-wise gradient normalization.

\bibitem{xie2025adamexploitsellinftygeometryloss}
Shuo Xie, Mohamad~Amin Mohamadi, and Zhiyuan Li.
\newblock Adam exploits $\ell_\infty$-geometry of loss landscape via coordinate-wise adaptivity, 2025.

\bibitem{xie2025structured}
Shuo Xie, Tianhao Wang, Sashank Reddi, Sanjiv Kumar, and Zhiyuan Li.
\newblock Structured preconditioners in adaptive optimization: A unified analysis.
\newblock {\em arXiv preprint arXiv:2503.10537}, 2025.

\bibitem{xie2025tale}
Shuo Xie, Tianhao Wang, Beining Wu, and Zhiyuan Li.
\newblock A tale of two geometries: Adaptive optimizers and non-euclidean descent.
\newblock {\em arXiv preprint arXiv:2511.20584}, 2025.

\bibitem{xie2026controlled}
Tian Xie, Haoming Luo, Haoyu Tang, Yiwen Hu, Jason~Klein Liu, Qingnan Ren, Yang Wang, Wayne~Xin Zhao, Rui Yan, Bing Su, et~al.
\newblock Controlled llm training on spectral sphere.
\newblock {\em arXiv preprint arXiv:2601.08393}, 2026.

\bibitem{xie2025mhc}
Zhenda Xie, Yixuan Wei, Huanqi Cao, Chenggang Zhao, Chengqi Deng, Jiashi Li, Damai Dai, Huazuo Gao, Jiang Chang, Liang Zhao, et~al.
\newblock mhc: Manifold-constrained hyper-connections.
\newblock {\em arXiv preprint arXiv:2512.24880}, 2025.

\bibitem{yang2025qwen3}
An~Yang, Anfeng Li, Baosong Yang, Beichen Zhang, Binyuan Hui, Bo~Zheng, Bowen Yu, Chang Gao, Chengen Huang, Chenxu Lv, et~al.
\newblock Qwen3 technical report.
\newblock {\em arXiv preprint arXiv:2505.09388}, 2025.

\bibitem{yang2026prismstructuredoptimizationanisotropic}
Yujie Yang.
\newblock Prism: Structured optimization via anisotropic spectral shaping, 2026.

\bibitem{yang2008principal}
Zhirong Yang and Jorma Laaksonen.
\newblock Principal whitened gradient for information geometry.
\newblock {\em Neural Networks}, 21(2-3):232--240, 2008.

\bibitem{yuan2024mars}
Huizhuo Yuan, Yifeng Liu, Shuang Wu, Xun Zhou, and Quanquan Gu.
\newblock Mars: Unleashing the power of variance reduction for training large models.
\newblock {\em arXiv preprint arXiv:2411.10438}, 2024.

\bibitem{yun2023riemannian}
Jihun Yun and Eunho Yang.
\newblock Riemannian sam: Sharpness-aware minimization on riemannian manifolds.
\newblock {\em Advances in Neural Information Processing Systems}, 36:65784--65800, 2023.

\bibitem{zeng2025glm}
Aohan Zeng, Xin Lv, Qinkai Zheng, Zhenyu Hou, Bin Chen, Chengxing Xie, Cunxiang Wang, Da~Yin, Hao Zeng, Jiajie Zhang, et~al.
\newblock Glm-4.5: Agentic, reasoning, and coding (arc) foundation models.
\newblock {\em arXiv preprint arXiv:2508.06471}, 2025.

\bibitem{zhang2025beyond}
Binchi Zhang, Zaiyi Zheng, Zhengzhang Chen, and Jundong Li.
\newblock Beyond the permutation symmetry of transformers: The role of rotation for model fusion.
\newblock {\em arXiv preprint arXiv:2502.00264}, 2025.

\bibitem{zhang2025adagrad}
Minxin Zhang, Yuxuan Liu, and Hayden Schaeffer.
\newblock Adagrad meets muon: Adaptive stepsizes for orthogonal updates.
\newblock {\em arXiv preprint arXiv:2509.02981}, 2025.

\bibitem{zhang2026teontensorizedorthonormalizationlayerwise}
Ruijie Zhang, Yequan Zhao, Ziyue Liu, Zhengyang Wang, Dongyang Li, Yupeng Su, Sijia Liu, and Zheng Zhang.
\newblock Teon: Tensorized orthonormalization beyond layer-wise muon for large language model pre-training, 2026.

\bibitem{zhang2024transformers}
Yushun Zhang, Congliang Chen, Tian Ding, Ziniu Li, Ruoyu Sun, and Zhiquan Luo.
\newblock Why transformers need adam: A hessian perspective.
\newblock {\em Advances in neural information processing systems}, 37:131786--131823, 2024.

\bibitem{zhao2022symmetry}
Bo~Zhao, Nima Dehmamy, Robin Walters, and Rose Yu.
\newblock Symmetry teleportation for accelerated optimization.
\newblock {\em Advances in neural information processing systems}, 35:16679--16690, 2022.

\bibitem{zhao2022symmetries}
Bo~Zhao, Iordan Ganev, Robin Walters, Rose Yu, and Nima Dehmamy.
\newblock Symmetries, flat minima, and the conserved quantities of gradient flow.
\newblock {\em arXiv preprint arXiv:2210.17216}, 2022.

\bibitem{zhao2023improving}
Bo~Zhao, Robert~M Gower, Robin Walters, and Rose Yu.
\newblock Improving convergence and generalization using parameter symmetries.
\newblock {\em arXiv preprint arXiv:2305.13404}, 2023.

\bibitem{zhao2025symmetry}
Bo~Zhao, Robin Walters, and Rose Yu.
\newblock Symmetry in neural network parameter spaces.
\newblock {\em arXiv preprint arXiv:2506.13018}, 2025.

\bibitem{zhao2024galore}
Jiawei Zhao, Zhenyu Zhang, Beidi Chen, Zhangyang Wang, Anima Anandkumar, and Yuandong Tian.
\newblock Galore: Memory-efficient llm training by gradient low-rank projection.
\newblock {\em arXiv preprint arXiv:2403.03507}, 2024.

\bibitem{zhou2020stochastic}
Dongruo Zhou, Pan Xu, and Quanquan Gu.
\newblock Stochastic nested variance reduction for nonconvex optimization.
\newblock {\em Journal of machine learning research}, 21(103):1--63, 2020.

\bibitem{zhu2024apollo}
Hanqing Zhu, Zhenyu Zhang, Wenyan Cong, Xi~Liu, Sem Park, Vikas Chandra, Bo~Long, David~Z Pan, Zhangyang Wang, and Jinwon Lee.
\newblock Apollo: Sgd-like memory, adamw-level performance.
\newblock {\em arXiv preprint arXiv:2412.05270}, 2024.

\bibitem{ziyin2025parameter}
Liu Ziyin, Yizhou Xu, Tomaso Poggio, and Isaac Chuang.
\newblock Parameter symmetry potentially unifies deep learning theory.
\newblock {\em arXiv preprint arXiv:2502.05300}, 2025.

\end{thebibliography}

\appendixtitle

\section{SOAP, SPlus, and Muon as Rotated Optimizers}\label{app: rotated_optimizers}

This appendix derives how SOAP \citep{vyas2024soap}, SPlus \citep{frans2025stable}, and Muon \citep{jordan2024muon,maes2024understanding,ahn2025dion}
reduce to the strict rotated-optimizer form
\begin{align*}
\Delta \mW_t \propto -\,\mR_t\, f_t(\mR_t^\top \mG_t;\text{state}_t),
\end{align*}
after two simplifying steps. First, we impose one-sided (left) rotations by setting the right rotation to $\mI_n$.
Second, we remove auxiliary accumulation designs to expose the core interface between a rotation $\mR_t$ and a base map $f_t$:
we replace the EMA buffers used to estimate rotations (e.g., $\text{EMA}[\mG_t\mG_t^\top]$ and $\text{EMA}[\mG_t^\top\mG_t]$)
and the first-moment momentum by instantaneous quantities. For SOAP, we still retain the Adam-style second-moment state
and make it explicit as part of $f_t$.

Throughout, for a PSD matrix $\mP=\mU\mD\mU^\top$, we write $\texttt{Eigenvectors}(\mP):=\mU$.

\subsection{SOAP \citep{vyas2024soap}}

SOAP is Adam-style elementwise normalization performed in an eigen-rotated coordinate system.
In a two-sided form, define the PSD statistics and their eigen-rotations by
\begin{align*}
\mC_{L,t} &= \beta_3\,\mC_{L,t-1} + (1-\beta_3)\,\mG_t\mG_t^\top, \qquad
\mQ_{L,t} = \texttt{Eigenvectors}(\mC_{L,t}),\\
\mC_{R,t} &= \beta_3\,\mC_{R,t-1} + (1-\beta_3)\,\mG_t^\top \mG_t, \qquad
\mQ_{R,t} = \texttt{Eigenvectors}(\mC_{R,t}),
\end{align*}
and rotate the gradient into this basis:
\begin{align*}
\mX_t := \mQ_{L,t}^\top \mG_t \mQ_{R,t}.
\end{align*}
SOAP maintains Adam-style moments in the rotated coordinates:
\begin{align*}
\mM_t &= \beta_1\,\mM_{t-1} + (1-\beta_1)\,\mX_t,\\
\mV_t^\top &= \beta_2\,\mV_{t-1} + (1-\beta_2)\,\mX_t \odot \mX_t,
\end{align*}
and applies the normalized step in the original coordinates by rotating back:
\begin{align*}
\Delta \mW_t \propto -\,\mQ_{L,t}\Big(\mM_t \oslash \sqrt{\mV_t^\top}\Big)\mQ_{R,t}^\top.
\end{align*}

Impose one-sided (left) rotation by setting $\mQ_{R,t}=\mI_n$ and define $\mR_t:=\mQ_{L,t}$.
Then $\mX_t=\mR_t^\top \mG_t$, and the update becomes
\begin{align*}
\Delta \mW_t \propto -\,\mR_t\Big(\mM_t \oslash \sqrt{\mV_t^\top}\Big),
\qquad
\mM_t = \beta_1\,\mM_{t-1} + (1-\beta_1)\,\mX_t,
\qquad
\mV_t^\top = \beta_2\,\mV_{t-1} + (1-\beta_2)\,\mX_t \odot \mX_t.
\end{align*}

Now replace the rotation-estimation buffers by instantaneous quantities and remove first-moment momentum,
\begin{align*}
\mC_{L,t} \leftarrow \mG_t\mG_t^\top,
\qquad
\mR_t \leftarrow \texttt{Eigenvectors}(\mG_t\mG_t^\top),
\qquad
\mM_t \leftarrow \mX_t,
\qquad
\mX_t := \mR_t^\top \mG_t,
\end{align*}
while retaining the second-moment state $\mV_t^\top$:
\begin{align*}
\mV_t^\top = \beta_2\,\mV_{t-1} + (1-\beta_2)\,\mX_t \odot \mX_t.
\end{align*}
Substituting $\mM_t=\mX_t$ into the one-sided update gives
\begin{align*}
\Delta \mW_t \propto -\,\mR_t\Big(\mX_t \oslash \sqrt{\mV_t^\top}\Big)
= -\,\mR_t\, f_t(\mX_t;\mV_t^\top),
\end{align*}
where:
\begin{align*}
f_t(\mX_t;\mV_t^\top) := \mX_t \oslash \sqrt{\mV_t^\top},
\qquad
\mV_t^\top = \beta_2\,\mV_{t-1} + (1-\beta_2)\,\mX_t \odot \mX_t,
\qquad
\mX_t=\mR_t^\top \mG_t.
\end{align*}
with eigen-rotations $\mR_t=\texttt{Eigenvectors}(\mG_t\mG_t^\top)$. This is our standard rotated optimization form given in \Cref{eq: intro}.

\subsection{SPlus \citep{frans2025stable}}

SPlus applies an elementwise sign map in an eigen-rotated coordinate system.
In a two-sided form, define the PSD statistics and eigen-rotations
\begin{align*}
\mC_{L,t} &= \beta_3\,\mC_{L,t-1} + (1-\beta_3)\,\mG_t\mG_t^\top, \qquad
\mQ_{L,t} = \texttt{Eigenvectors}(\mC_{L,t}),\\
\mC_{R,t} &= \beta_3\,\mC_{R,t-1} + (1-\beta_3)\,\mG_t^\top \mG_t, \qquad
\mQ_{R,t} = \texttt{Eigenvectors}(\mC_{R,t}),
\end{align*}
and maintain a momentum buffer
\begin{align*}
\mM_t = \beta_1\,\mM_{t-1} + (1-\beta_1)\,\mG_t.
\end{align*}
The two-sided SPlus update can then be written as
\begin{align*}
\Delta \mW_t \propto -\,\mQ_{L,t}\,\text{Sign}\!\big(\mQ_{L,t}^\top \mM_t \mQ_{R,t}\big)\,\mQ_{R,t}^\top.
\end{align*}

Impose one-sided (left) rotation by setting $\mQ_{R,t}=\mI_n$ and define $\mR_t:=\mQ_{L,t}$:
\begin{align*}
\Delta \mW_t \propto -\,\mR_t\,\text{Sign}(\mR_t^\top \mM_t).
\end{align*}

Remove auxiliary accumulations by replacing the rotation-estimation buffer and momentum by instantaneous quantities:
\begin{align*}
\mC_{L,t} \leftarrow \mG_t\mG_t^\top,
\qquad
\mR_t \leftarrow \texttt{Eigenvectors}(\mG_t\mG_t^\top),
\qquad
\mM_t \leftarrow \mG_t.
\end{align*}
Then $\mX_t:=\mR_t^\top \mG_t$ and
\begin{align*}
\Delta \mW_t \propto -\,\mR_t\,\text{Sign}(\mX_t)
= -\,\mR_t\, f(\mX_t),
\qquad
f(\mX):=\text{Sign}(\mX),
\qquad
\mX_t=\mR_t^\top \mG_t,
\end{align*}
which is the rotated-optimizer form (\Cref{eq: intro}) with $\mR_t=\texttt{Eigenvectors}(\mG_t\mG_t^\top)$.

\subsection{Muon/Spectral descent \citep{carlson2015preconditioned, jordan2024muon}}

The idealized update of Muon is given by:
\begin{align*}
\mM_t &= \beta\,\mM_{t-1} + (1-\beta)\,\mG_t,\\
\mM_t &= \mU_t \mS_t \mV_t^\top,\\
\Delta \mW_t &\propto -\,\mU_t\mV_t^\top.
\end{align*}

Setting rotation as $\mR_t:=\mU_t$ and removing gradient accumulation, we have 
\begin{align*}
\mG_t = \mU_t \mS_t \mV_t^\top, \quad \mX_t:= \mR_t^\top \mG_t = \mU_t^\top(\mU_t \mS_t \mV_t^\top)=\mS_t\mV_t^\top.
\end{align*}
Consider the row-normalized base optimizer projection
\begin{align*}
f_{\text{RN}}(\mX) := \sqrt{n}\,Q(\mX)^{-1}\mX,
\qquad
Q(\mX) := \mathrm{Diag}\!\big(\|\mX_{1,:}\|_2,\ldots,\|\mX_{m,:}\|_2\big).
\end{align*}
For $\mX_t=\mS_t\mV_t^\top$, row $i$ equals $\sigma_i(\mV_{:,i}^{\text{svd}})^\top$ and has norm $\sigma_i$, hence
\begin{align*}
Q(\mX_t)=\mathrm{Diag}(\sigma_1,\ldots,\sigma_m),
\qquad
Q(\mX_t)^{-1}\mX_t = \mV_t^\top.
\end{align*}
Therefore
\begin{align*}
\mR_t\, f_{\text{RN}}(\mX_t)
= \mU_t \cdot \sqrt{n}\,\mV_t^\top
\propto \mU_t\mV_t^\top,
\end{align*}
and Muon is now equivalent to rotated-optimizer (\Cref{eq: intro}) with $\mR_t=\mU_t$ and $f_t=f_{\text{RN}}$.

\section{Experiment Details}

\subsection{Sigma-MoE setup}
\label{subapp: sigma setup}
\subsubsection{Model architecture}
Sigma is an MoE model based on the DeepSeek V3. \Cref{tab:Sigma config} presents the detailed model setup. 
\begin{table}[h]
    \centering
    \caption{Hyperparameters of the Sigma MoE Architecture}
    \label{tab:Sigma config}
    \begin{tabular}{lr}
        \toprule
        \textbf{Hyperparameter} & \textbf{Value} \\
        \midrule
        \multicolumn{2}{c}{\textit{General Architecture}} \\
        \midrule
        Vocabulary Size & 128,256 \\
        Hidden Size ($d_{model}$) & 1,024 \\
        Number of Layers & 22 \\
        Max Position Embeddings & 2,048 \\
        Hidden Activation & SiLU \\
        RMS Norm $\epsilon$ & $1 \times 10^{-6}$ \\
        \midrule
        \multicolumn{2}{c}{\textit{Attention Mechanism}} \\
        \midrule
        Attention Heads & 16 \\
        Key/Value Heads & 16 \\
        Head Dimension & 64 \\
        RoPE $\theta$ & 10,000 \\
        \midrule
        \multicolumn{2}{c}{\textit{Mixture-of-Experts (MoE)}} \\
        \midrule
        Total Routed Experts & 46 \\
        Shared Experts & 1 \\
        Experts per Token & 4 \\
        MoE Intermediate Size & 512 \\
        Scoring Function & Softmax \\
        Auxiliary Loss $\alpha$ & 0.001 \\
        \bottomrule
    \end{tabular}
\end{table}

\subsubsection{Experiment setup}
\paragraph{Hyperparameters} We follow the guideline (\cref{exp: guidelines}) and perform an end-to-end search of learning rates for AdamW among $[2.5E-4, 5E-4, 7.5E-4, 1E-3]$. $5E-4$ is chosen for its best performance. We also select $\beta_1=0.9$, $\beta_2=0.95$. For all \name{} families, Eigen families, as well as Muon/Dion \arosink{} we follow \cref{sec:hyperparam} to transfer the optimal AdamW learning rate. We use $\beta_1=0.95$ (following standard optimal configurations for Muon) for \name{} families, Eigen families, as well as Muon/Dion. We use $5$ steps of Sinkhorn, and Cholesky QR (\cref{sec:chol_qr}) with $\epsilon=1E-7$ for \arosink{}. For hybrid setup, AdamW will also be used to update 1D parameters and the last layer of the network, where we reuse the tuned AdamW hyperparameters. For all methods, the cosine scheduler will decrease the learning rate to the $10\%$ of its original value.  We use $0.1$ for weight decays. 

\paragraph{Training setup}
A100 NVIDIA GPUs with DDP are used to train Sigma-MoE. We use a sequence length of $2$K and a total batch size per step of $4M$ tokens. We also enable gradient checkpointing for memory efficiency. A cosine scheduler with $600$ linear warm-up steps is chosen.

\paragraph{Optimizer Setup}
Apart from \name{} (Full Model), we assume the hybrid setup for all optimizers. For \name{}, we use SCQR (\cref{sec:chol_qr}) as the main solver for the rotation for efficiency. For Eigen-Sinkhorn and Dion, we rely on the default Pytorch QR solver. 

\subsection{Sigma-MoE additional analysis}
\label{subapp: Sigma additional results}
\Cref{fig:sigma training curve} shows late-stage training loss trajectories for different optimizers. The curves separate cleanly into three groups: (i) unrotated baselines, (ii) eigen-rotation methods (eigen-rotated updates paired with different base rules), and (iii) the ARO family. This mirrors the qualitative pattern reported in \Cref{subsec: universal effective} (Finding~1), and shows that the same grouping persists at larger scale. Notably, the ARO rotation used here is not derived from any standard norm or metric (e.g., spectral-norm-related constructions). Its strong performance therefore provides indirect evidence that the gains of recent matrix optimizers are not solely explained by a particular prescribed norm. Instead, these results are consistent with the view that much of the benefit comes from the rotation mechanism itself, and that spectral norm-derived methods may work partly because they implicitly induce a favorable rotation. Overall, the Sigma-MoE results reinforce Finding~2 in \Cref{subsec: universal effective} under scaling.

Moreover, \Cref{fig: sigma training curve at early stage} shows the early-stage loss trajectories for different optimizers. During the first $3$B tokens, the \name{} family underperforms AdamW and orthogonalization-based methods, often by a substantial margin. This indicate that the advantages of \name{} emerge over long training horizons and may be missed by short-horizon evaluations that focus primarily on early-stage behavior.

\begin{figure}
    \centering
    \begin{minipage}[t]{0.3\textwidth}
    \includegraphics[width=1\linewidth]{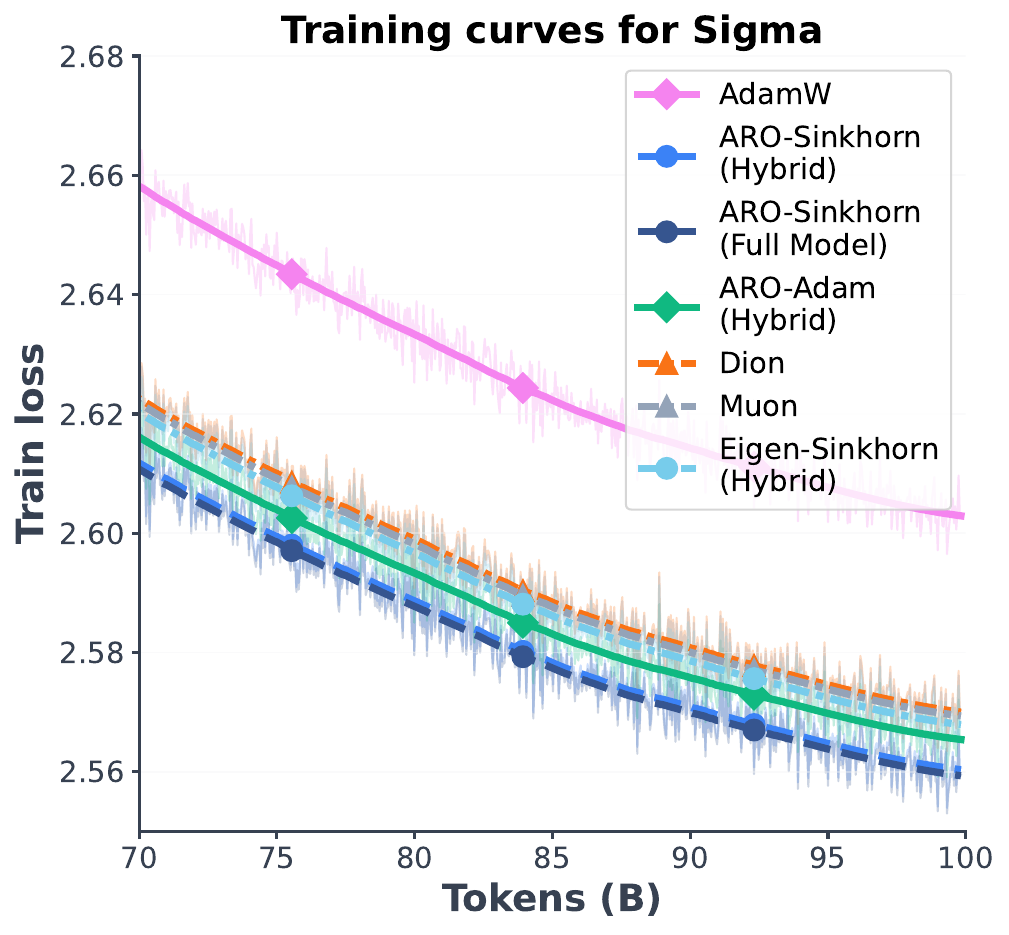}
    \caption{(Sigma-MoE-2B) The training curve of different optimizers (smoothed and actual) }
    \label{fig:sigma training curve}
    \end{minipage}\hfill
    \begin{minipage}[t]{0.3\textwidth}
        \includegraphics[width=1\linewidth]{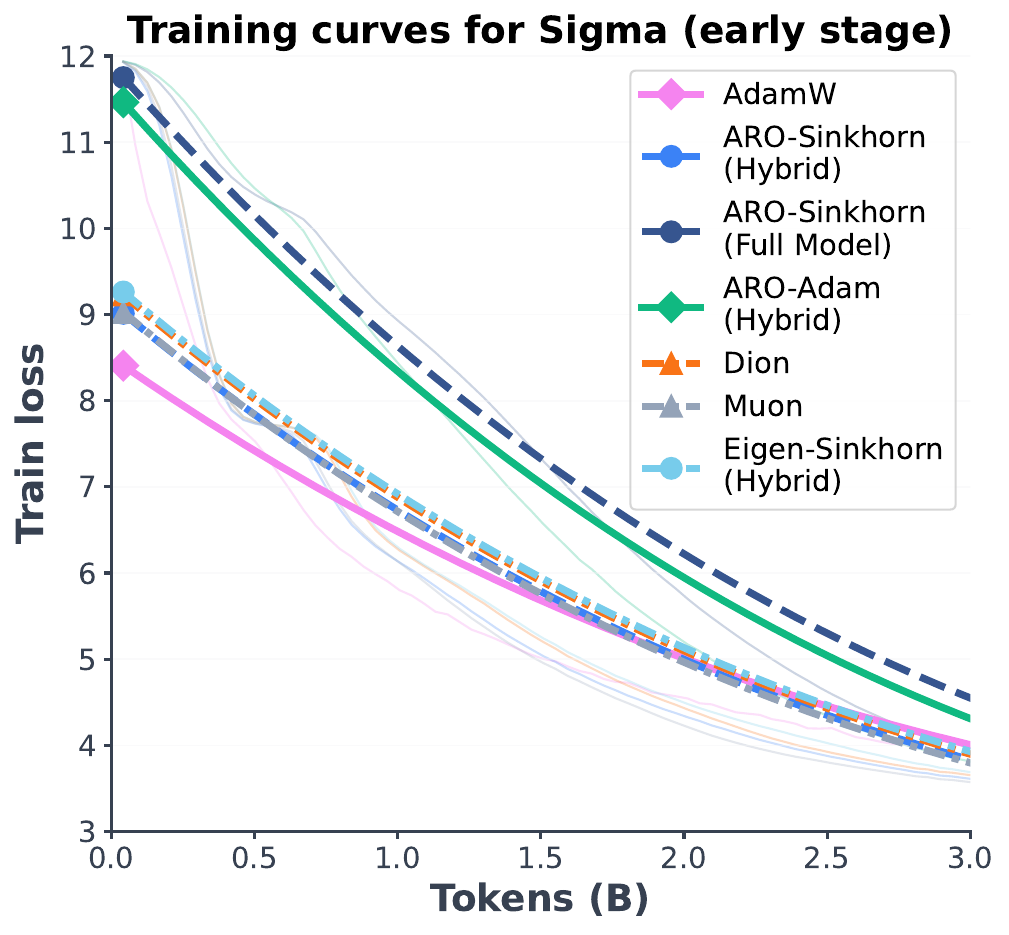}
        \caption{(Sigma-MoE-2B) The training curves at early stage (smoothed and actual).}
        \label{fig: sigma training curve at early stage}
    \end{minipage}\hfill
    \begin{minipage}[t]{0.3\textwidth}
        \includegraphics[width=1\linewidth]{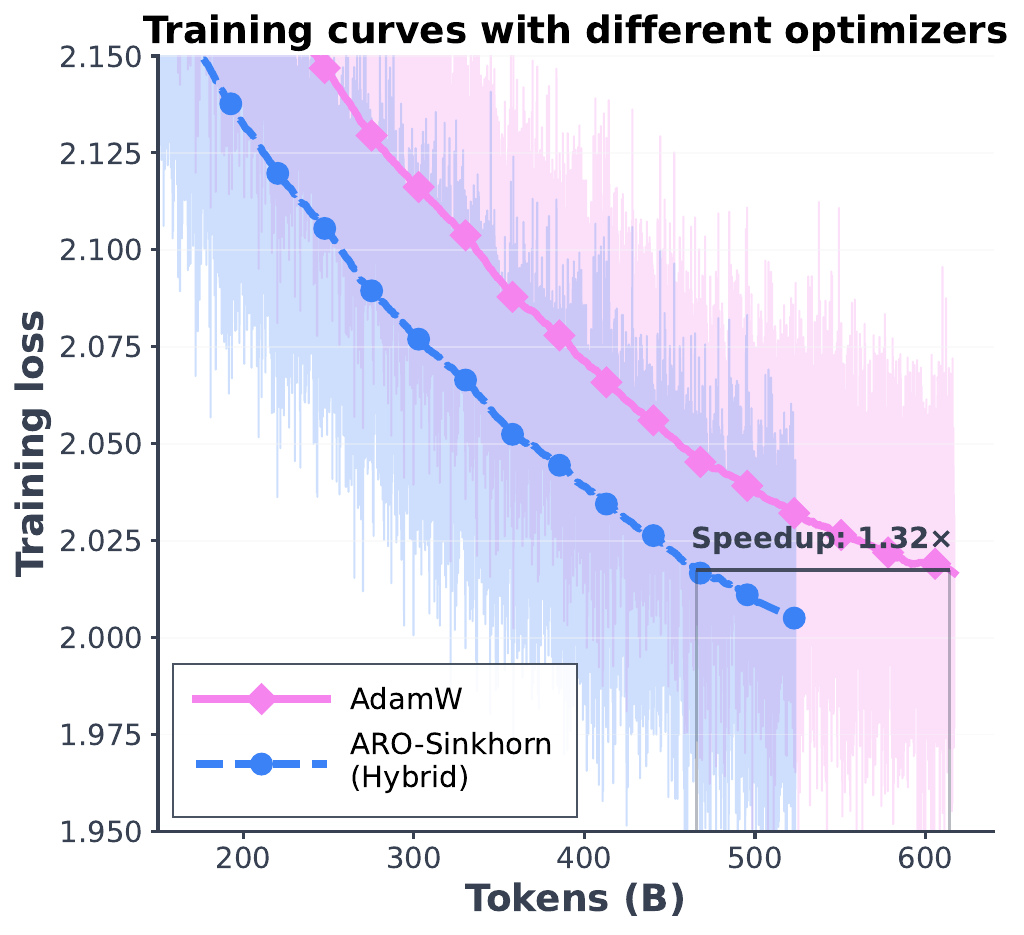}
        \caption{(Qwen 8B) The training curve of \name{} and AdamW (smoothed and actual)}
        \label{fig: Qwen training curve}
    \end{minipage}
\end{figure}
\subsection{Qwen3 8B setup}
\label{subapp: Qwen3 8B setup}
\paragraph{Architecture}
We use a Qwen3 8B architecture, with total parameter count at around 8.2 billion.  See \Cref{tab: qwen3_run_arch}.

\begin{table}[h]
    \centering
    \caption{Hyperparameters of the Qwen3 Architecture}
    \label{tab: qwen3_run_arch}
    \begin{tabular}{lr}
        \toprule
        \textbf{Hyperparameter} & \textbf{Value} \\
        \midrule
        \multicolumn{2}{c}{\textit{General Architecture}} \\
        \midrule
        Vocabulary Size & 151{,}936 \\
        Hidden Size  & 4{,}096 \\
        FFN Intermediate Size & 12{,}288 \\
        Number of Layers & 36 \\
        Max Position Embeddings & 4{,}096 \\
        Hidden Activation & SiLU \\
        RMS Norm $\epsilon$ & $1 \times 10^{-6}$ \\
        \midrule
        \multicolumn{2}{c}{\textit{Attention Mechanism}} \\
        \midrule
        Attention Heads & 20 \\
        Key/Value Heads & 20 \\
        Head Dimension & 128 \\
        RoPE $\theta$ & 1{,}000{,}000 \\
        \bottomrule
    \end{tabular}
\end{table}

\paragraph{Hyperparameters and optimizer setup}
Since end-to-end learning rate tuning for AdamW is not feasible on 8B + overtrain setups, we follow the 8B \citep{grattafiori2024llama} and use learning rate $3\times 10^{-4}$ for AdamW. For \name{} and Muon, we follow the same procedure to transfer AdamW learning rates \cref{sec:hyperparam}. For AdamW, we use $\beta_1 = 0.9$, $\beta_2 = 0.95$, and $\beta_1=0.95$ for Muon and \arosink{}.  For \arosink{}, we use hybrid setup with SCQR solver.  We use $5$ Sinkhorn steps. We also update ARO's rotation matrix at every step using shifted Cholesky QR (\cref{sec:chol_qr}) with $\epsilon=1E-7$. We use $0.1$ for decoupled weight decay across all methods.

\paragraph{Training setup}
We leverage B200 NVIDIA GPUs with a sequence length of $4K$, and a total token batch size per step of $14M$ tokens. We do not use gradient check-pointing.  The model is trained with $46000$ steps.  We also adopt the cosine scheduler with $400$ linear warm-up steps ($0.8\%$ of total training steps, comparable to $0.7\%$ used in \citep{grattafiori2024llama}); and the learning rate is decreased to $10\%$ in the end of training.

\subsection{Qwen3 8B additional results}
\label{subapp: Qwen3 8B additional results}
\Cref{fig: Qwen training curve} plots the training curve of \arosink{} and AdamW. The speedup factor are computed based on the smoothed loss through rolling averages with a smooth window $300$. At the end stage, \arosink{} can achieves a speedup of $1.32\times$, which corresponds well with the speedup from validation curve.

\section{Rotation as a stability-speed tradeoff}
\label{app: role of rotation}

\subsection{Setups and assumptions}
\label{subapp: denoising setups and assumptions}

In addition to the setups introduced in \cref{subsubsec: discussion denoising: setups}, we require the following assumptions for the theoretical analysis.

\paragraph{Noise distribution and geometry}
Since we are in a stochastic setting, we first introduce an assumption on the geometry of mini-batch noise distribution.

\begin{assumption}[Noise distribution]
We assume the noise $\noise\in\R^{m\times n}$ follows a zero-mean multivariate Gaussian distribution with shared column covariance $\cov$. Furthermore, we assume $\cov$ and $\mEG\mEG^T$ share the same eigenspace, i.e., $\EVD(\cov) = \EVD(\mEG\mEG^T)$.
\label{assumpt: noise geometry}
\end{assumption}

\begin{remarks}
\begin{remark}[Noise geometry justification]
    For the noise geometry, \cite{mori2022power, wang2023theoretical} show that the mean of mini-batch noise under uniform sampling is $0$ and its covariance can be closely approximated by $c\mF$, where $\mF$ is the empirical Fisher matrix and $c$ is a constant. Although $\mF$ should be computed by flattening all parameters, \cite{gong2025towards} shows that $\mEG\mEG^T$ is the optimal approximation to empirical Fisher matrix under the identical block-diagonal structure assumption and the Frobenius norm. This is also the strcutral assumption behind the Muon optimizer \citep{gong2025towards}. In \cref{assumpt: noise geometry}, we do not require such a strong claim; instead, we assume the noise covariance and $\mEG\mEG^T$ share the same eigenspace but can differ in eigenvalues.
\end{remark}
\end{remarks}

\paragraph{Low signal-to-noise ratio (SNR) regime}
We also define a low \emph{signal-to-noise ratio} (SNR) scenario, where noise strength dominates over the gradient signal. Depending on the data and experimental setup, this can occur at mid or late stages of training. Later in \cref{subapp: rotation optimal variance reduction}, we show through a synthetic experiment, that the following analysis may not be sensitive to this assumption. We leave the quantitative analysis to future work. 

\begin{assumption}[Noise dominated regime (low SNR)]
    Under the low-SNR regime, we assume that for a gradient matrix $\mEG$, noise covariance $\cov$, and rotation matrix $\mR$, the following approximation holds:
    \begin{align}
        \frac{(\mR^T\mEG_i)_j}{\sqrt{2(\mR^T\cov\mR)_{jj}}}\approx 0,
    \end{align}
    where $\mEG_i$ is the $i^{\text{th}}$ column of $\mEG$. Furthermore, for any differentiable function $h(x)$, when $x\approx 0$, we ignore terms beyond first order in the Taylor expansion.
    \label{assump: low SNR}
\end{assumption}

\begin{remarks}
    \begin{remark}[SNR interpretation]
        Since the mean of mini-batch noise is $0$ under uniform sampling, its strength is characterized by the covariance $\cov$. The matrix $\mR^T\cov\mR$ is the column covariance of the rotated noise $\mR^T\noise_i$. Therefore, the above ratio compares (elementwise) the strength of the rotated gradient to the marginal variance of the rotated noise.
    \end{remark}
\end{remarks}

\paragraph{Relative eigenvalue ordering}
From \cref{assumpt: noise geometry}, the eigenvalues of $\cov$ and $\mEG\mEG^T$ may differ. It is still useful to describe their relative ordering via the following notion.

\begin{definition}[Ratio Co-Monotonic (RCM)]
For two positive sequences $\lambda_1, \ldots, \lambda_m$ and $D_1, \ldots, D_m$, we say they satisfy RCM if $D_i>D_j$ implies $\frac{\lambda_i}{D_i}> \frac{\lambda_j}{D_j}$.
    \label{def: ratio co-monotonic}
\end{definition}
A direct consequence of RCM is that when $D_i>D_j$, it implies that $\lambda_i > \lambda_j$.

\subsection{Failure of Polar projection scheme}
\label{subapp: failure of polar}

\paragraph{Monotonic non-decreasing property (noise-free objective)}
It is straightforward to show that the alternating projection scheme has a monotonic non-decreasing property for the corresponding (noise-free) objective. First, the global maximizer for step (2) is
\begin{align*}
    \mR^* \;=\; \argmax_{\mR_t\in\mathcal{O}} \tr(\mNG_t\mA_{t-1}\mR_t^\top)
    \;=\; \polar(\mNG_t\mA_{t-1}).
\end{align*}
This guarantees that $\mR^*$ will not decrease the objective $\tr(\mNG_t f(\mNG_t^\top\mR_{t-1})\mR^{*\top})$. Next, we can rewrite the objective as $\tr(\mR^{*\top}\mNG_t f(\mNG_t^\top\mR_{t-1}))$. Since $f$ is defined as the dual norm solution, the optimal $f$ by definition is
\begin{align*}
    f(\mR^{*\top}\mNG_t)
    \;=\;
    \argmax_{\|\mZ_t\|\leq 1} \left\langle \mR^{*\top}\mNG_t,\; \mZ_t \right\rangle,
\end{align*}
which is exactly equivalent to substituting $\mR^*$ into $f$. Therefore, step (1) also guarantees a non-decreasing objective. Overall, multiple steps of this projection scheme will not decrease the objective we aim to maximize.

\paragraph{Empirical objective and alignment score}
Under the stochastic setup, the empirical objective $\tilde{\mathcal{J}}$ is no longer equivalent to the alignment score $\alignscore$. However, as mentioned in the beginning of \cref{subsec: discussion denoising}, $\alignscore$ controls the alignment w.r.t true descent direction $\mEG$ and training stability. This mismatch results in a crucial consequence that maximizing $\tilde{\mathcal{J}}$ may not lead to a desirable behavior. Intuitively, maximizing $\tilde{\mathcal{J}}$ pushes the update $\Delta \mW$ to align with the target $\mM$ or $\mNG$. When $\mM$ or $\mNG$ is misaligned with the ground-truth $\mEG$, the resulting update can cause negative $\alignscore$, which harms training.


To formalize this behavior, we first define \emph{misalignment}.

\begin{definition}[Momentum misalignment]
    At training time $t$, we say momentum misaligns with the full-batch gradient $\mEG$ when $\tr(\mEG^T\mM)<0$. Otherwise, we say momentum aligns with $\mEG$.
    \label{def: momentum misalignment}
\end{definition}

Next, we define the angle between $\mEG$ and $\mM$ as $\theta_{GM}$. Using this angle, we can define a \emph{danger region} boundary: when $\theta_{GM}$ is larger than the boundary angle, $\mM$ lies inside the danger region and the corresponding $\Delta \mW$ produced by the polar solution is guaranteed to be misaligned with $\mEG$. This behavior is described by the following theorem.

\begin{theorem}
    Assume the parameter update takes the form $\Delta \mW \propto \mR f(\mR^T\mM)$ (\cref{eq: ARO update}), and define
    \begin{align*}
        \cos(\theta_{GM})
        \;=\;
        \frac{\tr(\mEG^T\mM)}{\|\mEG\|_F\|\mM\|_F}
    \end{align*}
    and
    \begin{align*}
        k
        \;=\;
        \frac{\tr(\mM f(\mR^T\mM)^T\mR^T)}{\|\mM\|_F\|f(\mR^T\mM)^T\|_F}.
    \end{align*}
    Then $\Delta \mW$ is guaranteed to misalign with $\mEG$ when $\cos(\theta_{GM})<-\frac{1}{\sqrt{1+k^2}}$ and momentum misaligns with $\mEG$.
    \label{thm: failure of polar}
\end{theorem}

This theorem implies that when maximizing the empirical objective (i.e., the numerator of $k$), the value of $k$ increases, which decreases $\frac{1}{\sqrt{1+k^2}}$. As a result, the condition $\cos(\theta_{GM})<-\frac{1}{\sqrt{1+k^2}}$ becomes easier to satisfy, meaning it is easier for $\mM$ to fall into the danger region and for $\Delta \mW$ to become guaranteed misaligned, leading to instantaneous loss increase.
Note that this theorem characterizes a \emph{worst-case} scenario: it does not imply that if $\theta_{GM}$ is outside the danger region then $\Delta \mW$ must align with $\mEG$. In general, a larger danger region makes misalignment more likely. The same argument also holds when replacing $\mM$ by $\mNG$. A detailed proof is provided in \cref{subapp: proof of failure of polar}.

\subsection{Rotation as an optimal $\var(\alignscore)$ and alignment magnitude minimizer}
\label{subapp: rotation optimal variance reduction}

The alignment score $\alignscore$ measures how $\Delta \mW$ aligns with the ground truth $\mEG$:
\begin{align}
    \alignscore
    \;=\;
    \left\langle\mEG, \mR f(\mR^T\mNG)\right\rangle.
    \label{eq: alignment score with noise G}
\end{align}
This metric determines whether the current update increases or decreases the loss under an infinitesimal learning rate. Therefore, $\var(\alignscore)$ controls the stability of the training. Empirically (see \cref{subsubsec: discussion denoise: rotation as variancec reduction}), Eigen rotation significantly reduces the $\var(\alignscore)$, which correlates strongly with improved performance. Below we validate this stability improvement mechanism theoretically under the stochastic setup.

\begin{theorem}[Optimal $\var(\alignscore)$ reduction]
Under \cref{assump: low SNR} (low-SNR), \cref{assumpt: noise geometry} and $f=\sign$, let $\bm{\lambda}=[\lambda_1, \ldots, \lambda_m]$ and $\mD=[D_1,\ldots, D_m]$ be the positive, non-repeated, descending-ordered eigenvalues of $\mEG\mEG^T$ and $\cov$, respectively. We further assume the column noises are mutually independent. Define the variance of alignment score as
\begin{align*}
    \var(\alignscore) = \E[\alignscore^2] - \E[\alignscore]^2.
\end{align*}
Then,
\begin{align}
    \mR^* \;=\; \EVD(\mEG\mEG^T)
    \label{eq: optimal align score variance reduction rotation}
\end{align}
is a local minimizer of the above variance when $\bm{\lambda}$ and $\mD$ satisfy RCM.
\label{thm: optimal align score variance reduction}
\end{theorem}

\paragraph{A limitation of $\var(\alignscore)$ minimization}
Variance reduction comes with an undesirable property: Eigen rotation not only minimizes the variance, but it also minimizes the expected alignment score, resulting in overly conservative update.

\begin{theorem}[Rotation as alignment magnitude minimization]
Under the assumptions of \cref{thm: optimal align score variance reduction}, define the expected alignment magnitude as
\begin{align*}
    \dot{\loss}
    \;=\;
    \E\left[\left\langle\mEG,\; \mR\sign(\mR^T\mNG)\right\rangle\right].
\end{align*}
Then,
\begin{equation}
    \mR^* \;=\; \EVD(\mEG\mEG^T)
\end{equation}
is a local minimizer of $\dot{\loss}$ under RCM.
\label{thm: rotation as loss rate minimizer}
\end{theorem}

\paragraph{Sensitivity to dominated noise assumption}
One of the key assumptions in the above theorem is \cref{assump: low SNR}, where high order terms during the analysis can be ignored. However, in our MNIST setup, the $\var(\alignscore)$ reduction effect seem to persist even at the early stage of training. We conducted a simple synthetic experiment to show that eigen-rotation under relatively high signal-to-noise ratio can still significantly reduce the $\var(\alignscore)$ and $\E[\alignscore]$, although it may not be the minimizer. 

We simulate $\mEG$ and noise covariance in a structured manner. First, we build
$$
D_i \propto i^{-\alpha_D},\qquad r_i \propto i^{-\alpha_{\text{ratio}}},\qquad \lambda_i = D_i r_i,\quad i=1,\dots,m,
$$
where $\alpha_D$ and $\alpha_{\text{ratio}}$ are two hyperparameters, with 1 and 0.1, respectively. 

Next, sample random orthonormal basis $\mU$ and $\mV$,respectively, and define
$$
\mEG = s_g\,\mU\operatorname{diag}(\sqrt{\lambda})\mV^\top,
$$
where $s_g$ controls the signal strength. 

For noise, define
$$
\mSigma=s_\sigma\,\mU \operatorname{diag}(D)\mU^\top,
$$
where $s_{\sigma}$ controls the noise strength.
Then, we will sample one $\mEG$ with $s_g=1$, followed by $1000$ $\mNG$ with $s_\sigma=0.1$. This mimic the scenario where signal strength is higher than noise magnitude.

Then, we consider the following baselines for rotation: (1) eigen-rotation $\EVD(\mEG\mEG^\top)$; (2) ARO rotation with identity initialization, where $\mR_{t-1}$ in \cref{eq: gf} are set to $\mI$; (3) ARO rotation with perturbed eigen-rotation, where $\mR_{t-1}$ is set to the perturbed eigen-rotation. (3) aims to mimic the actual ARO rotation, where $\mR_{t-1}$ is from the previous solution, which we argue is more closed to the eigen-rotation than identity. 

To check if they reach the minimum, we test the above rotations against (1) $1000$ random rotations (top row of \cref{fig: variance mean synthetic exp}), and (2) $1000$ random rotations close to eigen-rotation (bottom row of \cref{fig: variance mean synthetic exp}). \Cref{fig: variance mean synthetic exp} shows the histogram of comparisons. We can observe even under the signal-to-noise ratio of $10$, eigen-rotation still achieves close to minimum of both $\var(\alignscore)$ and $\E[\alignscore]$. The two ARO variations achieves higher $\E[\alignscore]$ but without increasing the $\var(\alignscore)$ too much, especially when $\mR_{t-1}$ is initialized to be the perturbed eigen-rotation. This can be checked by inspecting the relative position w.r.t the histogram in the bottom row. 

\begin{figure}
    \centering
    \begin{minipage}{0.45\textwidth}
        \includegraphics[width=\linewidth]{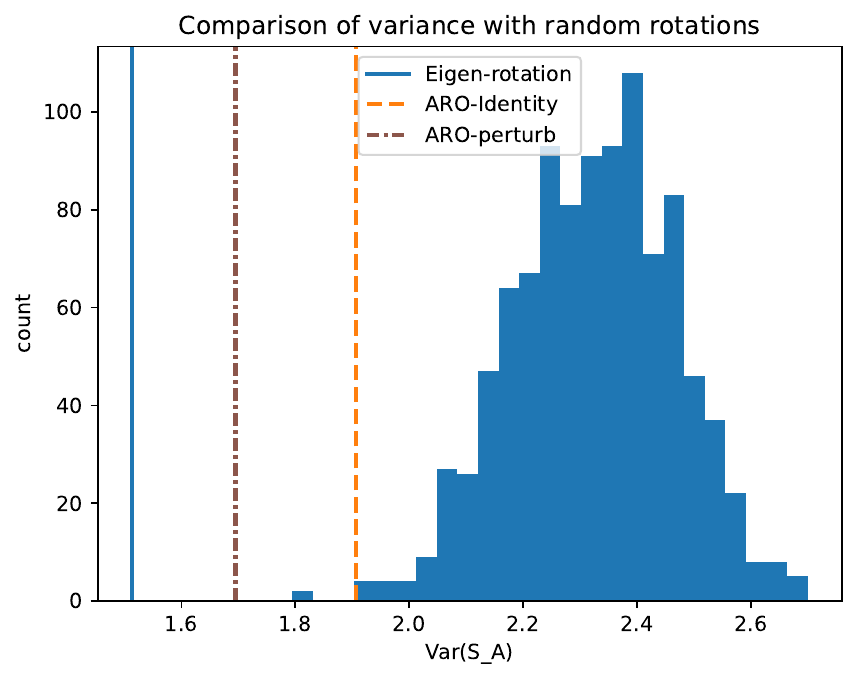}
        \label{fig: synthetic random variance}
    \end{minipage}\hfill
    \begin{minipage}{0.45\textwidth}
        \includegraphics[width=\linewidth]{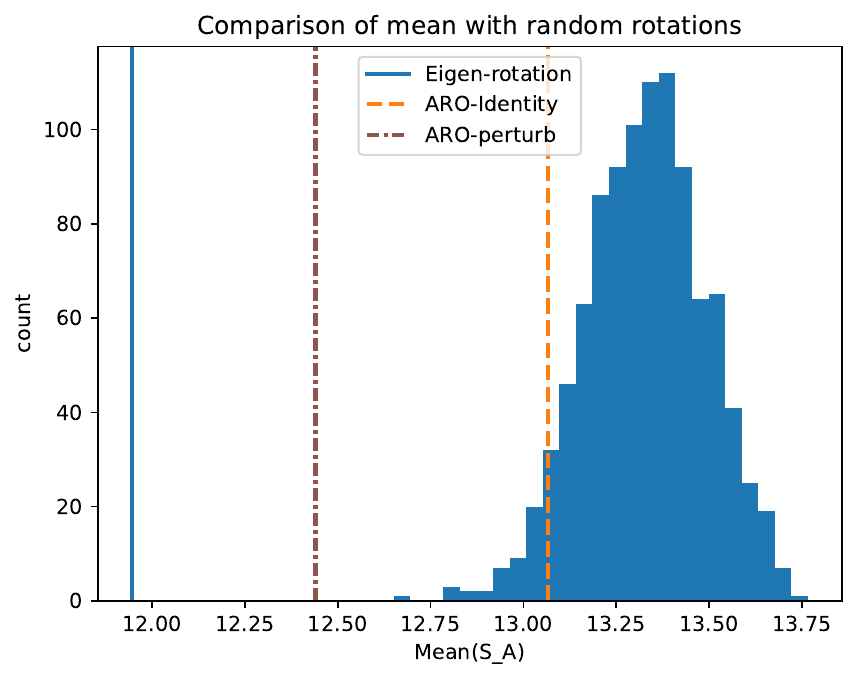}
        \label{fig: synthetic random mean}
    \end{minipage}\\
    \begin{minipage}{0.45\textwidth}
        \includegraphics[width=\linewidth]{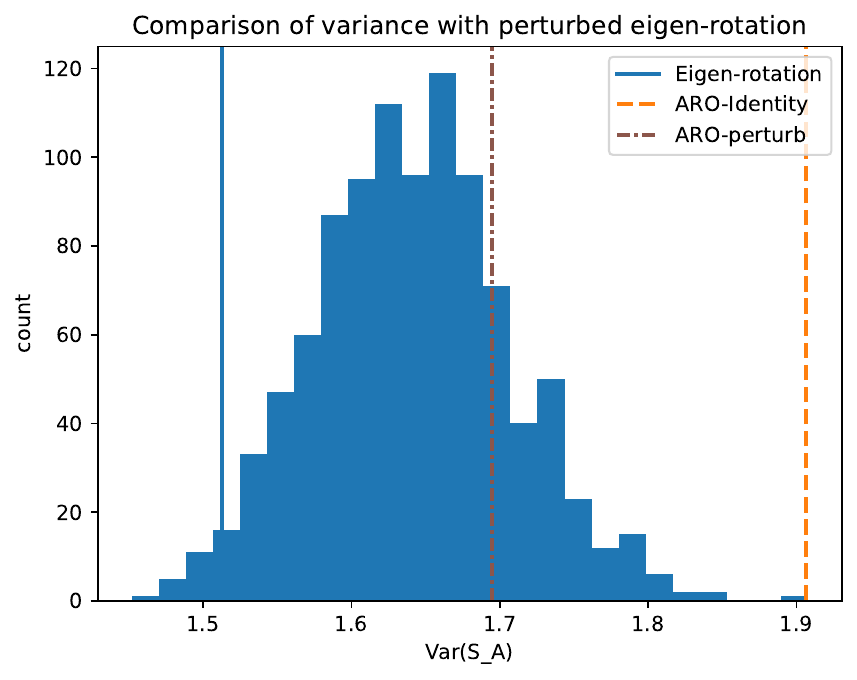}
        \label{fig: synthetic local variance}
    \end{minipage}\hfill
    \begin{minipage}{0.45\textwidth}
        \includegraphics[width=\linewidth]{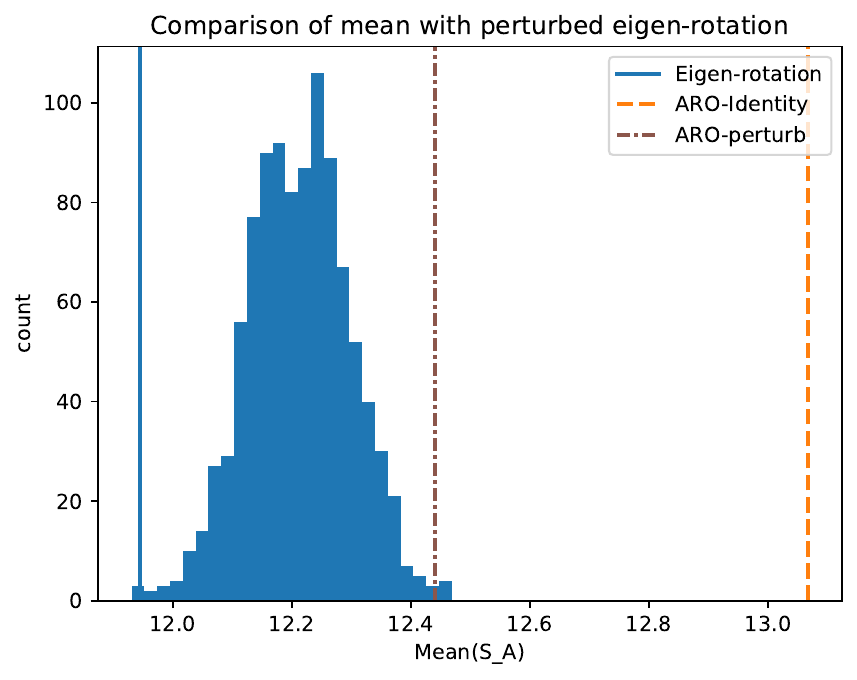}
        \label{fig: synthetic local mean}
    \end{minipage}
    \caption{The histogram shows the $\var(\alignscore)$ and $\E[\alignscore]$ under the synthetic samples of $\mEG$ and $\mNG$. }
    \label{fig: variance mean synthetic exp}
\end{figure}

\subsection{MNIST experiment setup}
\label{subapp: MNIST setup}

This section outlines the MNIST experimental setup.

\paragraph{Dataset and model}
We use the original MNIST dataset and randomly truncate the training set to $20000$ images to make computation of the full-batch gradient $\mEG$ feasible. We use a two-layer MLP with $256$ hidden units and ReLU \citep{agarap2018deep} as the classifier.

\paragraph{Hyperparameters}
For all optimizers, we follow the fairness guidelines in \cref{exp: guidelines}. Specifically, we use the same learning rate (lr$=10^{-3}$) throughout, and follow \cref{sec:hyperparam} to align their update RMS norms. We use $\beta_1=0.95$ and $\beta_2=0.99$ for all optimizers. We train the network for $3000$ steps with batch size $512$ to ensure full convergence. We also use a cosine learning-rate scheduler (similar to LLM training) that decays to $10\%$ of the initial learning rate, with $100$ steps of linear warm-up.
\section{Proofs of \cref{app: role of rotation}}
\label{app: proofs denoise}

\subsection{Proof of \cref{thm: failure of polar}}
\label{subapp: proof of failure of polar}

\begin{proof}
Let
$$
\mV \;:=\; \mR f(\mR^T\mM).
$$
Since $\Delta \mW \propto \mV$, it suffices to show $\langle \mEG,\mV\rangle < 0$ under the conditioned mentioned in the theorem. 

Before the proof details, we will summarize the logic and strategy behind it. We aims to first separate and simplify the alignment $\langle\mEG, \mV\rangle$ into two components: one is orthogonal to momentum, and one is not. Then, we can derive an upper bound of $\langle\mEG, \mV\rangle$ through Cauchy inequality. Then, we analyzie the condition when this upper bound will be negative, such that the original alignment $\langle\mEG,\mV\rangle$ is guaranteed to be negative.

Define the scalar
$$
\alpha \;:=\; \frac{\tr(\mEG^T\mM)}{\|\mM\|_F^2},
$$
and the orthogonal component
$$
\mEG_{\perp} \;:=\; \mEG - \alpha \mM,
\qquad\text{so that}\qquad
\mEG = \alpha \mM + \mEG_{\perp}.
$$
Then $\mEG_{\perp}$ is orthogonal to $\mM$:
\begin{align*}
\tr(\mEG_{\perp}^T\mM)
&= \tr\!\big((\mEG-\alpha\mM)^T\mM\big) \\
&= \tr(\mEG^T\mM) - \alpha\,\tr(\mM^T\mM) \\
&= \tr(\mEG^T\mM) - \frac{\tr(\mEG^T\mM)}{\|\mM\|_F^2}\,\|\mM\|_F^2 \\
&= 0.
\end{align*}

Now, by substitution, we have
\begin{align*}
\langle \mEG,\mV\rangle
&= \langle \alpha\mM + \mEG_{\perp}, \mV\rangle \\
&= \alpha \langle \mM,\mV\rangle + \langle \mEG_{\perp},\mV\rangle.
\end{align*}

We first rewrite $\langle \mM,\mV\rangle$:
\begin{align*}
\langle \mM,\mV\rangle
&= \tr(\mM^T \mR f(\mR^T\mM)) \\
&= \tr\!\big(\mM f(\mR^T\mM)^T \mR^T\big).
\end{align*}
Hence, with the definition of $k$,
$$
\langle \mM,\mV\rangle
= \|\mM\|_F\,\|f(\mR^T\mM)^T\|_F\, k.
$$

Next, bound the orthogonal term using Cauchy--Schwarz:
$$
\langle \mEG_{\perp},\mV\rangle \le \|\mEG_{\perp}\|_F \,\|\mV\|_F
= \|\mEG_{\perp}\|_F\,\|f(\mR^T\mM)^T\|_F.
$$

So we obtain the upper bound
\begin{align*}
\langle \mEG,\mV\rangle
\le \|f(\mR^T\mM)^T\|_F \Big(\alpha \|\mM\|_F k + \|\mEG_{\perp}\|_F\Big).
\end{align*}

Now compute $\|\mEG_{\perp}\|_F$ explicitly. Expanding the square,
\begin{align*}
\|\mEG_{\perp}\|_F^2
&= \|\mEG - \alpha \mM\|_F^2 \\
&= \tr\!\big((\mEG-\alpha\mM)^T(\mEG-\alpha\mM)\big) \\
&= \|\mEG\|_F^2 - 2\alpha\,\tr(\mEG^T\mM) + \alpha^2 \|\mM\|_F^2 \\
&= \|\mEG\|_F^2 - 2\frac{\tr(\mEG^T\mM)}{\|\mM\|_F^2}\tr(\mEG^T\mM)
   + \Big(\frac{\tr(\mEG^T\mM)}{\|\mM\|_F^2}\Big)^2 \|\mM\|_F^2 \\
&= \|\mEG\|_F^2 - \frac{\tr(\mEG^T\mM)^2}{\|\mM\|_F^2}.
\end{align*}
Let $c:=\cos(\theta_{GM})=\frac{\tr(\mEG^T\mM)}{\|\mEG\|_F\|\mM\|_F}$. Then
$$
\frac{\tr(\mEG^T\mM)^2}{\|\mM\|_F^2}
= \frac{(\|\mEG\|_F^2\|\mM\|_F^2 c^2)}{\|\mM\|_F^2}
= \|\mEG\|_F^2 c^2,
$$
so
$$
\|\mEG_{\perp}\|_F^2 = \|\mEG\|_F^2(1-c^2)
\qquad\Rightarrow\qquad
\|\mEG_{\perp}\|_F = \|\mEG\|_F\sqrt{1-c^2}.
$$
Also
$$
\alpha = \frac{\tr(\mEG^T\mM)}{\|\mM\|_F^2}
= \frac{\|\mEG\|_F\|\mM\|_F c}{\|\mM\|_F^2}
= \frac{\|\mEG\|_F}{\|\mM\|_F}c.
$$
Plugging these into the upper bound yields
\begin{align*}
\langle \mEG,\mV\rangle
&\le \|f(\mR^T\mM)^T\|_F
\Big(\|\mEG\|_F c\, k + \|\mEG\|_F\sqrt{1-c^2}\Big) \\
&= \|\mEG\|_F\,\|f(\mR^T\mM)^T\|_F \Big(kc + \sqrt{1-c^2}\Big).
\end{align*}

Next, we verify the condition when $kc+\sqrt{1-c^2}<0$ so that $\langle\mEG, \mV\rangle<0$. 

Assume the momentum mis-aligns with $\mEG$, i.e. $c=\cos(\theta_{GM})<0$.
If
$$
c < -\frac{1}{\sqrt{1+k^2}},
$$
then squaring both sides (both are negative) gives
$$
c^2 > \frac{1}{1+k^2}
\;\Rightarrow\;
(1+k^2)c^2 > 1
\;\Rightarrow\;
k^2c^2 > 1-c^2
\;\Rightarrow\;
|kc| > \sqrt{1-c^2}.
$$
Since $c<0$ and $k\ge 0$ since $f$ is defined by dual norm (\cref{subsec: normed steepest descent}), we have $|kc|=-kc$, hence
$$
-kc > \sqrt{1-c^2}
\;\Leftrightarrow\;
kc+\sqrt{1-c^2}<0.
$$
Therefore,
$$
\langle \mEG,\mV\rangle
\le \|\mEG\|_F\,\|f(\mR^T\mM)^T\|_F \Big(kc + \sqrt{1-c^2}\Big) < 0.
$$
This implies $\langle \mEG, \Delta\mW\rangle < 0$ (,
i.e. $\Delta \mW$ is guaranteed to mis-align with $\mEG$ under the stated condition.
\end{proof}

\subsection{Proof of \cref{thm: optimal align score variance reduction}}
\label{subapp: proof of minimium variance reduction}

Before proving this theorem, we need to introduce a useful lemma that can compute the analytic form of the a specific expectation:
\begin{lemma}[Arcsine law for signs of a bivariate normal]
\label{lem:arcsine}
Let $(X,Y)$ be jointly Gaussian with
$$
\E[X]=\E[Y]=0,\qquad \Var(X)=\Var(Y)=1,\qquad \Corr(X,Y)=\rho\in[-1,1].
$$
Then
$$
\E[\sign(X)\sign(Y)] = \frac{2}{\pi}\arcsin(\rho).
$$
\end{lemma}
    
\begin{proof}
Because $(X,Y)$ is standard bivariate normal with correlation $\rho$, throught change of variable, we can define a standard normals $(U,V)\sim\mathcal{N}(\mathbf{0},\mI_2)$ such that
$$
X = U,\qquad Y = \rho U + \sqrt{1-\rho^2}\,V.
$$

Due to the property of sign, there are limited possible combinations of $\sign(X)\sign(Y)$, and we can examine each of them and compute their corresponding expectations. 


Then
$$
\sign(X)\sign(Y)=
\begin{cases}
+1,& \text{if } (X>0 \text{ and }Y>0)\text{ or }(X<0\text{ and }Y<0),\\
-1,& \text{if } (X>0 \text{ and }Y<0)\text{ or }(X<0\text{ and }Y>0).
\end{cases}
$$
Because $\Pr(X=0\text{ or }Y=0)=0$ for a continuous normal, we ignore this scenario.

Hence
\begin{align*}
\E[\sign(X)\sign(Y)]
&= (+1)\Pr(X>0,Y>0) + (+1)\Pr(X<0,Y<0)\\
&\qquad +(-1)\Pr(X>0,Y<0) + (-1)\Pr(X<0,Y>0).
\end{align*}

But by symmetry of the normal distribution, we have
$$
\Pr(X>0,Y>0)=\Pr(X<0,Y<0),
\qquad
\Pr(X>0,Y<0)=\Pr(X<0,Y>0),
$$
so the above expectation becomes
\begin{align*}
\E[\sign(X)\sign(Y)]
&= 2\Pr(X>0,Y>0) - 2\Pr(X>0,Y<0).
\end{align*}
Also:
$$
\Pr(X>0,Y>0)+\Pr(X>0,Y<0)=\Pr(X>0)=\frac{1}{2}.
$$
Therefore $\Pr(X>0,Y<0)=\frac{1}{2}-\Pr(X>0,Y>0)$ and
\begin{align}
\E[\sign(X)\sign(Y)]
&= 2\Pr(X>0,Y>0) - 2\left(\frac{1}{2}-\Pr(X>0,Y>0)\right)\nonumber\\
&= 4\Pr(X>0,Y>0) - 1.
\label{eq:signprod_in_terms_of_quadrant_prob}
\end{align}

So it remains to compute $\Pr(X>0,Y>0)$.
Using the $(U,V)$ representation:
$$
X>0 \iff U>0,
\qquad
Y>0 \iff \rho U + \sqrt{1-\rho^2}\,V > 0
\iff
V > -\frac{\rho}{\sqrt{1-\rho^2}}\,U.
$$
Thus
$$
\Pr(X>0,Y>0)=\Pr\!\left(U>0,\; V > -aU\right),
\qquad
a:=\frac{\rho}{\sqrt{1-\rho^2}}.
$$

Now use the fact that $(U,V)$ is a standard 2D normal, which is rotationally symmetric:
its density is proportional to $\exp(-(r^2)/2)$ and $r^2=u^2+v^2$.
Therefore the direction angle $\Theta\in[0,2\pi)$ of the random vector $(U,V)$ is uniform on $[0,2\pi)$.

The region $\{U>0,\; V>-aU\}$ is a wedge in the $(u,v)$ plane:
\begin{itemize}
\item $U>0$ means the half-plane to the right of the $v$-axis, i.e.\ angles $\Theta\in(-\pi/2,\pi/2)$.
\item $V>-aU$ is the half-plane above the line $v=-au$, which is a line through the origin with angle
$\phi$ satisfying $\tan(\phi)=-a$. One convenient angle is $\phi = -\arctan(a)$.
\end{itemize}
Therefore,
$$
\text{wedge angle} = \frac{\pi}{2} - \left(-\arctan(a)\right)
= \frac{\pi}{2} + \arctan(a).
$$
Therefore, the probability of falling into this wedge is
$$
\Pr(X>0,Y>0)=\frac{\frac{\pi}{2}+\arctan(a)}{2\pi}
=\frac{1}{4}+\frac{1}{2\pi}\arctan(a).
$$
Finally, note that
$$
a=\frac{\rho}{\sqrt{1-\rho^2}}
\quad\Longrightarrow\quad
\arctan(a)=\arctan\!\left(\frac{\rho}{\sqrt{1-\rho^2}}\right)=\arcsin(\rho),
$$

Thus
$$
\Pr(X>0,Y>0)=\frac{1}{4}+\frac{1}{2\pi}\arcsin(\rho).
$$
Plugging into \eqref{eq:signprod_in_terms_of_quadrant_prob} gives
$$
\E[\sign(X)\sign(Y)]
=4\left(\frac{1}{4}+\frac{1}{2\pi}\arcsin(\rho)\right)-1
=\frac{2}{\pi}\arcsin(\rho).
$$
\end{proof}

Now, let's prove the main theorem. This is a rather long and involved proof. However, the main complexity is algebraic manipulation, and the logic is rather simple. The core strategy is to 
\begin{itemize}
    \item Simplify and re-write the alignment variance into its element-wise format.
    \item Then, under the noise dominate assumption, the above term can be further simplified to a format that involves the expectation of two $\sign(\cdot)$, as shown in \cref{lem:arcsine}. 
    \item Then, we leverage a tool, called 2D Given's rotation, and show that when the rotation is eigen-based rotation, a single 2D rotation will increase the above objective.
    \item Then, we show any disjoint seqeunce of such 2D rotation will also increase the variance. This is achieved by examining the Hessian cross term. 
    \item With the above analysis, we can show eigen-based rotation is a local minimizer of the variance. 
\end{itemize}

\begin{proof}[Proof of \cref{thm: optimal align score variance reduction}]
Define
$$
A := \mR^T\mEG\in\R^{m\times n},
\qquad
Z := \mR^T\noise\in\R^{m\times n}.
$$
Then
$$
\mR^T\mNG =  A+Z.
$$
Also,
\begin{align*}
\alignscore = \langle A,\;\underbrace{\sign(A+Z)}_{H}\rangle.
\end{align*}

Then, we decompose the $\alignscore$ into its elementwise format:
\begin{equation}
\label{eq:Y_expand}
\alignscore=\langle A,H\rangle = \sum_{i=1}^m\sum_{j=1}^n A_{ij}H_{ij}.
\end{equation}

Next, let's rewrite the $\var(\alignscore)$ also into its element-wise format:
From \eqref{eq:Y_expand}, 
$$
\E[\alignscore^2]
=\sum_{i=1}^m\sum_{j=1}^n\sum_{k=1}^m\sum_{l=1}^n
A_{ij}A_{kl}\,\E[S_{ij}H_{kl}].
$$
Similarly,
$$
\E[\alignscore]
= \E\left[\sum_{i,j}A_{ij}H_{ij}\right]
= \sum_{i,j}A_{ij}\E[H_{ij}],
$$
so
$$
\E[\alignscore]^2
=\sum_{i,j,k,l} A_{ij}A_{kl}\,\E[H_{ij}]\,\E[H_{kl}].
$$
Therefore, the alignment variance is 
\begin{align}
\var(\alignscore)=\sum_{i,j,k,l} A_{ij}A_{kl}\underbrace{\Big(\E[H_{ij}H_{kl}] - \E[H_{ij}]\E[H_{kl}]\Big)}_{\Cov(H_{ij},H_{kl})}
\label{eq:AV_quad_sum}
\end{align}

By assumption, the noise columns $\noise_{j}$ are independent across $j$.
Since $Z=\mR^T\noise$ and $\mR$ is deterministic, the columns $Z_{j}=\mR^T\noise_{j}$ are also independent across $j$.

Fix indices $(i,j)$ and $(k,l)$ with $j\neq l$.
Because $Z_{j}$ and $Z_{l}$ are independent, $H_{ij}$ and $H_{kl}$ are independent,
hence
$$
\Cov(H_{ij},H_{kl})=0\quad\text{when }j\neq l.
$$
Therefore the quadruple sum \eqref{eq:AV_quad_sum} reduces to:
\begin{equation}
\label{eq:AV_same_column}
\var(\alignscore)
=\sum_{j=1}^n\sum_{i=1}^m\sum_{k=1}^m
A_{ij}A_{kj}\,\Cov(H_{ij},H_{kj}).
\end{equation}

Fix a column $j$ and two rows $i,k$.
Let
$$
X := A_{ij}+Z_{ij},\qquad Y := A_{kj}+Z_{kj}.
$$
Then
$$
H_{ij}=\sign(X),\qquad H_{kj}=\sign(Y).
$$
The pair $(Z_{ij},Z_{kj})$ is jointly Gaussian with mean $0$ and covariance determined by
$$
\cov_R := \mR^T\cov\mR.
$$
By definition, the correlation of $(Z_{ij},Z_{kj})$ is
\begin{equation}
\label{eq:rho_def}
\rho_{ik}(\mR) := \frac{(\cov_R)_{ik}}{\sqrt{(\cov_R)_{ii}(\cov_R)_{kk}}}.
\end{equation}

Now, $X$ and $Y$ are jointly Gaussian with mean vector $(A_{ij},A_{kj})$ and the same covariance as $(Z_{ij},Z_{kj})$.
Under the low-SNR assumption, the standardized means
$$
\frac{A_{ij}}{\sqrt{2(\Sigma_R)_{ii}}}\approx 0,\qquad
\frac{A_{kj}}{\sqrt{2(\Sigma_R)_{kk}}}\approx 0,
$$
so we approximate $(X,Y)$ by the zero-mean pair $(Z_{ij},Z_{kj})$ at leading order.
Thus
$$
\Cov(\sign(X),\sign(Y))
\approx
\Cov(\sign(Z_{ij}),\sign(Z_{kj})).
$$
But $\E[\sign(Z_{ij})]=0$ and $\E[\sign(Z_{kj})]=0$ by symmetry,
so
$$
\Cov(\sign(Z_{ij}),\sign(Z_{kj}))
=\E[\sign(Z_{ij})\sign(Z_{kj})].
$$
Applying Lemma~\ref{lem:arcsine},
we obtain
\begin{equation}
\label{eq:Cov_sign_arcsin}
\Cov(H_{ij},H_{kj})
\approx
\frac{2}{\pi}\arcsin\!\big(\rho_{ik}(\mR)\big).
\end{equation}

Thus, we can greatly simplify the $\var(\alignscore)$ (\eqref{eq:AV_same_column}) by substituting \eqref{eq:Cov_sign_arcsin}:
\begin{align}
\var(\alignscore)= \frac{2}{\pi}\sum_{j=1}^n\sum_{i=1}^m\sum_{k=1}^m
A_{ij}A_{kj}\,\arcsin\!\big(\rho_{ik}(\mR)\big).
\label{eq:AV_triple_sum}
\end{align}

Next, we can express $\arcsin(\rho_{ik}(R))$ into a matrix format for later analysis.
Define the diagonal matrix of coordinate standard deviations of $\cov_R$:
$$
\mD_R := \diag\big((\Sigma_R)_{11},\dots,(\Sigma_R)_{mm}\big)\in\R^{m\times m}.
$$
Define the correlation matrix
$$
C_R := \mD_R^{-1/2}\cov_R \mD_R^{-1/2},
$$
so $(C_R)_{ik}=\rho_{ik}(\mR)$.
Define $\Phi_R:=\arcsin(C_R)$  by applying $\arcsin$ in a element-wise manner.

Thus, the $\var(\alignscore)$ can be further re-written as a compact matrix multiplication format. 

For a fixed column $j$, let $a_j:=A_{j}\in\R^m$. Then
$$
\sum_{i=1}^m\sum_{k=1}^m A_{ij}A_{kj}\,\arcsin(\rho_{ik}(\mR))
= \sum_{i,k} (a_j)_i (a_j)_k (\Phi_R)_{ik}
= a_j^T \Phi_R a_j.
$$
Then, \eqref{eq:AV_triple_sum} becomes
\begin{equation}
\label{eq:AV_sum_quadratic}
\var(\alignscore)= \frac{2}{\pi}\sum_{j=1}^n a_j^T \Phi_R a_j.
\end{equation}
Therefore, we have
$$
\sum_{j=1}^n a_j^T \Phi_R a_j
= \sum_{j=1}^n \tr(a_j^T \Phi_R a_j)
= \tr(\Phi_R AA^T).
$$
Also, $A=\mR^T\mG$.
Therefore
\begin{equation}
\label{eq:AV_trace_form}
\var(\alignscore)\approx \frac{2}{\pi}\tr\!\Big(\Phi_R\,\mR^T(\mEG\mEG^T)\mR\Big)
\end{equation}

Next, we will examine whether the eigen-based rotation serves as the local minimum by analyzing the its interactions with $\Phi$.

By the shared-eigenspace assumption, $\mEG\mEG^T=\mU\bm{\lambda}\mU^T$ and $\cov=\mU \bm{D}\mU^T$.
Let $\mR=\mU\mQ$ with some rotation $\mQ\in O(m)$.
Then
$$
\mR^T(\mEG\mEG^T)\mR
= \mQ^T \mU^T(\mU\bm{\lambda}\mU^T)\mU\mQ
= \mQ^T\bm{\lambda}\mQ,
$$
and
$$
\mR^T\cov\mR
= \mQ^T \mU^T(\mU \bm{D}\mU^T)\mU\mQ
= \mQ^T \bm{D}\mQ.
$$
So \eqref{eq:AV_trace_form} becomes a function of $\mQ$:
\begin{equation}
\label{eq:AV_of_Q}
\var(\alignscore)\approx \frac{2}{\pi}\tr\!\Big(\Phi_R\;\mQ^T\bm{\lambda}\mQ\Big),
\end{equation}
where
$$
\Phi_R:=\arcsin\!\Big(\Corr(\mQ^T \bm{D}\mQ)\Big).
$$

Next, we plan to show that any disjoint set of 2D rotation (i.e. Givens rotation) will increase $\var(\alignscore)$. If that is the case, then $\mQ=\mI$ is the local minimum. 

Fix any pair of indices $1\le p<q\le m$, and define the Givens rotation $\mG_{pq}(\theta)\in O(m)$
that rotates only the $(p,q)$-coordinates:
$$
\mG_{pq}(\theta) =
\begin{bmatrix}
1 & & & & & \\
& \ddots & & & & \\
& & \cos\theta & \cdots & -\sin\theta & \\
& & \vdots & \ddots & \vdots & \\
& & \sin\theta & \cdots & \cos\theta & \\
& & & & & \ddots
\end{bmatrix},
$$
i.e.\ it equals the identity on all coordinates except the $2\times2$ block on rows/cols $(p,q)$.

Let's assume that $\mQ$ is such Givens rotation, that rotate the $(p,q)$ plane by $\theta$, and we define
$$
\mQ(\theta):=\mG_{pq}(\theta),\qquad c:=\cos\theta,\quad s:=\sin\theta.
$$

\subparagraph{Step A: Compute $\mQ(\theta)^T\bm{\lambda}\mQ(\theta)$.}
Because $\bm{\lambda}$ is diagonal, the transform
differs from $\bm{\lambda}$ only in the $(p,q)$-block. A direct multiplication gives:
\begin{align}
\label{eq:B_block}
B(\theta):=\mQ(\theta)^T\bm{\lambda}\mQ(\theta)
\quad\Rightarrow\quad
\begin{cases}
B_{pp}(\theta)=\lambda_p c^2+\lambda_q s^2,\\
B_{qq}(\theta)=\lambda_p s^2+\lambda_q c^2,\\
B_{pq}(\theta)=B_{qp}(\theta)=(\lambda_p-\lambda_q)cs,
\end{cases}
\end{align}
and all other off-diagonal entries of $B(\theta)$ are zero, and other diagonal remains intact. 

\subparagraph{Step B: Compute $\cov(\theta):=\mQ(\theta)^T \bm{D}\mQ(\theta)$ and the induced correlation.}
Similarly, since $\bm{D}$ is diagonal, $\cov(\theta)$ differs from $\bm{D}$ only in the $(p,q)$-block:
\begin{align}
\label{eq:Sigma_block}
\cov(\theta):=\mQ(\theta)^T \bm{D}\mQ(\theta)
\quad\Rightarrow\quad
\begin{cases}
\cov_{pp}(\theta)=D_p c^2+D_q s^2,\\
\cov_{qq}(\theta)=D_p s^2+D_q c^2,\\
\cov_{pq}(\theta)=\cov_{qp}(\theta)=(D_p-D_q)cs,
\end{cases}
\end{align}
and all other off-diagonal entries are zero, and other diagonals remain intact. 

By definitionm the correlation coefficient between coordinates $p$ and $q$:
\begin{equation}
\label{eq:rho_pq_exact}
\rho_{pq}(\theta)
:=
\frac{\cov_{pq}(\theta)}{\sqrt{\cov_{pp}(\theta)\cov_{qq}(\theta)}}
=
\frac{(D_p-D_q)cs}{\sqrt{(D_p c^2+D_q s^2)(D_p s^2+D_q c^2)}}.
\end{equation}
The correlation matrix $\Corr(\cov(\theta))$ has ones on the diagonal, and its only nonzero
off-diagonal entries are $(p,q)$ and $(q,p)$, both equal to $\rho_{pq}(\theta)$.

Therefore $\Phi(\mQ(\theta))=\arcsin(\Corr(\cov(\theta)))$ has:
$$
\Phi_{ii}(\theta)=\arcsin(1)=\frac{\pi}{2}\quad\text{for all }i,
\qquad
\Phi_{pq}(\theta)=\Phi_{qp}(\theta)=\arcsin(\rho_{pq}(\theta)),
$$
and all other off-diagonal entries are zero.

\subparagraph{Step C: Reduce the trace.}
Because $\Phi(\theta)$ and $B(\theta)$ are symmetric and only the $(p,q)$ off-diagonal entries can be nonzero,
\begin{align}
\var(\alignscore)
&:=\frac{2}{\pi}\tr\!\big(\Phi(\theta)B(\theta)\big)\nonumber\\
&=\frac{2}{\pi}\sum_{i=1}^m \Phi_{ii}(\theta)B_{ii}(\theta) + 2\Phi_{pq}(\theta)B_{pq}(\theta)\nonumber\\
&=\sum_{i=1}^m B_{ii}(\theta) + \frac{4}{\pi}\,\arcsin(\rho_{pq}(\theta))\,B_{pq}(\theta).
\label{eq:T_theta_decompose_derivative}
\end{align}
Moreover, $\sum_i B_{ii}(\theta)=\tr(B(\theta))=\tr(\bm{\lambda})$ is constant in $\theta$. Hence, the entire $\theta$-dependence is
$$
\var(\alignscore)=\underbrace{\tr(\bm{\lambda})}_{\text{constant}} \;+\;\frac{4}{\pi}\,g(\theta)\,h(\theta),
\quad\text{where}\quad
g(\theta):=\arcsin(\rho_{pq}(\theta)),\ \ h(\theta):=B_{pq}(\theta).
$$

\subparagraph{Step D: Compute the derivative of $h(\theta)=B_{pq}(\theta)$ at $\theta = 0$.}
\begin{equation}
\label{eq:Bpq_theta}
h(\theta)=B_{pq}(\theta)=(\lambda_p-\lambda_q)c(\theta)s(\theta).
\end{equation}
Differentiate:
$$
h'(\theta)=(\lambda_p-\lambda_q)\big(c'(\theta)s(\theta)+c(\theta)s'(\theta)\big).
$$
Therefore,
\begin{equation}
\label{eq:hprime0}
h'(0)=(\lambda_p-\lambda_q)\big(0\cdot 0 + 1\cdot 1\big)=\lambda_p-\lambda_q.
\end{equation}
Also $h(0)=(\lambda_p-\lambda_q)c(0)s(0)=0$.

\subparagraph{Step D.2: Compute $\rho_{pq}'(0)$.}
Define the numerator and denominator:
$$
u(\theta):=\Sigma_{pq}(\theta),
\quad
v(\theta):=\sqrt{\Sigma_{pp}(\theta)\Sigma_{qq}(\theta)}.
$$

First, evaluate at $\theta=0$:
$$
u(0)=(D_p-D_q)c(0)s(0)=0,
$$
$$
\Sigma_{pp}(0)=D_p,\quad \Sigma_{qq}(0)=D_q
\quad\Longrightarrow\quad
v(0)=\sqrt{D_pD_q}>0.
$$
Next, compute $u'(0)$:
$$
u'(\theta)=(D_p-D_q)\big(c'(\theta)s(\theta)+c(\theta)s'(\theta)\big),
$$
so
\begin{equation}
\label{eq:uprime0}
u'(0)=(D_p-D_q)\big(0\cdot 0 + 1\cdot 1\big)=D_p-D_q.
\end{equation}
Now compute $v'(0)$. Note that $\Sigma_{pp}(\theta)$ and $\Sigma_{qq}(\theta)$ depend only on
$c(\theta)^2$ and $s(\theta)^2$, hence they are \emph{even} functions of $\theta$.
Therefore their product $\Sigma_{pp}(\theta)\Sigma_{qq}(\theta)$ is even, and so is its square root $v(\theta)$.
The derivative of an even differentiable function at $0$ is $0$, hence
\begin{equation}
\label{eq:vprime0}
v'(0)=0.
\end{equation}

Now we have:
$$
\rho_{pq}'(\theta)=\frac{u'(\theta)v(\theta)-u(\theta)v'(\theta)}{v(\theta)^2}.
$$
And, 
\begin{equation}
\label{eq:rhoprime0}
\rho_{pq}'(0)=\frac{u'(0)v(0)-0\cdot v'(0)}{v(0)^2}
=\frac{(D_p-D_q)\sqrt{D_pD_q}}{(D_pD_q)}
=\frac{D_p-D_q}{\sqrt{D_pD_q}}.
\end{equation}
Also $\rho_{pq}(0)=u(0)/v(0)=0$.

\subparagraph{Step D.3: Differentiate $g(\theta)=\arcsin(\rho_{pq}(\theta))$ at $0$.}
We have
$$
g'(\theta)=\frac{\rho_{pq}'(\theta)}{\sqrt{1-\rho_{pq}(\theta)^2}}.
$$
At $\theta=0$, $\rho_{pq}(0)=0$, so $\sqrt{1-\rho_{pq}(0)^2}=1$, and 
\begin{equation}
\label{eq:gprime0}
g'(0)=\rho_{pq}'(0)=\frac{D_p-D_q}{\sqrt{D_pD_q}}.
\end{equation}
Also $g(0)=\arcsin(0)=0$.

\subparagraph{Step E: Compute the first and second derivatives of $\var(\alignscore)$ w.r.t.~$\theta$ at $0$.}
Recall
$$
\var(\alignscore)=\tr(\bm{\lambda})+\frac{4}{\pi}g(\theta)h(\theta).
$$
Therefore,
$$
\var(\alignscore)'=\frac{4}{\pi}\big(g'(\theta)h(\theta)+g(\theta)h'(\theta)\big).
$$
At $\theta=0$, $g(0)=0$ and $h(0)=0$, hence
\begin{equation}
\label{eq:Tprime0}
\var(\alignscore)'|_{\theta=0}=\frac{4}{\pi}\big(g'(0)\cdot 0 + 0\cdot h'(0)\big)=0.
\end{equation}
Differentiate again:
\begin{align*}
\var(\alignscore)''
&= \frac{4}{\pi}\big(g''(\theta)h(\theta)+g'(\theta)h'(\theta)+g'(\theta)h'(\theta)+g(\theta)h''(\theta)\big)\\
&=\frac{4}{\pi}\big(g''(\theta)h(\theta)+2g'(\theta)h'(\theta)+g(\theta)h''(\theta)\big).
\end{align*}
At $\theta=0$, again $g(0)=h(0)=0$, so the first and last terms vanish and we obtain
\begin{equation}
\label{eq:Tsecond0}
\var(\alignscore)''|_{\theta=0}=\frac{4}{\pi}\cdot 2g'(0)h'(0)=\frac{8}{\pi}g'(0)h'(0).
\end{equation}
Substitute \eqref{eq:hprime0} and \eqref{eq:gprime0} into \eqref{eq:Tsecond0}:
\begin{equation}
\label{eq:var_second_derivative}
\var(\alignscore)''|_{\theta=0}
= \frac{8(D_p-D_q)(\lambda_p-\lambda_q)}{\pi\sqrt{D_pD_q}}.
\end{equation}

\subparagraph{Step F: $\mQ=\mI$ is the local minimizer under RCM}
Under the RCM condition and the fact that $(D_i)$ is sorted in descending order,
$D_p\ge D_q$ implies $\lambda_p\ge \lambda_q$, hence
$$
(D_p-D_q)(\lambda_p-\lambda_q)\ge 0.
$$
Therefore \eqref{eq:var_second_derivative} implies the second derivative at $\theta=0$
is nonnegative for every pair $(p,q)$.
Because also $\var{\alignscore}'=0$ by \eqref{eq:Tprime0}, it follows that $\theta=0$ is a local minimizer of
$\var(\alignscore)$ for every 2D Givens rotation direction.

But this alone is not enough. Next, we need to show for arbitrary multi-plane rotation (i.e.~a sequence of 2D Givens rotation with disjoint set of planes), this claim is still true. 

To achieve this, we need to examine the off-diagonal terms in the Hessian matrix of $\var(\alignscore)$ w.r.t.~rotation $\theta$ at arbitrary set of planes $(p,q)$ and $(r,s)$. 

\begin{equation}
\label{eq:StepF_objective_def}
F(\mQ)
:=
\tr\!\Big(\Phi(\mQ)\,\mQ^T\bm{\lambda}\,\mQ\Big) = \frac{\pi}{2}\var(\alignscore),
\qquad
\Phi(\mQ):=\arcsin\!\Big(\Corr(\mQ^T \mD\,\mQ)\Big).
\end{equation}

Next, we show some properties of this term. 

\subparagraph{(i) Sign-flip invariance: $F(\mS\mQ\mS)=F(\mQ)$.}
Let $\mS=\diag(s_1,\dots,s_m)$ be any diagonal sign matrix with $s_i\in\{\pm1\}$, so that
$\mS^T=\mS$ and $\mS^2=\mI$. We claim
\begin{equation}
\label{eq:StepF_signflip_invariance}
F(\mS\mQ\mS)=F(\mQ)\qquad\text{for all }\mQ\in O(m).
\end{equation}

To prove \eqref{eq:StepF_signflip_invariance}, first note that since $\bm{\lambda}$ and $D$ are diagonal,
they commute with $\mS$, hence $\mS\bm{\lambda}\mS=\bm{\lambda}$ and $\mS D \mS = D$.
Define
$$
B(\mQ):=\mQ^T\bm{\lambda}\mQ,
\qquad
\Sigma(\mQ):=\mQ^T D \mQ.
$$
Then
\begin{align}
B(\mS\mQ\mS)
&=\mS\,B(\mQ)\,\mS,
\label{eq:StepF_B_transform}\\
\Sigma(\mS\mQ\mS)
&=\mS\,\Sigma(\mQ)\,\mS.
\label{eq:StepF_Sigma_transform}
\end{align}
Next, we can show
$$
(\mS\Sigma\mS)_{ii}=\Sigma_{ii}.
$$
and,
\begin{align}
\Corr(\Sigma(\mS\mQ\mS))
&=\diag(\Sigma(\mS\mQ\mS))^{-1/2}\,\Sigma(\mS\mQ\mS)\,\diag(\Sigma(\mS\mQ\mS))^{-1/2}\notag\\
&=\diag(\Sigma(\mQ))^{-1/2}\,(\mS\Sigma(\mQ)\mS)\,\diag(\Sigma(\mQ))^{-1/2}\notag\\
&=\mS\Big(\diag(\Sigma(\mQ))^{-1/2}\Sigma(\mQ)\diag(\Sigma(\mQ))^{-1/2}\Big)\mS
=\mS\,\Corr(\Sigma(\mQ))\,\mS.
\label{eq:StepF_Corr_transform}
\end{align}
Based on the above, we prove the invariance:
\begin{equation}
\label{eq:StepF_arcsin_transform}
\Phi(\mS\mQ\mS)
=\arcsin(\Corr(\Sigma(\mS\mQ\mS)))
=\arcsin(\mS\,\Corr(\Sigma(\mQ))\,\mS)
=\mS\,\arcsin(\Corr(\Sigma(\mQ)))\,\mS
=\mS\,\Phi(\mQ)\,\mS.
\end{equation}
Finally, combine \eqref{eq:StepF_B_transform} and \eqref{eq:StepF_arcsin_transform}:
\begin{align*}
F(\mS\mQ\mS)
&=\tr\!\Big(\Phi(\mS\mQ\mS)\,B(\mS\mQ\mS)\Big)
=\tr\!\Big((\mS \Phi(\mQ)\mS)(\mS B(\mQ)\mS)\Big)\\
&=\tr\!\Big(\mS \Phi(\mQ) B(\mQ)\mS\Big)
=\tr\!\Big(\Phi(\mQ) B(\mQ)\mS\mS\Big)
=F(\mQ).
\end{align*}

\subparagraph{(ii) The evenness of alignment variance}
Fix two \emph{distinct} planes $(p,q)\neq(r,s)$ with $p<q$ and $r<s$.
Define the two-angle function
\begin{equation}
\label{eq:StepF_g_def}
g(u,v):=F\!\big(\mG_{pq}(u)\,\mG_{rs}(v)\big).
\end{equation}
Without loss of generality, assume
$t\in\{p,q\}$ and $t\notin\{r,s\}$.

Let $\mS=\diag(1,\dots,1,\underbrace{-1}_{t\text{-th entry}},1,\dots,1)$.
Then conjugation by $\mS$ flips the sign of the $(p,q)$-rotation angle and leaves the $(r,s)$-rotation unchanged:
\begin{equation}
\label{eq:StepF_givens_flip}
\mS\,\mG_{pq}(u)\,\mS=\mG_{pq}(-u),
\qquad
\mS\,\mG_{rs}(v)\,\mS=\mG_{rs}(v).
\end{equation}
Using \eqref{eq:StepF_signflip_invariance} and \eqref{eq:StepF_givens_flip}, we obtain
\begin{align*}
g(u,v)
&=F\!\big(\mG_{pq}(u)\mG_{rs}(v)\big)
=F\!\big(\mS\,\mG_{pq}(u)\mG_{rs}(v)\,\mS\big)\\
&=F\!\big((\mS\mG_{pq}(u)\mS)(\mS\mG_{rs}(v)\mS)\big)
=F\!\big(\mG_{pq}(-u)\mG_{rs}(v)\big)
=g(-u,v).
\end{align*}
Thus, for each fixed $v$, the function $u\mapsto g(u,v)$ is \emph{even}:
$$
g(u,v)=g(-u,v).
$$
Therefore, from the property of even function, we can differentiate:
$$
\frac{\partial g}{\partial u}(0,v)=0\quad\text{for all }v.
$$
Differentiate the above identity with respect to $v$ and then set $v=0$:
\begin{equation}
\label{eq:StepF_mixed_partial_zero}
\frac{\partial^2 g}{\partial v\,\partial u}(0,0)=0.
\end{equation}
Equation \eqref{eq:StepF_mixed_partial_zero} states exactly that the mixed quadratic term $uv$
vanishes at the origin for any pair of distinct rotation planes $(p,q)\neq(r,s)$.
Hence \emph{all} mixed second-order terms between different Givens angles vanish at $\mQ=\mI$.

\subparagraph{(iii) Second-order Taylor expansion has no cross-plane terms.}
Let $\theta=\{\theta_{pq}\}_{1\le p<q\le m}$ collect all small Givens angles, and fix any smooth local
parameterization of $O(m)$ near $\mI$ using a product of Givens rotations
$$
\mQ(\theta)=\prod_{p<q}\mG_{pq}(\theta_{pq}),
\qquad \mQ(\mathbf{0})=\mI,
$$
with some fixed order of multiplication. 
From \eqref{eq:StepF_mixed_partial_zero}, we have for any distinct $(p,q)\neq(r,s)$,
$$
\left.\frac{\partial^2 F(\mQ(\theta))}{\partial \theta_{pq}\,\partial \theta_{rs}}\right|_{\theta=\mathbf{0}}=0.
$$
Therefore, the second-order Taylor expansion of $F(\mQ(\theta))$ at $\theta=\mathbf{0}$ has no mixed quadratic terms:
\begin{equation}
\label{eq:StepF_taylor_sum_squares}
F(\mQ(\theta))
=F(\mQ(0))
+\frac12\sum_{p<q}\left.\frac{\partial^2 F(\mQ(\theta))}{\partial \theta_{pq}^2}\right|_{\theta=\mathbf{0}}\theta_{pq}^2
+o(\|\theta\|^2).
\end{equation}
Moreover, we have proved that each $\frac{\partial^2 F(\mQ(\theta))}{\theta^2_{pq}}\|_{\theta=0}$ be nonnegative under RCM.
Hence the quadratic form in \eqref{eq:StepF_taylor_sum_squares} is nonnegative for all small $\theta$,
and $\mQ=\mI$ is a local minimizer of $F$ under RCM.
\end{proof}

\subsection{Proof of \cref{thm: rotation as loss rate minimizer}}
\label{subapp: proof of minimium loss decrease rate}

\begin{proof}
This proof follows the identical logic as \cref{thm: optimal align score variance reduction}, with easier algebraic manipulations. 

\textbf{Step 1: Rewrite $\dot{\loss}$ in summation form.} 
First, 
\begin{align*}
\dot{\loss}
&=\E\!\left[\tr\!\left(\mEG^T\mR\sign(\mR^T\mNG)\right)\right].
\end{align*}
Then, we can easily show
\begin{align}
\dot{\loss} 
&=\sum_{i=1}^n \E[\mEG_i^T \Big(\mR\sign(\mR^T\mNG)\Big)_i]\\
&=\sum_{i=1}^n \E[\mEG_i^T \mR\,\sign(\mR^T\mNG_i)].
\label{eq:dotloss-colsum-givens}
\end{align}
Expanding the inner term for each $i$, we have
$$
(\mR^T\mEG_i)^T\,\sign(\mR^T\mNG_i)
=\sum_{j=1}^m (\mR^T\mEG_i)_j\,\sign\!\big((\mR^T\mNG_i)_j\big).
$$
Substituting into \eqref{eq:dotloss-colsum-givens} yields
\begin{align}
\dot{\loss}
=\sum_{i=1}^n\sum_{j=1}^m (\mR^T\mEG_i)_j\;
\E\!\left[\sign\!\big((\mR^T\mNG_i)_j\big)\right]
\label{eq:dotloss-double-sum-givens}
\end{align}

\textbf{Step 2: Compute $\E[\sign(\cdot)]$ and apply low-SNR.}  

Because $\noise_i\sim\mathcal N(\mathbf 0,\cov)$ and $\mR$ is orthogonal,
$$
\mR^T\noise_i\sim\mathcal N(\mathbf 0,\mR^T\cov\,\mR),
$$
so the scalar $(\mR^T\noise_i)_j$ is Gaussian with mean $0$ and variance $(\mR^T\cov\,\mR)_{jj}$:
$$
(\mR^T\noise_i)_j\sim \mathcal N\!\left(0,\;(\mR^T\cov\,\mR)_{jj}\right).
$$
Define
$$
\mu_{ij}:=(\mR^T\mEG_i)_j,
\qquad
\sigma_{j}^2(\mR):=(\mR^T\cov\,\mR)_{jj}.
$$
Then
$$
(\mR^T\mNG_i)_j=\mu_{ij}+Z_{ij},\qquad Z_{ij}\sim\mathcal N(0,\sigma_j^2(\mR)).
$$
Next, we compute $\E[\sign(\mu_{ij}+Z_{ij})]$ exactly. Since $Z_{ij}$ has a continuous density,
$P(\mu_{ij}+Z_{ij}=0)=0$, hence
\begin{align*}
\E[\sign(\mu_{ij}+Z_{ij})]
&=1\cdot P(\mu_{ij}+Z_{ij}>0)+(-1)\cdot P(\mu_{ij}+Z_{ij}<0)\\
&= P(\mu_{ij}+Z_{ij}>0)- P(\mu_{ij}+Z_{ij}<0).
\end{align*}
So
\begin{align*}
\E[\sign(\mu_{ij}+Z_{ij})]
&=P(Z_{ij}>-\mu_{ij})-P(Z_{ij}<-\mu_{ij})\\
&=\bigl(1-P(Z_{ij}\le -\mu_{ij})\bigr)-P(Z_{ij}<-\mu_{ij})\\
&=1-2P(Z_{ij}\le -\mu_{ij}).
\end{align*}
Let $\Phi$ be the standard normal CDF. Since $Z_{ij}/\sigma_j(\mR)\sim\mathcal N(0,1)$,
$$
P(Z_{ij}\le -\mu_{ij})=\Phi\!\left(\frac{-\mu_{ij}}{\sigma_j(\mR)}\right).
$$
Therefore,
\begin{align*}
\E[\sign(\mu_{ij}+Z_{ij})]
&=1-2\Phi\!\left(\frac{-\mu_{ij}}{\sigma_j(\mR)}\right)
=2\Phi\!\left(\frac{\mu_{ij}}{\sigma_j(\mR)}\right)-1.
\end{align*}
Using the identity $2\Phi(x)-1=\erf(x/\sqrt2)$, we obtain

\begin{align}
\E\!\left[\sign\!\big((\mR^T\mNG_i)_j\big)\right]
=\erf\!\left(\frac{(\mR^T\mEG_i)_j}{\sqrt{2(\mR^T\cov\,\mR)_{jj}}}\right).
\label{eq:sign-erf-givens}
\end{align}

Now apply Assumption~\ref{assump: low SNR}. Under low-SNR,
$$
\frac{(\mR^T\mEG_i)_j}{\sqrt{2(\mR^T\cov\,\mR)_{jj}}}\approx 0,
$$
and we ignore terms beyond first order in the Taylor expansion. Since
$$
\erf(x)=\frac{2}{\sqrt{\pi}}x+O(x^3)\quad\text{as }x\to 0,
$$
Substitution in \eqref{eq:sign-erf-givens} yields
\begin{align}
\E\!\left[\sign\!\big((\mR^T\mNG_i)_j\big)\right]
=\sqrt{\frac{2}{\pi}}\frac{(\mR^T\mEG_i)_j}{\sqrt{(\mR^T\cov\,\mR)_{jj}}}.
\label{eq:sign-lowSNR-givens}
\end{align}

\paragraph{Step 3: Reduce $\dot{\loss}$ (low-SNR) to a diagonal objective.}
Therefore, the $\dot{\loss}$ can be re-written as 
\begin{align*}
\dot{\loss}
&=\sqrt{\frac{2}{\pi}}\sum_{i=1}^n\sum_{j=1}^m
\frac{\big((\mR^T\mEG_i)_j\big)^2}{\sqrt{(\mR^T\cov\,\mR)_{jj}}}\\
&=\sqrt{\frac{2}{\pi}}\sum_{j=1}^m
\frac{\sum_{i=1}^n\big((\mR^T\mEG_i)_j\big)^2}{\sqrt{(\mR^T\cov\,\mR)_{jj}}}.
\end{align*}

Define $\mS:=\mEG\mEG^T\in\R^{m\times m}$. It is easy to show
\begin{align}
\sum_{i=1}^n\big((\mR^T\mEG_i)_j\big)^2=(\mR^T\mS\mR)_{jj}.
\label{eq:sum-squares-diag-givens}
\end{align}
Therefore,
\begin{align}
\dot{\loss}\approx \sqrt{\frac{2}{\pi}}\sum_{j=1}^m
\frac{(\mR^T\mS\mR)_{jj}}{\sqrt{(\mR^T\cov\,\mR)_{jj}}}.
\label{eq:dotloss-diag-form-givens}
\end{align}
Therefore, minimizing $\dot{\loss}$ under the low-SNR is equivalent to minimizing
\begin{align}
J(\mR):=\sum_{j=1}^m
\frac{(\mR^T\mS\mR)_{jj}}{\sqrt{(\mR^T\cov\,\mR)_{jj}}}.
\label{eq:J-R-givens}
\end{align}
Next, we follow the same logic as \cref{thm: optimal align score variance reduction}, and show eigen-based rotation is the local minimizer through Givens rotation. 
\paragraph{Step 4: Leveraging shared eigen-space.}
By Assumption~\ref{assumpt: noise geometry}, $\cov$ and $\mEG\mEG^T$ share the same eigen-space.
Hence
$$
\mS=\mU\mLambda\mU^T,\qquad \cov=\mU\mD\mU^T.
$$
with $\lambda_1\ge\cdots\ge\lambda_m\ge 0$ and $D_1\ge\cdots\ge D_m>0$.
Write $\mR=\mU\mQ$ with $\mQ\in\R^{m\times m}$ orthogonal. Then
\begin{align*}
\mR^T\mS\mR
=\mQ^T\bm{\lambda}\mQ,\;\;\;\;
\mR^T\cov\,\mR
=\mQ^T\mD\mQ.
\end{align*}
Substituting into \eqref{eq:J-R-givens} yields
\begin{align}
J(\mU\mQ)=F(\mQ):=\sum_{j=1}^m
\frac{(\mQ^T\bm{\lambda}\mQ)_{jj}}{\sqrt{(\mQ^T\mD\mQ)_{jj}}}.
\label{eq:F-Q-givens}
\end{align}
Thus it suffices to show that $\mQ=\mI$ is a local minimizer of $F(\mQ)$ over orthogonal $\mQ$.

\paragraph{Step 5: local minimum through Givens rotations}
Fix indices $1\le p<q\le m$. Define the Givens rotation $\mG_{pq}(\theta)\in\R^{m\times m}$ by
$$
\mG_{pq}(\theta)=\mI
\quad\text{except on rows and cols $(p,q)$ where it equals}\quad
\begin{bmatrix}
\cos\theta & -\sin\theta\\
\sin\theta & \cos\theta
\end{bmatrix}.
$$ similar to the proof of \cref{thm: optimal align score variance reduction}.

Consider the function
$$
\phi_{pq}(\theta):=F(\mG_{pq}(\theta)).
$$
Because $\mG_{pq}(\theta)$ is the identity outside the $(p,q)$-plane, for any $j\notin\{p,q\}$ we have
$$
(\mG_{pq}(\theta)^T\bm{\lambda}\mG_{pq}(\theta))_{jj}=\lambda_j,
\qquad
(\mG_{pq}(\theta)^T\mD\mG_{pq}(\theta))_{jj}=D_j,
$$
so the corresponding $F$ are constant.
Hence, we have
\begin{align}
\phi_{pq}(\theta)
=\sum_{j\neq p,q}\frac{\lambda_j}{\sqrt{D_j}}
+\frac{A_p(\theta)}{\sqrt{B_p(\theta)}}
+\frac{A_q(\theta)}{\sqrt{B_q(\theta)}},
\label{eq:phi-decomp-givens}
\end{align}
where
\begin{align*}
A_p(\theta)&:=(\mG_{pq}(\theta)^T\bm{\lambda}\mG_{pq}(\theta))_{pp},&
B_p(\theta)&:=(\mG_{pq}(\theta)^T\mD\mG_{pq}(\theta))_{pp},\\
A_q(\theta)&:=(\mG_{pq}(\theta)^T\bm{\lambda}\mG_{pq}(\theta))_{qq},&
B_q(\theta)&:=(\mG_{pq}(\theta)^T\mD\mG_{pq}(\theta))_{qq}.
\end{align*}
Because $\bm{\lambda}$ and $\mD$ are diagonal and $\mG_{pq}(\theta)$ only mixes coordinates $p$ and $q$,
a direct computation gives
\begin{align}
A_p(\theta)=\lambda_p\cos^2\theta+\lambda_q\sin^2\theta,\qquad
A_q(\theta)=\lambda_p\sin^2\theta+\lambda_q\cos^2\theta,
\label{eq:ApAq-givens}\\
B_p(\theta)=D_p\cos^2\theta+D_q\sin^2\theta,\qquad
B_q(\theta)=D_p\sin^2\theta+D_q\cos^2\theta.
\label{eq:BpBq-givens}
\end{align}

Define
$$
f_p(\theta):=\frac{A_p(\theta)}{\sqrt{B_p(\theta)}},\qquad
f_q(\theta):=\frac{A_q(\theta)}{\sqrt{B_q(\theta)}}.
$$

Differentiate \eqref{eq:ApAq-givens}--\eqref{eq:BpBq-givens}, we obtain
the derivative $f_p'(\theta)$ using the product/chain rules:
\begin{align}
f_p'(\theta)
&=\frac{d}{d\theta}\Big(A_p(\theta)\,B_p(\theta)^{-1/2}\Big)\\
&=A_p'(\theta)\,B_p(\theta)^{-1/2}
+A_p(\theta)\,\frac{d}{d\theta}\Big(B_p(\theta)^{-1/2}\Big)\\
&=A_p'(\theta)\frac{1}{\sqrt{B_p(\theta)}}
+A_p(\theta)\left(-\frac12\right)B_p(\theta)^{-3/2}B_p'(\theta)\\
&=\frac{A_p'(\theta)}{\sqrt{B_p(\theta)}}
-\frac12\,\frac{A_p(\theta)B_p'(\theta)}{B_p(\theta)^{3/2}} \\
&=\sin(2\theta)\left(
\frac{\lambda_q-\lambda_p}{\sqrt{B_p(\theta)}}
-\frac12\,\frac{A_p(\theta)(D_q-D_p)}{B_p(\theta)^{3/2}}
\right).
\label{eq:fp-prime}
\end{align}
Similarly,
\begin{align}
f_q'(\theta)
=\sin(2\theta)\left(
\frac{\lambda_p-\lambda_q}{\sqrt{B_q(\theta)}}
-\frac12\,\frac{A_q(\theta)(D_p-D_q)}{B_q(\theta)^{3/2}}
\right).
\label{eq:fq-prime}
\end{align}
Then, we have
$$
f_p'(0)=0,\qquad f_q'(0)=0,
$$
at $\theta =0$.
Hence ,$\phi_{pq}'(0)=0$. Because the $(p,q)$-Givens rotations generate the tangent directions of the orthogonal group at $\mI$, this shows that $\mQ=\mI$ is a stationary point of $F$.

\subparagraph{Second derivative at $\theta=0$.}
Re-written \eqref{eq:fp-prime} as
$$
f_p'(\theta)=\sin(2\theta)\,H_p(\theta),
\quad\text{where}\quad
H_p(\theta):=
\frac{\lambda_q-\lambda_p}{\sqrt{B_p(\theta)}}
-\frac12\,\frac{A_p(\theta)(D_q-D_p)}{B_p(\theta)^{3/2}}.
$$
Differentiate:
$$
f_p''(\theta)=2\cos(2\theta)\,H_p(\theta)+\sin(2\theta)\,H_p'(\theta).
$$
Evaluating at $\theta=0$ gives 
\begin{align}
f_p''(0)=2H_p(0).
\label{eq:fp-second-0}
\end{align}
With simple substitution, we obtain
\begin{align}
H_p(0)=\frac{\lambda_q-\lambda_p}{\sqrt{D_p}}
-\frac12\,\frac{\lambda_p(D_q-D_p)}{D_p^{3/2}},
\label{eq:Hp0}
\end{align}
and 
\begin{align}
f_p''(0)
=2\left(
\frac{\lambda_q-\lambda_p}{\sqrt{D_p}}
-\frac12\,\frac{\lambda_p(D_q-D_p)}{D_p^{3/2}}
\right).
\label{eq:fp-second-final}
\end{align}

Similarly,
\begin{align}
f_q''(0)
=2\left(
-\frac{\lambda_q-\lambda_p}{\sqrt{D_q}}
+\frac12\,\frac{\lambda_q(D_q-D_p)}{D_q^{3/2}}
\right).
\label{eq:fq-second-final}
\end{align}

Therefore,
\begin{align}
\phi_{pq}''(0)
&=f_p''(0)+f_q''(0)\notag\\
&=2(\lambda_q-\lambda_p)\left(\frac{1}{\sqrt{D_p}}-\frac{1}{\sqrt{D_q}}\right)
+(D_q-D_p)\left(\frac{\lambda_q}{D_q^{3/2}}-\frac{\lambda_p}{D_p^{3/2}}\right)\\
&=2c_{pq},
\label{eq:phi-second-0}
\end{align}
where
\begin{align}
c_{pq}:=(\lambda_q-\lambda_p)\left(\frac{1}{\sqrt{D_p}}-\frac{1}{\sqrt{D_q}}\right)
+\frac{D_q-D_p}{2}\left(\frac{\lambda_q}{D_q^{3/2}}-\frac{\lambda_p}{D_p^{3/2}}\right),
\label{eq:cpq-def-givens}
\end{align}

Because $D_1\ge\cdots\ge D_m$, for $p<q$ we have $D_p\ge D_q$.
Define
$$
a:=\sqrt{D_p},\qquad b:=\sqrt{D_q},\qquad
r_p:=\frac{\lambda_p}{D_p},\qquad r_q:=\frac{\lambda_q}{D_q}.
$$
Then $a\ge b>0$, $\lambda_p=r_p a^2$, and $\lambda_q=r_q b^2$.
Therefore, 
\begin{align*}
c_{pq}
&=(r_q b^2-r_p a^2)\frac{b-a}{ab}
+\frac{(b-a)(b+a)}{2}\left(\frac{r_q}{b}-\frac{r_p}{a}\right)\\
&=(b-a)\left(\frac{r_q b^2-r_p a^2}{ab}+\frac{b+a}{2}\left(\frac{r_q}{b}-\frac{r_p}{a}\right)\right)\\
&=(b-a)\cdot\frac{2(r_q b^2-r_p a^2)+(b+a)(r_q a-r_p b)}{2ab}\\
&=\frac{a-b}{2ab}\left(r_p(2a^2+ab+b^2)-r_q(a^2+ab+2b^2)\right)
\end{align*}

Now, since $a\ge b$, the prefactor $\frac{a-b}{2ab}\ge 0$.
Next rewrite the bracket as
\begin{align*}
r_p(2a^2+ab+b^2)-r_q(a^2+ab+2b^2)
&=\Big(r_p(a^2+ab+2b^2)-r_q(a^2+ab+2b^2)\Big)+r_p(a^2-b^2)\\
&=(r_p-r_q)(a^2+ab+2b^2)+r_p(a^2-b^2).
\end{align*}
Since $a^2+ab+2b^2>0$ and $a^2-b^2\ge 0$, it remains to note that $r_p-r_q\ge 0$.
Indeed, by the RCM condition (Definition~\ref{def: ratio co-monotonic}), for $D_p>D_q$ we have
$$
\frac{\lambda_p}{D_p}>\frac{\lambda_q}{D_q}\quad\Longleftrightarrow\quad r_p>r_q,
$$
Hence in all cases $r_p-r_q\ge 0$ and $r_p\ge 0$,
Therefore $c_{pq}\ge 0$, and
$$
\phi_{pq}''(0)=2c_{pq}\ge 0.
$$

\paragraph{Step 7: Conclude $\mQ=\mI$ is a local minimizer of $F$.}
We have shown for every $p<q$ that $\phi_{pq}'(0)=0$ and $\phi_{pq}''(0)\ge 0$, i.e., along every Givens-rotation direction through $\mI$, the first derivative vanishes and the second derivative is nonnegative. This represents that $\mQ=\mI$ is the local minimizer upto a single 2D rotation. Now, we have to prove that it is still a minimizer under a sequence of rotations with disjoint set of planes.  This argument is identical to \textbf{Step F} in the proof of \cref{thm: optimal align score variance reduction}. In the following, we include it for completeness.

\subparagraph{(i) A simple symmetry: $F$ is invariant under diagonal sign flips.}
For any diagonal sign matrix $\mS=\diag(s_1,\dots,s_m)$ with $s_k\in\{\pm 1\}$, we can easily show
\begin{align}
F(\mS\mQ\mS)=F(\mQ)\qquad\text{for all }\mQ\in\mathrm{O}(m).
\label{eq:signflip-invariance-step7}
\end{align}

\subparagraph{(ii) Mixed second-order terms vanish at $\mI$ by an evenness argument.}
Fix two \emph{distinct} planes $(p,q)\neq(r,s)$ with $p<q$ and $r<s$.
Consider the two-plane rotation
$$
g(u,v):=F\!\big(\mG_{pq}(u)\mG_{rs}(v)\big).
$$

Because $\{p,q\}\neq\{r,s\}$ as sets, there exists an index $t$ belonging to exactly one of the sets (i.e., $t\in\{p,q\}\triangle\{r,s\}$).
Assume $t\in\{p,q\}$ and $t\notin\{r,s\}$ (the other case is symmetric), and define the sign matrix
$$
\mS=\diag(1,\dots,1,\underbrace{-1}_{t\text{-th entry}},1,\dots,1).
$$
Then conjugation by $\mS$ flips the angle in the $(p,q)$-plane and leaves the $(r,s)$-plane unchanged:
\begin{align}
\mS\,\mG_{pq}(u)\,\mS=\mG_{pq}(-u),
\qquad
\mS\,\mG_{rs}(v)\,\mS=\mG_{rs}(v).
\label{eq:givens-flip-step7}
\end{align}
Using \eqref{eq:signflip-invariance-step7} and \eqref{eq:givens-flip-step7},
\begin{align*}
g(u,v)
&=F\!\big(\mG_{pq}(u)\mG_{rs}(v)\big)
=F\!\big(\mS\,\mG_{pq}(u)\mG_{rs}(v)\,\mS\big)\\
&=F\!\big((\mS\mG_{pq}(u)\mS)(\mS\mG_{rs}(v)\mS)\big)
=F\!\big(\mG_{pq}(-u)\mG_{rs}(v)\big)
=g(-u,v).
\end{align*}
Thus $g(u,v)$ is an \emph{even} function of $u$ for every fixed $v$:
$$
g(u,v)=g(-u,v).
$$
Differentiate with respect to $u$ and set $u=0$ to get
$$
\frac{\partial g}{\partial u}(0,v)=0\quad\text{for all small }v.
$$
Differentiating this identity with respect to $v$ and then setting $v=0$ yields
\begin{align}
\frac{\partial^2 g}{\partial v\,\partial u}(0,0)=0.
\label{eq:mixed-uv-zero-step7}
\end{align}
This is from the property of even function w.r.t. $u$. 
Equation \eqref{eq:mixed-uv-zero-step7} exactly says that there is \emph{no} mixed quadratic term $uv$ at the origin when we turn on two different Givens angles. Since the choice of $(p,q)\neq(r,s)$ was arbitrary, \emph{all} mixed second-order terms between distinct Givens planes vanish.

\subparagraph{(iii) Second-order Taylor expansion has no cross-plane terms}
Let $\theta=\{\theta_{pq}\}_{1\le p<q\le m}$ collect all small Givens angles, and consider any smooth local parameterization of $\mathrm{SO}(m)$ near $\mI$ by these Given's rotations.

From previous analysis, we know for $(p,q)\neq(r,s)$,
$$
\left.\frac{\partial^2 F(\mQ(\theta))}{\partial \theta_{pq}\,\partial \theta_{rs}}\right|_{\theta=0}=0.
$$
Therefore, the second-order Taylor expansion of $F(\mQ(\theta))$ at $\theta=0$ has \emph{no} mixed quadratic terms and takes the form
\begin{align}
\psi(\theta)
=F(\mQ(0))+\frac12\sum_{p<q}\left.\frac{\partial^2F(\mQ(\theta))}{\partial \theta_{pq}^2}\right|_{\theta=0}\theta_{pq}^2
+o(\|\theta\|^2)
=F(\mQ(0))+\sum_{p<q} c_{pq}\,\theta_{pq}^2+o(\|\theta\|^2).
\label{eq:taylor-sum-of-squares-step7}
\end{align}
Since each $c_{pq}\ge 0$, the quadratic term $\sum_{p<q} c_{pq}\theta_{pq}^2$ is nonnegative for every $\theta$.

Therefore, we conclude that $\mR=\mU\mQ$ is locally minimized at $\mQ=\mI$ under the RCM condition.

\end{proof}

\end{document}